\documentclass[letterpaper]{article} % DO NOT CHANGE THIS
\usepackage{aaai23}  % DO NOT CHANGE THIS
\usepackage{times}  % DO NOT CHANGE THIS
\usepackage{helvet}  % DO NOT CHANGE THIS
\usepackage{courier}  % DO NOT CHANGE THIS
\usepackage[hyphens]{url}  % DO NOT CHANGE THIS
\usepackage{graphicx} % DO NOT CHANGE THIS
\usepackage{natbib}  % DO NOT CHANGE THIS AND DO NOT ADD ANY OPTIONS TO IT
\usepackage{caption} % DO NOT CHANGE THIS AND DO NOT ADD ANY OPTIONS TO IT
\usepackage{algorithm}
\usepackage{algorithmic}
\usepackage{wrapfig}
\usepackage{newfloat}
\usepackage{listings}
\DeclareCaptionStyle{ruled}{labelfont=normalfont,labelsep=colon,strut=off} % DO NOT CHANGE THIS
\floatstyle{ruled}
\newfloat{listing}{tb}{lst}{}
\floatname{listing}{Listing}
\usepackage{tikz}
\usetikzlibrary{arrows}
\usepackage{tkz-graph}

\title{Optimism in Face of a Context:\\
Regret Guarantees for Stochastic Contextual MDP}
\author {
    Orin Levy,\textsuperscript{\rm 1}
    Yishay Mansour \textsuperscript{\rm 1,2}
}
\affiliations {
    \textsuperscript{\rm 1} Tel Aviv University\\
    \textsuperscript{\rm 2} Google Research, Tel Aviv\\
    orinlevy@mail.tau.ac.il, mansour.yishay@gmail.com
}

\usepackage{amsthm}

\newcommand{\R}{\mathbb{R}} %reals
\newcommand{\N}{\mathbb{N}} % naturals
\newcommand{\Reg}{\text{E.Regret}}
\newcommand{\Regrv}{\text{Regret}}

\newcommand{\E}{\mathbb{E}} % expectation
\newcommand{\Prob}{\mathbb{P}} % probability
\newcommand{\D}{\mathcal{D}} % distribution
\newcommand{\C}{\mathcal{C}} % context space
\newcommand{\F}{\mathcal{F}} % function class 
\newcommand{\G}{\mathcal{G}} % function class 
\newcommand{\Fp}{\mathcal{P}} % DYNAMICS class 
\newcommand{\M}{\mathcal{M}} % mdp
\newcommand{\Hist}{\mathbb{H}} % history
\newcommand{\Mhat}{\widehat{\mathcal{M}}}

\newcommand{\I}{\mathbb{I}} %indicator
\newcommand{\B}{\mathcal{B}} %Bernoulli rv
\newcommand{\V}{\mathbb{V}} % variance
\usepackage{microtype}
\usepackage{graphicx}
\usepackage{subfigure}
\usepackage{booktabs} 
\usepackage{hyperref}

\usepackage{amsmath}
\usepackage{amssymb}
\usepackage{mathtools}
\usepackage{amsthm}
\usepackage{natbib}

\usepackage[capitalize,noabbrev]{cleveref}

\newtheorem{theorem}{Theorem}[section]
\newtheorem{lemma}[theorem]{Lemma}
\newtheorem{claim}[theorem]{Claim}
\newtheorem{corollary}[theorem]{Corollary}
\newtheorem{definition}[theorem]{Definition}
\newtheorem{remark}[theorem]{Remark}
\newtheorem{conclusion}[theorem]{Conclusion}
\newtheorem{assumption}[theorem]{Assumption}
\newtheorem{observation}[theorem]{Observation}

\begin{document}

\maketitle

\begin{abstract}
    We present regret minimization algorithms for stochastic contextual MDPs under minimum reachability assumption, using an access to an offline least square regression oracle.
    We analyze three different settings: where the dynamics is known, where the dynamics is unknown but independent of the context and the most challenging setting where the dynamics is unknown and context-dependent. For the latter, our algorithm obtains regret bound of
    $\widetilde{O}( (H+{1}/{p_{min}})H|S|^{3/2}\sqrt{|A|T\log(\max\{|\mathcal{G}|,|\mathcal{P}|\}/\delta)})$ with probability $1-\delta$, where $\mathcal{P}$ and $\mathcal{G}$ are finite and realizable function classes used to approximate the dynamics and rewards respectively, $p_{min}$ is the minimum reachability parameter, $S$ is the set of states, $A$ the set of actions, $H$ the horizon, and $T$ the number of episodes.
    To our knowledge, our approach is the first optimistic approach applied to contextual MDPs with general function approximation (i.e., without additional knowledge regarding the function class, such as it being linear and etc.).
    We present a lower bound of $\Omega(\sqrt{T  H |S| |A| \ln(|\mathcal{G}|)/\ln(|A|)})$, on the expected regret which holds even in the case of known dynamics.
    Lastly, we discuss an extension of our results to CMDPs without minimum reachability, that obtains $\widetilde{O}(T^{3/4})$ regret. 
\end{abstract}

\section{Introduction}
Markov decision processes (MDPs) have been extensively studied, and are commonly used to describe dynamic environments. MDPs characterize a variety of real-life tasks and applications including: advertising, healthcare, games, robotics and more, where at each episode an agent interacts with the environment with the goal of maximizing her return. (See, e.g., \citet{Sutton2018,MannorMT-RLbook}.)

In many applications, in each episode, there are additional external factors that affect the environment, which we refer to as the \emph{context}.
One way to handle this is to extend the state space to include the context. This approach has the disadvantage of greatly increasing the state space, and, as a result, the complexity of learning and even the representation of a policy. An alternative approach, is to keep a small state space, and regard the context as an additional side-information. Contextual Markov Decision Process (CMDP) describes such a model, where for each context there is a potentially different optimal policy \cite{hallak2015contextual}.

CMDPs are useful to model many user-driven applications, where the context is a user-related information which influences the optimal decision making.
%
% One natural application is in healthcare. We can model the interaction with a given patient using MDP. For a given medical treatment, the expected outcome of a patient is highly dependent on his medical history and other personal parameters, which we model as her context. For example, the success probability of a given treatment might heavily depend on the patient's age and weight. Another example is potential drug interactions that can be crucial, and clearly depend on the patient.
% %that there are medicines that might be fatal when given to patients with high blood pressure, but have excellent influence on patients that do not suffer that.  
% We abstract the patient's medical history and any other relevant information as the context.
% The benefit of using a CMDP is the fact that most patients behave similarly, although the context space may be large, and there might be unforeseen connection between the context and the outcomes. CMDP allow to share information and behavior between different contexts in a natural way.
%patient (i.e., her context) is revealed to the doctor before deciding on a treatment strategy, and crucially affects it. Hence, this scenario can be described by a CMDP.
%
One natural application is in recommendation systems,
%
%Assume that a recommendation-based web site would like to optimize the content and ads presented to the user, to maximize his browsing time. Then, the behavior of a specific user can be modeled by a MDP. 
%
where two different users might behave completely different from one another, hence, a single MDP can not describe them both. 
%
%However, 
In those systems, users behavior can be described using a side information about them, such as age, gender, interest fields and hobbies. This information
is referred to as the \emph{context} which influences the environment.
%
%Using a side information about the users, such as age, gender, interest fields and hobbies, is refer to as \emph{context}.
%it defines a user-dependent environment. 
%This environment is deeply influenced by the user's context and describe his behavior well.
CMDP defines a mapping from context to a related MDP, and the optimal policy given a context is the optimal policy in the related MDP.
%Then, planning an optimal user-dependent policy given it's context yields a policy that achieves the site's goal. 
%
%\Yishay{One paragraph on CMDP}
%
% CMDPs were first presented in \citet{hallak2015contextual}, where the case of relatively small number of latent contexts was studied. 
% \citet{modi2018markov} study the generalization in CMDP using covering numbers of the observed contexts.
% \citet{modi2020no} present regret minimization algorithm for CMDP under GLM assumption. 
% In contrast, our focus is on minimizing regret using an access to a least square regression oracle, but without any structural knowledge regarding the CMDP. 
% (We review additional works in~\cref{subsec:related-work}.)

\noindent\textbf{Our contributions.}
We present regret minimization algorithms for CMDP under three different settings: (1) known dynamics, (2) unknown context-independent dynamics and (3) unknown context-dependent dynamics, which is the most challenging.
In all settings we assume an access a least square regression oracle, and finite function classes $\G$ and $\Fp$ used to approximate the rewards and dynamics, respectively. In addition, we assume minimum reachability, where any policy for any context has a probability of at least $p_{min}$ to reach any state.
%(e.g. see~\cref{min-reachability}).
For the known dynamics setting we obtain  
$ \widetilde{O}(( H+ {1}/{p_{min}})\cdot|S|\sqrt{T |A|  \log({|\G|}/{\delta}}))$  
regret.
For the unknown context-independent dynamics we obtain %$\widetilde{O}
% (( H+ {1}/{p_{min}})\cdot|S| \sqrt{  TH   |A|}\log(|\G|/\delta))$ 
$\widetilde{O} \Big(
H^{1.5}|S| \sqrt{T |A|}\log({1}/{\delta})
+ 
(H+1/{p_{min}} )\cdot |S| \sqrt{ |A| T\log(|\G|/\delta)}\Big)$
regret.
For the unknown context-dependent dynamics we obtain regret of $\widetilde{O}(( H+ {1}/{p_{min}})\cdot H|S|^{3/2}\sqrt{|A|T\log(\max\{|\G|,|\Fp|\}/\delta)})$.
All of the bounds hold with high probability.
% For the unknown context-independent dynamics we show an $\widetilde{O}\left(\max\left\{H|S| \sqrt{T |A|}\cdot\log{1}/{\delta},\; {1}/{p_{min}}\sqrt{ \log(|\F|/\delta) T  |S| |A|}\right\}\right)$ 
% expected
% regret.
%For all of the settings, we also obtain identical expected regret bounds.
% , which holds with high probability.
We also show a lower bound of $\Omega(\sqrt{T H |S| |A| \ln(|\G|)/\ln(|A|)})$ on the expected regret.
%, where $\F = \G^S$.
 Lastly, we discuss an extension of our results to CMDP without minimum reachability, that obtains $\widetilde{O}(T^{3/4})$ regret, in~\cref{sec:extention}.

Our approach applies the ``optimism in face of uncertainty'' principle to CMDPs and achieves a sub-linear 
%expected
regret.
%To our knowledge, our approach is the first application of the ``optimism in face of uncertainty'' principle to CMDPs and achieves a sub-linear 
Our algorithms and analysis were inspired by the optimistic approach of~\citet{xu2020upper} for learning contextual multi armed bandits using least square regression oracle.
%The main difference is that in CMAB, there is no dynamics. 
We extended their approach to handle CMDPs and even a context-dependent dynamics.
%In addition, we present a lower bound for learning CMDP using function approximation.

\subsection{Related Work}\label{subsec:related-work}
% \noindent{\bf Related work.}\\
%\input{AAAI2023/Main/related-work}
%
%
\noindent\textbf{Contextual Reinforcement Leaning.}
CMDP was first introduce by \citet{hallak2015contextual}. \citet{modi2018markov} gives a general framework for deriving generalization bounds for smooth CMDPs and finite contextual linear combination of MDPs. \citet{modi2020no}
gives a regret bound of $\widetilde{O}(\sqrt{T})$ for Generalized Linear Models (GLMs). Our regret function approximation framework is more general than GLM.

\citet{foster2021statistical} 
present a new statistical complexity measure
%the decision-estimation coefficient, 
for interactive decision making, and show an application of it to obtain $\widetilde{O}(\sqrt{T})$ regret for Contextual RL. They assume an access to an online estimation oracle with regret guarantees, that maximizes over models and policies together. It is unclear when is this oracle implementable in polynomial time. In contrast, we make a significantly weaker and standard assumption regarding an offline regression oracle. Another difference is that we use an optimistic approach while they use inverse gap weighting. (More details later.)
%
% denote it $\textbf{Est}$.
% Their reference model class is the class of all contextual MDPs $\mathcal{M}$ and randomized policies $\Pi$.
% Given the observations and played policies up to the current time step, the online estimation oracle returns an estimated CMDP. Given the current context, they use it to compute a distribution over policies, and sample the played policy from it. They obtain $\tilde{O}(\sqrt{T \cdot \textbf{Est}})$ regret.
% The main disadvantages of their approach are
% that their online estimation oracle is very strong and might be computationally inefficient. It is unclear whether their algorithmic results can be extended to support offline oracles for estimation.
% In addition, the sample complexity of such an oracle is unclear. 
% Moreover, the relation between their new complexity measure and known complexity measures for function approximation (i.e, VC/Pseudo/Fat-shattering/Natrajan-dimension) that are commonly-used in offline supervised learning is unclear.
% Their results are are very general and capture many RL settings. 
% In contrast, we use a standard and efficient offline least square oracle to build an approximated optimistic CMDP from scratch. Hence, we have much refined assumptions which allows us to use standard and common offline supervised learning tolls.

\citet{jiang2017contextual} present OLIVE which is sample efficient for Contextual Decision Processes (CDP) with low Bellman rank. We do not make any assumptions on the Bellman rank.

\citet{levy2022learning} consider the sample complexity of learning CMDPs using function approximation. They provide the first general and efficient reduction from CMDP to offline supervised learning. Their sample complexity varies from $\widetilde{O}(1/\epsilon^2)$ to $\widetilde{O}(1/\epsilon^8)$, depending on the setting. We, in contrast, consider regret minimization and obtain $\widetilde{O}(\sqrt{T})$ regret under  the minimum reachability assumption.

% They assume an access to an ERM oracle and derive sample complexity bounds to compute $\epsilon$-optimal policy under four different settings: where the dynamics known and unknown, and context dependent or context-free. They assume no additional structural assumption regarding the CMDP. Their method provide the first general and efficient reduction from CMDP to offline supervised learning. However, their sample complexity bounds are not optimal, as our lower bound shows.

% We, in contrast, consider online learning problem where the goal is regret minimization. We obtain $\sqrt{T}$ regret under  the minimum  reachability assumption while \citet{levy2022learning} has no such an assumption but obtain higher dependence on $T$.

\noindent\textbf{Contextual Bandits.} Contextual bandits (CMAB) are a natural extension of the Multi-Arm Bandit (MAB), augmented by a context which influences the rewards \cite{Slivkins-book,MAB-book}. \citet{agarwal2014taming} use efficiently an optimization oracle  to derive an optimal regret bound. 
Regression based approaches appear in \citet{agarwal2012contextual,foster2018practical,foster2020beyond,simchi2021bypassing}.
We differ from CMAB, since our main challenge is the dynamics, and the need to optimize future rewards, which is the case in most RL settings.\\
\citet{xu2020upper} present the first optimistic algorithm for CMAB. They assume an access to a least-square regression oracle and achieve $\widetilde{O}(\sqrt{T |A| \log |\F|})$ regret, where $\F$ is a finite and realizable function class uses to approximate the rewards. 
%They also show a result for infinite function class using covering numbers analysis. 
Our algorithms and analysis are inspired by their optimistic approach and we extend it to CMDP.

%\Orin{Is that elaborated enough?}

\noindent\textbf{Inverse Gap Weighting (IGW) technique.}\label{IGM}
\citet{foster2020beyond,simchi2021bypassing}
apply the IGW technique to CMAB and obtain $\widetilde{O}(\sqrt{T|A|})$ regret, assuming an access to a least square regression oracle.
However, we do not see any straight-forward extension of their approach to CMDP which is both computationally efficient and has an optimal regret, under the same least-square oracle assumption (even when the dynamics is known to the learner).
%In more detail, 
%
% A naive application of IGM for CMDP would yields an exponential (in $|S||A|$) regret and is computationally inefficient. 
% The naive algorithm is as follows:
% consider the following application of IGM to CMDP. 
% At time $t$:
% (1) Compute an approximation for the context-dependent rewards function $\hat{f}_t$.
% (2) Observe a context $c_t$ and choose policy $\pi_{t}$ according to a distribution over deterministic policies. The probability of a policy $\pi$ is proportional to $1/{(V^{\pi^\star_t}_t(s_0) - V^\pi_t(s_0))}$, where $\pi^\star_t$ is the optimal policy for the rewards $\hat{f}_t$. 
% (3) Experience trajectory and update the function approximation.
% Using the analysis of \cite{foster2020beyond,simchi2021bypassing} one can obtain $\tilde{O}(\sqrt{T \cdot 2^{|S||A|}})$ regret and it is computationally inefficient.
% A similar approach using the $Q$ function will also similarly fail.
% Consider at time $t$ selecting a stochastic policy such that $\pi_{t}(a|s)$ is proportional to $1/{(Q^{\pi_t}_t(s, \pi^\star_t(s)) - Q^{\pi_t}_t(s, a))}$. This attempt will fail due to the lack of optimism and the changing-per-context optimal policy.
%
\citet{foster2021statistical} apply IGW to CMDP and obtain optimal regret. However they use the much strong online estimation oracle as discussed above. 
% (For extended related work overview, see~\cref{Appendix:Extended-Related-Work}.)
%
% \noindent\textbf{Additional works.} There are works that considered the case were the contexts are unobservable \cite{Latent-context-MDP,eghbal-zadeh2021learning,eghbal2021context}, latent states \cite{KrishnamurthyAL16}, spectral methods \cite{sprctelCMDP2016}, transfer learning \cite{zhang2020transfer}, and more. All those issue are somewhat unrelated to our main focus.

\noindent\textbf{Paper organization.} \cref{sec:prelimineries} contains the notations we use, and our assumptions.~\cref{sec:KCDD,sec:UCFD,sec:UCDD} contain an outline of our algorithms and regret analysis for each one of the settings. 
\cref{sec:LB} presents our lower bound and \cref{sec:extention} sketches an extension of our result to CMDPs without minimum reachability. We discuss our results in \cref{sec:disscution}.
The supplementary material
%of this works
can be found in~\citet{levy2022optimism}.

\section{Preliminaries and Notations}\label{sec:prelimineries}
\noindent\textbf{Markov Decision Process (MDP)}
   is a tuple $(S,A,P,r,s_0, H)$, where (1) $S$ is a finite state space,  (2) $A$ is a finite action space, (3) $s_0\in S$ is the unique start state, 
   (4) $P(\cdot|s,a)$ defines the transition probability function, i.e.,  $P(s' | s,a)$ is the probability that we reach state $s'$ given that we are in state $s$ and perform action $a$,
   (5) $R(s,a)\in[0,1]$ is a  random variable  for the reward of performing action $a$ in state $s$, and $r(s,a)$ is its expectation, i.e., $r(s,a) = \mathbb{E}[R(s,a)|s,a] $, and 
   (6) $H$ is the finite horizon.
    %parameter.
    
    The state space is decomposed into $H+1$ disjoint subsets (layers) $S_0, S_1, \ldots, S_{H-1}, S_H$ such that transitions are only possible between consecutive layers (i.e., loop-free). 
     There is a unique final state, i.e., $S_H = \{s_H\}$, with reward $0$.

\noindent\textbf{Policy.}
%\begin{definition}
    A \emph{stochastic policy} $\pi$ is a mapping from states to distribution over actions, i.e., $\pi : S \to \Delta(A)$.  
    A \emph{deterministic policy} $\pi$ is a mapping from states to actions, i.e., $\pi : S \to A$.
%\end{definition}

\noindent\textbf{Occupancy measure}~(see e.g., \citet{puterman2014markov,zimin2013online}).
Let $q_h(s,a | \pi, P)$ denote the probability of reaching state $s\in S$ and performing action $a\in A$ at time $h \in [H]$ of an episode generated using  policy $\pi$ and dynamics $P$. Let $q_h(s | \pi,P)=\sum_{a\in A}q_h(s,a | \pi, P)$ be the probability to visit state $s\in S$ at time $h$. 
%
%By definition, for any $ h \in [H ]$, state $s \in S$, and action $a \in A$ we have $ q_h(s,a| \pi, P) = q_h(s|\pi, P) \cdot \pi(a|s)$, and 
%
%Note that for any level $h\in[H]$,  policy $\pi$, and dynamics $P$ 
%the occupancy measure $q_h(s,a| \pi, P)$ is a distribution over $S\times A$.
%, i.e., $\sum_{s\in S_h,a\in A} q_h(s,a|P,\pi)=1$.
%Note that in a layered and loop-free MDP, for any policy $\pi$ that $q_h(s| \pi, P)>0$ if and only if $s \in S_h$. 

% The generalization of occupancy measure to context-dependent policy is straight forward when considering $q_h(s,a| \pi_c; P^c_\star)$.

%Let $q_h(s| P, \pi)$ denote he probability to visit state $s$ at time $h$ of an episode when the agent playing according to $\pi$ on the environment defined by the dynamics $P$.

\noindent\textbf{Episode and trajectory.}
At the start of each episode we select a policy $\pi$.
%interacts with the environment in episodes of length $H$. 
The episode starts at the unique initial state $s_0$. In state $s_h \in S_h$, we play action $a_h\sim\pi(\cdot | s_h)$, observe a reward $r_h \sim R(s_h, a_h)$ and move to $s_{h+1} \sim P(\cdot| s_h, a_h)$. We generate
%
%Thus, an episode can be represented by a
a trajectory $\sigma_{H+1} = (s_0, a_0, r_0, s_1, \ldots, s_{H-1}, a_{H-1}, r_{H-1},s_H)$ of length $H+1$.
%, where 
%for all $h \in \{1, 2, \ldots, H\}$, 
%$s_h \sim P(\cdot| s_{h-1}, a_{h-1})$ and $r_h \sim R(s_h, a_h)$ 
% (W.l.o.g. we assume that $r(s_H,a_H) = 0$ for any $s_H\in S_H$ and $a_h\in A$ so we can omit it).

\noindent\textbf{Value functions.}
Given a policy $\pi$ and a MDP 
    $
        M 
        = 
        (S,A,P,r,s_0, H)
    $, 
the
$h \in [H-1] $ stage value function of a state $s \in S_h$ is defined as 
    $
        V^{\pi}_{M,h} (s)
        = 
        \mathbb{E}_{\pi, M} 
        [
        \sum_{k=h}^{H-1} r(s_k, \pi(s_k))|s_h = s ]
    $ and for an action ${a \in A}$ we have
    $
        Q^{\pi}_{M,h} (s,a)
        = 
        \mathbb{E}_{\pi, M}
        [
        \sum_{k=h}^{H-1} r(s_k, \pi(s_k))|s_h = s, a_h=a ]
    $.
% For completeness we define
% $
%         V^\pi_{M,H}(s) = 0,\;\; \forall s \in S
% $.    
When $h = 0$ we denote $V^{\pi}_{M,0}(s_0) := V^{\pi}_M (s_0)$.
%
% \noindent\textbf{Q function.}
% Given a policy $\pi$ and a MDP 
%     $
%         M 
%         = 
%         (S,A,P,r,s_0, H)
%     $, 
% the
% $h \in [H-1] $ stage  Q-function of a state $s \in S_h$ and an action $a \in A$ is defined as 
%     $
%         {Q^{\pi}_{M,h} (s,a)
%         = 
%         \mathbb{E}_{\pi, M}
%         \Big[
%         \sum_{k=h}^{H-1} r(s_k, \pi(s_k))|s_h = s, a_h=a \Big]}
%     $.
% For completeness we define
% $
%         Q^\pi_{M,H}(s,a) = 0,\;\; \forall (s,a) \in S\times A
% $.    
%For brevity, when $h = 0$ we denote $Q^{\pi}_{M,0}(s_0,a) := Q^{\pi}_M (s_0,a)$.
%
% Recall the Bellman's equations for the Q-function:
% for all $h \in [H-1] $, state $s \in S_h$ and an action $a \in A$ it holds that
% \[
%     Q^{\pi}_{M,h}(s,a) = r(s,a) + \mathop{\E}_{s' \sim P(\cdot|s,a)}\left[V^\pi_{M,h+1}(s')\right].
% \]

\noindent\textbf{Optimal policy}
   % An optimal policy 
   $\pi^\star_M$ for MDP $M$ satisfies, for every stage $h \in [H-1]$ and a state $s\in S_h$, 
    $
        \pi^\star_{M,h} (s) \in
        \arg \max_{\pi}\{V^{\pi}_{M,h}(s)\}
    $, and w.l.o.g it is a deterministic policy. 
%Recall that for the finite horizon return there always exists an optimal policy which is a deterministic policy.

\noindent\textbf{Planning.}
Given an MDP $M= (S, A, P, r, s_0, H)$ the algorithm $\texttt{Planning}(M)$ returns an optimal policy $\pi^\star_M$ and its value $V^\star_M(s_0)$ and runs in time $O(|S|^2\; |A|\; H)$.

%\noindent\textbf
%\subsection{Contextual MDP (CMDP) }%(context-free-dynamics and context-dependent-dynamics)}
%We define CMDP as follows.

%\begin{definition}\label{def: cmdp}{
\noindent\textbf{Contextual Markov Decision Process (CMDP)} is a tuple $(\mathcal{C},S, A, \mathcal{M})$ where $\mathcal{C}\subseteq \mathbb{R}^{d'}$ is the context space, $S$  the state space and $A$  the action space. The mapping $\mathcal{M}$  maps a context $c\in \mathcal{C}$ to a MDP
    $
        \mathcal{M}(c) 
        =
        (S, A, P^c_\star, r^c_\star,s_0, H)
    $, 
where $r^c_\star(s,a) = \E[R^c_\star(s,a)|c,s,a]$, $R^c_\star(s,a) \sim \D_{c,s,a}$.
%and $R^c_\star(s,a) \in [0,1]$.
% We assume there is a distribution $\D$ over the contexts space such that $c \sim \D$. In addition, we assume the context space is finite but potentially huge, mainly for mathematical convenience. Our results can be naturally expended to infinite context space
%\end{definition}

There is an unknown distribution $\mathcal{D}$ over the context space $\mathcal{C}$, and for each episode a context is sampled i.i.d. from $\mathcal{D}$.
For mathematical convenience, we assume the context space is finite (but potentially huge). Our results naturally extend to infinite contexts space.

\noindent\textbf{Context-Independent and Context-Dependent dynamics.}
A CMDP has a \emph{context-independent} dynamics when the context effects only the rewards function, while the dynamics are identical for all contexts, i.e., $ P^c_\star = P$ for any context $c$.
% Namely, there exits a transition probabilities matrix $P$ such that for all $c \in \mathcal{C} $ we have $ P^c_\star = P$. 
%In contrast,
A \emph{context-dependent} dynamics has a potentially different dynamics $P^c_\star$ for each context $c$.
Hence, the partition of the states space into layers is also context-dependent. We denote by $S^c_h$ the $h$ layer of context $c$.
%We consider both context-independent and context-dependent dynamics.

% \noindent\textbf{Layered and loop-free CMDP.}
% A layered and loop-free CMDP is a CMDP for which $\M(c)$ is a layered and loop-free MDP, for every context $c\in \C$.
% We remark that the partition into layers is context-dependent.
% For every context $c \in \C$, we denote by $S^c_0, S^c_1, \ldots ,S^c_{H-1}, S_H$ the disjoint layers, and $\forall c \in \C:\;\; S = \bigcup_{h \in [H]}S^c_h$. 
% We assume that for any context $c$ we have $S^c_0 = \{s_0\}$, and $S^c_H = \{s_H\}$,
% meaning are a unique start and final states, whose identical for all of the contexts. We assume w.l.o.g $s_H$ has no rewards.
% Clearly, this assumption does not limit the generality of our results.

%\begin{definition}[Context-dependent policy]
\noindent\textbf{Context-dependent policies.} A stochastic context-dependent policy $\pi = \left( \pi(c;\cdot): S \to \Delta(A) \right)_{c \in \mathcal{C}}$ maps a context $c \in \mathcal{C}$ to a stochastic policy $\pi(c;\cdot) : S \to \Delta(A)$.
    A deterministic context-dependent policy $\pi = \left( \pi(c;\cdot): S \to A) \right)_{c \in \mathcal{C}}$ maps a context $c \in \mathcal{C}$ to a policy $\pi(c;\cdot) : S \to A$.    
    Let $\Pi_\C$ denote the class of all deterministic context-dependent policies.
%\end{definition}
%
%\begin{remark}
%     The generalization of occupancy measure to context-dependent policy is straight forward when considering $q_h(s,a| \pi(c;\cdot ), P^c_\star)$ for each context $c \in \mathcal{C}$.
% \end{remark}
%
%\begin{definition}[Optimal context-dependent policy]
    A 
    %(deterministic)
    context-dependent policy $\pi^\star \in \Pi_\C$ is \emph{optimal} if for all $c \in \C$ it holds that
    $
        \pi^\star(c;\cdot) \in 
        \arg\max_{\pi}
        %: S \to A} 
        V^{\pi(c;\cdot)}_{\M(c)}(s_0)
    $
    \footnote{As for non-contextual MDP, there always exists a deterministic context-dependent policy that is optimal.}.
    %where $\pi(c; \cdot)$ maps a state to an action given the context is $c$, i.e., $\pi(c; \cdot): S \to A$.
%\end{definition}

% \begin{remark}
%     A policy $\pi \in \Pi_{\C}$ that satisfies the following for \textbf{every} context $c \in \C$ is optimal:
%     \[
%         \pi(c;\cdot) \in 
%         \arg\max_{\pi' \in S \to A}  V^{\pi'}_{\M(c)}(s_0)
%     \]
% \end{remark}

\noindent\textbf{Minimum reachability.}
%\label{min-reachability}
We assume that there exists $p_{min} \in (0,1]$ such that for every $c \in \C$, $h \in [H-1]$ and $s_h \in S^c_h$, any context-dependent policy $\pi \in \Pi_{\C}$ satisfies
$
    q_h(s_h | \pi(c;\cdot), P^c_\star) \geq p_{min}.
$
%We remark that $p_{min}$ is \textbf{unknown} to the learner in this section.
%
Let $q(s|\pi(c;\cdot),P^c_\star)$ denote the probability of visiting state $s$ when playing $\pi$ on the dynamics $P^c_\star$.
When the dynamics is layered and loop-free, then 
$
    q(s|\pi(c;\cdot),P^c_\star) = q_h(s | \pi(c;\cdot), P^c_\star) \geq p_{min} 
$ iff $ s \in S^c_h$.
We remark that our minimum reachability assumption is more refined than that usually used in RL literature, that $P^c_\star(s'|s,a) \geq p_{min}$ (see, e.g.,~\citet{wei2021last}) for every context $c $ and $(s,a,s')$.
Clearly, this requirement implies our minimum reachability, but the other direction does not necessarily hold.
An \emph{example} for a large class of (non-layered) CMDPs that satisfies that assumption is as follows.
(1) At the initial step, for all $c \in \C, a \in A, s' \in S: P^c_{\star,0}(s'|s_0,a) \geq p_{min}$.
(2) For every step $h >0$, the transition probability matrix $P^c_{\star,h}(\cdot|\cdot,a)$ is double stochastic for all $c \in \C$ and $a \in A$.
This guarantees that for any policy $\pi$, the occupancy measure is at least $p_{min}$. 
% Thus, for any policy $\pi$ and context $c \in \C$ it holds that $q_h(s | \pi, P^c_\star) \geq p_{min} $ for all $h \in [H]$ and $s \in S$.

% We require that $q(s|\pi(c;\cdot, P^c_\star)) \geq p_{min}$ for any policy $\pi \in \Pi_\C$, context $c \in \C$ and state $s \in S$, while  The latter is a string structural assumption that does limit the generality. 
% \Orin{To revise the order}

\noindent\textbf{Interaction protocol.}
    In each episode $t=1,2,...,T$ the agent:
    (1) Observes context $c_t \in \mathcal{C}$.
    (2) Chooses a policy $\pi_t$ (based on $c_t$ and the observed history). 
    (3) Observes a trajectory of $\pi_t$  in $\mathcal{M}(c_t)$.

\noindent\textbf{Trajectories and History.}{
Each episode is of length $H$.
A trajectory $\sigma = (c; s_0, a_0, r_0, \ldots, s_{h-1},a_{h-1},r_{h-1})$ is generated using the dynamics $P^c_\star$ and the played policy $\pi(c;\cdot)$.
We denote the history up to time $t-1$ by $\Hist_{t-1} = (\sigma^1, \ldots, \sigma^{t-1})$ where $\sigma^i$ is the trajectory observed in time $i \in [t-1]$, i.e., $\sigma^i= (c_i, s^i_0, a^i_0, r^i_0, \ldots, s^i_{H-1},a^i_{H-1},r^i_{H-1},s^i_H)$.
}

\noindent\textbf{Offline least square regression (LSR) oracle}
%We assume an access to an offline least-square regression oracle, which 
solves the optimization problem
$
    {\hat{f} \in \arg\min_{f \in \F }\sum_{i=1}^n (f(x_i) - y_i)^2}
$, 
given a data set $D = \{(x_i, y_i)\}_{i=1}^n$.
We remark that there exist function classes for them the LSR  oracle can be implemented efficiently. Clearly, this holds for linear functions.
%where for all $i \in [n]$.
%$(x_i,y_ii)$ are drawn from some distribution $Q$. 

\noindent\textbf{Reward function approximation.}\label{par:reward-function-approx}
We consider a finite function class $\G \subseteq (\C \times A \to [0,1])$ to approximate the context-dependent rewards function of each state $s \in S$.
Many times it would be more convenient to consider a finite function class 
$\F = \G^{S}$ where $f\in \F$ are functions of the form 
$f(c,s,a) = g_s(c,a)$ where $g_s \in \G$. Note that, $\log(|\F|) = |S|\log(|\G|)$.
Our algorithms get as input the finite function class ${\F \subseteq (\C \times S \times A \to [0,1]})$.
%, $|\F| < \infty$.
Each function $f \in \F$ maps context $c \in \C$, state $s \in S$ and action $a \in A$ to a (approximate) reward $r \in [0,1]$.
We use $\F$ to approximate the context-dependent rewards function using the LSR oracle under the following realizability assumption.
\begin{assumption}[rewards realizability]\label{assm:rewards-realizability}
    We assume that $\F$ is realizable, meaning, there exists a function $f_\star \in \F$ such that 
    $
       f_\star(c,s,a) = r^c_\star(s,a) = \E[R^c_\star(s,a)|c,s,a]
    $. 
\end{assumption}
%For technical reasons, we assume that the function class $\F$ \emph{tensorizes}, meaning, there is a base function class $\G \subseteq (\C \times A \to [0,1])$ such that $\F = \G^{S}$ in the sense $\F$ consists if functions of the form 
%$f(c,s,a) = g_s(c,a)$ where $g_s \in \G$.
%Hence, $\log(|\F|) = |S|\log(|\G|)$
%(the assumption is standard, see e.g.~\citet{foster2020beyond} for contextual bandits).
For mathematical convenience, we state our algorithms and regret upper bounds in terms of the cardinality of $|\F|$, and use the cardinality of $|\G|=S^{-1}\log |\F|$ for our lower bound. We present a comparison between the bounds in~\cref{sec:disscution}. 

\noindent\textbf{Dynamics function approximation.}
For the unknown context-independent dynamics case we simply use a tabular approximation (see~\cref{sec:UCFD}). For the unknown context-dependent case, our algorithm gets as input a finite  function class $\Fp \subseteq  (S \times ( S \times A \times \C) \to [0,1])$, where every function $P \in \Fp$ satisfies 
$
    \sum_{s' \in S} P(s'|s,a,c)  = 1$ for all $c \in \C$ and $(s,a) \in S \times A
$.
We use $\Fp$ to approximate the context-dependent dynamics using LSR oracle under the following realizability assumption. We denote $P^c(s'|s,a) = P(s'|s,a,c)$ for all $P \in \Fp$. 
%
% \begin{assumption}[Dynamics Realizability]
% Let $P$ be the true dynamics $P(s'|s,a,c) $ for a context $c$.
% Given a context $c$ we define the dynamics $P$ induced by $c$ as $P^c_\star$, i.e., $P^c_\star(s'|s,a)=P(s'|s,a,c) $.
%
% For every context $c \in \C$, and state-action pair $(s,a) \in S \times A$ we define a  random variable 
% $\B(P^c_\star,s,a)\in S$ which returns the next state $s'$ that observed after action $a$ was played in state $s$ and the dynamics is $P^c_\star$.
% Observe that $\B$ satisfies the following:
% $
%     \B(P^c_\star,s,a) \sim P^c_\star(\cdot|s,a)
% $ and
% $
%     \E\left[\I[s' = \B(P^c_\star,s,a)] \right]
%     =
%     P^c_\star(s'|s,a)
%     \;\; \forall s' \in S.
% $
% We assume that $\Fp$ is realizable, meaning there exist a function $P_\star \in \Fp$ such that 
% \[
%     P^c_\star(s'|s,a) = P^c_\star(s'|s,a) = \E[\I[s' = \B(P^c_\star,s,a)|c,s,a,s'], \forall c \in \C,\; (s,a,s') \in S \times A \times S,
% \] 
% where $P^c_\star(s'|s,a)=P_\star(s'|s,a,c)$.
% \end{assumption}

\begin{assumption}[dynamics realizability]\label{assm:dynamics-realizability}
%Let $P$ be the true dynamics $P(s'|s,a,c) $ for a context $c$.
We assume that $\Fp$ is realizable, meaning, there exists a function $P_\star \in \Fp$ which is the true context-dependent dynamics.
    % \[
    %   P^c_\star(s'|s,a) = P^c(s'|s,a) = \E[\I[s' = \B(P^c_\star,s,a)|c,s,a,s'],
    % \] 
    % where $P^c_\star(s'|s,a)=P_\star(s'|s,a,c)$.
\end{assumption}

\noindent\textbf{Learning goal. }
Our goal is to minimize the regret,
relative to the optimal context-dependent policy $\pi^\star$,
which defined as
% $
%     {\Reg_T(\text{ALG}):=
%     \E \left[ \sum_{t=1}^T \E_c \left[ 
%     V^{\pi^\star(c;\cdot)}_{\M(c)}(s_0)
%     -
%     V^{\pi_t(c;\cdot)}_{\M(c)}(s_0)
%     \right] \right]}
% $, 
$\Regrv_T:= \sum_{t=1}^T 
    V^{\pi^\star(c_t;\cdot)}_{\M(c_t)}(s_0)
    -
    V^{\pi_t(c_t;\cdot)}_{\M(c_t)}(s_0),
$
where $c_t \in \C$, $\pi^\star$ is an optimal context-dependent policy 
and $\pi_t \in \Pi_\C$ are the context and the selected policy at round $t$.
 We denote the expected regret as 
 $
     {\Reg_T:=
     \E \left[ \Regrv_T \right]}
 $ where the expectation is over the contexts, the randomization of the algorithm and the history.

% \noindent\textbf{Mathematical notations.}
% We denote expectation by $\mathbb{E}[\cdot]$, variance by $\V[\cdot]$ and probabilities by $\mathbb{P}[\cdot]$. The indicator function is $\mathbb{I}[G]$ returns $1$ if event $G$ holds and $0$ otherwise. 

\section{Known Context-Dependent Dynamics}\label{sec:KCDD}
In this section, we present a regret minimization algorithm for contextual MDPs under the minimum reachability assumption, where the context-dependent dynamics $P^c_\star$ is known to the learner.
%, for every context $c \in \C$.
We remark that the minimum reachability parameter $p_{min}$ is unknown to the learner. 
This section sets the main building blocks of our approach, which we will later extend to handle the unknown dynamics cases.

\noindent\textbf{Algorithm outline.}
For the first $|A|$ rounds, 
%are an initialization rounds, in which for all 
in each round $i \in\{1,2,\ldots,|A|\}$ the agent plays the policy $\pi_i \in \Pi_\C$ that always selects action $a_i$, regardless of the context and the state.
At every round $t>|A|$ we approximate the context-dependent rewards function using a least-square minimizer.
Using it, we build an ``optimistic in expectation'' rewards function, and compute an optimal policy for that optimistic model. We run it to generate a trajectory and update the oracle. 
Here, we take an advantage of the ability to compute the optimal policy $\pi_k(c;\cdot)$ for every context $c \in \C$ separately, for all $k = |A|+1, \ldots, t$, to obtain computationally efficient algorithm. (We discuss this challenge later.)
%For more details, 
%See~\cref{alg:RM-CMDP-KD-main,alg:RM-CMDP-KD}.

\begin{algorithm}
    \caption{Regret Minimization for CMDP with Known Dynamics (RM-KD)}
    \label{alg:RM-CMDP-KD-main}
    \begin{algorithmic}[1]
        \STATE
        { 
            \textbf{inputs:}  MDP parameters: 
                $S$, $A$, $P_\star$, $s_0$, $H$. %( a unique start state).
                Confidence $\delta>0$ and tuning parameters $\{\beta_t\}_{t=1}^T$.
        } 
        \STATE{\textbf{initialization:}
        in round $i\leq |A|$
        run $\pi_i(c;s) := a_i$}
        \FOR{round $t = |A| + 1, \ldots, T$}
            \STATE{ 
            $
                \hat{f}_t \in 
                \arg\min_{f \in \F}
                \sum_{i=1}^{t-1} \sum_{h=0}^{H-1} ( f(c_i,s^i_h,a^i_h) - r^i_h)^2
            $ is being computed using the LSR oracle
            }
            \STATE{observe a fresh context $c_t \sim \D$}
            \FOR{$k= |A|+1 ,|A|+2, \ldots, t$}
                \STATE{compute for all $(s,a) \in S \times  A:
                    $
                    $    
                        \widehat{r}_k^{c_t}(s, a) = \hat{f}_k(c_t,s,a) + \frac{\beta_k}{  \sum_{i=1}^{k-1}\I[a = \pi_i(c_t;s)]q(s | \pi_i(c_t;\cdot),P^{c_t}_\star)}
                    $ 
                    }
                \STATE{define the optimistic approximated MDP\\$\Mhat_k(c_t) = (S, A, P^{c_t}_\star, \hat{r}^{c_t}_k, s_0, H)$}
                \STATE{compute $\pi_k(c_t;\cdot) \in \arg\max_{\pi \in S \to A}V^{\pi}_{\Mhat_k(c_t)}(s_0)$ using a planning algorithm}
            \ENDFOR{}    
            \STATE{play $\pi_t(c_t;\cdot)$ and update oracle using\\${\sigma^t = (c_t, s^t_0, a^t_0, r^t_0, s^t_1, \ldots, s^t_{H-1}, a^t_{H-1}, r^t_{H-1},s^t_H) }$}
            \ENDFOR{}
    \end{algorithmic}
\end{algorithm}

\begin{remark}
    Since the CMDP is layered, for every context $c \in \C$, layer $h \in [H]$ and state $s_h\in S_h^c$ we have $q_h(s_h | \pi_i(c;\cdot),P^{c}_\star)=q(s_h | \pi_i(c;\cdot),P^{c}_\star)$.
    For convenience, in~\cref{alg:RM-CMDP-KD-main} we use $q$ to compute the approximated rewards function, but in the regret analysis we use $q_h$.
\end{remark}

\noindent\textbf{Analysis outline.} 
Our analysis consists of four main steps.

\noindent\textbf{Step 1}:\label{step-1}  
establish uniform convergence bound over any $t \geq 2$ and a fixed sequence of functions $f_2,f_3,\ldots \in \F$. 
(See~\cref{lemma:UC-lemma-5}). 
Our bound implies for the least square minimizers sequence $\{\hat{f}_t\}_{t=|A|+1}^T$ and any $\delta \in (0,1)$,
that with probability at least $1-\delta/2$ for all $t\geq 2$ it holds that
\begin{align*}
    \sum_{i=1}^{t-1} 
    \sum_{h=0}^{H-1} \mathop{\E}_{c_i,s^i_h,a^i_h}
    \Big[(\hat{f}_t(c_i, s^i_h, a^i_h)-&f_\star(c_i, s^i_h, a^i_h) )^2 | \Hist_{i-1}\Big]
    \\
    &\leq
    68 H \log(4|\F|t^3/\delta).
\end{align*}

\noindent\textbf{Step 2:}\label{step-2} 
construct a confidence bound over the value of any given policy w.r.t the true rewards function $f_\star$ and the least square minimizer at round $t$, $\hat{f_t}$. The confidence bound holds with high probability, in expectation over the contexts (i.e., ``optimism in expectation''). 
Formally, in~\cref{lemma:CB-policy}
we show that with probability at least $1-\delta/2$ for all $t>|A|$ and any policy $\pi \in \Pi_\C$ it holds that  
\begin{align*}
    |\E_c[V^{\pi(c;\cdot)}_{\M(c)}(s_0)] - \E_c&[V^{\pi(c;\cdot)}_{\M^{(\hat{f}_t,P_\star)}(c)}(s_0)]|
    \\
    & 
    \leq
    \sqrt{\phi_t(\pi)}
    \cdot \sqrt{ 68H\log(4|\F|t^3/\delta)},
\end{align*}
where $\phi_t(\pi)$ is the contextual potential of $\pi$ at round $t$, which is defined as 
$
    \phi_t(\pi) : = \E_c \left[ \sum_{h=0}^{H-1} \sum_{s_h \in S^c_h} \frac{q_h(s_h|\pi(c;\cdot) ,P^c_\star)}{ \sum_{i=1}^{t-1} \I[\pi(c;s_h)= \pi_i(c;s_h)]q_h(s_h | \pi_i(c;\cdot),P^c_\star)}\right]
$ and 
%is the contextual potential of $\pi$ at round $t$,
$\M^{(f,P_\star)}(c) = (S,A,P^c_\star,f(c,\cdot,\cdot),s_0, H)$ for any $f \in \F$. The true MDP is $\M(c) = \M^{(f_\star,P_\star)}(c)$. Also, $\pi_i$ is the selected policy at round $i$.

\noindent\textbf{Step 3:}\label{step-3}  
relax the confidence bound of step $2$ to be additive. 
In~\cref{lemma:sq-trick} we show that under the good event of step 2, 
%with probability at least $1-\delta/2$,
for all $t>|A|$ and any policy $\pi \in \Pi_\C$, for $\beta_t = \sqrt{ \frac{17 t \log(4|\F|t^3 /\delta)  }{|S||A|}}$ it holds that
\begin{align*}
    |\E_c[V^{\pi(c;\cdot)}_{\M(c)}(s_0) ]-  \E_c&[V^{\pi(c;\cdot)}_{\M^{(\hat{f}_t,P_\star)}(c)}(s_0)]| 
    \\
    &
    \leq
    \beta_t \cdot (\phi_t(\pi)
    +    {H |S| |A|}/{t} )
    .
\end{align*}

\noindent\textbf{Step 4:}\label{step-4} 
bound the cumulative contextual potential $\phi_t$
over every round $t =|A|+1, |A|+2, \ldots ,T$.
In~\cref{lemma:Contextual-Potential} we show that  
For any sequence of selected policies $\{\pi_t \in \Pi_\C\}_{t=1}^T$ it holds that 
$$
    \sum_{t=|A| +1}^T \phi_t(\pi_t)\leq {|S||A|}{p^{-1}_{min}}(1+ \log(T/|A|))
    .
$$
% In addition, for any fixed policy $\pi \in \Pi_\C$, we have an improved bound of 
% $
%     \sum_{t=|A| +1}^T \phi_t(\pi)\leq \frac{H|A|}{p_{min}}(1+ \log(T/|A|))
% $.
%
By combining all the steps and applying Azuma's inequality, we obtain the following regret bound.
(See~\cref{thm:regret-bound-kd} and for a bound in terms of $|\G|$ see~\cref{corl:regret-kd-}).
%
%, which holds with high probability.
\begin{theorem}[regret bound]\label{thm:regret-KD-main}
    For any $T > |A|$, finite function class $\F$ and $\delta \in (0,1)$, let $\beta_t = \sqrt{ \frac{17 t \log(4|\F|t^3/\delta) }{|S||A|}}$ for all $t \in [T]$.
    Then, with probability at least $1-\delta$ we have
    $$
        \Regrv_T(\text{RM-KD})
        \le
        % \\
        \widetilde{O}(
        (  p^{-1}_{min} + H )\sqrt{T |S||A|  \log{|\F|}/{\delta}}
        )
    .
    $$
\end{theorem}
We remark that in all of our algorithms, for $T \le |S||A|$ the regret is trivially bounded by $|S||A|H$. 

\subsection{Main Technical Challenges and Our Technique}\label{subsec:main-technical-challenges}
Following steps $2$ and $3$, a natural ``optimistic in expectation'' strategy is to select at round $t$
\begin{align*}
    \pi_t &\in \arg\max_{\pi \in \Pi_\C}
     \left\{ \E_c[V^{\pi(c;\cdot)}_{\M^{(\hat{f}_t,P_\star)}(c)}(s_0)] 
    + \beta_t \cdot \phi_t(\pi)
    \right\}
    \\
    &
    = \arg\max_{\pi \in \Pi_\C}
    \left\{\E_c[V^{\pi(c;\cdot)}_{\Mhat_t(c)}(s_0)]\right\}.
\end{align*}
This approach has an obvious three major drawbacks.\\
(1) The distribution over the contexts, $\D$, is unknown. Hence, we cannot compute $\E_c\left[V^{\pi(c;\cdot)}_{\Mhat_t(c)}(s_0)\right]$, for any policy $\pi$.\\
(2) Even when $\D$ is known, computing  $\pi_t \in \Pi_\C$ is intractable when the context space $\C$ is large.\\
(3) The representation of a context-dependent policy $\pi_t$ scales with the size of the context space $|\C|$, which can be huge. 

We overcome these hurdles using two observations.
The first observation is that
$$
    \max_{\pi \in \Pi_\C}\left\{\E_c \left[V^{\pi(c;\cdot)}_{\Mhat_t(c)}(s_0)\right]\right\} =
    \E_c \left[ \max_{\pi(c;\cdot) \in S \to A} V^{\pi(c;\cdot)}_{\Mhat_t(c)}(s_0)  \right]
    .
$$
We conclude that 
to compute a context-dependent policy $\pi_t \in \Pi_\C$ which maximizes LHS, we can compute for each context $c \in \C$ separately, a policy
$\pi_t(c;\cdot): S \to A$ that is optimal for $\Mhat_t(c)$.
For each context $c \in \C$ separately, solving the maximization problem in RHS can be done efficiently using a standard planning algorithm. 

The second observation is that in every round $t$, we do not have to know the full representation of $\pi_k$, for all $k \leq t$, but only
%. Notice that at every round $t \leq T$, we only use 
the mappings $\{\pi_k(c_i;\cdot)\}_{k=1}^t$ for the observed contexts $\{c_i\}_{i=1}^t$. 
By taking an advantage of these two observations,
at every round $t \leq T$, we compute $\pi_k(c_t;\cdot)$ for all $k \leq t$, which can be done in $poly(|S|,|A|, H,t)$ using a planning algorithm. Using this, we obtain an efficient algorithm which is independent of $|\C |$.

%By that, we overcome the hardness of fully computing $\pi_k$, and avoid the memory-overload of saving $\{\pi_k\}_{k=1}^T$ full representation. 

\section{Unknown Context-Independent Dynamics}\label{sec:UCFD}
In this section, we assume the dynamics is unknown to the learner, but is independent of the context. Meaning, for all $c \in \C$, $P^c_\star = P_\star$.
We also assume the learner knows the (context-independent) partition of the states space to layers, $S = \{S_0, \ldots, S_H\}$, and the minimum reachability $p_{min}$.

\noindent\textbf{Algorithm overview.}
Similarly to~\cref{alg:RM-CMDP-KD-main}, we define an optimistic-in-expectation rewards function, but, since the dynamics is unknown, we replace $q(s|\pi_i(c_t;\cdot),P^{c_t}_\star)$ with its lower bound $p_{min}$.
We denote by $N_t(s,a)$ and $N_t(s,a,s')$ the number of visits to $(s,a)$ and $(s,a,s')$, respectively, up to round $t$.
%, for all $(s,a,s') \in S \times A \times S$.
To approximate the dynamics, we use a tabular approximation and maintain the following confidence bounds over it, denote them by $\xi_t(s,a) = 2\sqrt{\frac{|S| + 2 \log(4 |S||A|T^2/\delta)}{\max\{1, N_t(s,a)\}}}$, for all $(s,a)\in S \times A$.  
At round $t$, we compute an optimistic model w.r.t the rewards function $\widehat{r}^{c_t}_t$ and a deterministic optimal policy $\pi_t(c_t;\cdot)$ for it
%using Algorithm FOA 
% (See Appendix C in \cite{levy2022optimism} for an efficient computation of the optimistic MDP and policy).
% Algorithm FOA gets a rewards function as an input, and computes the dynamics and deterministic policy which maximize the value of the resulting MDP
, under the constraints that the optimistic dynamics is within the confidence interval. 
(See~\cref{Appendix:UCID} for an efficient implementation).
We remark that the resulting optimistic approximated dynamics is context-dependent, since it was computed w.r.t the context-dependent approximated rewards function.
%For more details, see~\cref{Appendix:UCID}. 
%
\begin{algorithm}
    \caption{(sketch) Regret Minimization for Unknown Context Independent Dynamics (RM-UCID)}
    \label{alg:RM-UCID-main}
    \begin{algorithmic}[1]
        \FOR{round $t >|A|$}
            \STATE{
            $
                \hat{f}_t \in 
                \arg\min_{f \in \F}
                \sum_{i=1}^{t-1} \sum_{h=0}^{H-1} ( f(c_i,s^i_h,a^i_h) - r^i_h)^2
            $ 
            is computed using the LSR oracle
            }
            \STATE{compute the empirical model for all $(s,a,s'): \Bar{P}_t(s'|s,a) = \frac{N_t(s,a,s')}{\max\{1, N_t(s,a)\} }$}
            \STATE{observe a fresh context $c_t \sim \D$}
            \FOR{$k= |A|+ 1, \ldots, t$}
                \STATE{compute for all
                    $
                        (s,a) \in S \times A:$\\
                    $
                        \widehat{r}_k^{c_t}(s, a) = \hat{f}_k(c_t,s,a) +
                        \frac{\beta_k}{ p_{min} \sum_{i=1}^{k-1}\I[a = \pi_i(c_t;s)]}
                    $}
                \STATE{compute an optimistic model $\Mhat_k(c_t) = (S, A, \widehat{P}^{c_t}_k, \hat{r}^{c_t}_k, s_0, H)$ and policy $\pi_k(c_t;\cdot)$}
                %using the sub-routine \texttt{FOA}}
                \COMMENT{See~\cref{alg:foa} in~\cref{Appendix:UCID}}
            \ENDFOR{}    
            \STATE{play $\pi_t(c_t;\cdot)$, observe trajectory $\sigma^t$ and update oracle}
            %\STATE{update counters and LSR oracle using $\sigma^t$}
            \ENDFOR{}
    \end{algorithmic}
\end{algorithm}

\noindent\textbf{Analysis overview.}
We construct confidence intervals for both the dynamics and rewards. 
%For the rewards, we analyse the error caused by the rewards approximation similarly to the known dynamics case. 
For the analysis, we define an intermediate CMDPs where for all $t>|A|$ and $c \in \C$: 
% $\M^{(\hat{r}_t,P_\star)}(c) = (S, A, P_\star, \hat{r}^{c}_t, s_0, H)$
(1) $\M^{(f,\widehat{P}_t)}(c) = (S, A, \widehat{P}^c_t, f(c,\cdot,\cdot), s_0, H)$, $f \in \F$ and $\widehat{P}^c_t$ is the optimistic dynamics w.r.t $\widehat{r}^c_t$ defined in~\cref{alg:RM-UCID-main}.
(2) $\M^{(f,P_\star)}(c) = (S, A, P_\star, f(c,\cdot,\cdot), s_0, H)$, $f \in \F$ and $P_\star$ is the true dynamics.
(3) $\M^{(\widehat{r}_t, P_\star)}(c) = (S,A,P_\star,\widehat{r}^c_t,s_0, H)$.
Let $\pi^\star \in \Pi_\C$ be an optimal policy of the true CMDP.

\noindent\textbf{Analysing the error caused by the rewards approximation.}
%Using a similar analysis to the analysis showed 
Similar to the analysis for the known dynamics (\cref{sec:KCDD}), we show in~\cref{lemma:sq-trick-UCID}
that with high probability,
%at least $1-\delta/4$ 
for all $t>|A|$, and any policy $\pi \in \Pi_\C$ the following holds:
\begin{align*}
    | 
    \E_c[V^{\pi(c;\cdot)}_{\M(c)}(s_0)]
    -
    \E_c&[V^{\pi(c;\cdot)}_{\M^{(\hat{f}_t ,P_\star)}(c)}(s_0)]
    |
    \\
    &\hspace{10mm} \leq 
   \beta_t \left(\phi_t(\pi) + {H|S||A|}/{t} \right),
\end{align*}
where we abuse the contextual potential in round $t$ as
$
    \phi_t(\pi):=\E_c[ \sum_{h = 0}^{H-1} \mathop{\sum}_{s_h \in S_h} 
    \frac{ q_h(s_h, \pi(c;s_h)| \pi(c;\cdot),P_\star)}{  p_{min}\sum_{i=1}^{t-1}\I[ \pi(c;s_h) = \pi_i(c;s_h)] } ]
$.

\noindent\textbf{Analysing the error caused by the dynamics approximation.}
We show that with high probability the following good event holds.
For all $t > |A|$ and $(s,a) \in S\times A$, we have that 
$\|\Bar{P}_t(\cdot|s,a) - P_\star(\cdot|s,a)\|_1 \leq \xi_t(s,a)$.
In~\cref{lemma:optimality-of-pi-t-and-p-t} we show that 
under this good event, our optimistic approximated model $\Mhat_t(c) = (S,A,\widehat{P}^c_t,\widehat{r}^c_t,s_0, H)$ and the selected policy $\pi_t(c;\cdot)$ satisfy for all $c \in \C$ and $t>|A|$ that
$
    V^{\pi_t(c;\cdot)}_{\Mhat_t(c)}(s_0)
    \geq 
     V^{\pi^\star(c;\cdot)}_{\M(\widehat{r}^c_t,P_\star)}(s_0)
$.
When combining the latter inequality with the confidence bounds over the rewards, we obtain 
(see~\cref{lemma:optimism}), 
for all $t>|A|$ that
\begin{align*}
    \E_c [ V^{\pi^\star(c;\cdot)}_{\M(c)}(s_0)]
    -\E_c [V^{\pi_t(c;\cdot)}_{\Mhat_t(c)} (s_0) ]
    \leq 
    \beta_t \cdot
    {H|S||A|}/{t} 
    .
\end{align*}
Moreover,
in~\cref{lemma:UCID-dynamixs-value-error-over-t} 
we show that under good event of the dynamics approximation, for $T>|S||A|$ with high probability it holds that 
\begin{align*}
        \sum_{t= |A| + 1}^T \E_c[&V^{\pi_t(c;\cdot)}_{\M^{(\hat{f}_t, \widehat{P}_t)}(c)}] - \E_c[V^{\pi_t(c;\cdot)}_{\M^{(\hat{f}_t ,P_\star)}(c)}]
        \\
        &\hspace{8mm}\leq 
        {O}(H^{1.5}|S|\sqrt{|A|T  }\log(|S||A|T^2/\delta))
    .
\end{align*}
Lastly, we bound $\sum_{t= |A|+1}^T \phi_t(\pi_t)$ similarly to~\cref{sec:KCDD}.
\\
By combining all the above, and applying Azuma's inequality, we obtain the following regret bound
(see~\cref{thm:regret-bound-ucid} and for bound in terms of $|\G|$ see~\cref{corl:regret-bound-ucid-g}).
\begin{theorem}[regret bound]
    For any $T >|S||A|$, finite function class $\F$ and $\delta \in (0,1)$, for the choice of $\beta_t = \sqrt{ \frac{17 t\log(8|\F|t^3/\delta)}{ |S||A|}}$ for all $t$, 
    with probability at least $1-\delta$, 
    \begin{align*}
        \Regrv_T(\text{RM-UCID}&)
        \leq
        \widetilde{O} \Big(
         H^{1.5}|S| \sqrt{T |A|}\log({1}/{\delta})
         \\
         + &
         (H+{p^{-1}_{min}} )\cdot \sqrt{ 
        |S| |A| T\log(|\F|/\delta)}
        \Big)
        .  
    \end{align*}
\end{theorem}
For more details see~\cref{Appendix:UCID}.

\section{Unknown Context-Dependent Dynamics}\label{sec:UCDD}
In this section, we consider the most challenging case, where the dynamics is unknown and context-dependent.
We assume an access to a finite  function class $\Fp \subseteq  (S \times (S \times A \times \C) \to [0,1])$, for which every function $P \in \Fp$ satisfies 
$
    \sum_{s' \in S} P(s'|s,a,c)  = 1, \;\; \forall (s,a,c) \in S \times A \times \C
$.
%, through the LSR oracle.
We use $\Fp$ to approximate the context-dependent dynamics under the dynamics realizability assumption (\cref{assm:dynamics-realizability}).

\noindent\textbf{Algorithm outline.}
In Algorithm RM-UCDD (\cref{alg:RM-CMDP-UCDD-main}),
we approximate both the rewards and the dynamics using a LSR oracle. The first $|A|$ rounds are initialization rounds, as before. 
At round $t> |A|$, we compute the approximated rewards function for the context $c_t$ as is done in previous sections. For the dynamics approximation, we use the least square minimizer $\widehat{P}_t$. 
We define the approximated model for $c_t$, compute an optimal policy $\pi_t(c_t;\cdot)$ for it and run it to generate a trajectory and update the oracles.
We feed the LSR oracle for the dynamics with samples of the form $((c_t, s^t_h,a^t_h, s'), \I[s'=s^t_{h+1}])$ for all $t\le T$, $h \in [H-1]$ for every $s' \in S$, where $\I$ is an indicator function.
\begin{algorithm}
    \caption{Regret Minimization for CMDP with Unknown Context-Dependent Dynamics}
    \label{alg:RM-CMDP-UCDD-main}
    \begin{algorithmic}[1]
        \STATE
        { 
            \textbf{inputs:} MDP parameters: 
                $S $, $A$, $H$, $s_0$. Confidence  $\delta>0$ and tuning parameters $\{\beta_t,\gamma_t\}_{t=1}^T$. Minimum reachability parameter $p_{min}>0$.
        } 
        \STATE{\textbf{initialization:} in round $i\leq |A|$
        run $\pi_i(c;s) := a_i$}
        \FOR{round $t =  |A| + 1, \ldots, T$}
            \STATE{
            $
                \hat{f}_t \in 
                \arg\min_{f \in \F}
                \sum_{i=1}^{t-1} \sum_{h=0}^{H-1} ( f(c_i,s^i_h,a^i_h) - r^i_h)^2
            $ is computed using the LSR oracle 
            }
            \STATE{ also compute  
            $
                {\widehat{P}_t \in 
                \arg\min_{\widetilde{P} \in \Fp}
                \sum_{i=1}^{t-1} \sum_{h=0}^{H-1} \sum_{s' \in S}}$
                $
                ( \widetilde{P}^{c_i} (s'| s^i_h,a^i_h) - \I[s' = s^i_{h+1}])^2
            $ using the LSR oracle
            }
            \STATE{observe a fresh context $c_t \in \C$}
            \FOR{$k=|A|+1, |A|+ 2, \ldots, t$}
                \STATE{compute for all 
                    $
                        (s,a)\in S \times A :$\\
                    $
                        \hat{r}_k^{c_t}(s, a) = \hat{f}_k(c_t,s,a) + 
                        \frac{\beta_k + H|S|\gamma_k}{  p_{min}\sum_{i=1}^{k-1}\I[a = \pi_i(c_t;s)]}
                    $
                    }
                %\STATE{    where $\beta_k=$}
                \STATE{define $\Mhat_k(c_t) = (S, A, \widehat{P}^{c_t}_k, \hat{r}^{c_t}_k, s_0, H)$}
                \STATE{
                compute $\pi_k(c_t,\cdot) \in \arg\max_{\pi \in S \to A}V^{\pi}_{\Mhat_k(c_t)}(s_0)$ using planning algorithm}
            \ENDFOR{}    
            \STATE{play $\pi_t(c_t,\cdot)$, observe trajectory $\sigma^t$ and update oracles}
            \ENDFOR{}
    \end{algorithmic}
\end{algorithm}

\noindent\textbf{Analysis outline.}
In the analysis, we define the following intermediate MDPs for any context $c \in \C$:
(1) $\M^{(\widehat{r}_t, P)}(c) = (S,A,P^c,\widehat{r}^c_t,s_0, H)$ 
%for any $t>|A|$ 
for context-dependent dynamics $P \in \Fp$, where $\widehat{r}^c_t$ is the approximated rewards function in round $t$, which defined in~\cref{alg:RM-CMDP-UCDD-main}. By definition, $\Mhat_t(c) = \M^{(\widehat{r}_t, \widehat{P}_t)}(c)$.
(2) $\M^{(f, P)}(c) = (S,A,P^c,f(c,\cdot,\cdot),s_0, H)$ for any $f \in \F$ and $P \in \Fp$.
%where $P^c_\star$ is the true dynamics associated with the context $c$.
By definition, $\M(c) = \M^{(f_\star, P_\star)}(c)$.
%In addition, we use the notations of $\phi_t$ and $\psi_t$ defined in previous sections.
We denote by $\psi_t(\pi)$ the contextual potential at round $t$, which is defined as
$
    \psi_t(\pi) : = \E_c [ \sum_{h=0}^{H-1} \sum_{s_h \in S^c_h} \frac{q_h(s_h|\pi(c;\cdot) ,\widehat{P}^c_t)}{ p_{min}\sum_{i=1}^{t-1} \I[\pi(c;s_h)= \pi_i(c;s_h)]}]
$.

\noindent\textbf{Analysing the error caused by the rewards approximation.}
Similar to the known dynamics setting (\cref{sec:KCDD}), 
we show 
in~\cref{lemma:sq-trick-rewards-UCDD}
that with %high 
probability at least $1-\delta/4$, 
for all $t>|A|$ and $\pi \in \Pi_\C$ the following holds.
\begin{align*}
    |\E_c[V^{\pi(c;\cdot)}_{\M^{(\hat{f}_t ,\widehat{P}_t)}(c)}(s_0)] &-\E_c[V^{\pi(c;\cdot)}_{\M^{({f}_\star ,\widehat{P}_t)}(c)}(s_0)]|
    \\
    &\hspace{10mm}\leq  
    \beta_t (\psi_t(\pi) + {H|S||A|}/{t} )
    .
\end{align*}
% 1.$
%     {\E_c[V^{\pi(c;\cdot)}_{\M^{(\hat{r}_t ,P_\star)}(c)}(s_0) -V^{\pi(c;\cdot)}_{\M(c)}(s_0)]
%     \leq 
%    \beta_t (2\psi_t(\pi) + \frac{H|S||A|}{t} )}
% $\\ 
% 2.$
%     {\E_c[V^{\pi(c;\cdot)}_{\M(c)}(s_0) -V^{\pi(c;\cdot)}_{\M^{(\hat{r}_t ,P_\star)}(c)}(s_0)]
%     \leq 
%    \beta_t (\psi_t(\pi) + \frac{H|S||A|}{t} )}
% $

\noindent\textbf{Analysing the error of the dynamics approximation.}

\noindent\textbf{Key observation.}
Let $\B$ be a random variable which generates the next state $s_{h+1}$ given the true dynamics associated with $c$, $P^c_\star$, the state $s_h$ and the action $a_h$.
The random variable 
$\B(P^c_\star,s_h,a_h)$ is distributed $P^c_\star(\cdot|s_h,a_h)$.
Our observation is that since the CMDP is layered, given the context $c_t$ state $s_h^t$ and action $a_h^t$, we have that the random variables
$\B(P^{c_t}_\star,s_h^t,a_h^t)$ and $( s^t_0, a^t_0, s^t_1, \ldots, s^t_{h-1}, a^t_{h-1}) $ are independent random variables.
Using that observation, we are able to extend our uniform convergence bound to the dynamics approximation. Hence, we can apply the four steps strategy above for the dynamics approximation as well.

\noindent\textbf{Step 1:}\label{step-1-dynamics}  
establish uniform convergence bound over any $t \geq 2$ and a fixed sequence of functions $P_2,P_3,\ldots \in \Fp$.
(See~\cref{lemma:UC-lemma-5-dynamics}).
The bound implies that for the least square minimizers sequence $\{\widehat{P}_t\}_{t=|A|+1}^T$ 
%and any $\delta \in (0,1)$, with probability at least $1-\delta/4$ 
with high probability, for all $t>|A|$,
%the following holds,
\begin{align*}
    \sum_{i=1}^{t-1} \sum_{h=0}^{H-1} 
    \mathop{\E}_{c_i,s^i_h,a^i_h}
    \Big[
    %\sum_{s'\in S} (
    \|\widehat{P}^{c_i}_t(\cdot| s^i_h, a^i_h)& - P^{c_i}_\star(\cdot| s^i_h, a^i_h)\|^2_2
    %)^2 
    \;\Big| \Hist_{i-1}\Big]
    \\    
    &\leq 
    72H|S|\log(8|\Fp|t^3/\delta)
    .
\end{align*}

\noindent\textbf{Step 2:}\label{step-2-dynamics} 
construct a confidence bound over the value of any given policy w.r.t the approximated and true dynamics, where the rewards function is $f_\star$.
%the approximated rewards at round $t$.
The confidence bound holds with high probability, in expectation over the contexts.\\ 
Formally, in~\cref{lemma:CB-policy-dynamics}
we show that with probability at least $1-\delta/4$, for all $t>|A|$ and any policy $\pi \in \Pi_\C$ we have 
\begin{align*}
    |\E_c[&V^{\pi(c;\cdot)}_{\M^{(f_\star,P_\star)}(c)}(s_0)]  - \E_c[V^{\pi(c;\cdot)}_{\M^{(f_\star, \widehat{P}_t)}(c)}(s_0)]|
    \\
    &\hspace{14mm}\leq
   \sqrt{H|S| \psi_t(\pi)}
    \cdot \sqrt{72 H^2|S|\log(8|\Fp|t^3/\delta)}
    .   
\end{align*}

\noindent\textbf{Step 3:}\label{step-3-dynamics}  
relax the confidence bound in step $2$ to be additive.
In~\cref{lemma:sq-trick-UCDD} 
we show that under the good event of step 2 
for all $t>|A|$ and any policy $\pi \in \Pi_\C$, for $\gamma_t = \sqrt{ \frac{18t \log(8|\Fp|t^3/\delta) }{|S| |A|}}$ it holds that, 
\begin{align*}
   |\E_c[V^{\pi(c;\cdot)}_{\M^{(f_\star,P_\star)}(c)}(s_0)] & - \E_c[V^{\pi(c;\cdot)}_{\M^{(f_\star, \widehat{P}_t)}(c)}(s_0)]|
    \\
    &\hspace{4mm} \leq 
    \gamma_t  H  |S| \left( \psi_t(\pi) + {H |S|  |A|}/{t}\right)  
    .
\end{align*}

\noindent\textbf{Step 4:} bound the sum of contextual potential functions similarly to shown for the rewards, in previous sections.

Using all the above, we obtain the optimism lemma (\cref{lemma:optimism-UCDD}) 
which states that under the good events of step 2 for both the dynamics and rewards approximation, for all $t>|A|$,
    \begin{align*}
        \E_c[V^{\pi^\star(c;\cdot)}_{\M(c)}(s_0) -  V^{\pi_t(c;\cdot)}_{\Mhat_t(c)}(s_0)]
        \leq 
        \frac{(H |S| \gamma_t + \beta_t) H|S||A|}{t}
        ,
    \end{align*}
yielding the following regret bound. (See~\cref{thm:regret-bound-ucdd,corl:regret-bound-ucdd-g}.)

\begin{theorem}[regret bound]
For any $\delta \in (0,1)$, $T>|A|$ and finite function classes $\F$ and $\Fp$, for the choice of $\beta_t = \sqrt{ \frac{17t  \log(8|\F|t^3/\delta) }{|S||A|}}$ and  $\gamma_t = \sqrt{ \frac{18t \log(8|\Fp|t^3/\delta) }{|S| |A|}}$ for all $t$, with probability at least $1-\delta$
%, for $p_{min} \leq 1/|S|$,
it holds that 
\begin{align*}
    &\Regrv_T(\text{RM-UCDD})
    \leq
    \\
    &\widetilde{O}( 
           (H+1/{p_{min}}) H|S|^{3/2}\sqrt{|A|T\log(\max\{|\F|,|\Fp|\}/{\delta})}
        ).
\end{align*}
\end{theorem}

\section{Lower bound}\label{sec:LB}
\begin{figure}
    \centering
    \caption{Lower bound illustration}\label{fig:lower-bound-main}
    \includegraphics[scale=1]{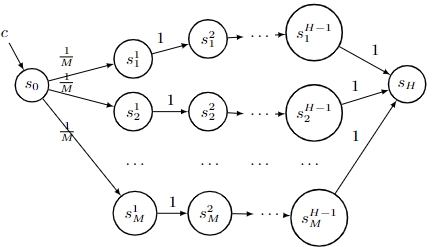}
\end{figure}

We present a lower bound for layered CMDP, where the dynamics is known and context-independent, which based on the lower bound for CMAB presented by~\citet{agarwal2012contextual}, in which $K=|A|$, $\G \subseteq (\C \times A \to [0,1])$ and $N \in \N$.

\begin{theorem}[Theorem $5.1$,~\citet{agarwal2012contextual}]\label{thm:lower-bound-CMAB-main}
For every $N $ and $K $ such that $\ln N / \ln K \leq T$, and every algorithm $\mathfrak{A}$, there exist a functions class $\G $ of 
cardinality at most $N$ and a distribution $D(c,r)$  for which
the realizability assumption holds, but the expected regret of $\mathfrak{A}$ is $\Omega(\sqrt{KT \ln  N/ \ln K})$.
\end{theorem}
%
%We show a lower bound for horizon $H \geq 2$.
%
\begin{theorem}[Lower bound for CMDP]
    Let $\delta \in (0,1)$, horizon $H \geq 2$ and $M, N \in \N$. 
    Let $T \geq 8 M \log ({|S|}/{\delta})+ 2M\ln N / \ln |A|$ and consider a CMDP $(\C,S,A, \M)$ for which ${|S| = M \cdot (H-1) + 2}$.

    Then, for any algorithm $\mathfrak{A}$, there exist a  base function class $\G \subseteq (\C \times A \to [0,1])$ of cardinality at most $N$,
    %where $\F = \G^S$, 
    and a distribution $D(c,s,a,r)$ for which the realizability assumption holds for $\F= \G^S$ and, with probability at least $1-\delta$, the expected regret of $\mathfrak{A}$ is 
    %$\Omega(\sqrt{TH|S||A|\ln|\F|/ \ln|S||A|})$.
    $\Omega (
    \sqrt{T  H |S| |A| \ln(N)/\ln(|A|)} )$.
\end{theorem}

\noindent\emph{proof idea.}
Solving the CMDP illustrated in~\cref{fig:lower-bound-main} is equivalent to solving $M(H-1) + 1$ CMAB problems. Hence, the theorem follows by~\cref{thm:lower-bound-CMAB-main}.(See~\cref{Appendix:LB}.)

\section{Extension: Remove the Reachability}\label{sec:extention}
The minimum reachability assumption allows us to limit the exploration-exploitation trade-off only to the actions selection, since any state is reached with probability at least $p_{min}>0$. 
% Hence, we avoid the need to explicitly explore the dynamics. In particular, in Algorithm RM-UCDD we avoid the need to add bonus that encourages dynamics exploration. 
%We remark that although the reachability assumption,
% Nevertheless, the exploration task is far from trivial, due to the context-dependent environment and the lack any additional knowledge regarding the function class.
% We suggest an extension that remove that assumption, and  obtain sub-linear regret of $\widetilde{O}(T^{3/4})$.
% An interesting direction for future research is to obtain $\widetilde{O}(\sqrt{T})$ regret bound, without our reachability assumption.
In this section we sketch a derivation of $\widetilde{O}(T^{3/4})$ regret, without the reachability assumption.

A first step towards removing the reachability assumption is to consider a dynamics class that is mixed with the uniform distribution with probability $\rho>0$.
%problem of agnostic function approximation for the dynamics. 
%
For every $P \in \Fp$ let $S^c_h(P^c)$ denote the $h$-layer defined by the transition matrix $P^c$.
Using $\Fp$ we define $\Fp(\rho)$, where for each dynamics $P^c\in \Fp$ there is a dynamics $\widetilde{P}^c\in \Fp(\rho)$, where in time $h$ with probability $\rho$ we transition to a random state in $S^c_h$.
%
% In each round, the oracle solves the optimization problem over the dynamics class $\Fp(\rho)$, where each function $\widetilde{P} \in \Fp(\rho)$ satisfies for a related function $P \in \Fp$ that for all $h \in [H-1]$ and $(c,s,a,s') \in \mathcal{C}\times S^c_h(P^c) \times A \times S^c_{h+1}(P^c)$  it holds that 
% $
%     \widetilde{P}^c(s'|s,a) = (1- \rho)P^c(s'|s,a) + {\rho}/{|S^c_{h+1}(P^c)|}
% $.
%
Assume that we have an access to a LSR oracle that gets as inputs a parameter $\rho$ and a realizable function class $\Fp$.
Let $P_\star \in \Fp$ denote the true context-dependent dynamics, and $\widetilde{P}_\star \in \Fp(\rho)$ be the related dynamics in $\Fp(\rho)$.\\
Please notice the following observations:\\
(1) $\widetilde{P}_\star$ has the minimum reachability property for $p_{min} = \rho/|S|$, even if $P_\star$ does not have it. \\
(2) For every context $c \in \C$, layer $h \in [H-1]$ and state-action $(s,a) \in S^c_h \times A$, it holds that 
$
    \| P^c_\star(\cdot|s,a) - \widetilde{P}^c_\star(\cdot|s,a)\|_1 \leq 2 \rho
$. \\
For $\rho<1/2$ this also implies that
the function class $\Fp(\rho)$ has an agnostic approximation error of at most $2\rho$ (w.r.t the square loss).\\
(3) Let any rewards function $r \in [0,1]$, context $c \in \C$ and policy $\pi$. Then, the value of $\pi$ on the model defined by $(r, P^c_\star)$ and the value of $\pi$ on the model defined by $(r, \widetilde{P}^c_\star)$ differ by at most $\widetilde{O}(\rho H^2)$. 
% \\---
% Recall the true dynamics is $P_\star$ (no reachbility assumption), and $\widetilde{P}_\star$ is the related dynamics (with reachability).
% The above implies that 
% for any policy, the expected cumulative reward differs by at most $O(\rho H^2)$.
This implies that the optimal policy for one of them is near optimal for the other.\\
By (2), in this setting, our uniform convergence bound of Step 1 has an additional error of $2\rho$, 
which yields (approximately) an additional term of $\widetilde{O}(\rho H^2|S|)$ in the additive confidence bound of a policy (Step 3) for the dynamics approximation.\\ 
Hence, the overall regret bound is 
$
    \widetilde{O}( \rho H^2|S|T + H^2|S|^{3/2}\sqrt{T|A|\log(\max\{|\F|,|\Fp|\}/\delta)}|S|/ \rho )
$.\\ 
For $\rho \approx  |S|^{3/4} T^{-1/4}$, we obtain a regret bound of $\widetilde{O}(H^2|S|^{7/4}T^{3/4}\sqrt{|A|\log(\max\{|\F|,|\Fp|\}/\delta)})$.
%for the unknown context-dependent dynamics case.
%
Therefore, our approach yields a sub-linear regret bound that does not depend on the minimum reachability parameter $p_{min}$, which is now a tuned parameter.
%, for the most difficult setting of unknown and context-dependent dynamics.

\section{Discussion}\label{sec:disscution}
%In this paper, we analyse stochastic CMDP in three different  settings, under minimum readability assumption.
To the best of our knowledge, this work is the first that obtains sub-liner regret bounds using general function approximation (i.e., without additional structural assumption regarding the CMDP or the function classes) and to present an expected regret lower bound.
Our results can be naturally extended to infinite function classes using covering numbers analysis (see e.g.,~\citet{shalev2014UnderstandingMLBook}).
%We leave that for future work.
%
Our algorithms has $poly(|S|,|A|,H, T)$ running time and space complexity, assuming an efficient least-square regression oracle.\\
\textbf{The main advantages of our technique:} (1) We present a novel confidence interval for general function approximation in CMDPs.
(2) We use an access to a standard offline least-square regression oracle, which we call only $O(T)$ times.
(3) Our algorithms do no fully represent the selected context-dependent policy at each time step, as the representation of it scales linearly in the context space size $|\C|$, which can be huge, but rather compute it only for the observed contexts.
%(4) We believe the our optimistic approach can be extend to other related setting.

\noindent\textbf{Tightness of our bounds.}
Consider our regret upper bounds in terms of the base  class $\G$ cardinality, recalling that $\F = \G^S$.
\emph{Known context-dependent} dynamics:
$
    \widetilde{O}(
    (H+{p^{-1}_{min}})|S|\sqrt{T|A|\log({|\G|}/{\delta}}))
$. 
For the \emph{Unknown context-independent} dynamics:
$\widetilde{O} \Big(
H^{1.5}|S| \sqrt{T |A|}\log({1}/{\delta})
+ 
(H+{p^{-1}_{min}} )\cdot |S| \sqrt{ |A| T\log(|\G|/\delta)}\Big)
$. 
\emph{Unknown context-dependent} dynamics:
$
    \widetilde{O}( 
    (H+{p^{-1}_{min}}) H|S|^{3/2}\sqrt{|A|T\log(\max\{|\G|,|\Fp|\}/{\delta})})
$.
On the other hand, recall 
our lower bound is $\Omega (\sqrt{T  H |S| |A| \ln(|\G|)/\ln(|A|)} )$. While our dependency in $T$, $|A|$ and $|\G|$ is near-optimal, bridging the gap in $|S|$, $H$ and $p_{min}$ is an important open question.

\section*{Acknowledgements}
This project has received funding from the European Research Council (ERC) under the European Union’s Horizon 2020 research and innovation program (grant agreement No. 882396), by the Israel Science Foundation(grant number 993/17), Tel Aviv University Center for AI and Data Science (TAD), and the Yandex Initiative for Machine Learning at Tel Aviv University.

\bibliography{references}

\newpage
\appendix
\onecolumn

\newpage\section{Extended Related Work}\label{Appendix:Extended-Related-Work}
Bellow we provide an extended related literature review.

\noindent\textbf{Contextual Reinforcement Leaning.}
CMDP was first introduce by \citet{hallak2015contextual}. \citet{modi2018markov} gives a general framework for deriving generalization bounds as a function of the covering number for smooth CMDPs and contextual linear combination of $d$ MDPs, where $d$ is a finite number of MDPs, and the context-space is the $d-1$ probability simplex.  \citet{modi2020no}
give a regret analysis for Generalized Linear Models (GLMs). Our function approximation framework is much more general than GLM and smooth CMDP.

\citet{foster2021statistical} 
present new statistical complexity measure, the decision-estimation coefficient, for interactive decision making. They show an application of it to obtain regret upper bound for Contextual RL. They assume an access to an online estimation oracle with regret guarantees, denote it $\textbf{Est}$.
Their reference model class is the class of all contextual MDPs $\mathcal{M}$ and randomized policies $\Pi$.
Given the observations and played policies up to the current time step, the online estimation oracle returns an estimated CMDP. Given the current context, they use it to compute a distribution over policies, and sample the played policy from it. They obtain $\tilde{O}(\sqrt{T \cdot \textbf{Est}})$ regret.

The main disadvantages of their approach are
that their online estimation oracle is very strong and might be computationally inefficient. It is unclear whether their algorithmic results can be extended to support offline oracles for estimation.
In addition, the sample complexity of such an oracle is unclear. 
Moreover, the relation between their new complexity measure and known complexity measures for function approximation (i.e, VC/Pseudo/Fat-shattering/Natrajan dimension) that are commonly-used in offline supervised learning is unclear.
Their results are very general and capture many RL settings. 
In contrast, we use a standard and efficient offline least square oracle to build an approximated optimistic CMDP from scratch. Hence, we have much refined assumptions which allows us to use standard tools of supervised learning.

\citet{jiang2017contextual} consider Contextual Decision Processes (CDP) with low Bellman rank, and present  OLIVE, which is sample efficient for CDPs with a small Bellman rank. We do not make any assumptions on the Bellman rank.

\citet{levy2022learning} consider learning CMDPs using function approximation. They assume an access to an ERM oracle and derive sample complexity bounds to compute $\epsilon$-optimal policy under four different settings: where the dynamics known and unknown, and context dependent or context-free. They assume no additional structural assumption regarding the CMDP. Their method provide the first general and efficient reduction from CMDP to offline supervised learning. However, their sample complexity bounds are not optimal, as our lower bound shows.
We, in contrast, consider online learning problem where the goal is regret minimization. We obtain $\sqrt{T}$ regret under the minimum reachability assumption while \citet{levy2022learning} has no such an assumption but obtain much higher dependency on $T$.

Contextual MDPs are an extension of Contextual multi-armed bandits (CMAB) to MDPs. 
%\noindent\textbf{Contextual Bandits.} Contextual bandits (
CMAB are a natural extension of the Multi-Arm Bandit (MAB), augmented by a context which influences the rewards \cite{Slivkins-book,MAB-book}. \citet{agarwal2014taming}  use efficiently an optimization oracle  to derive an optimal regret bound. 
Squared-loss regression based approaches appear in \citet{agarwal2012contextual,foster2018practical,foster2020beyond,simchi2021bypassing}.
We differ from CMAB, since our main challenge is the dynamics, and the need to optimize future rewards, which is the case in most RL settings.

\citet{xu2020upper} present the first optimistic algorithm for CMAB. They assume an access to a least-square regression oracle and achieve $\tilde{O}(\sqrt{T |A| \log |\F|})$ regret, where $\F$ is a finite and realizable function class uses to approximate the rewards. They also show a result for infinite function class using covering numbers analysis. We adapt their approach and extend it to CMDP.

\noindent\textbf{Inverse Gap Weighting (IGW) technique.} \citet{foster2020beyond,simchi2021bypassing}
apply the IGW technique to CMAB and obtain $\tilde{O}(\sqrt{T|A|})$ regret, assuming an access to an online/offline least square regression oracle, respectively. However, we do not see any straight-forward extension of their approach to CMDP which is both computationally efficient and has an optimal regret, under the same least-square oracle assumption (even when the dynamics is known to the learner).
In more detail, 
consider the following application of IGW to CMDP. 
At time $t$:
(1) Compute an approximation for the context-dependent rewards function $\hat{f}_t$.
(2) Observe a context $c_t$ and choose policy $\pi_{t}$ according to a distribution over deterministic policies. The probability of a policy $\pi$ is proportional to $1/{(V^{\pi^\star_t}_t(s_0) - V^\pi_t(s_0))}$, where $\pi^\star_t$ is the optimal policy for the rewards $\hat{f}_t$. 
(3) Experience trajectory and update the function approximation.
Using the analysis of \citet{foster2020beyond,simchi2021bypassing} one can obtain $\tilde{O}(\sqrt{T \cdot 2^{|S||A|}})$ regret and it is computationally inefficient.
A similar approach using the $Q$ function will also similarly fail.
Consider at time $t$ selecting a stochastic policy such that $\pi_{t}(a|s)$ is proportional to $1/{(Q^{\pi_t}_t(s, \pi^\star_t(s)) - Q^{\pi_t}_t(s, a))}$. This attempt will fail due to the lack of optimism and the changing-per-context optimal policy.
\citet{foster2021statistical} apply IGW to CMDP and obtain optimal regret. However they use the strong online estimation oracle discussed above.

\noindent\textbf{Additional works.} There are works that considered the case were the contexts are unobservable \cite{Latent-context-MDP,eghbal-zadeh2021learning,eghbal2021context}, latent states \cite{KrishnamurthyAL16}, spectral methods \cite{sprctelCMDP2016}, transfer learning \cite{zhang2020transfer}, and more. All those issue are somewhat unrelated to our main focus.

\section{Known Dynamics}\label{Appendix:KD}

In this section we present our algorithm and regret result for the known context-dependent dynamics setting. We present this section to introduce the main building blocks in our approach, which we will later extend to the unknown dynamics cases.
We operate under the following assumptions.

\noindent\textbf{Known context-dependent dynamics.}{
We assume that there is a mapping $P_\star$ from any context $c\in \C$ to a dynamics $P^c_\star$. In the known dynamics case the learner is familiar with  that mapping. 
Recall we assume that for each context $c$, the related MDP $\M(c)$ is layered and loop free.
We remark that the partition to layers is context-dependent. Since the dynamics is known, the partition of the state space into layers can be easily computed for every context $c \in \C$.
%This is true since if state $s$ belongs to layer $h$ then $h$ is the only index for which we have $q_h(s| \pi_s(c;\cdot),P^c_\star) >0$ where $\pi_s(c;\cdot)$ is the policy with the highest probability to visit state $s$ when the dynamics is $P^c_\star$.
Hence, we assume that for each context, the learner also knows the partition to layers, which naturally reveals to her when the dynamics is known.
%
% When the dynamics is known, for any given context $c \in \C$, and state $s \in S$ we can compute $\pi_s(c;\cdot)$, the policy that maximizes the probability of reaching state $s$ under the dynamics $P^c_\star$, and the appropriate occupancy measure $q_h(s| \pi_s(c;\cdot),P^c_\star) $ for every layer $h \in [H]$, in $poly(|S|,|A|,H)$ time using a standard planing algorithm.
\\
For every context $c$, we denote by $S^c_0, S^c_1, \ldots ,S^c_{H-1}, S^c_H$ the disjoint layers, and for all $c \in \C$, $ S = \bigcup_{h \in [H]}S^c_h$. 
Recall that for every context $c$ we have that $S^c_0 = \{s_0\}$ and $S^c_H = \{s_H\}$, meaning there are unique start and final states, which are identical for all of the contexts.
Clearly, the above assumptions do not limit the generality of our results. 
}
%
% Recall the definition of a context-dependent policy:
% \begin{definition}[context-dependent policy]
%     A stochastic context-dependent policy $\pi = \left( \pi(c;\cdot): S \to \Delta(A) \right)_{c \in \mathcal{C}}$ maps a context $c \in \mathcal{C}$ to a policy $\pi(c;\cdot) : S \to \Delta(A)$.  
%     A deterministic context-dependent policy $\pi = \left( \pi(c;\cdot): S \to A) \right)_{c \in \mathcal{C}}$ maps a context $c \in \mathcal{C}$ to a policy $\pi(c;\cdot) : S \to A$.    
%     Let $\Pi_\C$ denote the class of all deterministic context-dependent policies.
% \end{definition}
% %
% Using the notation of context-dependent policy, we define the following reachability assumption for contextual MDPs.

\noindent\textbf{Minimum reachability.}{
Recall the definition of a context-dependent policy.\\
A stochastic context-dependent policy $\pi = \left( \pi(c;\cdot): S \to \Delta(A) \right)_{c \in \mathcal{C}}$ maps a context $c \in \mathcal{C}$ to a policy $\pi(c;\cdot) : S \to \Delta(A)$.\\
A deterministic context-dependent policy $\pi = \left( \pi(c;\cdot): S \to A) \right)_{c \in \mathcal{C}}$ maps a context $c \in \mathcal{C}$ to a policy $\pi(c;\cdot) : S \to A$.\\
Let $\Pi_\C$ denote the class of all deterministic context-dependent policies.
\\
Using the notation of context-dependent policy, we define the following reachability assumption for contextual MDPs.
We assume that there exists \textbf{$p_{min} \in (0,1]$} such that for every context $c \in \C$, layer $h \in [H]$ and a state $s_h \in S^c_h$ and any context-dependent policy $\pi \in \Pi_{\C}$ it holds that 
$
    q_h(s_h | \pi(c;\cdot), P^c_\star) \geq p_{min}
$. 
We remark that $p_{min}$ is not (explicitly) given to the learner in this section.
Recall $q(s|\pi(c;\cdot),P^c_\star)$ denotes the probability of visiting state $s \in S$ when playing $\pi(c;\cdot)$, which is the policy that $\pi$ defines for the context $c$, on the dynamics $P^c_\star$.
When the dynamics is layered and loop-free, then 
$
    q(s|\pi(c;\cdot),P^c_\star) = q_h(s | \pi(c;\cdot), P^c_\star) \geq p_{min}
$ if and only if $s \in S_h$.
%
%\Orin{We might be able to remove that assumption since we indeed have $\pi_t$ is the counter}
}
\noindent\textbf{Reward function approximation.}{
%\label{par:reward-function-approx}
We consider a finite function class $\G \subseteq (\C \times A \to [0,1])$.
Many times it would be more convenient to consider a finite function class 
$\F = \G^{S}$ where $f\in \F$ are functions of the form 
$f(c,s,a) = g_s(c,a)$ where $g_s \in \G$. Note that, $\log(|\F|) = |S|\log(|\G|)$.
Our algorithm gets as input the finite function class ${\F \subseteq \C \times S \times A \to [0,1]}$.
%, $|\F| < \infty$.
Each function $f \in \F$ maps context $c \in \C$, state $s \in S$ and action $a \in A$ to a (approximate) reward $r \in [0,1]$.
We use $\F$ to approximate the context-dependent rewards function using the LSR oracle under the  realizability assumption (\cref{assm:rewards-realizability}). The assumption states that 
% \begin{assumption}[rewards realizability]
% %\label{assm:rewards-realizability}
%     We assume that $\F$ is realizable, meaning,
there exists a function $f_\star \in \F$ such that 
    $
       f_\star(c,s,a) = r^c_\star(s,a) = \E[R^c_\star(s,a)|c,s,a]
    $. \\
% \end{assumption}
For mathematical convenience, we state our algorithms and regret upper bounds in terms of the cardinality of $|\F|$.
}

\noindent In the following, we use the \textbf{value function representation using occupancy measures}.
For any (non-contextual) policy $\pi$ and MDP $M = (S,A,P,r,s_0,H)$ the value function at any state $s \in S$ and layer $h' \in [H]$ can be written as
\[
    V^{\pi}_{M,h'}(s) 
    =
    \sum_{h = h'}^{H-1} \sum_{s_h \in S_h} \sum_{a_h \in A}
    q_h(s_h, a_h | \pi, P) \cdot r(s_h, a_h)
    .
\]

\subsection{Main Technical Challenges of the Optimism in Expectation Approach}\label{subsec:opt-in-exp}
In~\cref{lemma:sq-trick} we will later present, we show that with high probability, for all $t>|A|$ it holds that
\begin{align*}
    \left|\E_c\left[V^{\pi(c;\cdot)}_{\M(c)}(s_0)\right] - \E_c\left[V^{\pi(c;\cdot)}_{\M^{(\hat{f}_t,P_\star)}(c)}(s_0)\right]\right|
    \leq &
    \beta_t \cdot \E_c \left[ \sum_{h=0}^{H-1} \sum_{s_h \in S^c_h} \frac{  q_h(s_h|\pi(c;\cdot) ,P^c_\star)}{\sum_{i=1}^{t-1} \I[\pi(c;s_h)= \pi_i(c;s_h)]q_h(s_h | \pi_i(c;\cdot),P^c_\star)}\right]
    \\
    &+
    \beta_t \frac{H  |S| |A|}{t},
\end{align*}
where $\hat{f}_t$ is the least square minimizer at round $t$, $\M^{(\hat{f}_t,P_\star)}(c) = (S,A,P^c_\star,\hat{f}_t(c,\cdot,\cdot),s_0,H)$,
$\pi_i$ is the context-dependent policy selected in round $i \in [t-1]$ and $\beta_t$ is a parameter related to the size of the function class $\F$.
% Consider Algorithm RM-KD (\cref{alg:RM-CMDP-KD}).
% In every round $t > |A|$, given the function approximation of the rewards up to time $t$, context context $c_t \in \C$,
% we compute an approximated MDP associated with $c_t$, $\Mhat_t(c_t) = (S,A,P^{c_t}_\star,\hat{r}_t^{c_t}, s_0, H)$ where for all $h \in [H-1]$, $s_h \in S^{c_t}_h$ and $a_h \in A$,
% $$
%     % \forall h \in [H-1], \; s_h \in S^{c_t}_h,\; a_h \in A: \;\;
%     \hat{r}_t^{c_t}(s_h, a_h) := \hat{f}_t (c_t,s_h,a_h) + \frac{\beta_t}{ \sum_{i=1}^{t-1}\I[a_h = \pi_i(c_t;s_h)]q_h(s_h| \pi_i(c_t;\cdot),P^{c_t}_\star)
%     }.
% 
Consider the approximated MDP $\Mhat_t(c)$ defined in Algorithm RM-KD (\cref{alg:RM-CMDP-KD}).
For every context-dependent policy $\pi \in \Pi_{\C}$, consider the following explicit representation of the value function for any given context $c \in \C$.
\begin{equation}\label{eq:V-t-def}
    \begin{split}
        &V^{\pi(c;\cdot)}_{\Mhat_t(c)}(s_0)
        \\
        = &
        \sum_{h = 0}^{H-1} \sum_{s_h \in S^c_h} \sum_{a_h \in A}
        q_h(s_h, a_h | \pi(c;\cdot),P^c_\star) \cdot \widehat{r}_t^c(s_h, a_h)
        \\
        %\tag{Since $\pi$ is a deterministic context-dependent function}
        = &
        \sum_{h = 0}^{H-1} \sum_{s_h \in S^c_h} \sum_{a_h \in A}
        q_h(s_h, a_h | \pi(c;\cdot),P^c_\star) \cdot 
        \left( \hat{f}_t (c,s_h,a_h) + \frac{\beta_t}{\sum_{i=1}^{t-1}\I[a_h = \pi_i(c;s_h)]q_h(s_h|\pi_i(c;\cdot),P^c_\star) }\right)
        \\
        = &
        \sum_{h = 0}^{H-1} \sum_{s_h \in S^c_h} \sum_{a_h \in A}
        q_h(s_h|  \pi(c;\cdot), P^c_\star) \I[a_h = \pi(c;s_h)] \cdot 
        \left( \hat{f}_t (c,s_h,a_h) + \frac{\beta_t}{\sum_{i=1}^{t-1}\I[a_h = \pi_i(c;s_h)]q_h(s_h|\pi_i(c;\cdot),P^c_\star) }\right)
        \\
        = &
        \sum_{h = 0}^{H-1} \sum_{s_h \in S^c_h} 
        q_h(s_h|\pi(c;\cdot), P^c_\star) \cdot 
        \left( \hat{f}_t (c, s_h,  \pi(c;s_h)) + \frac{\beta_t}{  \sum_{i=1}^{t-1}\I[ \pi(c;s_h) = \pi_i(c;s_h)]q_h(s_h|\pi_i(c;\cdot),P^c_\star) }\right)
        \\
        = &
        V^{\pi(c;\cdot)}_{\M^{(\hat{f}_t,P_\star)}(c)}(s_0) +  \sum_{h=0}^{H-1} \sum_{s_h \in S^c_h} \frac{ q_h(s_h|\pi(c;\cdot),P^c_\star) \cdot \beta_t}{  \sum_{i=1}^{t-1} \I[\pi(c;s_h)= \pi_i(c;s_h)]q_h(s_h | \pi_i(c;\cdot),P^c_\star)}
        ,        
    \end{split}
\end{equation}
where the third identity is since $\pi$ is a deterministic policy. The forth identity is since the non-zero terms are where $a_h = \pi(c;s_h)$.

Hence, following the above, 
%~\cref{lemma:CB-policy,lemma:sq-trick},
a natural ``optimistic in expectation'' approach is to choose at round $t$ the policy
    \begin{align*}
        \pi_t &\in \arg\max_{\pi \in \Pi_\C}
        \left\{
            \E_c[V^{\pi(c;\cdot)}_{\M^{(\hat{f}_t,P_\star)}(c)}(s_0)] 
            +
            \beta_t \cdot
            \E_c \left[ \sum_{h=0}^{H-1} \sum_{s_h \in S^c_h} \frac{ q_h(s_h|\pi(c;\cdot),P^c_\star)}{  \sum_{i=1}^{t-1} \I[\pi(c;s_h)= \pi_i(c;s_h)]q_h(s_h | \pi_i(c;\cdot),P^c_\star)}\right]
            %\cdot \sqrt{ 68\log(4|\F|t^3(H)^2/\delta)}
        \right\}
        \\
        &= \arg\max_{\pi \in \Pi_\C}
        \{
            \E_c[V^{\pi(c;\cdot)}_{\Mhat_t(c)}(s_0)]
        \},
    \end{align*}
where the equality follows from the value function derivation in~\cref{eq:V-t-def}.
%, for an appropriate choice of $\beta_t$.
% This approach has an obvious drawback which is that the distribution over the context $\D$ is unknown. 
% We overcome that hurdle using the observation that we can maximize the value function for every context $c \in \C$ separately.
% Since the following trivially holds,
% \begin{align*}
%     \max_{\pi \in \C\times S \to A}\left\{\E_c \left[V^{\pi(c;\cdot)}_{\Mhat(c)}(s_0)\right]\right\} =
%     \E_c \left[ \max_{\pi(c;\cdot) \in S \to A} \{ V^{\pi(c;\cdot)}_{\Mhat(c)}(s_0) \} \right],
% \end{align*}
% We can compute a policy $\pi :\C \times (S \to A)$ such that $\pi(c;\cdot): S \to A$ is optimal for $\Mhat_t(c)$, for every context $c \in \C$ separately.
% By the above $\pi$ is also optimal in expectation, since it maximizes the value for each context separately.
% For any given context $c \in \C$, we can compute an optimal policy $ \pi(c;\cdot) \in \arg \max_{\pi' \in  S \to A}  V^{\pi'}_{\Mhat(c)}(s_0)$ using standard planing algorithms. The computation time of the planing oracle is in $poly(|S|,|A|, H)$.

This approach has three major drawbacks.
\begin{enumerate}
    \item The distribution over the context $\D$ is unknown. Hence, we cannot compute $\E_c\left[V^{\pi(c;\cdot)}_{\Mhat_t(c)}(s_0)\right]$, for any policy $\pi$.
    \item Even when $\D$ is known, computing  $\pi_t \in \Pi_\C$ is intractable when the context space $\C$ is large.
    \item The representation of a context-dependent policy $\pi$ scales with the size of the context space $|\C|$, which can be huge.
\end{enumerate}
We overcome those hurdles using two observations.

\noindent\textbf{Observation 1:} it holds that
$$
    \max_{\pi \in \Pi_\C}\left\{\E_c \left[V^{\pi(c;\cdot)}_{\Mhat_t(c)}(s_0)\right]\right\} =
    \E_c \left[ \max_{\pi(c;\cdot) \in S \to A} \{ V^{\pi(c;\cdot)}_{\Mhat_t(c)}(s_0) \} \right]
    .
$$
We conclude that 
to compute a context-dependent policy $\pi_t \in \Pi_\C$ which maximizing LHS, we can compute for each context $c \in \C$ separately, a policy
$\pi_t(c;\cdot): S \to A$ that is optimal for $\Mhat_t(c)$.
For each context $c \in \C$ separately, solving the maximization problem in RHS can be done efficiently using a standard planning algorithm. 

\noindent\textbf{Observation 2:} in each round $t \leq T$, we do not have to know the full representation of $\pi_k$, for all $k \leq t$.
Notice that at every round $t \leq T$, we only use the mappings $\{\pi_k\}_{k=1}^t$ for the contexts $\{c_k\}_{k=1}^t$. 

By taking an advantage of those two observation,
at every round $t \leq T$, we compute $\pi_k(c_t;\cdot)$ for all $k \leq t$, which can be done in $poly(|S|,|A|, H,t)$ using a planing algorithm.
By that, we overcome the hardness of fully computing $\pi_k$, and avoid the memory-overload of saving $\{\pi_k\}_{k=1}^T$ full representation.

\subsection{Regret Minimization Algorithm}

For the known dynamics case, our algorithm works as follows.\\
Rounds $i= 1,2, \ldots, |A|$ are initialization rounds, where for each action $a_i \in A$ in turn, the agent simply runs the policy $\pi_i$ that always selects action $a_i$, for any  context $c_i \in \C$ and state $s \in S$.\\
In rounds $t = |A|+1, \ldots , T$ the agent observes context $c_t \sim \D$. Then the agent performs the following computations.
For every (previous) round $k  = |A|+1, \ldots ,t-1$ the agent computes the policy $\pi_k(c_t;\cdot)$
which is the optimal policy of $\Mhat_k(c_t)$. I.e., the optimal policy for $c_t$ with respect to the approximated model that uses the function $\hat{f}_k$, which was computed in round $k \leq t$. The agent computes $\pi_k(c_t;\cdot)$ in the following manner.
\begin{enumerate}
    \item The agent computes the optimistic reward function of round $k$ for the context $c_t$, using $\hat{f}_k$ and the policies $\{\pi_i(c_t;\cdot)\}_{i=1}^{k-1}$. The rewards function is
    \[
        \forall (s,a) \in S \times  A: \;\;
        \widehat{r}_k^{c_t}(s, a) = \hat{f}_k(c_t,s,a) + 
        \frac{\beta_k}{  \sum_{i=1}^{k-1}\I[a = \pi_i(c_t;s)]q(s | \pi_i(c_t;\cdot),P^{c_t}_\star)}. 
    \]
    We remark that since $P^{c_t}_\star$ is known, $q(s |\pi_i(c_t;\cdot),P^{c_t}_\star)$ can also be computed in polynomial time in $|S|, |A|, H$, for all $s \in S$.
    \item The agent defines the approximated MDP for $c_t$ at round $k$, $\Mhat_k(c_t) = (S, A, P^{c_t}_\star, \widehat{r}^{c_t}_k, s_0, H)$.
    \item The agent computes ${\pi_k(c_t,\cdot) \in \arg\max_{\pi \in S \to A}V^{\pi}_{\Mhat_k(c_t)}(s_0)}$ using a planning algorithm.
\end{enumerate}
% (1) The agent computes the optimistic reward function of round $k$ for the context $c_t$, using $\hat{f}_k$ and the policies $\{\pi_i(c_t;\cdot)\}_{i=1}^{k-1}$. The rewards function is
% \[
%     \forall (s,a) \in S \times  A: \;\;
%     \hat{r}_k^{c_t}(s, a) = \hat{f}_k(c_t,s,a) + 
%     \frac{\beta_k}{  \sum_{i=1}^{k-1}\I[a = \pi_i(c_t;s)]q(s | \pi_i(c_t;\cdot),P^{c_t}_\star)}. 
% \]
% We remark that since $P^{c_t}_\star$ is known, $q(s |\pi_i(c_t;\cdot),P^{c_t}_\star)$ can also be computed in polynomial time in $|S|, |A|, H$, for all $s \in S$.\\
% %
% (2) The agent defines the approximated MDP for $c_t$ at round $k$, $\Mhat_k(c_t) = (S, A, P^{c_t}_\star, \hat{r}^{c_t}_k, s_0, H)$.\\
% (3) The agent computes ${\pi_k(c_t,\cdot) \in \arg\max_{\pi \in S \to A}V^{\pi}_{\Mhat_k(c_t)}(s_0)}$ using a planning algorithm.
%
The agent computes $\pi_t(c_t;\cdot)$ in the same way, and then run it to generate trajectory $\sigma^t$. Lastly, the agent updates the least square regression (LSR) oracle using $\sigma^t$. 
For more details, see~\cref{alg:RM-CMDP-KD}.
\begin{algorithm}
    \caption{Regret Minimization for CMDP with Known Dynamics (RM-KD)}
    \label{alg:RM-CMDP-KD}
    \begin{algorithmic}[1]
        \STATE
        { 
            \textbf{inputs:}
            \begin{itemize}
                \item MDP parameters: 
                $S$ 
                %= \{S_0, S_1, \ldots , S_H\}$
                , $A$, $P_\star$, $s_0$, $H$. %( a unique start state).
                \item Confidence parameter $\delta$ and tuning parameters $\{\beta_t\}_{t=1}^T$.
                %\item Minimum reachability parameter $p_{min}$.
            \end{itemize}
        } 
        \STATE{\textbf{initialization:} for the first $|A|$ rounds, for each action $a_i$ in turn, run once the policy $\pi_i(c,s) = a_i 
        %\;\; \forall c \in \C, s \in S
        $ that
        at any state $s$ plays action $a_i$, regardless of the context $c$.}
        \FOR{round $t = |A| + 1, \ldots, T$}
            \STATE{compute
            $
                \hat{f}_t \in 
                \arg\min_{f \in \F}
                \sum_{i=1}^{t-1} \sum_{h=0}^{H-1} ( f(c_i,s^i_h,a^i_h) - r^i_h)^2
            $ using the LSR oracle
            }
            \STATE{observe context $c_t \sim \D$}
            \FOR{$k= |A|+1 ,|A|+2, \ldots, t$}
                \STATE{compute 
                    % $
                    %     \forall h \in [H-1], \; s_h \in S^{c_t}_h,\; a_h \in A: \;\;
                    % $
                    % $$
                    %     \hat{r}_k^{c_t}(s_h, a_h) = \hat{f}_k(c_t,s_h,a_h) + 
                    %     \frac{\beta_k}{  \sum_{i=1}^{k-1}\I[a_h = \pi_i(c_t,s_h)]q_h(s_h | \pi_i(c_t,\cdot),P^{c_t}_\star)} 
                    % $$
                    % new def for rewards
                    $
                        \forall (s,a) \in S \times  A: \;\;
                        \widehat{r}_k^{c_t}(s, a) = \hat{f}_k(c_t,s,a) + \frac{\beta_k}{  \sum_{i=1}^{k-1}\I[a = \pi_i(c_t,s)]q(s | \pi_i(c_t,\cdot),P^{c_t}_\star)}
                    $ 
                    }
                \STATE{define $\Mhat_k(c_t) = (S, A, P^{c_t}_\star, \widehat{r}^{c_t}_k, s_0, H)$}
                \STATE{compute $\pi_k(c_t,\cdot) \in \arg\max_{\pi \in S \to A}V^{\pi}_{\Mhat_k(c_t)}(s_0)$ using a planning algorithm}
           \ENDFOR{}    
            \STATE{play $\pi_t(c_t,\cdot)$ and update oracle using $\sigma^t = (c_t, s^t_0, a^t_0, r^t_0, s^t_1, \ldots, s^t_{H-1}, a^t_{H-1}, r^t_{H-1},s^t_H) $}
           \ENDFOR{}
    \end{algorithmic}
\end{algorithm}

\begin{observation}
    For any $|A|+1 \leq t \leq T$, given the policy $\pi_t(c_t;\cdot)$, the trajectory $\sigma^t$ is independent from the previous trajectories $\sigma^1, \ldots \sigma^{t-1}$.
\end{observation}

\begin{observation}\label{obs:q-vs-q_h}
    Since the CMDP is layered, for every $c \in \C$, $h \in [H]$ and $s_h\in S_h^c$ it holds that
    $$q_h(s_h | \pi_i(c;\cdot),P^{c}_\star)=q(s | \pi_i(c;\cdot),P^{c}_\star).$$
\end{observation}

\begin{corollary}
    Following~\cref{obs:q-vs-q_h}, the following two definitions of the rewards function are equivalent.
    \begin{equation}\label{eq:def-r-q-h} 
    \begin{split}
        \forall h \in [H-1], \; s_h \in S^c_h,\; a_h \in A:\;\;
        \widehat{r}_k^{c}(s_h, a_h) = \hat{f}_k(c,s_h,a_h) + 
        \frac{\beta_k}{  \sum_{i=1}^{k-1}\I[a_h = \pi_i(c,s_h)]q_h(s_h | \pi_i(c;\cdot),P^{c}_\star)},
    \end{split}
    \end{equation}
    and
    \begin{equation}\label{eq:def-r-q}
    \begin{split}
        \forall (s,a) \in S \times  A: \;\;
        \widehat{r}_k^{c}(s, a) = \hat{f}_k(c,s,a) + 
        \frac{\beta_k}{  \sum_{i=1}^{k-1}\I[a = \pi_i(c,s)]q(s | \pi_i(c;\cdot),P^{c}_\star)}.
    \end{split}
    \end{equation}
\end{corollary}
    % For convenience, in~\cref{alg:RM-CMDP-KD} we use $q(s | \pi_i(c;\cdot),P^{c}_\star)$ to define the approximated rewards function, but in the regret analysis we use $q_h(s_h | \pi_i(c;\cdot),P^{c}_\star)$.
    
\noindent For convenience, in~\cref{alg:RM-CMDP-KD} we use~\cref{eq:def-r-q} the approximated rewards function, but in the regret analysis we use~\cref{eq:def-r-q-h}.

\subsection{Regret Analysis}

\subsubsection{Analysis Outline.}

In the following regent analysis, for any $t > |A|$ we define the following MDPs for every context $c \in \C$.
\begin{enumerate}
    \item $\M^{(f,P_\star)}(c) = (S,A,P^c_\star,f(c,\cdot,\cdot),s_0,H)$, for any $f \in \F$.
    \item $\M^{(f_\star,P_\star)}(c)$ is the true model, where $f_\star(c, \cdot, \cdot) = r^c_\star$ is the true context dependent rewards and $P^c_\star$ is the true context-dependent dynamics, which we also denote by $\M(c)$.
    \item $\Mhat_t(c)=\M^{(\widehat{r}_t,P_\star)}(c)$ is the approximated MDP defined in~\cref{alg:RM-CMDP-KD}.
\end{enumerate}

\noindent Inspired by the analysis of~\cite{xu2020upper} for CMAB, our analysis consists of four main steps.

\paragraph{Step 1:} 
Establish uniform convergence bound for any $t \geq 2$ and a fixed sequence of functions $f_2,f_3,\ldots \in \F$ which states the following (see~\cref{lemma:UC-lemma-5}).\\
For any $\delta \in (0,1)$, with probability at least $1-\delta/2$ it holds that 
    \begin{align*}
        \sum_{i=1}^{t-1} \mathop{\E}
        \left[  \sum_{h=0}^{H-1} (f_t(c_i, s^i_h, a^i_h)-f_\star(c_i, s^i_h, a^i_h) )^2 | \Hist_{i-1}
        \right]
        &\leq
        68 H \log(4|\F|t^3/\delta)
        \\
        + &
        2 \sum_{i=1}^{t-1} \sum_{h=0}^{H-1} (f_t(c_i, s^i_h, a^i_h)- r^i_h )^2 - (f_\star(c_i, s^i_h, a^i_h) -r^i_h )^2).
    \end{align*}
To derive the above, we use~\cref{lemma:4.2-agrwl,lemma:UC-F}.

\paragraph{Step 2:}
Constructing confidence bound over the value of any given context-dependent policy with respect to the true rewards function $f_\star$ and the least square minimizer at round $t$, $\hat{f_t}$. The confidence bound holds with high probability.
%, in expectation over the contexts (i.e., ``optimism in expectation'').
\\
In~\cref{lemma:CB-policy} we show that with probability at least $1-\delta/2$ for all $t>|A|$ and every context-dependent policy $\pi \in \Pi_\C$ it holds that
\begin{align*}
    \left|\E_c\left[V^{\pi(c;\cdot)}_{\M(c)}(s_0)\right] - \E_c\left[V^{\pi(c;\cdot)}_{\M^{(\hat{f}_t,P_\star)}(c)}(s_0)\right]\right|
    \leq
    \sqrt{\phi_t(\pi)}
    \cdot \sqrt{ 68H\log(4|\F|t^3/\delta)},    
\end{align*}
where $\phi_t(\pi)$ is the contextual potential of policy $\pi \in \Pi_\C$ at time $|A|<t \leq T$ which is defined as follows.   
$$
    \phi_t(\pi) : = \E_c \left[ \sum_{h=0}^{H-1} \sum_{s_h \in S^c_h} \frac{q_h(s_h|\pi(c;\cdot) ,P^c_\star)}{ \sum_{i=1}^{t-1} \I[\pi(c;s_h)= \pi_i(c;s_h)]q_h(s_h | \pi_i(c;\cdot),P^c_\star)}\right].
$$

\paragraph{Step 3:}
Relax the confidence bound of step 2 to be additive.
In~\cref{lemma:sq-trick} we show that with probability at least $1-\delta/2$,
for all $t>|A|$ and any policy $\pi \in \Pi_\C$ for $\beta_t = \sqrt{ \frac{17 t \log(4|\F|t^3 /\delta) }{|S||A|}}$ it holds that
    $$
        \left|\E_c[V^{\pi(c;\cdot)}_{\M(c)}(s_0)] - \E_c[V^{\pi(c;\cdot)}_{\M^{(\hat{f}_t,P_\star)}(c)}(s_0)]\right| 
        \leq
        \beta_t \cdot \phi_t(\pi)
        +   \beta_t\cdot \frac{H \;|S|\; |A|}{t}
        .
    $$

\paragraph{Step 4:} 
Upper bound the cumulative contextual potential $\phi_t$
over every round $t >|A|$.\\ 
In~\cref{lemma:Contextual-Potential}
we show that for the sequence of selected context-dependent policies $\{\pi_t \in \Pi_\C\}_{t=1}^T$ it holds that
$$
    \sum_{t=|A| +1}^T \phi_t(\pi_t)\leq \frac{|S||A|}{p_{min}}(1+ \log(T/|A|))
    .
$$
% In addition, if for all $t$, $\pi_t= \pi$, for any context dependent policy $\pi \in \Pi_c$, we have an improved bound where the $|S|$ factor is replaced with $H$ factor,
% $$
%     \sum_{t=|A| +1}^T \phi_t(\pi)\leq \frac{H|A|}{p_{min}}(1+ \log(T/|A|))
%     .
% $$

\paragraph{Deriving regret bound.}
Using the results of steps 2 and 3, we derive the policy-difference lemma (\cref{lemma:policy-difference-KD}).
The lemma states that under the good event established in step 2, %
%the following bound, that holds with high probability
the following holds for any $t >|A|$.
    \begin{align*}
        \E_c\left[V^{\pi^{\star}(c;\cdot)}_{\M(c)}(s_0) \right]
        \leq
        \E_c\left[V^{\pi_t(c;\cdot)}_{\M(c)}\right]
        & +
        2 \beta_t \cdot
        % \E_c \left[ \sum_{h=0}^{H-1} \sum_{s_h \in S^c_h} \frac{q_h(s_h|\pi_t(c;\cdot),P^c_\star) }{\sum_{i=1}^{t-1} \I[\pi_t(c;s_h)= \pi_i(c;s_h)]q_h(s_h | \pi_i(c;\cdot),P^c_\star)}\right]
        \phi_t(\pi_t)
        +
        2 \beta_t \frac{H|S||A|}{t},  
    \end{align*}
where $\pi^\star \in \Pi_\C$ is an optimal context-dependent policy, and $\pi_t \in \Pi_\C$ is the selected context-dependent policy at round $t$.
\\
By summing the above inequality for all $t>|A|$ and applying the resoults of step $4$, in~\cref{thm:regret-bound-kd} 
%(\cref{thm:e-regret-KD}) 
we obtain a regret bound
%(and expected regret bound)
of
$$   
    %\Regrv_T(RM-KD) \leq 
    \tilde{O}\left(
    %\max\left\{ 
    \frac{ \sqrt{T |S||A| \log\frac{|\F|}{\delta}}}{p_{min} }+
    H  \sqrt{ T |S||A|\log\frac{|\F|}{\delta}}+
    |A| H
    %\right\}
    \right),
$$
which holds with probability at least $1-\delta$ for all $T \geq 1$.
We derive a regret bound in terms of $|\G|$ in~\cref{corl:regret-kd-}.
\\
In the following analysis, we use the explicit expression of the contextual potential $\phi_t$.

\subsubsection{Step 1: Establishing Uniform Convergence Bound Over \texorpdfstring{$\F$}{Lg}.}

The following lemma is an adaption of Lemma $4.2$ in \citet{agarwal2012contextual}.
\begin{lemma}\label{lemma:4.2-agrwl}
    Fix a function $f \in \F$. Suppose we sample context $c$ from the data distribution $\D$,  and rewards $r(s,a)$ from $\D_{c,s,a}$ for any state-action pair $(s,a)$.
    %from arbitrary distribution such that $r$ and $(s,a)$ are conditionally independent given $c$.
    Define the random variable
    \[
        Y_{c,s,a,r} = (f(c,s,a) - r(s,a))^2 - (f_\star(c,s,a) - r(s,a))^2.
    \]
    Then, the followings hold.
    \begin{enumerate}
        \item $\mathop{\E}_{c,s,a,r}[Y_{c,s,a,r}] = \mathop{\E}_{c,s,a}\left[( f(c,s,a) - f_\star(c,s,a))^2 \right] $.
        \item $\mathop{\V}_{c,s,a,r}[Y_{c,s,a,r}] \leq 4\mathop{\E}_{c,s,a,r}[Y_{c,s,a,r}]$
        \item $\V \left[\sum_{h=0}^{H-1} Y_{c,s_h,a_h,r} \right] \leq 4 H \E \left[ \sum_{h=0}^{H-1} Y_{c,s_h,a_h,r} \right]$.
    \end{enumerate}
    % \[
    %     \mathop{\E}_{c,s,a,r}[Y_{c,s,a,r}] = \mathop{\E}_{c,s,a}\left[( f(c,s,a) - f_\star(c,s,a))^2 \right]    
    % \]
    % and
    % \[
    %      \mathop{\V}_{c,r,s,a}[Y_{c,s,a,r}] \leq 4\mathop{\E}_{c,s,a,r}[Y_{c,s,a,r}].
    % \]
\end{lemma}

\begin{proof}
    Let us rearrange the definition of $Y_{c,s,a,r}$ as 
    \begin{equation}\label{eq:Y-eq}
        Y_{c,s,a,r} = (f(c,s,a) - f_\star(c,s,a))(f(c,s,a) + f_\star(c,s,a) - 2r(s,a)).
    \end{equation}
    Hence, we have
    \begin{align*}
        \mathop{\E}_{c,s,a,r}[Y_{c,s,a,r}]
        = &
        \mathop{\E}_{c,s,a,r}[(f(c,s,a) - f_\star(c,s,a))(f(c,s,a) + f_\star(c,s,a) - 2r(s,a))]
        \\
        = &
        \mathop{\E}_{c,s,a} \left[\mathop{\E}_{r}\left[(f(c,s,a) - f_\star(c,s,a))(f(c,s,a) + f_\star(c,s,a) - 2r(s,a)) \;|c,s,a \right]\right]
        \\
        = &
        \mathop{\E}_{c,s,a}\left[(f(c,s,a) - f_\star(c,s,a))(f(c,s,a) + f_\star(c,s,a) - 2\mathop{\E}_{r}[r(s,a)\;|c,s,a])\right]
        \\
        = &
        \mathop{\E}_{c,s,a}\left[ (f(c,s,a) - f_\star(c,s,a))^2 \right],    
    \end{align*}
    where the third identity uses that $r(s,a) \sim \D_{c,s,a}$ and the forth identity uses the realizability assumption that $f_\star(c,s,a)=\mathop{\E}_{r}[r(s,a)|c,s,a]$,
    proving the first part of the lemma.
    
    For the second part, note that $f$, $f_\star$ and $r$ are all between $0$ and $1$. 
    Hence using~\cref{eq:Y-eq} we obtain
    \begin{equation}\label{ineq:awrawl-rewards}
        \begin{split}
            Y^2_{c,s,a,r}
            = &
            (f(c,s,a) - f_\star(c,s,a))^2(f(c,s,a) + f_\star(c,s,a) - 2r(s,a))^2
            \\
            \leq &
            4(f(c,s,a) - f_\star(c,s,a))^2,
        \end{split}
    \end{equation}
    yielding the second part of the lemma as
    \begin{align*}
        \mathop{\V}_{c,s,a,r}[Y_{c,s,a,r}]
        \leq 
        \mathop{\E}_{c,s,a,r}[Y^2_{c,s,a,r}]
        \leq
        4\mathop{\E}_{c,s,a}[(f(c,s,a) - f_\star(c,s,a))^2]
        =
        4\mathop{\E}_{c,s,a,r}[Y_{c,s,a,r}].
    \end{align*}

    %Let $r_h := r(s_h,a_h)$. 
    For the third part, by norms inequality and~\cref{ineq:awrawl-rewards} we obtain,
    \begin{align*}
        \left( \sum_{h=0}^{H-1}Y_{c,s_h,a_h,r_h} \right)^2
        \leq
        H\sum_{h=0}^{H-1}Y^2_{c,s_h,a_h,r_h}
        \leq 4 H \sum_{h=0}^{H-1}(f(c,s_h,a_h) - f_\star(c,s_h,a_h))^2, 
    \end{align*}
    yielding the third part of the lemma as
    \begin{align*}
        \mathop{\V}\left[\sum_{h=0}^{H-1} Y_{c,s_h,a_h,r}\right]
        \leq 
        \mathop{\E}\left[\left( \sum_{h=0}^{H-1}Y_{c,s_h,a_h,r} \right)^2\right]
        \leq
        4H \mathop{\E}\left[\sum_{h=0}^{H-1}(f(c,s_h,a_h) - f_\star(c,s_h,a_h))^2\right]
        =
        4H \mathop{\E}\left[\sum_{h=0}^{H-1} Y_{c,s_h,a_h,r}\right].
    \end{align*}
\end{proof}

\begin{lemma}[Freedman's inequality, \citet{bartlett2008high}]\label{lemma:fridman's-ineq}
    Suppose $Z_1, Z_2,\ldots, Z_t$ is a martingale difference sequence with $|Z_i| \leq b$ for all $i=1,2,\ldots,t$. Then for any $\delta< 1/e^2$ with probability at least $1-\log_2(t)\delta$,
    \[
        \sum_{i=1}^t Z_i \leq 
        4\sqrt{\sum_{i=1}^t \V[Z_i|Z_1,\ldots,Z_{i-1}]\log(1/\delta)} 
        +
        2b\log(1/\delta).
    \]
\end{lemma}

\begin{definition}\label{def:Y-f-RV}
    For every round $t \geq 1$, layer $h \in [H-1]$ and a function $f \in \F$ we define the random variable
    \[
        Y_{f,c_t,s^t_h,a^t_h,r^t_h} = (f(c_t, s^t_h,a^t_h) -
        r^t_h)^2 - (f_\star(c_t, s^t_h,a^t_h)
        -
        r^t_h)^2
    \]
    where 
    $(c_t, s^t_h, a^t_h) \sim \D(c_t) \cdot q_h(s^t_h, a^t_h | \pi_t(c_t; \cdot), P^{c_t}_\star)$ and $r^t_h =  R^{c_t}_\star(s^t_h,a^t_h) \sim \D_{c_t,s^t_h,a^t_h}$.
    %$c_t \sim \D$ is the context in round $t$, and for every layer $h \in [H-1]$, $(s^t_h,a^t_h,r^t_h)$ is a tuple of the trajectory generated in round $t$, $\sigma^t = (s^t_0,a^t_0,r^t_0, \ldots, s^t_H)$.
\end{definition}

% \begin{lemma}[uniform convergence over $\F$]\label{lemma:UC-F}
%     For a fixed $t \geq 2$ and a fixed $\delta_t \in (0,1/e^2)$ with probability at least $1-\log_2((t-1)H)\delta_t$, we have
%     \begin{align*}
%         &\sum_{i=1}^{t-1} \sum_{h=0}^{H-1} \mathop{\E}_{c_i,s^i_h,a^i_h}
%         \left[ (f(c_i, s^i_h, a^i_h)-f_\star(c_i, s^i_h, a^i_h) )^2 | \Hist_{i-1}
%         %, \sigma^i_{h-1}
%         \right]
%         \\
%         & =
%         \sum_{i=1}^{t-1}
%         \E \left[ \sum_{h=0}^{H-1}(f(c_i, s^i_h, a^i_h)-f_\star(c_i, s^i_h, a^i_h) )^2 \Big| \Hist_{i-1}\right]
%         % \\
%         % & =
%         % \sum_{i=1}^{t-1} \sum_{h=0}^{H-1} \mathop{\E}_{c_i,s^i_h,a^i_h}
%         % \left[ (f(c_i, s^i_h, a^i_h)-f_\star(c_i, s^i_h, a^i_h) )^2 | \Hist_{i-1}
%         % %, \sigma^i_{h-1}
%         % \right]
%         \\
%         & \leq
%         68 H \log(|\F|/\delta_t)
%         +
%         2 \sum_{i=1}^{t-1} \sum_{h=0}^{H-1} Y_{f,c_i,s^i_h,a^i_h,r^i_h}.
%     \end{align*}
%     uniformly over all $f \in \F$.
% \end{lemma}

In the proof of the following lemma, we use the following simple observations.
\begin{observation}\label{obs:KD-dist-of-y-c-s-a}
    For all $ i \in [t-1]$ and $ h \in \{0, \ldots H-1\}$ we have that
    $$(c_i, s^i_h, a^i_h) \sim \D(c_i) \cdot q_h(s^i_h, a^i_h| \pi_i(c_i;\cdot), P^{c_i}_\star).$$
    In addition, $\pi_i$ is determined completely by the history $\Hist_{i-1}$.
    Hence,
    %since $\pi_i$ is determined completely by the history $\Hist_{i-1}$,
    by linearity of expectation it holds for any $f \in \F$ that,
    $$
        \mathop{\E}\left[\sum_{h=0}^{H-1} Y_{f,c_i,s^i_h,a^i_h,r^i_h}\Big|\Hist_{i-1}\right]
        =
        \sum_{h=0}^{H-1} \mathop{\E}_{c_i, s^i_h, a^i_h,r^i_h}\left[Y_{f,c_i,s^i_h,a^i_h,r^i_h}\Big|\Hist_{i-1}\right]
        .
    $$
    Similarly,
    $$
        \mathop{\E}\left[\sum_{h=0}^{H-1} (f(c_i, s^i_h, a^i_h)-f_\star(c_i, s^i_h, a^i_h) )^2\Big|\Hist_{i-1}\right]
        =
        \sum_{h=0}^{H-1} \mathop{\E}_{c_i, s^i_h, a^i_h}\left[(f(c_i, s^i_h, a^i_h)-f_\star(c_i, s^i_h, a^i_h) )^2\Big|\Hist_{i-1}\right]
        .
    $$
\end{observation}

\begin{observation}\label{obs:KD-martingale-Y}    
    For any $f \in \F$,
    %since for all $i > |A|$, $\pi_i$ is a deterministic context-dependent policy that determined completely by the history $\Hist_{i-1}$,
    we have that $\{Z_i(f)\}_{i=1}^{t-1} $  is a martingale difference sequence of length $t-1$ where
    $$
        Z_i(f) := \mathop{\E}\left[\sum_{h=0}^{H-1} Y_{f,c_i,s^i_h,a^i_h,r^i_h}\Big|\Hist_{i-1}\right]
        - \sum_{h=0}^{H-1} Y_{f,c_i,s^i_h,a^i_h,r^i_h} 
    $$
   % is a martingale difference sequence of length $t-1$,
    and the filtration is $\{\Hist_i\}_{i=1}^{t-1}$
    ($\Hist_0$ is the empty history.).
    Recall that for all $i \geq 1$ we defined that $\Hist_i = (\sigma^1, \ldots, \sigma^{i})$.
    Clearly, $Z_i(f)$ is determined given $ \Hist_1, \ldots, \Hist_i$ and $\E[Z_i(f) |\Hist_1, \ldots,\Hist_{i-1}]=0$ for all $i \in \{1,2,\ldots, t-1\}$.
    
    In addition, since $ \mathop{\E}\left[\sum_{h=0}^{H-1} Y_{f,c_i,s^i_h,a^i_h,r^i_h}\Big|\Hist_{i-1}\right]$ is a constant given $\Hist_{i-1}$, it holds that
    \[
        \V\left[Z_i(f) \;|\Hist_{i-1}\right] = \V\left[ \sum_{h=0}^{H-1} Y_{f,c_i,s^i_h,a^i_h,r^i_h} \;\Big| \Hist_{i-1}\right].
    \]

\end{observation}

\begin{lemma}[uniform convergence over $\F$]\label{lemma:UC-F}
    For a fixed $t \geq 2$ and a fixed $\delta_t \in (0,1/e^2)$ with probability at least $1-\log_2((t-1)H)\delta_t$, it holds that
    \begin{align*}
        &\sum_{i=1}^{t-1} \sum_{h=0}^{H-1} \mathop{\E}_{c_i,s^i_h,a^i_h}
        \left[ (f(c_i, s^i_h, a^i_h)-f_\star(c_i, s^i_h, a^i_h) )^2 | \Hist_{i-1}
        %, \sigma^i_{h-1}
        \right]
        \\
        & =
        \sum_{i=1}^{t-1}
        \E \left[ \sum_{h=0}^{H-1}(f(c_i, s^i_h, a^i_h)-f_\star(c_i, s^i_h, a^i_h) )^2 \Big| \Hist_{i-1}\right]
        % \\
        % & =
        % \sum_{i=1}^{t-1} \sum_{h=0}^{H-1} \mathop{\E}_{c_i,s^i_h,a^i_h}
        % \left[ (f(c_i, s^i_h, a^i_h)-f_\star(c_i, s^i_h, a^i_h) )^2 | \Hist_{i-1}
        % %, \sigma^i_{h-1}
        % \right]
        \\
        & \leq
        68 H \log(|\F|/\delta_t)
        +
        2 \sum_{i=1}^{t-1} \sum_{h=0}^{H-1} Y_{f,c_i,s^i_h,a^i_h,r^i_h}.
    \end{align*}
    uniformly for all $f \in \F$.
\end{lemma}

%We now prove~\cref{lemma:UC-F}.
\begin{proof}
    Fix a function $f \in \F$, and consider the random variable defined in~\cref{def:Y-f-RV},
    \[
        Y_{f,c_t,s^t_h,a^t_h,r^t_h} = (f_t(c_i, s^i_h, a^i_h)- r^i_h )^2 - (f_\star(c_i, s^i_h, a^i_h) -r^i_h )^2.
    \]
    
    Notice that $|Y_{f,c_t,s^t_h,a^t_h,r^t_h}| \leq 1$ for any function $f \in \F$, round $t\geq 2$, layer $h \in [H-1]$, state $s^t_h \in S^{c_t}_h$, action $a^t_h \in A$ and observed reward $r^t_h$.
    Hence, 
    $$\left|\sum_{h=0}^{H-1} Y_{f,c_t,s^t_h,a^t_h,r^t_h}\right|\leq \sum_{h=0}^{H-1} |Y_{f,c_t,s^t_h,a^t_h,r^t_h}| \leq H.$$
    
    By Freedman's inequality (\cref{lemma:fridman's-ineq}), for $\delta_t <  1/e^2$, with probability at least ${1-\log_2(t-1)\delta_t/|\F|}$ it holds that
    \begingroup
    \allowdisplaybreaks
    \begin{align*}
         &\sum_{i=1}^{t-1} \mathop{\E}\left[ \sum_{h=0}^{H-1} Y_{f,c_i,s^i_h,a^i_h,r^i_h} \Big| \Hist_{i-1}\right]
        -
        \sum_{i=1}^{t-1} \sum_{h=0}^{H-1} Y_{f,c_i,s^i_h,a^i_h,r^i_h}
        \\
        & \leq
        \; 4 \sqrt{
        \sum_{i=1}^{t-1} \V \left[\sum_{h=0}^{H-1}
        Y_{f,c_i,s^i_h,a^i_h,r^i_h}\Big|\Hist_{i-1}
        %,\sigma^i_{h-1}
        \right]\log(|\F|/\delta_t)
        }
        +
        2H \log(|\F|/\delta_t)
        %, \;\;\; \forall f \in \F
        .
    \end{align*}
    \endgroup
    By~\cref{lemma:4.2-agrwl} for all $i \in \{1,2,\ldots,t-1\}$ it holds that
    \[
        \V\left[\sum_{h=0}^{H-1} Y_{f,c_i,s^i_h,a^i_h,r^i_h}\Big|\Hist_{i-1}\right]
        \le
        4H\E\left[ \sum_{h=0}^{H-1} Y_{f,c_i,s^i_h,a^i_h,r^i_h}\Big|\Hist_{i-1}\right].
    \]
    Therefore,
    \begingroup
    \allowdisplaybreaks
    \begin{align*}
         &\sum_{i=1}^{t-1} 
         \E \left[
         \sum_{h=0}^{H-1} Y_{f,c_i,s^i_h,a^i_h,r^i_h} \Big|\Hist_{i-1}
         \right]
        \\
        & \leq \; 
        4 \sqrt{
        \sum_{i=1}^{t-1} 
        \V \left[ \sum_{h=0}^{H-1} Y_{f,c_i,s^i_h,a^i_h,r^i_h}\Big|\Hist_{i-1}\right]\log(|\F|/\delta_t)}
        +
        2H \log(|\F|/\delta_t)
        +
        \sum_{i=1}^{t-1} \sum_{h=0}^{H-1} Y_{f,c_i,s^i_h,a^i_h,r^i_h}
        \\
        &\leq \;
        8 \sqrt{
        H \sum_{i=1}^{t-1} 
        \E \left[
        \sum_{h=0}^{H-1}
        Y_{f,c_i,s^i_h,a^i_h,r^i_h} \Big|\Hist_{i-1}\right]\log(|\F|/\delta_t)}
        +
        2H \log(|\F|/\delta_t)
        +
        \sum_{i=1}^{t-1} \sum_{h=0}^{H-1} Y_{f,c_i,s^i_h,a^i_h,r^i_h}
        .
        %, \;\;\; \forall f \in \F.
    \end{align*}
    \endgroup
    We have
      \begin{align*}
         &\sum_{i=1}^{t-1} 
         \E \left[
         \sum_{h=0}^{H-1} Y_{f,c_i,s^i_h,a^i_h,r^i_h} \Big|\Hist_{i-1}
         \right]
          -
        8 \sqrt{
        H \sum_{i=1}^{t-1} 
        \E \left[
        \sum_{h=0}^{H-1}
        Y_{f,c_i,s^i_h,a^i_h,r^i_h} \Big|\Hist_{i-1}\right]\log(|\F|/\delta_t)}
        \\
        &\leq
        2H \log(|\F|/\delta_t)
        +
        \sum_{i=1}^{t-1} \sum_{h=0}^{H-1} Y_{f,c_i,s^i_h,a^i_h,r^i_h}
        .
        %, \;\;\; \forall f \in \F.
    \end{align*}
    We add to both sides $16H \log(|\F|/\delta_t)$, and 
    this implies for any $f \in \F$ that
    \begin{equation}
        \left( \sqrt{\sum_{i=1}^{t-1} 
        \E \left[
        \sum_{h=0}^{H-1} Y_{f,c_i,s^i_h,a^i_h,r^i_h} \Big|
        \Hist_{i-1} \right]} - 4 \sqrt{H \log(|\F|/\delta_t)} \right)^2
        \le
        18 H \log(|\F|/\delta_t) + \sum_{i=1}^{t-1} \sum_{h=0}^{H-1} Y_{f,c_i,s^i_h,a^i_h,r^i_h}.
    \end{equation}
    Yielding, 
    %This implies that,
      \begin{equation}
         \sqrt{\sum_{i=1}^{t-1} 
        \E \left[
        \sum_{h=0}^{H-1} Y_{f,c_i,s^i_h,a^i_h,r^i_h} \Big|
        \Hist_{i-1} \right]} \leq 4 \sqrt{H \log(|\F|/\delta_t)} 
        +
        \sqrt{18 H \log(|\F|/\delta_t) + \sum_{i=1}^{t-1} \sum_{h=0}^{H-1} Y_{f,c_i,s^i_h,a^i_h,r^i_h}}.
    \end{equation}
    And then,
        \begin{equation}
         \sum_{i=1}^{t-1} 
        \E \left[
        \sum_{h=0}^{H-1} Y_{f,c_i,s^i_h,a^i_h,r^i_h} \Big|
        \Hist_{i-1} \right] \leq 
        \left(4 \sqrt{H \log(|\F|/\delta_t)} 
        +
        \sqrt{18 H \log(|\F|/\delta_t) + \sum_{i=1}^{t-1} \sum_{h=0}^{H-1} Y_{f,c_i,s^i_h,a^i_h,r^i_h}}\right)^2.
    \end{equation}  
    This inequality further implies (using that $(a+b)^2 \leq 2 (a^2 + b^2)$ for all $a,b \in \R$ ) that for any $f \in \F$ it holds that
    \begin{equation}\label{eq:lemma-6-final}
        \sum_{i=1}^{t-1} 
        \mathop{\E}\left[ \sum_{h=0}^{H-1} Y_{f,c_i,s^i_h,a^i_h,r^i_h}\Big|
        \Hist_{i-1}\right]
        \le
        68H\log(|\F|/\delta_t) + 2\sum_{i=1}^{t-1} \sum_{h=0}^{H-1} Y_{f,c_i,s^i_h,a^i_h,r^i_h}.
    \end{equation}
    
    Lastly, by combining~\cref{eq:lemma-6-final} with~\cref{lemma:4.2-agrwl} we obtain the lemma since,
    \begin{align*}
        &  \sum_{i=1}^{t-1} \mathop{\E}
        \left[  \sum_{h=0}^{H-1} (f(c_i, s^i_h, a^i_h)-f_\star(c_i, s^i_h, a^i_h) )^2 | \Hist_{i-1}
        %, \sigma^i_{h-1}
        \right]
        \\
        \tag{By~\cref{obs:KD-dist-of-y-c-s-a}}
        & =
        \sum_{i=1}^{t-1} \sum_{h=0}^{H-1} \mathop{\E}_{c_i,s^i_h,a^i_h}
        \left[ (f(c_i, s^i_h, a^i_h)-f_\star(c_i, s^i_h, a^i_h) )^2 | \Hist_{i-1}
        %, \sigma^i_{h-1}
        \right]
        \\
        \tag{\cref{lemma:4.2-agrwl}}
        & =
        \sum_{i=1}^{t-1} \sum_{h=0}^{H-1} 
        \mathop{\E}_{c_i,s^i_h,a^i_h,r^i_h}
        \left[ Y_{f,c_i,s^i_h,a^i_h,r^i_h} \;\Big|\; \Hist_{i-1}
        %, \sigma^i_{h-1}
        \right]
        \\
        \tag{By~\cref{obs:KD-dist-of-y-c-s-a}}
        & =
        \sum_{i=1}^{t-1}
        \E \left[
        \sum_{h=0}^{H-1} 
        Y_{f,c_i,s^i_h,s^i_h,r^i_h} \;\Big|\; \Hist_{i-1}
        \right]
        \\
        \tag{By~\cref{eq:lemma-6-final}}
        & \leq
        68H\log(|\F|/\delta_t)
        +
        2 \sum_{i=1}^{t-1} \sum_{h=0}^{H-1} Y_{f,c_i,s^i_h,a^i_h,r^i_h}.
    \end{align*}
\end{proof}

\begin{lemma}[uniform convergence over all sequences of estimators]\label{lemma:UC-lemma-5}
    %For an arbitrary contextual MDP algorithm and 
    For any $\delta \in (0,1)$, with probability at least $1-\delta/2$ it holds that
    \begin{align*}
        & \sum_{i=1}^{t-1} \mathop{\E}
        \left[  \sum_{h=0}^{H-1} (f_t(c_i, s^i_h, a^i_h)-f_\star(c_i, s^i_h, a^i_h) )^2 | \Hist_{i-1}
        %, \sigma^i_{h-1}
        \right]
        \\
        & =
        \sum_{i=1}^{t-1} \sum_{h=0}^{H-1} \mathop{\E}_{c_i,s^i_h,a^i_h}
        \left[ (f_t(c_i, s^i_h, a^i_h)-f_\star(c_i, s^i_h, a^i_h) )^2 | \Hist_{i-1}
        %,\sigma^i_{h-1}
        \right]\\
        & \leq
        68 H \log(4|\F|t^3/\delta)
        +
        2 \sum_{i=1}^{t-1} \sum_{h=0}^{H-1} (f_t(c_i, s^i_h, a^i_h)- r^i_h )^2 - (f_\star(c_i, s^i_h, a^i_h) -r^i_h )^2.
    \end{align*}
    simultaneously for all $t \geq 2$ and any fixed sequence of functions $f_2,f_3,\ldots \in \F$.
\end{lemma}

\begin{proof}
    For a fixed $\delta \in (0,1)$, take $\delta_t = \delta/4t^3$ and apply union bound to~\cref{lemma:UC-F} with all $t \geq 2$. We have,
    \begin{align*}
        \sum_{t=1}^\infty  \delta_t \log(t-1)
        =
        \sum_{t=1}^\infty  \delta/4 t^3  \log(t-1)
        \le
        \sum_{t=1}^\infty \frac{\delta}{4t^2}
        \le
        \frac{\delta}{2}.
    \end{align*}
    Hence, by~\cref{lemma:UC-F} with probability at least $1-{\delta}/{2}$ it holds that
    \begingroup
    \allowdisplaybreaks
    \begin{align*}
        &\sum_{i=1}^{t-1} \sum_{h=0}^{H-1} \mathop{\E}_{c_i,s^i_h,a^i_h}
        \left[ (f_t(c_i, s^i_h, a^i_h)-f_\star(c_i, s^i_h, a^i_h) )^2 | \Hist_{i-1}
        \right]
        \\
        & =
        \sum_{i=1}^{t-1} \mathop{\E}
        \left[  \sum_{h=0}^{H-1} (f_t(c_i, s^i_h, a^i_h)-f_\star(c_i, s^i_h, a^i_h) )^2 | \Hist_{i-1}
        %, \sigma^i_{h-1}
        \right]
        % \\
        % & =
        % \sum_{i=1}^{t-1} \sum_{h=0}^{H-1} \mathop{\E}_{c_i,s^i_h,a^i_h}
        % \left[ (f_t(c_i, s^i_h, a^i_h)-f_\star(c_i, s^i_h, a^i_h) )^2 | \Hist_{i-1}
        % \right]
        \\
        & \leq
        68H\log(4|\F|t^3/\delta)
        +
        2 \sum_{i=1}^{t-1} \sum_{h=0}^{H-1} Y_{f_t,c_t,s^i_h,a^i_h,r^i_h}
        \\
        & =
        68 H \log(4|\F|t^3 /\delta)
        +
        2 \sum_{i=1}^{t-1} \sum_{h=0}^{H-1} (f_t(c_i, s^i_h, a^i_h)- r^i_h )^2 - (f_\star(c_i, s^i_h, a^i_h) -r^i_h )^2.
    \end{align*}
    \endgroup
    simultaneously for every $t \geq 2$ and any fixed sequence of functions $f_2,f_3,\ldots \in \F$.
    %(one function for each value of $t$).
\end{proof}

\subsubsection{Step 2: Constructing Confidence Bound Over Policies With Respect to Rewards Approximation.}

\begin{remark}
    Recall $f_\star(c,s,a) = r^c_\star(s,a)$. Hence, for any policy $\pi \in \Pi_{\C}$ and a context $c \in \C$ we have the following,
    \begin{equation}
        \begin{split}
            V^{\pi(c; \cdot)}_{\M(c)} 
            = &
            \sum_{h = 0}^{H-1} \sum_{s_h \in S^c_h} \sum_{a_h \in A}
            q_h(s_h, a_h | \pi(c;\cdot),P^c_\star) \cdot r^c_\star(s_h, a_h)
            \\
            = &
            \sum_{h = 0}^{H-1} \sum_{s_h \in S^c_h} \sum_{a_h \in A}
            q_h(s_h, a_h | \pi(c;\cdot),P^c_\star) \cdot f_\star (c,s_h,a_h).
        \end{split}
    \end{equation}
\end{remark}

\begin{lemma}[confidence bound over policies w.r.t rewards approximation]\label{lemma:CB-policy}
    Consider~\cref{alg:RM-CMDP-KD} 
    that at each initialization round $t \leq |A|$, plays the policy that always choose action $a_t$, and at each round $t \geq |A|+1$ 
    %an admissible 
    %non-randomized contextual MDP algorithm
    selects $\pi_t$ based on the history $\Hist_{t-1}$.
    %at each round $t \geq |A|+1$ 
    
    Then, for any $\delta \in (0,1)$,
    with probability at least $1-\delta/2$ for all $t > |A|$ and any context-dependent policy $\pi \in \Pi_\C$ the following holds.
    %the estimation error of the expected value of $\pi$ caused by the rewards approximation is bounded by
    \begin{align*}
        &\Big|\E_c\left[V^{\pi(c;\cdot)}_{\M(c)}(s_0)\right] - \E_c\left[V^{\pi(c;\cdot)}_{\M^{(\hat{f}_t,P_\star)}(c)}(s_0)\right]\Big|
        \\
        & \leq
        \sqrt{\E_c \left[ \sum_{h=0}^{H-1} \sum_{s_h \in S^c_h} \frac{q_h(s_h|\pi(c;\cdot) ,P^c_\star)}{ \sum_{i=1}^{t-1} \I[\pi(c;s_h)= \pi_i(c;s_h)]q_h(s_h | \pi_i(c;\cdot),P^c_\star)}\right]}
        \cdot \sqrt{ 68H\log(4|\F|t^3/\delta)},    \end{align*}
    where $\M^{(f,P_\star)}(c) = (S,A,P^c_\star,f(c,\cdot,\cdot),s_0, H)$ for any $f \in \F$,
    and $\M(c) = \M^{(f_\star,P_\star)}(c)$.
\end{lemma}

\begin{proof}
    For any function $f \in \F$ we consider the random variable defined in~\cref{def:Y-f-RV} 
    \[
        Y_{f,c_t,s^t_h,a^t_h,r^t_h} = (f(c_t, s^t_h,a^t_h) -
        r^t_h
        %R^{c_t}_\star(s^t_h,a^t_h)
        )^2 - (f_\star(c_t, s^t_h,a^t_h)
        -
        r^t_h
        %R^{c_t}_\star(s^t_h,a^t_h)
        )^2
    \]
    for any $t \geq 1$, context $c_t \in \C$, layer $h \in [H-1]$, state $s^t_h \in S^{c_t}_h$ and action $a^t_h \in A$.
    We remark that $(c_t, s^t_h, a^t_h) \sim \D(c_t) \cdot q_h(s^t_h, a^t_h | \pi_t(c_t; \cdot), P^{c_t}_\star)$ and $r^t_h =  R^{c_t}_\star(s^t_h,a^t_h)$.
    \\
    We first prove the following auxiliary claim.
    \begin{claim}\label{claim:expectation-eq-given-hist}
        For all $t \geq 2$ it holds that
        \begin{align*}
            &\sum_{i=1}^{t-1} \sum_{h=1}^{H-1}
            \mathop{\E}_{c_i, s^i_h,a^i_h}
            \left[ (\hat{f}_t(c_i,s^i_h,a^i_h) - f_\star(c_i,s^i_h,a^i_h) )^2 |\Hist_{i-1}\right]
            =
            \\
            &\E_c \left[\sum_{i=1}^{t-1} \sum_{h=0}^{H-1} \sum_{s_h \in S^c_h} q_h(s_h | \pi_i(c;\cdot),P^c_\star) (\hat{f}_t(c,s_h, \pi_i(c;s_h)) - f_\star(c,s_h, \pi_i(c;s_h)))^2 \right].
        \end{align*}
    \end{claim}
    
    \begin{proof}
        Recall that for every round $ i\in \{1,2,\ldots,|A|\}$, $\pi_i$ is a deterministic policy that always plays action $a_i$. 
        For all $i > |A|$ we have that $\pi_i$ is determined completely by the history $\Hist_{i-1}$, and is a deterministic context-dependent policy. Hence, for any function $f \in \F$, round $i \in \{1, \ldots, t-1\}$ and layer $h \in [H-1]$ it holds that
        \begin{equation}\label{eq:lemma1-1}
            \begin{split}
                &\mathop{\E}_{c_i, s^i_h,a^i_h}
                \left[ (f(c_i,s^i_h,a^i_h) - f_\star(c_i,s^i_h,a^i_h) )^2 |\Hist_{i-1}\right]
                \\
                = &
                \mathop{\E}_{c_i, s^i_h}\left[ (f(c_i,s^i_h,\pi_i(c_i; s^i_h)) - f_\star(c_i,s^i_h,\pi_i(c_i; s^i_h)) )^2 |\Hist_{i-1}\right]
                \\
                = &
                \E_c \left[ 
                \E_{ s_h} \left[ (f(c,s_h,\pi_i(c; s_h)) - f_\star(c, s_h,\pi_i(c; s_h)) )^2 |\Hist_{i-1},c\right]\right]
                \\
                = &
                \E_c \left[ 
                 \sum_{s_h \in S^c_h} q_h(s_h | \pi_i(c;\cdot),P^c_\star)  (f(c,s_h,\pi_i(c; s_h)) - f_\star(c, s_h,\pi_i(c; s_h)) )^2 
                 %\Big |\Hist_{i-1}
                 \right] 
                ,
        \end{split}
    \end{equation}
    where the first identity holds since $a^i_h = \pi_i(c_i; s^i_h)$ and $\pi_i$ is determined deterministically given $\Hist_{i-1}$. 
    The second identity is since $c_i$ is independent of $\Hist_{i-1}$ but $s^i_h$ is dependent on $c_i$ and $\Hist_{i-1}$ (through $\pi_i$).
    %and that the dependency in $\Hist_{i-1}$ affects only $\pi_i$.
    The third identity is an explicit representation of the expectation over $s_h$ given the context $c$ and the history $\Hist_{i-1}$, since $\pi_i$ is determined completely and deterministically by $\Hist_{i-1}$.
   \\
    By summing over $i = 1,2,\ldots, t-1$ and $h \in [H-1]$ we obtain the claim since,
    \begin{align*}
        &\sum_{i=1}^{t-1} \sum_{h=1}^{H-1}
        \mathop{\E}_{c_i, s^i_h,a^i_h}
        \left[ (\hat{f}_t(c_i,s^i_h,a^i_h) - f_\star(c_i,s^i_h,a^i_h) )^2 |\Hist_{i-1}\right]
        =
        \\
        \tag{By~\cref{eq:lemma1-1}}
        = &
        \sum_{i=1}^{t-1} \sum_{h=1}^{H-1}
        \E_c \left[ 
        \sum_{s_h \in S^c_h} q_h(s_h | \pi_i(c;\cdot),P^c_\star)  (\hat{f}_t(c,s_h,\pi_i(c; s_h)) - f_\star(c, s_h,\pi_i(c; s_h)) )^2
        %\Big |\Hist_{i-1}
        \right]
        \\
        \tag{By linearity of expectation}
        = &
        \E_c \left[ 
        \sum_{i=1}^{t-1} \sum_{h=1}^{H-1}
        \sum_{s_h \in S^c_h} q_h(s_h | \pi_i(c;\cdot),P^c_\star)  (\hat{f}_t(c,s_h,\pi_i(c; s_h)) - f_\star(c, s_h,\pi_i(c; s_h)) )^2 
        %\Big |\Hist_{i-1}
        \right],
    \end{align*}
    as stated.

    \end{proof}
    
    Returning to the proof of the lemma, by~\cref{lemma:UC-lemma-5}, for any $\delta \in (0,1)$ we have with probability at least $1-\delta/2$ that
    \begingroup
    \allowdisplaybreaks
    \begin{align*}
        &\sum_{i=1}^{t-1} \sum_{h=0}^{H-1} \mathop{\E}_{c_i,s^i_h,a^i_h}
        \left[ ({f}_t(c_i, s^i_h, a^i_h)-f_\star(c_i, s^i_h, a^i_h) )^2 | \Hist_{i-1}
        %,\sigma^i_{h-1}
        \right]\\
        & \leq
        68H\log(4|\F|t^3/\delta)
        +
        2 \sum_{i=1}^{t-1} \sum_{h=0}^{H-1} ({f}_t(c_i, s^i_h, a^i_h)- r^i_h )^2 - (f_\star(c_i, s^i_h, a^i_h) -r^i_h )^2,
    \end{align*}
    \endgroup
    %simultaneously for all $t \geq |A|+1$ and any fixed sequence of functions ${f}_{|A|+1}, {f}_{|A|+2}, \ldots \in \F$.
    simultaneously for all $t \geq 2$ and any fixed sequence of functions ${f}_2,{f}_3,\ldots \in \F$. 
    %, where we used the definition of $Y$'s.
    %
    %Note that where $f_t = \hat{f}_t$,
    Since $\{{f}_t\}_{t=|A|+1}^T$ are the least square minimizers at each round $t$, it holds that
    \[
        \sum_{i=1}^{t-1} \sum_{h=0}^{H-1} (\hat{f}_t(c_i, s^i_h, a^i_h)- r^i_h )^2 - (f_\star(c_i, s^i_h, a^i_h) -r^i_h )^2 \leq 0.
    \]
    
    Recall that $\M(c) = \M^{(f_\star,P_\star)}(c)$ for all $c \in \C$.
    
    By all the above, with probability at least $1-\delta/2$ for all $t \geq |A|+1$ and any context-dependent policy $\pi \in \Pi_\C$ it holds that
    \begingroup
    \allowdisplaybreaks
    \begin{align*}
        &  \left| \E_c[V^{\pi(c;\cdot)}_{\M^{(\hat{f}_t,P_\star)}(c)}(s_0) - V^{\pi(c;\cdot)}_{\M^{(f_\star,P_\star)}(c)}(s_0)]\right|
        \\
        \tag{ By definition}
        = &
        \left| \E_c \left[ \sum_{h=0}^{H-1} \sum_{s_h \in S^c_h}\sum_{a_h \in A}  q_h(s_h,a_h | \pi(c;\cdot), P^c_\star)(\hat{f}_t(c,s_h, a_h) - f_\star(c,s_h, a_h)) \right] \right|
        \\
        \tag{$\pi(c;\cdot)$ is a deterministic policy for all $c \in \C$}
        = & 
        \left| \E_c \left[ \sum_{h=0}^{H-1} \sum_{s_h \in S^c_h}  q_h(s_h | \pi(c;\cdot), P^c_\star)(\hat{f}_t(c,s_h, \pi(c;s_h)) - f_\star(c,s_h, \pi(c;s_h))) \right] \right|  
        \\
        \tag{By triangle inequality}
        \leq &
        \E_c \left[ \sum_{h=0}^{H-1} \sum_{s_h \in S^c_h} 
        q_h(s_h | \pi(c;\cdot),P^c_\star)
        \left|f_t(c,s_h, \pi(c;s_h)) - f_\star(c,s_h, \pi(c;s_h))\right| \right]
        \\
        = &
        \E_c \Bigg[  \sum_{h=0}^{H-1} \sum_{s_h \in S^c_h}
        (\sqrt{ q_h(s_h | \pi(c;\cdot), P^c_\star)})^2
        \frac{  \sqrt{\sum_{i=1}^{t-1} \I[\pi(c;s_h)= \pi_i(c;s_h)] q_h(s_h | \pi_i(c;\cdot),P^c_\star)}}
        {\sqrt{\sum_{i=1}^{t-1} \I[\pi(c;s_h)= \pi_i(c;s_h)]q_h(s_h | \pi_i(c;\cdot),P^c_\star)}}
        \\
        \tag{Multiplication in $\frac{  \sqrt{\sum_{i=1}^{t-1} \I[\pi(c;s_h)= \pi_i(c;s_h)] q_h(s_h | \pi_i(c;\cdot),P^c_\star)}}
        {\sqrt{\sum_{i=1}^{t-1} \I[\pi(c;s_h)= \pi_i(c;s_h)]q_h(s_h | \pi_i(c;\cdot),P^c_\star)}}$}
        & \cdot
        \left|\hat{f}_t(c,s_h, \pi(c;s_h)) - f_\star(c,s_h, \pi(c;s_h))\right| \Bigg]
        \\
        \tag{Re-arranging}
        = & 
       \sum_{c \in \C} \sum_{h=0}^{H-1} \sum_{s_h \in S^c_h} \sqrt{\D(c)}
       (\sqrt{ q_h(s_h | \pi(c;\cdot), P^c_\star)})^2
        \frac{ 1}
        {\sqrt{\sum_{i=1}^{t-1} \I[\pi(c;s_h)= \pi_i(c;s_h)]q_h(s_h | \pi_i(c;\cdot),P^c_\star)}}
        \\
       &\cdot \sqrt{\D(c)} \sqrt{\sum_{i=1}^{t-1} \I[\pi(c;s_h)= \pi_i(c;s_h)] q_h(s_h | \pi_i(c;\cdot),P^c_\star)}
        |\hat{f}_t(c,s_h, \pi(c;s_h)) - f_\star(c,s_h, \pi(c;s_h))| 
      %  }
        \\
        \tag{By Cauchy-Schwartz inequality}
         \leq &
        \sqrt{\E_c \left[ \sum_{h=0}^{H-1} \sum_{s_h \in S^c_h}  \frac{q_h^2(s_h | \pi(c;\cdot) ,P^c_\star)}{\sum_{i=1}^{t-1} \I[\pi(c;s_h)= \pi_i(c;s_h)]q_h(s_h | \pi_i(c;\cdot),P^c_\star)}\right]}
        \\
        &\cdot \sqrt{ \E_c \left[ \sum_{h=0}^{H-1} \sum_{s_h \in S^c_h} \sum_{i=1}^{t-1}\I[\pi(c;s_h)= \pi_i(c;s_h)] q_h(s_h | \pi_i(c;\cdot),P^c_\star)(\hat{f}_t(c,s_h, \pi(c;s_h)) - f_\star(c,s_h, \pi(c;s_h)))^2 \right]}
        \\
        \leq &
        \tag{By $q^2_h(s_h | \pi(c;\cdot), P^c_\star) \leq q_h(s_h | \pi(c;\cdot), P^c_\star)$ and change of summing order}
        \sqrt{\E_c \left[ \sum_{h=0}^{H-1} \sum_{s_h \in S^c_h} \frac{q_h(s_h | \pi(c;\cdot), P^c_\star)}{\sum_{i=1}^{t-1} \I[\pi(c;s_h)= \pi_i(c;s_h)]q_h(s_h | \pi_i(c;\cdot),P^c_\star)}\right]}
        \\
        &\cdot \sqrt{ \E_c \left[\sum_{i=1}^{t-1}\sum_{h=0}^{H-1} \sum_{s_h \in S^c_h}   q_h(s_h | \pi_i(c;\cdot),P^c_\star) \I[\pi(c;s_h)= \pi_i(c;s_h)](\hat{f}_t(c,s_h, \pi(c;s_h)) - f_\star(c,s_h, \pi(c;s_h)))^2 \right]}
        \\
        \tag{The non-zero terms are where $\pi_i(c;s_h) = \pi(c;s_h)$}
        = &
        \sqrt{\E_c \left[ \sum_{h=0}^{H-1} \sum_{s_h \in S^c_h} \frac{q_h(s_h | \pi ,P^c_\star)}{\sum_{i=1}^{t-1} \I[\pi(c;s_h)= \pi_i(c;s_h)]q_h(s_h | \pi_i(c;\cdot),P^c_\star)}\right]}
        \\
        &\cdot \sqrt{ \E_c \left[\sum_{i=1}^{t-1} \sum_{h=0}^{H-1} \sum_{s_h \in S^c_h}  q_h(s_h | \pi_i(c;\cdot),P^c_\star) \I[\pi(c;s_h)= \pi_i(c;s_h)](\hat{f}_t(c,s_h, \pi_i(c;s_h)) - f_\star(c,s_h, \pi_i(c;s_h)))^2 \right]}
        \\
        \tag{Removing the indicators only increase the sum}
        \leq &
        \sqrt{\E_c \left[ \sum_{h=0}^{H-1} \sum_{s_h \in S^c_h} \frac{q_h(s_h | \pi(c;\cdot), P^c_\star)}{\sum_{i=1}^{t-1} \I[\pi(c;s_h)= \pi_i(c;s_h)]q_h(s_h | \pi_i(c;\cdot),P^c_\star)}\right]}
        \\
        &\cdot \sqrt{ \E_c \left[\sum_{i=1}^{t-1} \sum_{h=0}^{H-1} \sum_{s_h \in S^c_h} q_h(s_h | \pi_i(c;\cdot),P^c_\star) (\hat{f}_t(c,s_h, \pi_i(c;s_h)) - f_\star(c,s_h, \pi_i(c;s_h)))^2 \right]}
        \\
        \tag{By~\cref{claim:expectation-eq-given-hist}}
        = &
        \sqrt{\E_c \left[ \sum_{h=0}^{H-1} \sum_{s_h \in S^c_h} \frac{q_h(s_h | \pi(c;\cdot), P^c_\star)}{\sum_{i=1}^{t-1} \I[\pi(c;s_h)= \pi_i(c;s_h)]q_h(s_h | \pi_i(c;\cdot),P^c_\star)}\right]}
        \\
        &\cdot \sqrt{ 
        %\E_c \left[ 
        \sum_{i=1}^{t-1} \sum_{h=1}^{H-1}
        %\E_{\Hist_{i-1},\sigma^i_{h-1}} \left[
        \mathop{\E}_{c_i, s^i_h,a^i_h}
        \left[ (\hat{f}_t(c_i,s^i_h,a^i_h) - f_\star(c_i,s^i_h,a^i_h) )^2 |\Hist_{i-1}
        %, \sigma^i_{h-1}\right] \right]
        \right]} 
        \\
        \tag{By~\cref{lemma:UC-lemma-5} combined with the fact that $\hat{f}_t$ is the least-square minimizer}
        \leq &
        \sqrt{\E_c \left[ \sum_{h=0}^{H-1} \sum_{s_h \in S^c_h} \frac{q_h(s_h | \pi(c;\cdot), P^c_\star)}{\sum_{i=1}^{t-1} \I[\pi(c;s_h)= \pi_i(c;s_h)]q_h(s_h | \pi_i(c;\cdot),P^c_\star)}\right]}
        \cdot \sqrt{ 68 H \log(4|\F|t^3 /\delta)}. 
    \end{align*}
    \endgroup
Lastly, we remark that by choice of $\pi_i$ for $i \in \{1,2,\ldots,|A|\}$, and the minimum reachability assumption for any deterministic policy $\pi \in \Pi_\C$, layer $h \in [H-1]$, state $s_h \in S^c_h$ and $t > |A|$ it holds that
\begin{align*}
    \sum_{i=1}^{t-1} \I[\pi(c;s_h)= \pi_i(c;s_h)]q_h(s_h | \pi_i(c;\cdot),P^c_\star)
    \geq
    p_{min} \cdot \sum_{i=1}^{t-1} \I[\pi(c;s_h)= \pi_i(c;s_h)]
    \geq 
     p_{min} > 0,
\end{align*}
hence the above is well defined.
\end{proof}

\subsubsection{Step 3: Relax the Confidence Bound to be Additive.}
\begin{lemma}[the ``square trick'' relaxation]\label{lemma:sq-trick}
    Under the good event of~\cref{lemma:CB-policy},
    for all $t > |A|$ and any context-dependent policy $\pi \in \Pi_\C$ it holds that 
    \begin{align*}
        \left|\E_c[V^{\pi(c;\cdot)}_{\M(c)}(s_0)] - \E_c[V^{\pi(c;\cdot)}_{\M^{(\hat{f}_t,P_\star)}(c)}(s_0)]\right|
        %\leq
        % \\
        % \leq &
        % \sqrt{\E_c \left[ \sum_{h=0}^{H-1} \sum_{s_h \in S^c_h} \frac{q_h(s_h|\pi(c;\cdot) ,P^c_\star)}{\sum_{i=1}^{t-1} \I[\pi(c;s_h)= \pi_i(c;s_h)]q_h(s_h | \pi_i(c;\cdot),P^c_\star)}\right]}
        % \cdot \sqrt{ 68 H \log(4|\F|t^3 H/\delta)}
        \leq &
        \beta_t \cdot \E_c \left[ \sum_{h=0}^{H-1} \sum_{s_h \in S^c_h} \frac{  q_h(s_h|\pi(c;\cdot) ,P^c_\star)}{\sum_{i=1}^{t-1} \I[\pi(c;s_h)= \pi_i(c;s_h)]q_h(s_h | \pi_i(c;\cdot),P^c_\star)}\right]
        \\
        & +  \beta_t \cdot \frac{H  |S|  |A|}{t}
        ,
    \end{align*}
    where $\beta_t = \sqrt{ \frac{17 t \log(4|\F|t^3 /\delta) }{|S| |A|}}$.
\end{lemma}

\begin{proof}
    Consider the following derivation,
    for $\beta_t = \sqrt{ \frac{17 t \log(4|\F|t^3/\delta) }{|S||A|}}$.
    \begingroup
    \allowdisplaybreaks
    \begin{align*}
        & \left|\E_c\left[V^{\pi(c;\cdot)}_{\M(c)}\right] - \E_c\left[V^{\pi(c;\cdot)}_{\M^{(\hat{f}_t,P_\star)}(c)}\right]\right| 
        \\
        \tag{By~\cref{lemma:CB-policy}}
        \leq &
        \sqrt{\E_c \left[ \sum_{h=0}^{H-1} \sum_{s_h \in S^c_h} \frac{q_h(s_h|\pi(c;\cdot) ,P^c_\star)}{\sum_{i=1}^{t-1} \I[\pi(c;s_h)= \pi_i(c;s_h)]q_h(s_h | \pi_i(c;\cdot),P^c_\star)}\right]}
        \cdot \sqrt{ 68 H \log(4|\F|t^3/\delta)}
        \\
        = &
        \sqrt{\E_c \left[ \sum_{h=0}^{H-1} \sum_{s_h \in S^c_h} \frac{\beta_t \cdot q_h(s_h|\pi(c;\cdot) ,P^c_\star)}{\sum_{i=1}^{t-1} \I[\pi(c;s_h)= \pi_i(c;s_h)]q_h(s_h | \pi_i(c;\cdot),P^c_\star)}\right]}
        \cdot \sqrt{\frac{1}{\beta_t} 68 H \log(4|\F|t^3/\delta)}
        \\
        \tag{By $\beta_t$ choice}
        = & 
        \sqrt{\E_c \left[ \sum_{h=0}^{H-1} \sum_{s_h \in S^c_h} \frac{\beta_t \cdot q_h(s_h|\pi(c;\cdot) ,P^c_\star) }{\sum_{i=1}^{t-1} \I[\pi(c;s_h)= \pi_i(c;s_h)]q_h(s_h | \pi_i(c;\cdot),P^c_\star)}\right]}
        \cdot \sqrt{ H \sqrt{\frac{ |S| |A|}{t}} \frac{68\log(4|\F|t^3/\delta)}{\sqrt{17\log(4|\F|t^3/\delta)}}}
        \\
        = & 
        \sqrt{\E_c \left[ \sum_{h=0}^{H-1} \sum_{s_h \in S^c_h} \frac{\beta_t \cdot q_h(s_h|\pi(c;\cdot) ,P^c_\star) }{\sum_{i=1}^{t-1} \I[\pi(c;s_h)= \pi_i(c;s_h)]q_h(s_h | \pi_i(c;\cdot),P^c_\star)}\right]}
        \cdot \sqrt{4 H \cdot \sqrt{\frac{ |S||A|}{t}} \frac{ (\sqrt{17\log(4|\F|t^3/\delta)})^2}{\sqrt{17\log(4|\F|t^3/\delta)}}}
        \\
        = & 
        \sqrt{\E_c \left[ \sum_{h=0}^{H-1} \sum_{s_h \in S^c_h} \frac{\beta_t \cdot q_h(s_h|\pi(c;\cdot) ,P^c_\star) }{\sum_{i=1}^{t-1} \I[\pi(c;s_h)= \pi_i(c;s_h)]q_h(s_h | \pi_i(c;\cdot),P^c_\star)}\right]}
        \cdot \sqrt{4 H \cdot \sqrt{\frac{|S| |A|}{t}}  \sqrt{17\log(4|\F|t^3/\delta)}}
        \\
        = & 
        \sqrt{2 \cdot \E_c \left[ \sum_{h=0}^{H-1} \sum_{s_h \in S^c_h} \frac{\beta_t \cdot q_h(s_h|\pi(c;\cdot) ,P^c_\star) }{\sum_{i=1}^{t-1} \I[\pi(c;s_h)= \pi_i(c;s_h)]q_h(s_h | \pi_i(c;\cdot),P^c_\star)}\right]}
        \\
        & \cdot \sqrt{2 H \cdot \sqrt{\frac{|S|\cdot |A|}{t}}\sqrt{\frac{t^2}{ |S|^2\cdot |A|^2}}  \sqrt{\frac{ |S|^2\cdot |A|^2}{t^2}} \sqrt{17\log(4|\F|t^3/\delta)}}
        \\
        = & 
        \sqrt{2 \cdot \E_c \left[ \sum_{h=0}^{H-1} \sum_{s_h \in S^c_h} \frac{\beta_t \cdot q_h(s_h|\pi(c;\cdot) ,P^c_\star) }{\sum_{i=1}^{t-1} \I[\pi(c;s_h)= \pi_i(c;s_h)]q_h(s_h | \pi_i(c;\cdot),P^c_\star)}\right]}
        \cdot \sqrt{2 H \cdot \beta_t \sqrt{\frac{|S|^2\cdot |A|^2}{t^2}}}
        \\
        = & 
        2\cdot \sqrt{ \E_c \left[ \sum_{h=0}^{H-1} \sum_{s_h \in S^c_h} \frac{\beta_t \cdot q_h(s_h|\pi(c;\cdot) ,P^c_\star) }{\sum_{i=1}^{t-1} \I[\pi(c;s_h)= \pi_i(c;s_h)]q_h(s_h | \pi_i(c;\cdot),P^c_\star)}\right]}
        \cdot \sqrt{ \beta_t \cdot \frac{H  |S| |A|}{t}}
        \\
        \tag{Since $2ab\leq a^2+b^2$}
        \leq &
        \E_c \left[ \sum_{h=0}^{H-1} \sum_{s_h \in S^c_h} \frac{\beta_t \cdot q_h(s_h|\pi(c;\cdot) ,P^c_\star)}{\sum_{i=1}^{t-1} \I[\pi(c;s_h)= \pi_i(c;s_h)]q_h(s_h | \pi_i(c;\cdot),P^c_\star)}\right]
        +  \beta_t \cdot \frac{H  |S|  |A|}{t}.
    \end{align*}
    \endgroup
\end{proof}

\subsubsection{Step 4: Bounding the Sum of Contextual Potential Functions.}

Let us define the contextual potential function in round $t$, for $T \geq t>|A|$.
\begin{definition}
    We denote by $\phi_t(\pi)$ the contextual potential of a context-dependent policy $\pi \in \Pi_\C$ at round $|A|<t \leq T$ which is defined as follows.   
    $$
        \phi_t(\pi) : = \E_c \left[ \sum_{h=0}^{H-1} \sum_{s_h \in S^c_h} \frac{q_h(s_h|\pi(c;\cdot) ,P^c_\star)}{ \sum_{i=1}^{t-1} \I[\pi(c;s_h)= \pi_i(c;s_h)]q_h(s_h | \pi_i(c;\cdot),P^c_\star)}\right],
    $$
    where $\{\pi_t \in \Pi_{\C}\}_{t=1}^T$
    is the sequence of context dependent policies selected by~\cref{alg:RM-CMDP-KD}.
\end{definition}

\noindent In the following lemma, we bound the sum of contextual potential functions, over the rounds $t = |A|+1, \ldots, T$.
\begin{lemma}[contextual potential]\label{lemma:Contextual-Potential}
    Let $\{\pi_t \in \Pi_{\C}\}_{t=1}^T$
    be the sequence of context dependent policies selected by~\cref{alg:RM-CMDP-KD}.
    Then, for all $T > |A|$ the following holds.
    \begin{align*}
        \sum_{t=|A| +1}^T\phi_t(\pi_t)
        & =  
        \sum_{t=|A| +1}^T 
        \E_c \left[ 
        \sum_{h =0}^{H-1}
        \sum_{s_h \in S^c_h}
        \frac{q_h(s_h | \pi_t(c;\cdot),P^c_\star)}{\sum_{i=1}^{t-1} \I[\pi_t(c;s_h) = \pi_i(c;s_h)]q_h(s_h | \pi_i(c;\cdot),P^c_\star)} \right]
        \\
        & \leq 
        \frac{|S||A|}{p_{min}}( 1+ \log(T/|A|))
        .
        \end{align*} 
\end{lemma}

\begin{proof}
    For any fixed context $c \in \C$, the following holds.
    \begingroup
    \allowdisplaybreaks
    \begin{align*}
        &\sum_{t=|A| +1}^T
        \sum_{h =0}^{H-1}
        \sum_{s_h \in S^c_h}
        \frac{q_h(s_h | \pi_t(c;\cdot),P^c_\star)}{\sum_{i=1}^{t-1} \I[\pi_t(c;s_h) = \pi_i(c;s_h)]q_h(s_h | \pi_i(c;\cdot),P^c_\star)}
        \\
        \tag{By the minimum reachability assumption}
        \leq &
        \sum_{t=|A| +1}^T \sum_{h =0}^{H-1} \sum_{s_h \in S^c_h}  
        \frac{q_h(s_h| \pi_t(c;\cdot),P^c_\star)}{p_{min}\sum_{i=1}^{t-1} \I[\pi_t(c;s_h) = \pi_i(c;s_h)]}
        \\
        \tag{$q_h(s_h| \pi_t(c;\cdot),P^c_\star)\leq 1$ and change of summing order}
        \leq &
        \frac{1}{p_{min}} \sum_{h =0}^{H-1} \sum_{s_h \in S^c_h} 
        \sum_{t=|A| +1}^T
        \frac{1}{\sum_{i=1}^{t-1} \I[\pi_t(c;s_h) = \pi_i(c;s_h)]}
        \\
        \leq &
        \frac{1}{p_{min}} \sum_{h =0}^{H-1} \sum_{s_h \in S^c_h} 
        \sum_{a_h \in A}
        \sum_{i=1}^{\sum_{t=1}^T  \I[\pi_t(c;s_h) = a_h] } \frac{1}{i}
        \\
        \tag{Since $\sum_{i=1}^n \frac{1}{i} \leq 1+\log(n)$}
        \leq &
        \frac{1}{p_{min}}
        \sum_{h =0}^{H-1} \sum_{s_h \in S^c_h} \sum_{a_h \in A} \left(1 + \log\left(\sum_{t=1}^T \I[\pi_t(c;s_h) = a_h]\right) \right)
        \\
        = &
         \frac{|S||A|}{p_{min}} + \sum_{h =0}^{H-1} \sum_{s_h \in S^c_h} |A| \cdot \frac{1}{|A|}\sum_{a_h \in A}  \log\left(\sum_{t=1}^T \I[\pi_t(c;s_h) = a_h]\right)
        \\
         \tag{By Jansen's inequality, since $\log$ is concave}
         \leq &
        \frac{|S||A|}{p_{min}} + \sum_{h =0}^{H-1} \sum_{s_h \in S^c_h} |A| \cdot   \log\left(\frac{1}{|A|}\sum_{a_h \in A}\sum_{t=1}^T \I[\pi_t(c;s_h) = a_h]\right)
        \\
        = &
        \frac{|S||A|}{p_{min}} + \sum_{h =0}^{H-1} \sum_{s_h \in S^c_h} |A| \cdot   \log\left(\frac{1}{|A|}\sum_{t=1}^T \sum_{a_h \in A} \I[\pi_t(c;s_h) = a_h]\right)
        \\
        \tag{For all $t$, $\sum_{a_h \in A} \I[\pi_t(c;s_h) = a_h]=1$ since $\pi_t$ is a deterministic policy}
         = &
        \frac{|S||A|}{p_{min}} + \sum_{h =0}^{H-1} \sum_{s_h \in S^c_h} |A| \cdot   \log\left(\frac{T}{|A|}\right)
        \\
        = &
        \frac{|S||A|}{p_{min}}( 1 +  \log(T/|A|)).
    \end{align*}
    \endgroup
    By taking an expectation over $c$ on both sides of the inequality, we obtain the lemma.
\end{proof}

\subsubsection{Deriving The Regret Bound.}
\begin{lemma}[equivalence of policies]\label{lemma:Policies-equivalence}
    For any mapping $\widetilde{\M}$ from a context $c \in \C$ to a MDP, it holds that
    \[
        \max_{\pi\in \Pi_\C}
        \E_c\left[ V^{\pi(c; \cdot)}_{\widetilde{\M}(c)}(s_0) \right]
        =
        \E_c\left[ \max_{\pi \in S \to A}
        V^{\pi}_{\widetilde{\M}(c)}(s_0) \right].
    \]
\end{lemma}

\begin{proof}
    Holds trivially.
\end{proof}

\begin{lemma}[policy difference]\label{lemma:policy-difference-KD}
    Under the good event of~\cref{lemma:CB-policy} 
    %(Confidence bound over policies w.r.t rewards approximation)
    , for any $t >|A|$ it holds that
    \begin{align*}
        \E_c\left[V^{\pi^{\star}(c;\cdot)}_{\M(c)}(s_0) \right]
        \leq
        \E_c\left[V^{\pi_t(c;\cdot)}_{\M(c)}\right]
        & +
        2 \beta_t \cdot
        \E_c \left[ \sum_{h=0}^{H-1} \sum_{s_h \in S^c_h} \frac{q_h(s_h|\pi_t(c;\cdot),P^c_\star) }{\sum_{i=1}^{t-1} \I[\pi_t(c;s_h)= \pi_i(c;s_h)]q_h(s_h | \pi_i(c;\cdot),P^c_\star)}\right]
        \\
        & +
        2 \beta_t \frac{H|S||A|}{t}.  
    \end{align*}
    where $\pi^\star \in \Pi_\C$ is an optimal context-dependent policy, and $\pi_t$ is the context-dependent policy that was selected in round $t$.
\end{lemma}

\begin{proof}
    %Let $\pi^\star \in \Pi_{\C}$ be an optimal policy and $\pi_t \in \Pi_{\C}$ be the chosen policy in time $t$.
    Under the good event of~\cref{lemma:CB-policy}, using~\cref{lemma:sq-trick}, for every $t > |A|$, consider the following derivation.
    \begingroup
    \allowdisplaybreaks
    \begin{align*}
        &\E_c[V^{\pi^{\star}(c;\cdot)}_{\M(c)}(s_0) ]
        \\
        \tag{By~\cref{lemma:sq-trick} applied for $\pi^\star$}
        &\leq 
        \E_c[V^{\pi^{\star}(c;\cdot) }_{\M^{(\hat{f}_t,P_\star)}(c)}(s_0)]
        % \\
        % & 
        +
        \E_c \left[ \sum_{h=0}^{H-1} \sum_{s_h \in S^c_h} \frac{\beta_t \cdot q_h(s_h|\pi^\star(c;\cdot) ,P^c_\star) }{\sum_{i=1}^{t-1} \I[\pi^\star(c;s_h)= \pi_i(c;s_h)]q_h(s_h | \pi_i(c;\cdot),P^c_\star)}\right]
        +  
        \beta_t \frac{H |S| |A|}{t}
        \\
        \leq &
        \max_{\pi \in \Pi_{\C}}
        \left\{ 
            %\underbrace{
            \E_c[V^{\pi(c;\cdot)}_{\M^{(\hat{f}_t,P_\star)}(c)}(s_0) ]
            +
            \E_c \left[ \sum_{h=0}^{H-1} \sum_{s_h \in S^c_h} \frac{\beta_t  \cdot q_h(s_h|\pi(c;\cdot) ,P^c_\star) }{\sum_{i=1}^{t-1} \I[\pi(c;s_h)= \pi_i(c;s_h)]q_h(s_h | \pi_i(c;\cdot),P^c_\star)}\right]
            %}_{=\E_c\left[V^{\pi(c;\cdot)}_{\Mhat_t(c)}(s_0)\right] \text{ By equation~(\cref{eq:V-t-def})}}
        \right\} 
        +
        \beta_t \frac{H |S| |A|}{t}
        \\
        \tag{By~\cref{eq:V-t-def}}
        = &
        \max_{\pi \in \Pi_{\C}}
        \left\{ 
            \E_c\left[V^{\pi(c;\cdot)}_{\Mhat_t(c)}(s_0)\right]
        \right\} 
        +
        \beta_t \frac{H |S| |A|}{t}
        \\
        \tag{By~\cref{lemma:Policies-equivalence} applied on the mapping $\Mhat_t$ and $\pi_t$ choice}
        = &
        \E_c\left[V^{\pi_t(c;\cdot)}_{\Mhat_t(c)}(s_0)\right]
        +
        \beta_t \frac{H |S| |A|}{t}
        \\
        \tag{By~\cref{eq:V-t-def}}
        = &
        \E_c[V^{\pi_t(c;\cdot)}_{\M^{(\hat{f}_t,P_\star)}(c)}]
        +
        \E_c \left[ \sum_{h=0}^{H-1} \sum_{s_h \in S^c_h} \frac{\beta_t  \cdot q_h(s_h|\pi_t(c;\cdot) ,P^c_\star)}{\sum_{i=1}^{t-1} \I[\pi_t(c;s_h)= \pi_i(c;s_h)]q_h(s_h | \pi_i(c;\cdot),P^c_\star)}\right]
        +
        \beta_t \frac{H |S| |A|}{t}
        \\
        \tag{By~\cref{lemma:sq-trick} applied for $\pi_t$}
        \leq &
        \E_c\left[V^{\pi_t(c;\cdot)}_{\M(c)}\right]
        +
        2 \beta_t \cdot
        \E_c \left[ \sum_{h=0}^{H-1} \sum_{s_h \in S^c_h} \frac{q_h(s_h|\pi_t(c;\cdot),P^c_\star) }{\sum_{i=1}^{t-1} \I[\pi_t(c;s_h)= \pi_i(c;s_h)]q_h(s_h | \pi_i(c;\cdot),P^c_\star)}\right]
        +
        2 \beta_t \frac{H |S| |A|}{t},
    \end{align*}
    \endgroup
    as stated.
\end{proof}

Consider the regret, which defined as
$$
   \Regrv_T(\text{ALG}) 
    :=
    \sum_{t=1}^T V^{\pi^\star(c_t;\cdot)}_{\M(c_t)}
    -
    V^{\pi_t(c_t;\cdot)}_{\M(c_t)}.
$$
The following theorem establish our regret bound, which holds with high probability.
\begin{theorem}[regret bound]\label{thm:regret-bound-kd}
     For every $T \geq 1$, finite functions class $\F$ and $\delta \in (0,1)$ let ${\beta_t = \sqrt{ \frac{17 t \log(4|\F|t^3/\delta)}{|S||A|}}}$ for all $t \in [T]$.
     Then, with probability at least $1-\delta$ it holds that
    $$
        \Regrv_T(RM-KD)
        \leq
        % \tilde{O}\left(
        % \max\left\{ 
        % \frac{ \sqrt{T  |S| |A|\log\frac{|\F|}{\delta}}}{p_{min} },\;
        % H\sqrt{ T  |S| |A|  \log\frac{|\F|}{\delta}},\;|A| H \right\}\right)
        %%
        \widetilde{O}\left(
        %\max\left\{ 
        \frac{ \sqrt{T |S||A| \log\frac{|\F|}{\delta}}}{p_{min} }+
        H  \sqrt{ T |S||A|\log\frac{|\F|}{\delta}}+
        |A| H
        %\right\}
        \right)
        .
    $$
\end{theorem}

\begin{proof}
    We prove the theorem under the good event stated in~\cref{lemma:CB-policy}, which holds with probability at least $1-{\delta}/{2}$.
    
    By the 
    %good event and the 
    policy difference lemma (\cref{lemma:policy-difference-KD}), under the good event, 
    %with probability at least $1-\delta/2$ 
    for all $t > |A|$ it holds that
    \begin{equation}\label{eq:thm-1-bound-reg}
        \begin{split}
        \E_c\left[V^{\pi^{\star}(c;\cdot)}_{\M(c)}(s_0) - V^{\pi_t(c;\cdot)}_{\M(c)}(s_0)\right]
        \leq &  
        2 \beta_t \cdot
        \E_c \left[ \sum_{h=0}^{H-1} \sum_{s_h \in S^c_h}
        \frac{ q_h(s_h|\pi_t(c;\cdot),P^c_\star) }{\sum_{i=1}^{t-1} \I[\pi_t(c;s_h)= \pi_i(c;s_h)]q_h(s_h | \pi_i(c;\cdot),P^c_\star)}\right]
        \\
        & +
        2 \beta_t \frac{H |S| |A|}{t}.
        \end{split}
    \end{equation}
    
    % In expectation we have
    % \begingroup
    % \allowdisplaybreaks
    % \begin{align*}
    %     \Reg_T\left(RM-KD \right)
    %     = &
    %     \E \left[ \sum_{t=1}^T 
    %     \E_{c}\left[V^{\pi^{\star}(c;\cdot)}_{\M(c)}(s_0)  - V^{\pi_t(c;\cdot)}_{\M(c)}(s_0) \right] \right]
    %     \\
    %     \tag{By linearity of expectation, since the algorithm is deterministic.}
    %     = &
    %     \sum_{t=1}^T \E_{\Hist_{t-1}}\left[   \E_{c}\left[V^{\pi^{\star}(c;\cdot)}_{\M(c)}(s_0)  - V^{\pi_t(c;\cdot)}_{\M(c)}(s_0)  \Big| \Hist_{t-1}\right] \right].
    % \end{align*}
    % \endgroup
    % Thus, since $\pi_t$ is determined completely by the history $\Hist_{t-1}$, by Azuma's inequality, with probability at least $1-\delta/2$, over the policies $\pi_t$, we have
    
    Recall that $\Hist_t = (\sigma^1, \ldots, \sigma^t)$.
     Consider the Martingale difference sequence $\{Y_t\}_{t=1}^T$ and the filtration $\{\Hist_{t}\}_{t=1}^T$ where 
    $$
        Y_t 
        := 
        V^{\pi^{\star}(c_t,\cdot)}_{\M(c_t)}(s_0)  - V^{\pi_t(c_t,\cdot)}_{\M(c_t)}(s_0)
        -
        %\E_{\Hist_{t-1}}\left[ 
        \E_{c_t}\left[ V^{\pi^{\star}(c_t,\cdot)}_{\M(c_t)}(s_0)  - V^{\pi_t(c_t,\cdot)}_{\M(c_t)}(s_0) \Big|\Hist_{t-1}\right] %\right]
        .
    $$
    %and $\Hist_0$ is the empty history.
    Clearly, for all $t$, $|Y_t|\leq 2H$, $Y_t$ is determined completely by $\Hist_t$ and $\E\left[ Y_t |\Hist_{t-1}\right] = 0$ since
     \begin{align*}
         \E\left[ Y_t |\Hist_{t-1}\right]
         & =
         \E_{c_t}\left[ 
         V^{\pi^{\star}(c_t,\cdot)}_{\M(c_t)}(s_0)  - V^{\pi_t(c_t,\cdot)}_{\M(c_t)}(s_0) 
         \Big|\Hist_{t-1}\right]
         -
        \E_{c_t}\left[ V^{\pi^{\star}(c_t,\cdot)}_{\M(c_t)}(s_0)  - V^{\pi_t(c_t,\cdot)}_{\M(c_t)}(s_0) \Big|\Hist_{t-1}\right] 
         \\
         \tag{Since $\pi_t$ is determined by $\Hist_{t-1}$, and $c_t$ and $\pi^\star$ are independent of the history $\Hist_{t-1}$}
        %  & =
        %   \E_{c_t}\left[ 
        %  V^{\pi^{\star}(c_t,\cdot)}_{\M(c_t)}(s_0)  - V^{\pi_t(c_t,\cdot)}_{\M(c_t)}(s_0) 
        %  \right]
        %  -
        %  E_{c_t}\left[ V^{\pi^{\star}(c_t,\cdot)}_{\M(c_t)}(s_0)  - V^{\pi_t(c_t,\cdot)}_{\M(c_t)}(s_0) \right]
         =& 0.
     \end{align*}
     
    Hence, by Azuma's inequality, with probability at least $1-\delta/2$ it holds that 
    \begin{align*}
        \Regrv_T(RM-KD) 
        =
        \sum_{t=1}^T V^{\pi^\star(c_t;\cdot)}_{\M(c_t)}
        -
        V^{\pi_t(c_t;\cdot)}_{\M(c_t)}
        \leq 
        \sum_{t=1}^T
        \E_{c_t}\left[ V^{\pi^{\star}(c_t;\cdot)}_{\M(c_t)}(s_0)  - V^{\pi_t(c_t;\cdot)}_{\M(c_t)}(s_0) \Big| \Hist_{t-1}\right]
        + 2H\sqrt{2T\log(4/\delta)}.
    \end{align*}
    
    Since $\pi_t$ is determined completely (and deterministically) by $\Hist_{t-1}$, and $c_t$ and $\pi^\star$ are independent of the history $\Hist_{t-1}$, we can omit the conditioning on $\Hist_{t-1}$.
    Thus, we obtain,
    \begingroup
    \allowdisplaybreaks
    \begin{align*}
     &\Regrv_T(RM-KD)
     \\
     = &
     \sum_{t=1}^T V^{\pi^\star(c_t;\cdot)}_{\M(c_t)}
    -
    V^{\pi_t(c_t;\cdot)}_{\M(c_t)}
    \\
    \tag{By Azuma's inequality, holds w.p. at least $ 1-\delta$.}
    \leq &
     %\underbrace{\leq}_{(2)} 
     \sum_{t=1}^T  \E_{c}\left[V^{\pi^{\star}(c;\cdot)}_{\M(c)}(s_0)  - V^{\pi_t(c;\cdot)}_{\M(c)}(s_0) \right] + 2H\sqrt{2 T\log(4/\delta)}
    \\
    % \tag{$(\star)$}
    \leq &
    \sum_{t = |A| + 1}^T  \E_{c}[V^{\pi^{\star}(c;\cdot)}_{\M(c)}(s_0)  - V^{\pi_t(c;\cdot)}_{\M(c)}(s_0) ] + 2H\sqrt{2 T\log(4/\delta)} 
    + 
    |A| H
    \\
   % \underbrace{\leq}_{(3)} &
   \leq &
    \sum_{t = |A| + 1}^T  \left(2 \beta_t \cdot
    \E_c \left[ \sum_{h=0}^{H-1} \sum_{s_h \in S^c_h} \frac{ q_h(s_h|\pi_t(c;\cdot),P^c_\star) }{\sum_{i=1}^{t-1} \I[\pi_t(c;s_h)= \pi_i(c;s_h)]q_h(s_h | \pi_i(c;\cdot),P^c_\star)}\right]
    + 
    2 \beta_t \frac{H |S| |A|}{t}\right) 
    \\
    \tag{By~\cref{eq:thm-1-bound-reg}}
    & +
    2H\sqrt{2 T\log(4/\delta)} 
    + 
    |A| H
    \\
    %\underbrace{\leq}_{(4)} &
    \leq &
    2\beta_T \cdot
    \sum_{t = |A| + 1}^T\E_c \left[\sum_{h=0}^{H-1} \sum_{s_h \in S^c_h} \frac{ q_h(s_h|\pi_t(c;\cdot),P^c_\star) }{\sum_{i=1}^{t-1} \I[\pi_t(c;s_h)= \pi_i(c;s_h)]q_h(s_h | \pi_i(c;\cdot),P^c_\star)}\right]
    \\
    \tag{Since $\beta_T \geq \beta_t$ for all $ t\leq T$.}
    & + 
    \sum_{t = |A| + 1}^T 2\beta_t \frac{H |S||A|}{t}
    + 
    2H\sqrt{2 T\log(4/\delta)} +
    |A| H
    \\
    \leq &
    %\underbrace{\leq}_{(5)} &
    2\beta_T \cdot  \frac{|S||A|}{p_{min}}(1+ \log(T/|A|))
    %( + |S||A| \log(T/|A|))
    + 
    \sum_{t = |A| + 1}^T 2\beta_t \frac{H |S| |A|}{t}
    \\
    \tag{By~\cref{lemma:Contextual-Potential}}
    & + 
    2H\sqrt{2 T\log(4/\delta)} +
    |A| H
    \\
    \tag{By $\beta_t$ choice}
    \leq &
    2 \frac{|S||A|}{p_{min}} \sqrt{ \frac{17 \log(4|\F|T^3 /\delta) T}{ |S||A|}}(1+\log(T/|A|)) 
    %(|S||A| + |S||A| \log(T/|A|))
    \\
    & +
    2\sqrt{17 \log(4|\F|T^3/\delta)} \sum_{t = |A| + 1}^T  \sqrt{\frac{t}{ |S| |A|}}\frac{H  |S| |A|}{t}
    \\
    & + 2H\sqrt{2 T\log(4/\delta)} +
    |A| H
    \\
    \tag{$\sum_{t=1}^T \frac{1}{\sqrt{t}} \leq 2 \sqrt{T}$}
    \leq &
    2 \frac{(1+ \log(T/|A|)}{p_{min}}\cdot \sqrt{ 17 \log(4|\F|T^3/\delta) \cdot T |S| |A|} 
    % +
    % 2 \frac{1}{p_{min}}\cdot \sqrt{ 17 \log(4|\F|T^3/\delta) \cdot T  |S| |A|} 
    %\\
    % & 
    %
    +
    4H \sqrt{17 \log(4|\F|T^3/\delta)  T |S| |A|} 
    \\
    & 
    + 
    2H\sqrt{2 T\log(4/\delta)} 
    +
    |A| H
    % \\
    % &\leq
    % 7 \frac{1}{p_{min}} \cdot H \cdot \log(T/|A|) \sqrt{ 17 \log(4|\F|T^3(H)^2/\delta) \cdot T \cdot |S|\cdot|A|} 
    % + 
    % |A|\cdot H
    \\
    &=
    % \tilde{O}\left(
    % \max\left\{ 
    % \frac{ \sqrt{T  |S| |A|\log\frac{|\F|}{\delta}}}{p_{min} },\;
    % H\sqrt{ T  |S| |A|  \log\frac{|\F|}{\delta}},\;
    % |A| H \right\}\right)
    \tilde{O}\left(
    %\max\left\{ 
    \frac{ \sqrt{T |S||A| \log\frac{|\F|}{\delta}}}{p_{min} }+
    H  \sqrt{ T |S||A|\log\frac{|\F|}{\delta}}+
    |A| H
    %\right\}
    \right)
    .
    \end{align*}
    \endgroup
    By union bound, the regret bound above holds with probability at least $1-\delta$.
\end{proof}

\begin{corollary}[regret bound in terms of $\G$]\label{corl:regret-kd-}
    For every $T \geq 1$, finite function class $\G$ ($\F = \G^S$) and $\delta \in (0,1)$, the following holds
    %let ${\beta_t = \sqrt{ \frac{17 t \log(4|\F|t^3/\delta)}{|S||A|}}}$ for all $t \in [T]$.
    with probability at least $1-\delta$, for the same choice of parameters $\{\beta_t\}_{t \in [T]}$.
    $$
        \Regrv_T(RM-KD)
        \leq
        \widetilde{O}\left(
        \frac{ |S| \sqrt{T |A| \log\frac{|\G|}{\delta}}}{p_{min} }+
        H|S|  \sqrt{ T |A|\log\frac{|\G|}{\delta}}+
        |A| H
        \right)
        .
    $$
\end{corollary}
\begin{proof}
    Plug $\log(|\F|) = |S| \log(|\G|)$ in the bound of~\cref{thm:regret-bound-kd}.
\end{proof}

\section{Unknown and Context-Independent Dynamics}\label{Appendix:UCID}
In this section, we consider the setting of unknown and context-independent dynamics. We operate under the following assumptions.
\\
\noindent\textbf{Unknown and context-independent dynamics.}{
Meaning, for every context $c \in \C$ we have $P^c_\star = P_\star$, i.e., all the contexts has the same dynamics. In addition, $P_\star$ is unknown to the learner. 
Recall we assume the CMDP is layered. 
Since the dynamics is context-independent so is the partition of the states space into layers. 
Hence, we denote by $S_0,S_1, \ldots S_H$ the disjoint layers of the CMDP. As before, $S_0 = \{s_0\}$, $S_H = \{s_H\}$, where $s_0$, $s_H$ are unique start and final states, respectively, and $S = \bigcup_{h \in [H]}S_h$.
We assume that the (context-independent) partition to layers is known to the learner.
}
\\
\noindent\textbf{Known minimum reachability parameter.}{
We assume that there exists $p_{min} \in (0,1]$ such that for each layer $h \in [H]$, state $s_h \in S_h$ and context $c \in \C$ for every context-dependent policy $\pi$ it holds that 
$
    q_h(s_h | \pi(c;\cdot), P_\star) \geq p_{min}
$. 
We remark in this section, we assume that $p_{min}$ is known to the learner.
}
\\
In our algorithms and proofs, we use the notation of the Q-value function.
\\
\noindent\textbf{Q-function.}{
Given a policy $\pi$ and a MDP 
    $
        M 
        = 
        (S,A,P,r,s_0, H)
    $, 
the $h \in [H-1] $ stage  Q-function of a state $s \in S_h$ and an action $a \in A$ is defined as 
    $$
        Q^{\pi}_{M,h} (s,a)
        = 
        \mathbb{E}_{\pi, M} 
        \Big[
        \sum_{k=h}^{H-1} r(s_k, a_k)|s_h = s, a_h=a \Big].
    $$
For completeness we define
$ Q^\pi_{M,H}(s,a) = 0$ for all $ (s,a) \in S\times A$.    
For brevity, when $h = 0$ we denote $Q^{\pi}_{M,0}(s_0,a) := Q^{\pi}_M (s_0,a)$.
Recall Bellman's equations for the Q-function. 
For every layer $h \in [H-1] $, state $s \in S_h$ and an action $a \in A$ it holds that
\[
    Q^{\pi}_{M,h}(s,a) = r(s,a) + \mathop{\E}_{s' \sim P(\cdot|s,a)}\left[V^\pi_{M,h+1}(s')\right].
\]
}

\subsection{Regret Minimization Algorithm}\label{subsec:UCID-ALG}

In Algorithm RM-UCID (\cref{alg:RM-UCID}),
the first $|A|$ rounds are initialization rounds, where in round $i \in \{1, 2, \ldots, |A|\}$ the agent plays the policy $\pi_i$ which always selects action $a_i$, i.e., $\pi_i(c;s) = a_i$ for every context $c \in \C$ and state $s \in S$. She updates the LSR oracle and the counters $N_i(s,a)$, $N_i(s,a,s')$ for all $(s,a,s')\in S \times A \times S$ using the observed trajectory $\sigma^i$.\\
At every round $t =  |A|+1 ,\ldots ,T$ the agent observes a context $c_t$ and computes the optimistic approximated model $\Mhat_t(c)$ and its optimal deterministic policy $\pi_t(c_t;\cdot)$.\\
Similarly to Algorithm RM-KD (i.e.,~\cref{alg:RM-CMDP-KD}),
to compute the optimistic rewards function for the context $c_t$, $\widehat{r}^{c_t}_t$, the agent needs to compute $\pi_k(c_t;\cdot)$ for all $k = |A|+1, \ldots, t-1$. To compute $\pi_k(c_t;\cdot)$ she needs to compute both $\widehat{r}^{c_t}_k$ and the optimistic dynamics $\widehat{P}^{c_t}_k$, for all $k = |A|+1, \ldots, t-1$.
Hence, the agent performs the following steps.\\
For all $k=|A|+1, \ldots , t-1$:
\begin{enumerate}
    \item The agent computes the empirical dynamics at round $k$, which defined as,${\Bar{P}_k(s'|s,a) = \frac{N_k(s,a,s')}{\max\{1, N_k(s,a)\} }}$ for all $(s,a,s') \in S\times A\times S$ .
    \item The agent computes an approximated rewards function $\widehat{r}^{c_t}_k$ for the context $c_t$ using the approximation computed in round $k$, $\hat{f}_k$, the policies $\{\pi_i(c_t;\cdot)\}_{i=1}^{k-1}$ and the minimum reachability parameter $p_{min}$.
    \item The agent computes the optimistic dynamics $\widehat{P}^{c_t}_k$ and a deterministic context-dependent policy $\pi_k(c_t;\cdot)$ by invoking Algorithm \texttt{FOA} (see~\cref{alg:foa}). We now briefly describe Algorithm \texttt{FOA}. Algorithm \texttt{FOA} finds an optimistic dynamics and policy, with respect to the rewards function $\widehat{r}^{c_t}_k$ as follows.
    \begin{itemize}
        \item[$(i)$] The first step is to solve the linear program in~\cref{eq:LP-q}, and by that compute dynamics $\widehat{P}^{c_t}_k$ and a stochastic policy $\pi^{\widehat{q}^{c_t}_k}$ that obtain a maximal value for the rewards function $\widehat{r}^{c_t}_k$, under the constraints that for all $(s,a) \in S \times A$, the optimistic dynamics satisfies that $\left\|\widehat{P}^{c_t}_k (\cdot|s,a) -\Bar{P}_k(\cdot|s,a)\right\|_1 \leq \xi_k(s,a)$, where $\xi_k$ is a confidence bound defined in the algorithm.
        \item[$(ii)$] The second step is to %invoke Algorithm \texttt{FDP} (see~\cref{alg:FDP}) which return a 
        compute a deterministic policy $\pi_k(c_t;\cdot)$ which 
        %obtains value that equals or higher than that  of $\pi^{\widehat{q}^{c_t}_k}$ on 
        is optimal for the MDP $\Mhat_t(c_t)$ which defined by the rewards function $\widehat{r}^{c_t}_k$ and the optimistic dynamics $\widehat{P}^{c_t}_k$. This can be done efficiently using a standard planning algorithm.
    \end{itemize}
    \item Lastly, the agent computes $\pi_t(c_t;\cdot)$ in the same manner described above. Then she runs $\pi_t(c_t;\cdot)$ to generate a trajectory $\sigma^t$ and update the LSR oracle and the counters using $\sigma^t$.
\end{enumerate}
For more details, see~\cref{alg:RM-UCID}.
\begin{algorithm}
    \caption{Regret Minimization for Unknown Context Independent Dynamics (RM-UCID)}
    \label{alg:RM-UCID}
    \begin{algorithmic}[1]
        \STATE
        { 
            \textbf{inputs:}
            \begin{itemize}
                \item MDP parameters: 
                $S= \{S_0, S_1, \ldots , S_H\}$
                , $A$, $H$, $s_0$
                \item Confidence parameter $\delta$ and tuning parameters $\{\beta_t\}_{t=1}^T$.
                \item Minimum reachability parameter $p_{min}$.
            \end{itemize}
        } 
        \STATE{initialize counters $N_t(s,a) = 0$, $N_t(s,a, s') = 0$ for all $(s,a,s') \in S \times A \times S,\; t\in [T]$.}
        
        % \STATE{Initialization: for the first $|S|\cdot |A|$ rounds, for each $(s,a)$ in turn, run a policy at state $s$ play action $a$, regardless of the current context, for each $(s,a) \in S \times A$ (and for any other state choose an action uniformly at random)}
        \FOR{ round $i =  1, \ldots, |A|$}
        \STATE{run policy $\pi_i$ which always selects action $a_i$}
        \STATE{update counters according to the observed tuples $(s^i_h, a^i_h,s^i_{h+1})$ for all $i \in [ |A|], h \in [H-1]$}
        \ENDFOR{}
        \FOR{ round $t = |A| + 1, \ldots, T$}
            \STATE{compute
            $
                \hat{f}_t \in 
                \arg\min_{f \in \F}
                \sum_{i=1}^{t-1} \sum_{h=0}^{H-1} ( f(c_i,s^i_h,a^i_h) - r^i_h)^2
            $ using the LSR oracle
            }
            \STATE{observe a fresh context $c_t \sim \D$}
            \FOR{$k= |A|+ 1, \ldots, t$}
                \STATE{compute for all 
                $(s,a,s') \in S\times A\times S: \;\;\Bar{P}_k(s'|s,a) = \frac{N_k(s,a,s')}{\max\{1, N_k(s,a)\} }$}            
                \STATE{compute 
                    $
                        % \forall h \in [H], \; 
                        % s_h \in S_h,\; a_h \in A: \;\;
                        \forall (s,a) \in S \times A:\;
                        \widehat{r}_k^{c_t}(s, a) = \hat{f}_k(c_t,s,a) +
                        %\min\left\{
                        %\frac{1}{p_{min}} \cdot
                        \frac{\beta_k}{ p_{min} \sum_{i=1}^{k-1}\I[a = \pi_i(c_t,s)]}
                        %, 1 \right\} 
                    $}
                % \STATE{Compute for all 
                % $(s,a,s') \in S\times A\times S: \;\;\Bar{P}_k(s'|s,a) = \frac{N_k(s,a,s')}{\max\{1, N_k(s,a)\} }$}
                \STATE{find optimistic model $\Mhat_k(c_t) = (S, A, \widehat{P}^{c_t}_k, \widehat{r}^{c_t}_k, s_0, H)$ and a policy for it $\pi_k(c_t;\cdot)$ using algorithm \texttt{FOA}~(\cref{alg:foa})}
            \ENDFOR{}    
            \STATE{play $\pi_t(c_t;\cdot)$ and observe trajectory $\sigma^t = (c_t, s^t_0, a^t_0, r^t_0, s^t_1, \ldots, s^t_{H-1}, a^t_{H-1}, r^t_{H-1},s^t_H) $.}
            \STATE{update the oracle and the counters using $\sigma^t$. The counters update is as follows.
            \begin{align*}
               &N_{t+1}(s,a) =  N_{t}(s,a) + \I[(s,a) \in \sigma^t],\\
               &N_{t+1}(s,a,s') =  N_{t}(s,a,s') + \I[(s,a,s') \in \sigma^t]
                \;\; \forall (s,a,s') \in S \times A \times S
            \end{align*}
            % and 
            % \begin{align*}
            %     &N^{t+1}(s_h,a_h) =  N^{t}(s_h,a_h) 
            %     \;\; \forall (s_h,a_h) \notin \sigma^t\\
            %     &N^{t+1}(s_h,a_h,s_{h+1}) =  N^{t}(s_h,a_h,s_{h+1})
            %     \;\; \forall (s_h,a_h,s_{h+1}) \notin \sigma^t
            % \end{align*}
            }
            \ENDFOR{}
    \end{algorithmic}
\end{algorithm}

% \begin{algorithm}
%     \caption{Find Deterministic Policy (FDP)}
%     \label{alg:FDP}
%     \begin{algorithmic}[1]
%         \STATE
%         { 
%             \textbf{inputs:}
%             \begin{itemize}
%                 \item Layered MDP $M = (S,A,P,r,s_0, H)$, where $S = \{S_0, S_1, \ldots, S_H\}$.
%                 \item Stochastic (time-invariant) Markovian policy $\pi$.
%             \end{itemize}
%         }
%         \FOR{$h = H, H-1, \ldots, 0$}
%             \STATE{compute $Q^{\pi}_{M,h}(s,a)$ for all $(s,a) \in S_h \times A$}
%             \COMMENT{can be computed in $poly(|S|,|A|,H)$ time.}
%             \FOR{$s \in S_h$}
%                 \STATE{choose $\pi'(s) \in \arg\max_{a \in A} Q^{\pi}_{M,h}(s,a) $}
%             \ENDFOR{}
%         \ENDFOR{}
%         \STATE{\textbf{return: } $\pi'$}
%     \end{algorithmic}
% \end{algorithm}

\begin{algorithm}
%[H]
    \caption{Find Optimistic Approximation (FOA)}
    \label{alg:foa}
    \begin{algorithmic}[1]
        \STATE
        { 
            \textbf{inputs:}
            \begin{itemize}
                \item MDP parameters: 
                $S= \{S_0, S_1, \ldots, S_H\}$, $A$, $s_0$, $H$.
                \item Confidence parameter $\delta$.
                \item Round $k$, counters $N_k(s,a)$ for all $(s,a) \in S \times A$, and the empirical dynamics $\Bar{P}_k$.
                %\item Optimistic rewards function $\widehat{r}_t^{c}$. 
                \item Context $c$ and the approximated rewards function $\widehat{r}_k^{c}$. 
            \end{itemize}
        }
        \STATE{compute confidence intervals over the dynamics 
        $
            \xi_k(s,a) = 2\sqrt{\frac{|S| + \lambda}{\max\{1, N_k(s,a)\}}},
        $ 
        where 
        % $
        %     \lambda = 2\log (4 |S||A|H^2 T^2/\delta)
        % $
        $
            {\lambda = 2\log (4 |S||A| T^2/\delta)}
        $
        .
        }
        \STATE{solve the following linear program
        \begin{equation}\label{eq:LP-q}
        \begin{split}
        &\max_{q \in [0,1]^{|S||A||S|}} \sum_{h=0}^{H-1} \sum_{s \in S_h} \sum_{a \in A} \widehat{r}^c_k(s,a) \cdot \sum_{s' \in S}q(s,a,s')
        \\
        &\qquad \text{subject to: }\\
        &\qquad q(s,a,s') \in [0,1],\;\; \forall h \in [H-1], (s,a,s') \in S \times A \times S\\
        &\qquad \sum_{s \in S_h} \sum_{a \in A} \sum_{s' \in S_{h+1}} q(s,a,s') = 1,\;\; \forall h \in [H-1]
        \\
        &\qquad \sum_{s' \in S_{h+1}}\sum_{a \in A} q(s,a,s')
        =
        \sum_{s' \in S_{h-1}} \sum_{a \in A}  q(s',a,s),\;\; \forall  h \in \{1, \ldots ,H-1\}, s \in S_h
        \\
        &\qquad \sum_{s' \in S_{h+1}}
        \left| 
            q(s,a,s') - \Bar{P}_k(s'|s,a) \cdot \sum_{s'' \in S_{h+1}}q(s,a,s'')
        \right| \leq 
        \xi_k(s,a) \cdot \sum_{s'' \in S_{h+1}}q(s,a,s'')
        \\
        &\qquad\;\; \forall h \in [H-1], (s,a) \in S_h \times A
    \end{split}
    \end{equation}
    }
    \STATE{denote by $\widehat{q}^c_k$ the solution of the LP in~\cref{eq:LP-q} and compute the induced policy and dynamics
    $$
        \pi^{\widehat{q}^c_k}(a|s) = \frac{\sum_{s' \in S_{h+1}} \widehat{q}^c_k(s,a,s')}{\sum_{a' \in A} \sum_{s'\in S_{h+1}} \widehat{q}^c_k(s,a',s')},\;\; \forall h \in [H-1], s \in S_h,a\in A
    $$
    and
    $$
     \widehat{P}^c_k(s'|s,a) = \frac{ \widehat{q}^c_k(s,a,s')}{ \sum_{s''\in S_{h+1}} \widehat{q}^c_k(s,a,s'')},\;\; \forall h \in [H-1], s \in S_h, a\in A
    $$
    let $\Mhat_k(c) = (S,A,\widehat{P}^c_k,\widehat{r}^c_k,s_0,H)$ 
    }
    \STATE{compute a deterministic optimal policy $\pi_k(c,\cdot) \in \arg\max_{\pi \in S \to A}V^{\pi}_{\Mhat_k(c)}(s_0)$ using a planning algorithm}
    %$\pi_k(c, \cdot) \gets \texttt{FDP}(\Mhat_k(c), \pi^{\widehat{q}^c_k})$}
    
    \STATE{\textbf{return: }
    $\pi_k(c;\cdot)$ and 
    %$\widehat{P}^c_k$.}
    $
    \Mhat_k(c) = (S,A,\widehat{P}^c_k,\widehat{r}^c_k,s_0,H)$ }
    %that solve the LP in~\cref{eq:olp}}
        
    \end{algorithmic}
\end{algorithm}

\begin{remark}
    The optimistic approximated dynamics at round $t$, $\widehat{P}_t$, is a context-dependent dynamics although the true dynamics $p_\star$ is not. This happens since the rewards function $\widehat{r}_t$ is context-dependent and for every context the optimistic dynamics is computed with respect to the rewards function of the related context. 
    However, by the constraints of the LP in~\cref{eq:LP-q}, the context-dependent dynamics induces the same partition of the stated to layers for all the context. Moreover, the partition of the states to layers is identical to the partition induced by $P_\star$ (which is given to the algorithm as input).
\end{remark}

\subsection{Regret Analysis}
For every $t >  |A|$ we define the following MDPs for every context $c \in \C$,
\begin{enumerate}
    \item $\M^{(f,P_\star)}(c) = (S,A,P_\star,f(c,\cdot,\cdot),s_0,H)$, for any $f \in \F$ and the true dynamics $P_\star$.
    \item $\M^{(f_\star,P_\star)}(c)$ is the true model, where $f_\star$ is the true context dependent rewards and $P_\star$ is the true dynamics, which we also denote by $\M(c)$.
    \item $\M^{(\widehat{r}_t,P_\star)}(c) = (S,A,P_\star,\widehat{r}^c_t,s_0,H)$ is the model where $\widehat{r}^c_t$ is the approximated rewards in round $t$ defined in~\cref{alg:RM-UCID} and $P_\star$ is the true dynamics.
    \item $\Mhat_t(c)=\M^{(\widehat{r}_t,\widehat{P}_t)}(c) $ is the approximated MDP in time $t > |A|$ defined in~\cref{alg:RM-UCID}. 
    \item $\M^{(f,\widehat{P}_t)}(c) = (S,A,\widehat{P}^c_t,f(c,\cdot,\cdot),s_0,H)$, for any $f \in \F$, where $\widehat{P}^c_t$ is the approximated optimistic dynamics associated with the context $c$ which is defined in~\cref{alg:RM-UCID}.
\end{enumerate}

\subsubsection{Analysis Outline.}
We analyse the regret of Algorithm RM-UCID under the following two good events.
\paragraph{Event $G_1$: confidence bound over policies w.r.t rewards estimation.}\label{par:good-event-rewards-UCID}{
Intuitively, the good event $G_1$ states that the following confidence bound over the expected value difference holds for every $t>|A|$ and context-dependent policy $\pi \in \Pi_\C$.
We show that $G_1$ holds with high probability.
\\
Formally, in~\cref{lemma:CB-policy-UCID} we prove that with probability at least $1 - {\delta}/{4}$, for all $t >|A|$ and every context-dependent policy $\pi \in \Pi_\C$  it holds that
    \begin{align*}
        &\left|\E_c[V^{\pi(c;\cdot)}_{\M(c)}(s_0) ] - \E_c[V^{\pi(c;\cdot)}_{\M^{(\hat{f}_t,P_\star)}(c)}(s_0) ]\right| 
        \\
        \leq &
        \sqrt{\E_c \left[ \sum_{h=0}^{H-1} \sum_{s_h \in S_h} \frac{q_h(s_h|\pi ,P_\star)}{\sum_{i=1}^{t-1} \I[\pi(s_h)= \pi_i(c;s_h)]q_h(s_h | \pi_i(c;\cdot),P_\star)}\right]}
        \cdot
        %\sqrt{ 68 \log(8|\F|t^3H^2/\delta)}
        \sqrt{ 68 H \log(8|\F|t^3/\delta)}
         .
    \end{align*}
    % where $\beta_t = \sqrt{ \frac{17 t \log(8|\F|t^3/\delta)}{ |S||A|}}$. 
    %(see  for a full proof).

\noindent We derive this good event similarly to shown for the known dynamics setting.  
%using a two steps analysis, that is identical to
%the analysis presented in the known dynamics setting.
It follows using the results of %~\cref{lemma:CB-policy-UCID,lemma:sq-trick-UCID}
~\cref{lemma:UC-lemma-5-UCID}.
%and~\cref{lemma:CB-policy-UCID}.
%For more details, see~\cref{Subsec:good-event-rewards}.
Then, under this good event, in~\cref{lemma:sq-trick-UCID} we derive an additive confidence bound, similarly to shown for the rewards in~\cref{Appendix:KD}.
}

\paragraph{Event $G_2$: Confidence interval over the empirical dynamics.}\label{par:good-event-dynamics-UCID}{
Intuitively, it states that the following confidence bound over the empirical dynamics $\Bar{P}_t$ holds, for every round $t \in \{1,2,\ldots, T\}$ and state-action pair $(s,a) \in S \times A$.
We show that $G_2$ holds with high probability.
Formally,
\begin{lemma}
    With probability at least $1-{\delta}/{4}$, for every $t \in \{1,2,\ldots, T\}$
    %, $k  \in \{1,2,\ldots, t\}$ 
    and state-action pair $(s,a) \in S \times A$, for $\xi_t(s,a) = 2\sqrt{\frac{|S| + 2\log (4 |S||A| T^2/\delta)}{\max\{1, N_t(s,a)\}}}$ the following holds.
    $$
        \| P_\star(\cdot | s,a) - \Bar{P}_t(\cdot |s,a)\|_1 \leq \xi_t(s,a)
        .
    $$ 
\end{lemma}

\begin{proof}
    The lemma holds by Bretagnolle Huber-Carol inequality~(see~\cref{lemma:Bretagnolle-Huber-Carol}) (using the tabular argument) and union bound for all ${t \in \{1,2,\ldots,T\}}$
    %, ${k \in \{1,2,\ldots,t\}}$ 
    and ${(s,a) \in S \times A}$.    
\end{proof}

}
In the following, we first analyse the error caused by the dynamic approximation.
%(see~\cref{subsubsec:dynamics-error-UCID}).
The analysis consists of a two main steps.

\noindent\textbf{Step 1:} In~\cref{lemma:optimality-of-pi-t-and-p-t} we show that under the good event $G_2$, for all $t > |A|$ and every context $c \in \C$, it holds that
    \begin{equation}\label{ineq:optimality}
        V^{\pi_t(c;\cdot)}_{\Mhat_t(c)}(s_0)
        \geq V^{\pi^\star(c;\cdot)}_{\M^{(\widehat{r}_t,P_\star)}(c)},
    \end{equation}
where $\pi^\star$ is the optimal context-dependent policy and $\pi_t$ is the selected context-dependent policy in round $t$.     
To prove the above inequality, we use the result of~\cref{lemma:LP-equivalence}.
    
\noindent\textbf{Step 2:} In~\cref{lemma:UCID-dynamixs-value-error-over-t} we show that under the good event $G_2$, with probability at least $1-\delta/4$ it holds that
    \begin{equation}\label{ineq:UCID-dynamixs-value-error-over-t}
        \sum_{t= |A| + 1}^T \E_c[V^{\pi_t(c;\cdot)}_{\M^{(\hat{f}_t, \widehat{P}_t)}(c)}] - \E_c[V^{\pi_t(c;\cdot)}_{\M^{(\hat{f}_t ,P_\star)}(c)}]
        \leq
        {O}(H^{1.5}|S|\sqrt{|A|T  }\log(|S||A|T^2/\delta))
        .
    \end{equation}

    To prove the lemma, we 
    % first show the approximate rewards function is bounded for every context $c \in \C$ (see~\cref{lemma:r-hat-bound-UCID}), and then
    use the Value Difference Lemma (\cref{lemma:val-diff-simple-P-bar}), and the good event assumption to obtain the above bound.

% \begin{remark}
%     Due to a computation mistake, the above resolute is $\sqrt{H}$ higher then stated in the paper. The resulted regret bound is higher by at most $\tilde{O}(H/\sqrt{|S|})$ factor then the bound stated in the paper.
% \end{remark}

To analyse the error caused by the rewards approximation %(see~\cref{subsubsec:rewards-error-UCID}))
, we repeat the four steps analysis presented for the known dynamics setting (see~\cref{Appendix:KD}). We use steps 1 and 2 (\cref{lemma:UC-lemma-5-UCID,lemma:CB-policy-UCID}) to derive the good even $G_1$.
In step 3 (see~\cref{lemma:sq-trick-UCID}), we relax the confidence bound to be additive.
In step 4 (see~\cref{lemma:Contextual-Potential-UCID}) we bound the sum of contextual potential functions.

Using the results of steps 3 and 4, in~\cref{lemma:UCID-rewards-value-error-over-t} we obtain for all $T>|A|$ that
\begin{equation}\label{ineq:UCID-rewards-value-error-over-t}
    \begin{split}
        &\sum_{t= |A| + 1}^T \E_c[V^{\pi_t(c;\cdot)}_{\M^{(\hat{f}_t ,P_\star)}(c)}] - \E_c[V^{\pi_t(c;\cdot)}_{\M(c)}]
        \\
        \leq & 
        2 H\sqrt{17 \log(8|\F|T^3/\delta) |S||A| T} 
        + (1+ \log(T/|A|)) \frac{\sqrt{ 17 \log(8|\F|T^3/\delta) T  |S|  |A|}}{p_{min}}
        .
    \end{split}
\end{equation} 
To derive the regret bound, using~\cref{ineq:optimality} we prove the optimism lemma (\cref{lemma:optimism}) which states that under the good events $G_1$ and $G_2$ for all $t>|A|$ it holds that
    \[
        \E_c \left[ V^{\pi^\star(c;\cdot)}_{\M(c)}(s_0)\right]
        -
        \E_c \left[ V^{\pi_t(c;\cdot)}_{\Mhat_t(c)}(s_0) \right]
        \leq
        % \beta_t \left(
        % \E_c\left[             
        % \sum_{h = 0}^{H-1} \sum_{s_h \in S_h} 
        % %\frac{1}{p_{min}}
        % \frac{q_h(s_h, \pi^\star(c;s_h)| \pi^\star(c;\cdot),P_\star)}{  p_{min}\sum_{i=1}^{t-1}\I[ \pi^\star(c;s_h) = \pi_i(c;s_h)] } \right]
        % +
        \beta_t\cdot\frac{H |S||A|}{t}
        %\right)
        .
    \]

Lastly, in~\cref{thm:regret-bound-ucid} 
we combine the result of the optimism lemma (\cref{lemma:optimism}), with~\cref{ineq:UCID-dynamixs-value-error-over-t,ineq:UCID-rewards-value-error-over-t} and the cumulative contextual potential bounds (\cref{lemma:Contextual-Potential-UCID}) to obtain a regret bound
%(and an expected regret bound)
of
    % \begin{align*}
    %     % &\Regrv_T( RM-UCID)
    %     % \\
    %     % \leq &          
    %      \tilde{O} \left(
    %      \max\left\{
    %      H|S| \sqrt{T |A|}\log\frac{1}{\delta},\; %dynamics
    %      H\sqrt{|S||A| T \log\frac{|\F|}{\delta}},\;%rewards
    %      \frac{\sqrt{ 
    %     17 \log(8|\F|T^3/\delta) T  |S| |A|}}
    %     {p_{min}},\;%rewards
    %      |A|H
    %     \right\}         
    %      \right)
    %     ,
    % \end{align*}
    \begin{align*}
        % \tilde{O} \left(
        %  %\max\left\{
        %  \max\{H, \sqrt{S}\}\cdot
        %  H\sqrt{T |A||S|}\log\frac{1}{\delta}+ %dynamics
        %  \max\{H, 1/p_{min}\}\cdot \sqrt{|S||A| T \log\frac{|\F|}{\delta}}+%rewards
        %  |A|H
        %  \right).
        \widetilde{O} \left(
         H^{1.5}|S| \sqrt{T |A|}\log\frac{1}{\delta}+ %dynamics
         (H+ 1/p_{min})\cdot \sqrt{|S||A| T \log\frac{|\F|}{\delta}}+%rewards
         |A|H
         \right),
    \end{align*}
which holds with high probability. In~\cref{corl:regret-bound-ucid-g} we derive a regrt bound in terms of $|\G|$.
We remark that in the following analysis, we use the explicit expression of the contextual potential $\psi_t$.

\subsubsection{Analysing the Error Caused by the Dynamics Approximation.}\label{subsubsec:dynamics-error-UCID}
In the following, we analyse the error caused by the tabular dynamics approximation.

\paragraph{Analysis of Algorithm FOA.}

The following lemmas shows that Algorithm FOA (\cref{alg:foa}) returns an optimistic dynamics and deterministic optimal policy, with respect to the approximated rewards function.

\begin{lemma}\label{lemma:LP-equivalence}
    Assume the good event $G_2$ holds. 
    Then, for all $k > |A|$ and a context $c \in \C$
    the linear program in~\cref{eq:LP-q}  
    is equivalent to the following constraint maximization problem,
    \begin{equation}\label{eq:olp}
    \begin{split}
        &\max_{\substack{(\pi,P )}
        } 
        %\max_{P^q} 
        V^{\pi}_{\M^{(\widehat{r}_k,P)}(c)}
        %{\Mhat_t(c)}
        (s_0) 
        \\
        &\qquad \text{subject to: }\\
        &\qquad \pi \in S \to \Delta(A)\\
        &\qquad \forall (s,a) \in S \times A: \;\;\;  P(\cdot |s, a )\in \Delta(S)
        \\
        &\qquad \forall (s,a)\in S \times A: \;\;\; 
        \|P(\cdot |s, a ) - \Bar{P}_k(\cdot |s, a )\|_1 \leq  \xi_k(s,a)
    \end{split}
    \end{equation}
    where $\Bar{P}_k$ is the the empirical dynamics, $\widehat{r}_k$ is the approximated reward function at round $k$ and we define $\M^{(\widehat{r}_k,P)}(c) := (S,A,P,\widehat{r}_k^{c}, s_0, H)$. %(P is a context-independent dynamics).
\end{lemma}

\begin{proof}
    Fix round $k > |A|$ and a context $c \in \C$.
    
    Consider the following extended definition of the occupancy measure presented by
    \citet{rosenberg2019online} for any dynamics $P$ and stochastic Markovian policy $\pi$,
    \[
        q^{P,\pi}(s, a, s'):=
        \Prob_{P,\pi}[s_h=s,a_h=a,s_{h+1}=s'], \;\; \forall h \in [H-1], (s,a,s') \in S_h \times A \times S_{h+1}.
    \]
    Let $\Delta(M)$ denote the set of all extended occupancy measures.\\ 
    We have that $q \in \Delta(M)$ if and only if $q$ satisfies the following requirements.
    \begin{enumerate}
        \item $\sum_{s \in S_h} \sum_{a \in A} \sum_{s' \in S_{h+1}} q(s,a,s') = 1,\;\forall h \in [H-1]$.
        \item $\sum_{s' \in S_{h+1}}\sum_{a \in A} q(s,a,s')
        =
        \sum_{s' \in S_{h-1}} \sum_{a \in A}  q(s',a,s),\;\forall  h \in \{1, \ldots ,H-1\}, s \in S_h$.
    \end{enumerate}
    % \begin{align*}
    %     &\sum_{s \in S_h} \sum_{a \in A} \sum_{s' \in _{h+1}} q(s,a,s') = 1,\;\; \forall h \in [H-1].
    %     \\
    %     & \sum_{s' \in S_{h+1}}\sum_{a \in A} q(s,a,s')
    %     =
    %     \sum_{s' \in S_{h-1}} \sum_{a \in A}  q(s',a,s),\;\; \forall  h \in \{1, \ldots ,H-1\}, s \in S_h.
    % \end{align*}
    
    \citet{rosenberg2019online} showed that $q \in \Delta(M)$ if and only if there exist a pair of stochastic Markovian policy and dynamics $(\pi^q,P^q)$ for which 
    $$
        \pi^q(a|s) = \frac{\sum_{s' \in S_{h+1}} q(s,a,s')}{\sum_{a' \in A} \sum_{s'\in S_{h+1}} q(s,a',s')},\;\; \forall h \in [H-1], s \in S_h,a\in A,
    $$
    and
    $$
     P^q(s'|s,a) = \frac{ q(s,a,s')}{ \sum_{s''\in S_{h+1}} q(s,a,s'')},\;\; \forall h \in [H-1], s \in S_h, a\in A.
    $$
    
    Using the extended occupancy measure definition, for $\M^{(\widehat{r}_k,P^q)}(c) = (S,A,P^q, \widehat{r}^c_k, s_0, H)$ and $\pi^q$
    %(where the state space is layered $S = \{S_0, \ldots ,S_H\}$, $S_0 = \{s_0\}$ and $S_H = \{s_H\}$)
    it holds that
    \begin{align*}
        V^{\pi^q}_{\M^{(\widehat{r}_k,P^q)}(c)}(s_0)
        =
        \sum_{h =0}^{H-1}
        \sum_{s \in S_h} \sum_{a \in A} \widehat{r}^c_k(s,a) \sum_{s' \in S_{h+1}} q(s,a,s').
    \end{align*}
    Hence, the objective functions of both maximization problems are equivalent.
    
    In the following we show that the constraints of both maximization problems are equivalent as well.
    \\
    Let $\Bar{P}_k$ be the empirical dynamics at round $k$.
    In the maximization problem in~\cref{eq:olp} we have the constraint 
    % We would like to limit the occupancy measure set to include only dynamics $\widehat{P}$ that satisfies
    $\|P(\cdot|s,a) - \Bar{P}_k(\cdot|s,a)\|_1 \leq \xi_k(s,a)$, for all $(s,a) \in S \times A$.
    \citet{rosenberg2019online} showed we can equivalently apply the following constraint on the extended occupancy measure $q$
    \begin{align*}
        \sum_{s' \in S_{h+1}}
        \left| 
            q(s,a,s') - \Bar{P}_k(s'|s,a) \cdot \sum_{s'' \in S_{h+1}}q(s,a,s'')
        \right| \leq 
        \xi_k(s,a) \cdot \sum_{s'' \in S_{h+1}}q(s,a,s''),\;\; \forall h \in [H-1], (s,a) \in S_h \times A
    \end{align*}
    and obtain that
    $\|P^q(\cdot|s,a) - \Bar{P}_k(\cdot|s,a)\|_1 \leq \xi_k(s,a)$, for all $(s,a) \in S \times A$.
    \\
    Hence, we showed the constrains on $(\pi,P)$ in~\cref{eq:olp} are equivalent to the constrains on $q$ in~\cref{eq:LP-q}, and so are objective functions we maximize.
    Hence, the two optimization problems are equivalent.
    
    % By all the above, the maximization problem in~\cref{eq:olp} is equivalent to the linear program~\cref{eq:LP-q}. 
\end{proof}

\begin{corollary}\label{corl:LP-equivalence}
    Assume the good event $G_2$ holds. 
    For all $k > |A|$ and any context $c \in \C$,
    let $\widehat{q}^c_k$ be an optimal solution for the LP in~\cref{eq:LP-q}.
    Consider the induced policy and dynamics,
    $$
        \pi^{\widehat{q}^c_k}(a|s) = \frac{\sum_{s' \in S_{h+1}} \widehat{q}^c_k(s,a,s')}{\sum_{a' \in A} \sum_{s'\in S_{h+1}} \widehat{q}^c_k(s,a',s')},\;\; \forall h \in [H-1], s \in S_h,a\in A,
    $$
    and
    $$
     \widehat{P}^c_k(s'|s,a) = \frac{ \widehat{q}^c_k(s,a,s')}{ \sum_{s''\in S_{h+1}} \widehat{q}^c_k(s,a,s'')},\;\; \forall h \in [H-1], s \in S_h, a\in A.
    $$
   
    Then, $\left(\pi^{\widehat{q}^c_k}, \widehat{P}^c_k\right)$ is an optimal solution of the maximization problem in~\cref{eq:olp}.
    
    Moreover, $\pi^{\widehat{q}^c_k}$ is an optimal stochastic policy for the MDP $\M^{(\widehat{r}_k,\widehat{P}^c_k) }(c):= (S,A,\widehat{P}^c_k,\widehat{r}_k^{c}, s_0, H)$.
\end{corollary}

\begin{proof}
    The first part of the corollary is immediately implied by~\cref{lemma:LP-equivalence}.
    
    For the second part, assume that there exist a round $k$ and a context $c$ such that $\pi^{\widehat{q}^c_k}$ is not an optimal policy of the MDP $\M^{(\widehat{r}_k,\widehat{P}^c_k) }(c)$ .
    Meaning,
    %for the context $c \in \C$ at round $t$,
    there exists a policy ${\pi':S \to \Delta(A)}$ for which
    $$V^{\pi'}_{\M^{(\widehat{r}_k,\widehat{P}^c_k) }(c)}(s_0) > V^{\pi^{\widehat{q}^c_k}}_{\M^{(\widehat{r}_k,\widehat{P}^c_k) }(c)}(s_0).$$
    Under the good event $G_2$, $\left(\pi', \widehat{P}^c_k\right)$ is a feasible solution to the maximization problem in~\cref{eq:olp}, 
    for the context $c$ at round $k$,
    while $\left(\pi^{\widehat{q}^c_k}, \widehat{P}^c_k\right)$ is an optimal solution.
    %to the maximization problem in~\cref{eq:olp}, for the context $c$ at round $t$.
    % for every policy $\pi: S \to \Delta(A)$ it holds that
    Hence, 
    $$V^{\pi'}_{\M^{(\widehat{r}_k,\widehat{P}^c_k) }(c)}(s_0) \leq V^{\pi^{\widehat{q}^c_k}}_{\M^{(\widehat{r}_k,\widehat{P}^c_k) }(c)}(s_0),$$ a contradiction.
\end{proof}

\begin{lemma}[optimality of $(\pi_t, \widehat{P}_t)$ w.r.t $\widehat{r}_t$]\label{lemma:optimality-of-pi-t-and-p-t}
    Assume the good event $G_2$ holds.
    Then, for all $t > |A|$ and every context $c \in \C$, it holds that
    \[
        V^{\pi_t(c;\cdot)}_{\Mhat_t(c)}(s_0)
        \geq V^{\pi^\star(c;\cdot)}_{\M^{(\widehat{r}_t,P_\star)}(c)}
        .
    \]
\end{lemma}

\begin{proof}
    Fix a round $t > |A|$ and a context $c \in \C$.
    Under the good event $G_2$ the true dynamics $P_\star$ satisfies
    for all $k \in \{1,2,\ldots,t\}$ that
    \[
        \|P_\star - \Bar{P}_k\|_1 \leq \xi_k(s,a),\;\; \forall (s,a) \in S\times A.
    \]
    Let $\widehat{q}^c_t$ be the solution of LP in~\cref{eq:LP-q} for round $t$ and the context $c$.
    Let $\pi^{\widehat{q}^c_t}$ and $\widehat{P}^c_t$ be the induced policy and dynamics.
    By~\cref{corl:LP-equivalence} we have that $(\pi^{\widehat{q}^c_t},\widehat{P}^c_t)$ is an optimal solution for the maximization problem in~\cref{eq:olp}. Since $(\pi^\star(c;\cdot),P_\star)$ is a feasible solution for~\cref{eq:olp}, by the optimality of $(\pi^{\widehat{q}^c_t},\widehat{P}^c_t)$ and $\Mhat_t(c)$ definition we have that
    \[
        V^{\pi^{\widehat{q}^c_t}}_{\Mhat_t(c)}(s_0)
        =
        V^{\pi^{\widehat{q}^c_t}}_{\M^{(\widehat{r}_t, \widehat{P}^c_t)}(c)}(s_0)
        \geq 
        V^{\pi^\star(c;\cdot)}_{\M^{(\widehat{r}_t,P_\star)}(c)}
        .
    \]
    
    Lastly, since $\pi_t(c;\cdot)$ is an optimal deterministic policy  
    %returned by Algorithm FDP (\cref{alg:FDP}) 
    of the MDP $\Mhat_t(c)$
    %and $\pi^{\widehat{q}^c_t}$,by~\cref{lemma:FDP-correctness} 
    and $\pi^{\widehat{q}^c_t}$ is an optimal stochastic policy for it, 
    it holds that
    \[
          V^{\pi_t(c;\cdot)}_{\Mhat_t(c)}(s_0) = V^{\pi^{\widehat{q}^c_t}}_{\Mhat_t(c)}(s_0),
    \]
    yielding the lemma as
    \[
         V^{\pi_t(c;\cdot)}_{\Mhat_t(c)}(s_0)
         = 
         V^{\pi^{\widehat{q}^c_t}}_{\Mhat_t(c)}(s_0)
         \geq
         V^{\pi^\star(c;\cdot)}_{\M^{(\widehat{r}_t,P_\star)}(c)}. 
    \]
\end{proof}

\paragraph{Bound on the Cumulative Error of the Approximated Dynamics.}

We now bound the cumulative value difference caused by the dynamics approximation.

\begin{lemma}[cumulative value error due to dynamics approximation]\label{lemma:UCID-dynamixs-value-error-over-t}
    Assume the good event $G_2$ holds.
    Let $\{\pi_t \in \Pi_\C\}_{t=1}^T$ be the sequence of selected deterministic context-dependent policies. Let $\{\hat{f}_t\}_{t=1}^T$ be the sequence of least square minimizers.
    Then, for any $T \geq |S||A|$ with probability at least $1-\delta/4$ it holds that
    \[
        \sum_{t= |A| + 1}^T \E_c[V^{\pi_t(c;\cdot)}_{\M^{(\hat{f}_t, \widehat{P}_t)}(c)}] - \E_c[V^{\pi_t(c;\cdot)}_{\M^{(\hat{f}_t ,P_\star)}(c)}]
        \leq
        {O}(H^{1.5}|S|\sqrt{|A|T  }\log(|S||A|T^2/\delta))
         .
    \]
\end{lemma}

\begin{proof}
    Assume the good event $G_2$ holds. Using Value-Difference~\cref{lemma:val-diff-efroni}
    % and~\cref{corl:val-boud-r-hat-UCID} 
    we obtain
    \begingroup
    \allowdisplaybreaks
    \begin{align*}
        &\sum_{t= |A| + 1}^T \E_c[V^{\pi_t(c;\cdot)}_{\M^{(\hat{f}_t, \widehat{P}_t)}(c)}] - \E_c[V^{\pi_t(c;\cdot)}_{\M^{(\hat{f}_t ,P_\star)}(c)}]
        \\
        \tag{Value Difference~\cref{lemma:val-diff-efroni}}
        = &
        \sum_{t=|A| + 1}^T 
        \E_{c} \left[ \E_{\pi_t(c;\cdot),P_\star}\left[
        \sum_{h=0}^{H-1} \sum_{s' \in S} 
        (\widehat{P}^{c}_t(s'|s_h, a_h) - P_\star(s'|s_h,a_h))V^{\pi_t(c;\cdot)}_{\M^{(\hat{f}_t, \widehat{P}_t)}(c), h+1}(s')
        \Bigg| s_0\right]\right]
        \\
        \tag{Since $\hat{f}_t \in[0,1]$}
        \leq &
        H \sum_{t= |A| + 1}^T 
        \E_{c} \left[ \E_{\pi_t(c;\cdot),P_\star}\left[
        \sum_{h=0}^{H-1}
        \|\widehat{P}^{c}_t(\cdot|s_h, a_h) - P_\star(\cdot|s_h,a_h)\|_1
        \Bigg| s_0\right]\right]
        \\
        \tag{By triangle inequality}
        \leq &
        H \sum_{t= |A| + 1}^T 
        \E_{c} \left[ \E_{\pi_t(c;\cdot),P_\star}\left[
        \sum_{h=0}^{H-1}
        \|\widehat{P}^{c}_t(\cdot|s_h, a_h) - \Bar{P}_t(\cdot|s_h,a_h)\|_1
        +
        \|\Bar{P}_t(\cdot|s_h, a_h) - P_\star(\cdot|s_h,a_h)\|_1
        \Bigg| s_0\right]\right]
        \\
        \tag{Since $G_2$ holds and $\|\widehat{P}^{c}_t(\cdot|s, a) - \Bar{P}_t(\cdot|s,a)\|_1 \leq \xi_t(s,a)$ for all $(s,a)$ }
        \leq &
         H \sum_{t= |A| + 1}^T 
        \E_{c} \left[ \E_{\pi_t(c;\cdot),P_\star}\left[
        \sum_{h=0}^{H-1}
        2 \xi_t(s_h,a_h)
        \Bigg| s_0\right]\right]
        \\
        \tag{Representation using occupancy measures}
        = &
        2 H \;  \sum_{t= |A| + 1}^T 
        \E_{c} \left[
        \sum_{h=0}^{H-1} \sum_{s_h \in S_h} \sum_{a_h \in A}
        q_h(s_h, a_h| \pi_t(c;\cdot),P_\star)\cdot \xi_t(s_h,a_h)
        \right]
        \\
        \tag{By $\xi_t$ definition}
        = &
        4 H \sqrt{|S|+2\log (4 |S||A|T^2/\delta)} \;  \sum_{t= |A| + 1}^T 
        \E_{c} \left[
        \sum_{h=0}^{H-1} \sum_{s_h \in S_h} \sum_{a_h \in A}
        \frac{q_h(s_h, a_h| \pi_t(c;\cdot),P_\star)}{\sqrt{\max\{1, N_t(s_h,a_h)\}}}
        \right]
        \\
        \tag{By adding non-negative terms and ``renaming'' the context}
        \leq &
        % 2 H \sqrt{|S|+2\log (4 |S||A|T^2/\delta)} \;  \sum_{t= 
        % %|A| + 
        % 1}^T 
        % \E_{c} \left[
        % \sum_{h=0}^{H-1} \sum_{s_h \in S_h} \sum_{a_h \in A}
        % \frac{q_h(s_h, a_h| \pi_t(c;\cdot),P_\star)}{\sqrt{\max\{1, N_t(s_h,a_h)\}}}
        % \right]
        % \\
        % \tag{By adding and subtracting indicators, see explanation bellow}
        %= &
        4 H \sqrt{|S|+2\log (4 |S||A|T^2/\delta)} \;  \sum_{t=
        %|A| + 
        1}^T 
        \E_{c_t} \left[
        \sum_{h=0}^{H-1} 
        \sum_{s_h \in S_h} \sum_{a_h \in A}
        \frac{q_h(s_h, a_h| \pi_t(c_t;\cdot),P_\star)}{\sqrt{\max\{1, N_t(s_h,a_h)\}}}
        \right]
        \\
        % = &
        % 4 H \sqrt{|S|+2\log (4 |S||A|T^2/\delta)} \;  \sum_{t= |A| + 1}^T 
        % \E_{c_t} \left[
        % \sum_{h=0}^{H-1} \sum_{s \in S_h} \sum_{a \in A}
        % \frac{q_h(s_h, a_h| \pi_t(c_t;\cdot),P_\star)- \I[c_t,t,h,s_h,a_h] + \I[c_t,t,h,s_h,a_h]}{\sqrt{\max\{1, N_t(s_h,a_h)\}}}\right] 
        = &
        % \tag{By adding and subtracting indicators, see explanation bellow}
        4 H \sqrt{|S|+2\log (4 |S||A|T^2/\delta)}  \sum_{t=
        %|A| +
        1}^T 
        \sum_{h=0}^{H-1}\Bigg( 
        \E_{c_t} \left[\sum_{s_h \in S_h} \sum_{a_h \in A}
        \frac{q_h(s_h, a_h| \pi_t(c_t;\cdot),P_\star)}{\sqrt{\max\{1, N_t(s_h,a_h)\}}}
        \right]
        \\
        \tag{By linearity of $\E$ and adding and subtracting indicators, see explanation bellow}
        & - \sum_{s_h \in S_h} \sum_{a_h \in A}
        \frac{\I[c_t,t,h,s_h,a_h]}{\sqrt{\max\{1, N_t(s_h,a_h)\}}}
        +
        \sum_{s_h \in S_h} \sum_{a_h \in A}
        \frac{\I[c_t,t,h,s_h,a_h]}{\sqrt{\max\{1, N_t(s_h,a_h)\}}} \Bigg)
        \\
        \tag{By Azuma's inequality, holds with probability at least $1-\delta/4$ }
        %\underbrace{\leq}_{see (i) bellow}
       \underbrace{\leq}_{(i)} &
       %\leq &
        4 H \sqrt{|S|+2\log (4 |S||A|T^2/\delta)} \left(\sqrt{2TH \log(8/\delta)} +\; 
        %\E_c \left[ 
        \sum_{t= 
        %|A| +
        1}^T 
        \sum_{h=0}^{H-1} \sum_{s_h \in S_h} \sum_{a_h \in A}
        \frac{\I[c_t,t,h,s_h,a_h]}{\sqrt{\max\{1, N_t(s_h,a_h)\}}}
        %\right] 
        \right)
        \\
        % \leq &
        % 8 H \sqrt{|S|+2\log (4 |S||A|T^2/\delta)} \left( \sqrt{2T H \log(8/\delta)} +\;
        % 2
        % %\E_c \left[
        % \sum_{t= 
        % %|A| + 
        % 1}^T 
        % \sum_{h=0}^{H-1} \sum_{s_h \in S_h} \sum_{a_h \in A}
        % \sum_{i=1}^{
        % \max\{1, N_t(s_h,a_h)\}} \frac{1}{\sqrt{i}}
        % %\right]
        % \right)
        % \\
        \tag{Since $S = \bigcup_{h=0}^{H-1}S_h$ and the MDP is layered}
        = &
        4 H \sqrt{|S|+2\log (4 |S||A|T^2/\delta)} \left(  \sqrt{2TH \log(8/\delta)} +\;
        %\E_c \left[
        % \underbrace{\sum_{t=
        % %|A| +
        % 1}^T 
        % \sum_{h=0}^{H-1} \sum_{s_h \in S_h} \sum_{a_h \in A}
        % \frac{\I[c_t,t,h,s_h,a_h]}{\sqrt{\max\{1, N_t(s_h,a_h)\}}}}_{\leq \sqrt{|S||A|T}}
        % %\right]
        \sum_{s \in S} \sum_{a \in A}
        \sum_{i=0}^{N_T(s,a)} \frac{1}{\sqrt{\max\{1, i\}}} 
        \right)
        \\
        \tag{$\sum_{i=1}^n \frac{1}{\sqrt{n}} \leq 2 \sqrt{n}$.}
        \leq &
        4 H \sqrt{|S|+2\log (4 |S||A|T^2/\delta)} \left(  \sqrt{2TH \log(8/\delta)} +\;
        |S||A| + 2
        \sum_{s \in S} \sum_{a \in A}
        \sqrt{N_T(s,a)} 
        \right)
        \\
        \tag{By Cauchy–Schwarz inequality}
        \leq &
        4 H \sqrt{|S|+2\log (4 |S||A|T^2/\delta)} \left(  \sqrt{2TH \log(8/\delta)} + |S||A| + 2\sqrt{|S||A| H T}\; \right)
        \\
        \tag{ For $T \geq |S||A|$}
        \leq &
        4 H \sqrt{|S|+2\log (4 |S||A|T^2/\delta)} \left(  \sqrt{2TH \log(8/\delta)} + 3\sqrt{|S||A| H T}\; \right)
        \\
        \tag{The MDP is layered hence $H \leq |S|$}
        \leq &
        4 H \sqrt{|S|+2\log (4 |S||A|T^2/\delta)} \left(  \sqrt{2T|S| \log(8/\delta)} + 3\sqrt{|S||A| H T}\; \right)
        \\
        \tag{$\log(8/\delta) \geq 1\;\; \forall \delta \in (0,1)$}
        \leq &
        4 H \sqrt{|S|+2\log (4 |S||A|T^2/\delta)} \left(  \sqrt{2T|S| \log(8/\delta)} + 3\sqrt{|S||A| H T \log(8/\delta)}\; \right)
        \\
        \leq &
        4 H \sqrt{|S|+2\log (4 |S||A|T^2/\delta)} \cdot 5\sqrt{|S||A| H T \log(8/\delta)}
        \\
        \leq &
        20 H \sqrt{|S|+2\log (8 |S||A|T^2/\delta)} \cdot \sqrt{|S||A| H T \log(8|S||A|T^2/\delta)}
        \\
        \tag{$\sqrt{a+b} \le \sqrt{a}+\sqrt{b}$}
        \leq &
        20  (\sqrt{|S|}+\sqrt{2\log (8 |S||A|T^2/\delta)}) \cdot H^{1.5}\sqrt{|S||A| T \log(8|S||A|T^2/\delta)}
        \\
        = &
        20  (H^{1.5}|S|\sqrt{|A| T \log(8|S||A|T^2/\delta)}+ H^{1.5}\sqrt{2|S| |A| T} \log(8|S||A|T^2/\delta)) 
        \\
        \leq &
        % 4 H \sqrt{|S|+2\log (4 |S||A|T^2/\delta)} \left(  \sqrt{2TH \log(8/\delta)} + 2\sqrt{|S||A| T}\; \right).
        {O}(H^{1.5}|S|\sqrt{|A|T  }\log(|S||A|T^2/\delta))
        .
    \end{align*}
    \endgroup
    
    % \textbf{Explanation for inequality (i)}.
    We denote by $\I[c_t,t,h,s_h,a_h]$ an indicator which indicates whether at time-step $h$ of round $t$, when playing $\pi_t(c_t;\cdot)$, the agent visited $s_h$ and played $a_h$ where the context is $c_t$.
    Inequality $(i)$ holds with probability at least $1-\delta/4$ by Azuma's inequality.
\end{proof}

\subsubsection{Deriving the Good Event for the Rewards Approximation.}\label{Subsec:good-event-rewards}
In the following, we derive the good event for the rewards approximation.

\paragraph{Step 1: Establishing uniform convergence bound over \texorpdfstring{$\F$}{Lg}.}

\begin{lemma}[uniform convergence over all sequences of estimators]\label{lemma:UC-lemma-5-UCID}
    For any $\delta \in (0,1)$, with probability at least $1-\delta/4$ it holds that
    \begin{align*}
        &\sum_{i=1}^{t-1} \mathop{\E}
        \left[  \sum_{h=0}^{H-1} (f_t(c_i, s^i_h, a^i_h)-f_\star(c_i, s^i_h, a^i_h) )^2 | \Hist_{i-1}\right]
        \\
        & =
        \sum_{i=1}^{t-1} \sum_{h=0}^{H-1} \mathop{\E}_{c_i,s^i_h,a^i_h}
        \left[ (f_t(c_i, s^i_h, a^i_h)-f_\star(c_i, s^i_h, a^i_h) )^2 | \Hist_{i-1}\right]\\
        & \leq
        %68 \log(8|\F|t^3H^2/\delta)
        68 H\log(8|\F|t^3/\delta)
        +
        2 \sum_{i=1}^{t-1} \sum_{h=0}^{H-1} (f_t(c_i, s^i_h, a^i_h)- r^i_h )^2 - (f_\star(c_i, s^i_h, a^i_h) -r^i_h )^2.
    \end{align*}
    The above holds simultaneously for any  $t \geq 2$ and a fixed sequence of functions $f_2,f_3,\ldots \in \F$.
    %(one for each time step $t$).
\end{lemma}

\begin{proof}
    For a fixed $\delta \in (0,1)$, take $\delta_t = \delta/8t^3$ and apply union bound to~\cref{lemma:UC-F} with all $t \geq 2$. 
    The proof is identical to~\cref{lemma:UC-lemma-5}
\end{proof}

\paragraph{Step 2: Constructing confidence bound over policies with respect to the rewards approximation.}

\begin{lemma}[confidence of policies w.r.t rewards approximation]\label{lemma:CB-policy-UCID}
    Consider~\cref{alg:RM-UCID}
    that at each initialization round $t \leq |A|$, plays the policy that always choose action $a_t$, and at each round $t \geq |A|+1$ 
    selects $\pi_t$ based on the history $\Hist_{t-1}$.
    
    Then, for any $\delta \in (0,1)$ with probability at least $1-\delta/4$  for all $t \geq |A|+1$ and every policy $\pi \in \Pi_\C$, it holds that
    %the estimation error of the expected value of $\pi$ on the true dynamics $P_\star$ is bounded by
    \begin{align*}
        &|\E_c[V^{\pi(c;\cdot)}_{\M^{(f_\star,P_\star)}(c)}(s_0)  ] - \E_c[V^{\pi(c;\cdot)}_{\M^{(\hat{f}_t,P_\star)}(c)}(s_0) ]|
        \\
        \leq &
        \sqrt{\E_c \left[ \sum_{h=0}^{H-1} \sum_{s_h \in S_h} \frac{q_h(s_h|\pi(c;\cdot) ,P_\star)}{ \sum_{i=1}^{t-1} \I[\pi(c;s_h)= \pi_i(c;s_h)]q_h(s_h | \pi_i(c;\cdot),P_\star)}\right]}
        \cdot \sqrt{ 68 H \log(8|\F|t^3/\delta)}.
    \end{align*}
    %where we define $\M_f(c) = (S,A,P^c_\star,r^c_\star=f(c,\cdot,\cdot), H,s_0)$ for any $f \in \F$.
\end{lemma}

\begin{proof}
   The Lemma is obtained using the same derivation presented in~\cref{lemma:CB-policy}, where for every context $c \in \C$ we have $P^c_\star = P_\star$, using~\cref{lemma:UC-lemma-5-UCID} (instead of~\cref{lemma:UC-lemma-5}).
\end{proof}

\subsubsection{Analysing the error due to the rewards Approximation.}\label{subsubsec:rewards-error-UCID}

In the following steps, we analyse the error caused by the rewards approximation, under the good event.

\paragraph{Step 3: Relax the confidence bound to be additive.}

\begin{lemma}[the ``square trick'' relaxation for unknown and context-free dynamics]\label{lemma:sq-trick-UCID}
    Under the good event $G_1$ (which stated in~\cref{lemma:CB-policy-UCID}),
    %Under the conditions of~\cref{lemma:CB-policy-UCID}
    for $\beta_t = \sqrt{ \frac{17 t \log(8|\F|t^3/\delta) }{ |S||A|}}$ we have for any $t >|A|$ and context-dependent policy $\pi \in \Pi_\C$ that
    \begin{align*}
        \left|\E_c\left[V^{\pi(c;\cdot)}_{\M(c)}(s_0) \right] - \E_c\left[V^{\pi(c;\cdot)}_{\M^{(\hat{f}_t,P_\star)}(c)}(s_0) \right]\right| 
        % \\
        % \leq &
        % \sqrt{\E_c \left[ \sum_{h=0}^{H-1} \sum_{s_h \in S_h} \frac{q_h(s_h|\pi(c;\cdot) ,P_\star)}{\sum_{i=1}^{t-1} \I[\pi(c;s_h)= \pi_i(c;s_h)]q_h(s_h | \pi_i(c;\cdot),P_\star)}\right]}
        % \cdot \sqrt{ 68 H \log(8|\F|t^3/\delta)}
        \leq &
        \E_c \left[ \sum_{h=0}^{H-1} \sum_{s_h \in S_h} \frac{\beta_t \cdot q_h(s_h|\pi(c;\cdot) ,P_\star)}{\sum_{i=1}^{t-1} \I[\pi(c;s_h)= \pi_i(c;s_h)]q_h(s_h | \pi_i(c;\cdot),P_\star)}\right]
        \\
        & +  \beta_t \frac{H|S||A|}{t}
        .
    \end{align*}
    %where $\beta_t = \sqrt{ \frac{17 t \log(8|\F|t^3/\delta) }{ |S||A|}}$.
    
    The above implies that
    % \begin{equation}\label{eq:r-PVSr-hat-P}
    %     \E_c[V^{\pi(c;\cdot)}_{\M(c)}(s_0)] - \E_c[V^{\pi(c;\cdot)}_{\M^{(\widehat{r}_t,P_\star)}(c)}(s_0)]
    %     \leq
    %     \beta_t \cdot \E_c\left[             
    %     \sum_{h = 0}^{H-1} \sum_{s_h \in S_h} 
    %     \frac{ q_h(s_h, \pi(c;s_h)| \pi(c;\cdot),P_\star)}{ p_{min} \sum_{i=1}^{t-1}\I[ \pi(c;s_h) = \pi_i(c;s_h)] } \right]
    %     +
    %     \beta_t \frac{H |S||A|}{t},
    % \end{equation}
    % and
    % \begin{equation}\label{eq:r-hat-PVPr-P}
    %     \E_c[V^{\pi(c;\cdot)}_{\M^{(\widehat{r}_t,P_\star)}(c)}(s_0)] - \E_c[V^{\pi(c;\cdot)}_{\M(c)}(s_0)]
    %     \leq
    %     2\beta_t \cdot \E_c\left[             
    %     \sum_{h = 0}^{H-1} \sum_{s_h \in S_h} 
    %     \frac{ q_h(s_h, \pi(c;s_h)| \pi(c;\cdot),P_\star)}{ p_{min} \sum_{i=1}^{t-1}\I[ \pi(c;s_h) = \pi_i(c;s_h)] } \right]
    %     +
    %     \beta_t \frac{H|S||A|}{t},
    % \end{equation}
    % where $\widehat{r}_t$ is the context-dependent rewards function defined in~\cref{alg:RM-UCID}.

    \begin{align*}
        \left|\E_c\left[V^{\pi(c;\cdot)}_{\M(c)}(s_0) \right] - \E_c\left[V^{\pi(c;\cdot)}_{\M^{(\hat{f}_t,P_\star)}(c)}(s_0) \right]\right| 
        \leq &
        \E_c \left[ \sum_{h=0}^{H-1} \sum_{s_h \in S_h} \frac{\beta_t \cdot q_h(s_h|\pi(c;\cdot) ,P_\star)}{p_{min}\sum_{i=1}^{t-1} \I[\pi(c;s_h)= \pi_i(c;s_h)]}\right]
        \\
        & +  \beta_t \frac{H|S||A|}{t}
        .
    \end{align*}
\end{lemma}

\begin{proof}
    Using the same derivation presented in~\cref{lemma:sq-trick},
    where $\beta_t = \sqrt{ \frac{17 t \log(8|\F|t^3/\delta) }{ |S||A|}}$ we obtain the first part of the lemma.
    For the second part, by minimum reachability it holds that
    
    \begin{equation*}\label{ineq:sq-trick+p-min-ucfd}
        \begin{split}
            \left|\E_c[V^{\pi(c;\cdot)}_{\M(c)}(s_0)] - \E_c[V^{\pi(c;\cdot)}_{\M^{(\hat{f}_t,P_\star)}(c)}(s_0)]\right|
            \leq &
            \E_c \left[ \sum_{h=0}^{H-1} \sum_{s_h \in S_h} \frac{\beta_t \cdot q_h(s_h|\pi(c;\cdot),P_\star)}{\sum_{i=1}^{t-1} \I[\pi(c;s_h)= \pi_i(c;s_h)]q_h(s_h | \pi_i(c;\cdot),P_\star)}\right]
            +  \beta_t \frac{H |S||A|}{t} 
            \\
            \leq &
            \E_c \left[ \sum_{h=0}^{H-1} \sum_{s_h \in S_h} \frac{\beta_t \cdot q_h(s_h|\pi(c;\cdot),P_\star)}{{p_{min}}\sum_{i=1}^{t-1} \I[\pi(c;s_h)= \pi_i(c;s_h)]}\right]
            +  \beta_t \frac{H|S||A|}{t}.
        \end{split}
    \end{equation*}

\end{proof}

\paragraph{Step 4: Bounding the sum of contextual potential functions.}

We consider the following two contextual potential functions in round $t$, for $T \geq t>|A|$ and a context-depended policy $\pi \in \Pi_\C$.
\\
We define $\psi_t(\pi)$ as follows.
\begin{definition}
    We denote by $\psi_t(\pi)$ the contextual potential of a context-dependent policy $\pi \in \Pi_\C$ at round $|A|<t \leq T$ using $p_{min}$ as follows.   
    $$
        \psi_t(\pi) : = \E_c \left[ \sum_{h=0}^{H-1} \sum_{s_h \in S_h} \frac{q_h(s_h|\pi(c;\cdot) ,\widehat{P}^c_t)}{ p_{min} \cdot \sum_{i=1}^{t-1} \I[\pi(c;s_h)= \pi_i(c;s_h)]}\right],
    $$
    where $\{\pi_t \in \Pi_{\C}\}_{t=1}^T$ and $\{\widehat{P}_t \}_{t=1}^T$
    are the sequences of context dependent policies and optimistic dynamics selected by~\cref{alg:RM-UCID}.
\end{definition}

In addition, we abuse the notation of $\phi_t(\pi)$ from~\cref{Appendix:KD}.
\begin{definition}
    We denote by $\phi_t(\pi)$ the contextual potential of a context-dependent policy $\pi \in \Pi_\C$ at round $|A|<t \leq T$ which is defined as follows.   
    $$
        \phi_t(\pi) : = \E_c \left[ \sum_{h=0}^{H-1} \sum_{s_h \in S_h} \frac{q_h(s_h|\pi(c;\cdot) ,P_\star)}{ p_{min}\cdot \sum_{i=1}^{t-1} \I[\pi(c;s_h)= \pi_i(c;s_h)]}\right],
    $$
    where $\{\pi_t \in \Pi_{\C}\}_{t=1}^T$
    is the sequence of context dependent policies selected by~\cref{alg:RM-UCID}.
\end{definition}

Recall that in our setting, $P^c_\star = P_\star$ and $S^c_h=S_h$ for every context $c \in \C$ and layer $h \in [H]$.
In the following lemma, we bound the sum of contextual potential functions, over the rounds $t = |A|+1, \ldots, T$.
\begin{lemma}[contextual potential lemma for unknown context-independent dynamics]\label{lemma:Contextual-Potential-UCID}
    Let ${\{\pi_t \in \Pi_{\C}\}_{t=1}^T}$
    be the sequence of context-dependent policies selected by~\cref{alg:RM-UCID}.
    % Let $\{\pi_t \in \Pi_{\C}\}_{t=1}^T$ be the sequence of played policies.
    % Let $\pi_t \in \Pi_{\C}$ the policy that played in time $t$ regardless of the context $c$ for $t = 1,\ldots, |A|$ and form round $t=  |A| + 1$ up to $T$, $\pi_t$ are given by any deterministic stationary policy.
    %Then for all $T > |A|$ we have for any policy $\pi \in \Pi_{\C}$
    Then, for all $T > |A|$ the followings hold.
    \begin{enumerate}
        \item
        For the sequence $\{\pi_t \in \Pi_{\C}\}_{t=1}^T$ it holds that
        \begin{align*}
            \sum_{t=|A| +1}^T
            \phi_t(\pi_t)
            = & 
            \sum_{t=|A| +1}^T   
            \E_c \left[ 
            \sum_{h =0}^{H-1}
            \sum_{s_h \in S_h}
            \frac{q_h(s_h | \pi_t(c;\cdot),P_\star)}{p_{min} \cdot \sum_{i=1}^{t-1} \I[\pi_t(c;s_h) = \pi_i(c;s_h)]} 
            \right]
            \\
            \leq & 
            \frac{|S||A|}{p_{min}}(1+\log(T/|A|))
            .
        \end{align*}
        \item
        For the sequences $\{\pi_t \in \Pi_{\C}\}_{t=1}^T$ and $\{\widehat{P}_t\}_{t=1}^T$ it holds that
        \begin{align*}
            \sum_{t=|A| +1}^T
            \psi_t(\pi_t)
            = & 
            \sum_{t=|A| +1}^T   
            \E_c \left[ 
            \sum_{h =0}^{H-1}
            \sum_{s_h \in S_h}
            \frac{q_h(s_h | \pi_t(c;\cdot),\widehat{P}^c_t)}{p_{min} \cdot \sum_{i=1}^{t-1} \I[\pi_t(c;s_h) = \pi_i(c;s_h)]} 
            \right]
            \\
            \leq & 
            \frac{|S||A|}{p_{min}}(1+\log(T/|A|))
        \end{align*}
    \end{enumerate}
\end{lemma}

The proof is similar to the proof of~\cref{lemma:Contextual-Potential}.
\begin{proof}
    For any fixed context $c \in \C$, the following holds.
    \begingroup
    \allowdisplaybreaks
    \begin{align*}
        &\sum_{t=|A| +1}^T \sum_{h =0}^{H-1} \sum_{s_h \in S_h}  
        \frac{q_h(s_h| \pi_t(c;\cdot),P_\star)}{p_{min}\sum_{i=1}^{t-1} \I[\pi_t(c;s_h) = \pi_i(c;s_h)]}
        \\
        \leq &
        \frac{1}{p_{min}} \sum_{h =0}^{H-1} \sum_{s_h \in S_h} 
        \sum_{t=|A| +1}^T
        \frac{1}{\sum_{i=1}^{t-1} \I[\pi_t(c;s_h) = \pi_i(c;s_h)]}
        \\
        \leq &
        \frac{1}{p_{min}} \sum_{h =0}^{H-1} \sum_{s_h \in S_h} 
        \sum_{a_h \in A}
        \sum_{i=1}^{\sum_{t=1}^T  \I[\pi_t(c;s_h) = a_h] } \frac{1}{i}
        \\
        \tag{Since $\sum_{i=1}^n \frac{1}{i} \leq 1+\log(n)$}
        \leq &
        \frac{1}{p_{min}}
        \sum_{h =0}^{H-1} \sum_{s_h \in S_h} \sum_{a_h \in A} \left(1 + \log\left(\sum_{t=1}^T \I[\pi_t(c;s_h) = a_h]\right) \right)
        \\
        = &
         \frac{|S||A|}{p_{min}} + \sum_{h =0}^{H-1} \sum_{s_h \in S_h} |A| \cdot \frac{1}{|A|}\sum_{a_h \in A}  \log\left(\sum_{t=1}^T \I[\pi_t(c;s_h) = a_h]\right)
        \\
         \tag{By Jansen's inequality, since $\log$ is concave}
         \leq &
        \frac{|S||A|}{p_{min}} + \sum_{h =0}^{H-1} \sum_{s_h \in S_h} |A| \cdot   \log\left(\frac{1}{|A|}\sum_{a_h \in A}\sum_{t=1}^T \I[\pi_t(c;s_h) = a_h]\right)
        \\
        = &
        \frac{|S||A|}{p_{min}} + \sum_{h =0}^{H-1} \sum_{s_h \in S_h} |A| \cdot   \log\left(\frac{1}{|A|}\sum_{t=1}^T \sum_{a_h \in A} \I[\pi_t(c;s_h) = a_h]\right)
        \\
        \tag{For all $t$, $\sum_{a_h \in A} \I[\pi_t(c;s_h) = a_h]=1$ since $\pi_t$ is a deterministic policy}
         = &
        \frac{|S||A|}{p_{min}} + \sum_{h =0}^{H-1} \sum_{s_h \in S_h} |A| \cdot   \log\left(\frac{T}{|A|}\right)
        \\
        = &
        \frac{|S||A|}{p_{min}}( 1 +  \log(T/|A|)).
    \end{align*}
    \endgroup
    By taking an expectation over $c$ on both sides of the inequality, we obtain the first part of the  lemma.

    For the second part of the lemma we similarly have
    \begingroup
    \allowdisplaybreaks
    \begin{align*}
        &\sum_{t=|A| +1}^T \sum_{h =0}^{H-1} \sum_{s_h \in S_h}  
        \frac{q_h(s_h| \pi_t(c;\cdot),\widehat{P}^c_t)}{p_{min}\sum_{i=1}^{t-1} \I[\pi_t(c;s_h) = \pi_i(c;s_h)]}
        \\
        \leq &
        \frac{1}{p_{min}} \sum_{h =0}^{H-1} \sum_{s_h \in S_h} 
        \sum_{t=|A| +1}^T
        \frac{1}{\sum_{i=1}^{t-1} \I[\pi_t(c;s_h) = \pi_i(c;s_h)]}
        \\
        \leq &
        \frac{1}{p_{min}} \sum_{h =0}^{H-1} \sum_{s_h \in S_h} 
        \sum_{a_h \in A}
        \sum_{i=1}^{\sum_{t=1}^T  \I[\pi_t(c;s_h) = a_h] } \frac{1}{i}
        \\
        \leq &
        \frac{1}{p_{min}}
        \sum_{h =0}^{H-1} \sum_{s_h \in S_h} \sum_{a_h \in A} \left(1 + \log \left(\sum_{t=1}^T \I[\pi_t(c;s_h) = a_h] \right) \right)
        \\
        \leq &
        \frac{|S||A|}{p_{min}}( 1 +  \log(T/|A|)),
    \end{align*}
    \endgroup
    By taking expectation over $c$ on both sides of the above inequality, we obtain the second part of the lemma.
\end{proof}

\begin{lemma}[cumulative value error due to rewards approximation]\label{lemma:UCID-rewards-value-error-over-t}
    Under the good event $G_1$, for all $T>|A|$ it holds that
    \begin{align*}
        \sum_{t= |A| + 1}^T \E_c[V^{\pi_t(c;\cdot)}_{\M^{(\hat{f}_t ,P_\star)}(c)}] -
        \E_c[V^{\pi_t(c;\cdot)}_{\M(c)}]
        \leq &
        2 H\sqrt{17 \log(8|\F|T^3/\delta) |S||A| T} 
        \\
        &+ (1+ \log(T/|A|)) \frac{\sqrt{ 17 \log(8|\F|T^3/\delta) T |S|  |A|}}{p_{min}}
        .
    \end{align*}

\end{lemma}

\begin{proof}
    Assume the good event $G_1$ holds and consider the following derivation.
    \begingroup
    \allowdisplaybreaks
    \begin{align*}
        &
        \sum_{t= |A| + 1}^T \E_c[V^{\pi_t(c;\cdot)}_{\M^{(\hat{f}_t ,P_\star)}(c)}] -\E_c[V^{\pi_t(c;\cdot)}_{\M(c)}]
        \\
        \tag{By~\cref{lemma:sq-trick-UCID}}
        \leq &
        \sum_{t= |A| + 1}^T \frac{\beta_t  H  |S| |A|}{t}+               
        \E_c\left[ \sum_{h = 0}^{H-1} \sum_{s_h \in S_h}
        \frac{\beta_t \cdot q_h(s_h, \pi_t(c;s_h)|\pi_t(c;\cdot), P_\star)}{  {p_{min}}\sum_{i=1}^{t-1}\I[ \pi_t(c;s_h) = \pi_i(c;s_h)] } \right]
        \\
        = &
        \tag{By $\beta_t$ choice}
        H \sqrt{17 \log(8|\F|T^3/\delta) |S||A|} \sum_{t = |A| + 1}^T \frac{1}{\sqrt{t}} 
        \\
        \tag{Since $\beta_T \geq \beta_t$ for all $t \leq T$.}
        & +
        \beta_T  
        \sum_{t = |A| + 1}^T
        \E_c\left[
        \sum_{h = 0}^{H-1} \sum_{s_h \in S_h}
        \frac{ q_h(s_h, \pi_t(c;s_h)|  \pi_t(c;\cdot),P_\star)}{{p_{min}}  \sum_{i=1}^{t-1}\I[ \pi_t(c;s_h) = \pi_i(c;s_h)] } \right]
        \\
        \leq &
        \tag{Since $\sum_{t=1}^T \frac{1}{\sqrt{t}} \leq 2\sqrt{T}$.}
        2H\sqrt{17 \log(8|\F|T^3/\delta) |S||A| T} 
        \\
        \tag{By~\cref{lemma:Contextual-Potential-UCID}, item $1$.}
        & + 
        \beta_T \frac{|S||A|}{p_{min}}( 1+  \log(T/|A|))
        \\
        \tag{By $\beta_T$ choice}
        = &
        2H\sqrt{17 \log(8|\F|T^3/\delta) |S||A| T}
        + \sqrt{ \frac{17 \log(8|\F|T^3/\delta) T}{ |S||A|}} \frac{|S||A| }{p_{min}}(1+  \log(T/|A|))
        \\
        = &
        2H\sqrt{17 \log(8|\F|T^3/\delta) |S||A| T} 
        + (1+ \log(T/|A|)) \frac{\sqrt{ 
        17 \log(8|\F|T^3/\delta) T  |S|  |A|}}{p_{min}},
    \end{align*}
    \endgroup
    as stated.
\end{proof}

\subsubsection{Regret Bound.}
In the following we derive our regret bound. The following identities yield a useful valve decomposition. 
Note that for every context $c \in \C$ and $t \in [T]$, the following holds for every deterministic context-dependent policy $\pi \in \Pi_\C$.
\begin{equation}\label{eq:V-t-def-ucid}
    \begin{split}
        &V^{\pi(c;\cdot)}_{\M(\widehat{r}_t, P_\star)(c)}(s_0)
        \\
        = &
        \sum_{h = 0}^{H-1} \sum_{s_h \in S_h} \sum_{a_h \in A}
        q_h(s_h, a_h | \pi(c;\cdot),P_\star) \cdot \hat{r}_t^c(s_h, a_h)
        \\
        %\tag{Since $\pi$ is a deterministic context-dependent function}
        = &
        \sum_{h = 0}^{H-1} \sum_{s_h \in S_h} \sum_{a_h \in A}
        q_h(s_h, a_h | \pi(c;\cdot),P_\star) \cdot 
        \left( \hat{f}_t (c,s_h,a_h) + \frac{\beta_t}{p_{min}\sum_{i=1}^{t-1}\I[a_h = \pi_i(c;s_h)]}\right)
        \\
        = &
        \sum_{h = 0}^{H-1} \sum_{s_h \in S_h} \sum_{a_h \in A}
        q_h(s_h| \pi(c;\cdot),P_\star) \I[a_h = \pi(c;s_h)] \cdot 
        \left( \hat{f}_t (c,s_h,a_h) + \frac{\beta_t}{p_{min}\sum_{i=1}^{t-1}\I[a_h = \pi_i(c;s_h)] }\right)
        \\
        = &
        \sum_{h = 0}^{H-1} \sum_{s_h \in S_h} 
        q_h(s_h| \pi(c;\cdot),P_\star) \cdot 
        \left( \hat{f}_t (c, s_h,  \pi(c;s_h)) + \frac{\beta_t}{  p_{min}\sum_{i=1}^{t-1}\I[ \pi(c;s_h) = \pi_i(c;s_h)]}\right)
        \\
        = &
        V^{\pi(c;\cdot)}_{\M^{(\hat{f}_t,P_\star)}(c)}(s_0) +  \sum_{h=0}^{H-1} \sum_{s_h \in S_h} \frac{ q_h(s_h|\pi(c;\cdot),P_\star) \cdot \beta_t}{  p_{min}\sum_{i=1}^{t-1} \I[\pi(c;s_h)= \pi_i(c;s_h)]}
        ,        
    \end{split}
\end{equation}
The above clearly holds when taking the expectation on both sides of the equation.

% Recall that for every $t >  |A|$ and $\pi \in \Pi_{\C}$ it holds that
% \begin{equation}\label{eq:V-t-def-UCFD}
%     \begin{split}
%         &V^{\pi(c;\cdot)}_{\Mhat_t(c)}(s_0)
%         = 
%         \sum_{h = 0}^{H-1} \sum_{s_h \in S_h} \sum_{a_h \in A}
%         q_h(s_h, a_h | \pi(c;\cdot), \widehat{P}^c_t) \cdot \widehat{r}_t^c(s_h, a_h)
%         \\
%         = &
%         \sum_{h = 0}^{H-1} \sum_{s_h \in S_h} \sum_{a_h \in A}
%         q_h(s_h, a_h | \pi(c;\cdot), \widehat{P}^c_t) \cdot 
%         \left( \hat{f}_t (c,s_h,a_h) 
%         + 
%         \min\left\{
%             \frac{1}{p_{min}} \cdot
%             \frac{\beta_k}{  \sum_{i=1}^{k-1}\I[a_h = \pi_i(c_t,s)]}, 1 \right\}\right)
%         \\
%         = &
%         \sum_{h = 0}^{H-1} \sum_{s_h \in S_h} 
%         q_h(s_h, \pi(c;s_h)| \pi(c;\cdot),\widehat{P}^c_t) \cdot 
%         \left( \hat{f}_t (c, s_h,  \pi(c;s_h))
%         +  
%         \min\left\{
%             \frac{1}{p_{min}} \cdot
%             \frac{\beta_k}{  \sum_{i=1}^{k-1}\I[\pi(c;s_h) = \pi_i(c_t,s)]}, 1 \right\}\right)
%         \\
%         = & 
%         V^{\pi(c;\cdot)}_{\M^{(\hat{f}_t,\widehat{P}_t)}(c)}(s_0) 
%         +  
%         \sum_{h = 0}^{H-1} \sum_{s_h \in S_h} 
%          q_h(s_h, \pi(c;s_h)|\pi(c;\cdot),\widehat{P}^c_t)
%          \cdot
%          \min\left\{
%             \frac{1}{p_{min}} \cdot
%             \frac{\beta_k}{  \sum_{i=1}^{k-1}\I[\pi(c;s_h) = \pi_i(c_t,s)]}, 1 \right\}
%         .        
%     \end{split}
% \end{equation}

 We start our regret derivation by proving an optimism lemma.
\begin{lemma}[optimism]\label{lemma:optimism}
    Let $\pi^\star \in \Pi_\C$ be a deterministic optimal policy of the true CMDP.
    Under the good events $G_1$ and $G_2$, for all $t > |A|$
    %\Orin{Need to choose the confidence bound $\lambda$ accordingly, and the logs in the $\beta_t$ def, and the sqrt in lemmas 1.6-1.9}
    it holds that 
    \[
        \E_c \left[ V^{\pi^\star(c;\cdot)}_{\M(c)}(s_0)\right]
        -
        \E_c \left[ V^{\pi_t(c;\cdot)}_{\Mhat_t(c)}(s_0) \right]
        \leq
        % \E_c\left[             
        % \sum_{h = 0}^{H-1} \sum_{s_h \in S_h} 
        % \frac{1}{p_{min}}\frac{\beta_t \cdot q_h(s_h, \pi^\star(c;s_h)| \pi^\star(c;\cdot),P_\star)}{  \sum_{i=1}^{t-1}\I[ \pi^\star(c;s_h) = \pi_i(c;s_h)] } \right]
        %+
        \beta_t \frac{H |S||A|}{t}.
    \]
\end{lemma}

\begin{proof}
    Assume the good events hold and consider the following derivation.
    \begingroup
    \allowdisplaybreaks
    \begin{align*}
        &\E_c \left[ V^{\pi^\star(c;\cdot)}_{\M(c)}(s_0) \right]
        \\
        \tag{\cref{lemma:sq-trick-UCID}}
        \leq &
        \E_c \left[ V^{\pi^\star(c;\cdot)}_{\M^{(\hat{f}_t,P_\star)}(c)}(s_0) \right] 
        + 
        \E_c\left[             
        \sum_{h = 0}^{H-1} \sum_{s_h \in S_h} 
        \frac{\beta_t \cdot q_h(s_h, \pi^\star(c;s_h)| \pi^\star(c;\cdot),P_\star)}{ {p_{min}} \sum_{i=1}^{t-1}\I[ \pi^\star(c;s_h) = \pi_i(c;s_h)] } \right]
        +
        \beta_t \frac{H |S| |A|}{t}
        \\
        \tag{By the value function decomposition, see~\cref{eq:V-t-def-ucid}}
        = &
        \E_c \left[ V^{\pi^\star(c;\cdot)}_{\M^{(\widehat{r}_t,P_\star)}(c)}(s_0) \right] 
        +
        \beta_t \frac{H |S| |A|}{t}
        \\
        \tag{By~\cref{lemma:optimality-of-pi-t-and-p-t}}
        \leq &
        % \E_c \left[ V^{\pi_t(c;\cdot)}_{\M^{(\widehat{r}_t,\widehat{P}_t)}(c)}(s_0) \right]
        % + 
        % \E_c\left[             
        % \sum_{h = 0}^{H-1} \sum_{s_h \in S_h} 
        % \frac{1}{p_{min}}\frac{\beta_t \cdot q_h(s_h, \pi^\star(c;s_h)| \pi^\star(c;\cdot),P_\star)}{  \sum_{i=1}^{t-1}\I[ \pi^\star(c;s_h) = \pi_i(c;s_h)] } \right]
        % +
        % \beta_t \frac{H\cdot |S|\cdot |A|}{t}
        % \\
        % = &
        % \tag{By $\Mhat_t$ definition.}
        \E_c \left[ V^{\pi_t(c;\cdot)}_{\Mhat_t(c)}(c)(s_0) \right]
        % + 
        % \E_c\left[             
        % \sum_{h = 0}^{H-1} \sum_{s_h \in S_h} 
        % \frac{1}{p_{min}}\frac{\beta_t \cdot q_h(s_h, \pi^\star(c;s_h)| \pi^\star(c;\cdot),P_\star)}{  \sum_{i=1}^{t-1}\I[ \pi^\star(c;s_h) = \pi_i(c;s_h)] } \right]
        +
        \beta_t \frac{H |S| |A|}{t} 
        ,
    \end{align*}
    \endgroup
    as stated.
    %Follows trivially by the Linear Programming for each context separately, and the confidence interval over the rewards. 
\end{proof}

Recall the regret, which defined as
$
   \Regrv_T(\text{ALG}) 
    :=
    \sum_{t=1}^T V^{\pi^\star(c_t;\cdot)}_{\M(c_t)}
    -
    V^{\pi_t(c_t;\cdot)}_{\M(c_t)}
$. The following theorem states the main result of the unknown context-independent dynamics setting.
\begin{theorem}[regret bound]\label{thm:regret-bound-ucid}
   For every $T \geq |S||A|$, finite functions class $\F$ and $\delta \in (0,1)$ with probability at least $1-\delta$ it holds that
    \begin{align*}
        \Regrv_T(RM-UCID)
        \leq
        \tilde{O} \left(
         H^{1.5}|S| \sqrt{T |A|}\log\frac{1}{\delta}+ %dynamics
         (H+ 1/p_{min})\cdot \sqrt{|S||A| T \log\frac{|\F|}{\delta}}+%rewards
         |A|H
         \right)
         ,
    \end{align*}
    for the choice of $\beta_t = \sqrt{ \frac{17 t \log(8|\F|t^3/\delta) }{ |S||A|}}$ for all $t \in [T]$. 
    %$|A|<t \leq T$.
\end{theorem}

\begin{proof}
    Similarly to shown in~\cref{eq:V-t-def-ucid}, for every context $c \in \C$ and $t \in [T]$ we have for any $\pi \in \Pi_\C$ that
    \begin{equation}\label{eq:V-t-def-ucid-reg}
    \begin{split}
        V^{\pi(c;\cdot)}_{\M(\widehat{r}_t, \widehat{P}_t)(c)}(s_0)
        = V^{\pi(c;\cdot)}_{\M^{(\hat{f}_t,\widehat{P}_t)}(c)}(s_0) +  \sum_{h=0}^{H-1} \sum_{s_h \in S^c_h} \frac{ \beta_t \cdot q_h(s_h|\pi(c;\cdot),\widehat{P}^c_t) }{ p_{min} \sum_{i=1}^{t-1} \I[\pi(c;s_h)= \pi_i(c;s_h)]}
        ,      
    \end{split}
\end{equation}
    where $\widehat{r}_t$ and $\widehat{P}_t$ are the optimistic context-dependent rewards and dynamics in round $t$, respectively.
    \\
    We bound the regret under the good events $G_1$ and $G_2$. Since $G_1 \cap G_2$ holds with probability at least $1 - \delta/2$, the theorem will follow by union bound.
    \\
    Consider the Martingale difference sequence $\{Y_t\}_{t=1}^T$ where the filtration is $\{\Hist_{t}\}_{t=1}^T$ for 
    $$
        Y_t 
        := 
        V^{\pi^{\star}(c_t;\cdot)}_{\M(c_t)}(s_0)  - V^{\pi_t(c_t;\cdot)}_{\M(c_t)}(s_0)
        -
        %\E_{\Hist_{t-1}}\left[ 
        \E_{c_t}\left[ V^{\pi^{\star}(c_t;\cdot)}_{\M(c_t)}(s_0)  - V^{\pi_t(c_t;\cdot)}_{\M(c_t)}(s_0) \Big|\Hist_{t-1}\right] %\right]
        .
    $$
    Clearly, for all $t$, $|Y_t|\leq 2H$, $Y_t$ is determined completely by the history $\Hist_t$ and $\E_{c_t}\left[ Y_t | \Hist_{t-1}\right] = 0$.
    Thus, by Azuma's inequality, with probability at least $1-\delta/4$ it holds that
    \begingroup
    \allowdisplaybreaks
    \begin{align*}
        &\Regrv_T( RM-UCID)
        \\
        = &
        \sum_{t=1}^T V^{\pi^\star(c_t;\cdot)}_{\M(c_t)}(s_0)  - V^{\pi_t(c_t;\cdot)}_{\M(c_t)}(s_0)
        \\
        \tag{By Azuma's inequality, holds w.p. at least $1-\delta/4$.}
        \leq &
        %\underbrace{\leq}_{(1)}&
        \sum_{t=1}^T \left( \E_{c_t}[V^{\pi^\star(c_t;\cdot)}_{\M(c_t)}(s_0) |\Hist_{t-1}] - \E_{c_t}[V^{\pi_t(c_t;\cdot)}_{\M(c_t)}(s_0) |\Hist_{t-1}] \right)
        +2H\sqrt{2T\log(8/\delta)}
        \\
        \tag{Since $\pi_t$ is determined by $\Hist_{t-1}$ and $c_t$, $\pi^\star$ is independent of the history, we can omit the conditioning in $\Hist_{t-1}$}
        = &
        \sum_{t=1}^T \left( \E_c[V^{\pi^\star(c;\cdot)}_{\M(c)}(s_0) ] - \E_c[V^{\pi_t(c;\cdot)}_{\M(c)}(s_0) ] \right)
        +2H\sqrt{2T\log(8/\delta)}
        \\
        \leq &
        \sum_{t=|A|+1}^T \left( \E_c[V^{\pi^\star(c;\cdot)}_{\M(c)}(s_0) ] - \E_c[V^{\pi_t(c;\cdot)}_{\M(c)}(s_0) ] \right)
        +2H\sqrt{2T\log(8/\delta)} + |A|H
        \\
        \leq &
        \sum_{t=|A|+1}^T \left( \E_c[V^{\pi_t(c;\cdot)}_{\Mhat_t(c)}(s_0) ]  - \E_c[V^{\pi_t(c;\cdot)}_{\M(c)}(s_0) ] \right)
        \\
        \tag{By the optimism lemma,~\cref{lemma:optimism}}
        & +
        \sum_{t = |A|+1}^T \beta_t \frac{H |S| |A|}{t} 
        +2H\sqrt{2T\log(8/\delta)} + |A| H
        \\
        \tag{By adding and subtracting $\E_c[V^{\pi_t(c;\cdot)}_{\M^{(\hat{f}_t ,P_\star)}(c)}(s_0) ]$}
        = &
         \sum_{t=|A|+1}^T \left( \E_c[V^{\pi_t(c;\cdot)}_{\M^{(\widehat{r}_t, \widehat{P}_t)}(c)}(s_0) ]  
         - \E_c[V^{\pi_t(c;\cdot)}_{\M^{(\hat{f}_t ,P_\star)}(c)}(s_0) ] 
         + \E_c[V^{\pi_t(c;\cdot)}_{\M^{(\hat{f}_t ,P_\star)}(c)}(s_0) ] 
         -\E_c[V^{\pi_t(c;\cdot)}_{\M(c)}(s_0) ] \right)
        \\
        & +
        \sum_{t = |A|+1}^T \beta_t \frac{H |S| |A|}{t} 
        +2H\sqrt{2T\log(8/\delta)} + |A| H
        \\
        \tag{By~\cref{eq:V-t-def-ucid-reg}}
        = &
        \sum_{t = |A|+1}^T  \left( \E_c[V^{\pi_t(c;\cdot)}_{\M^{(\hat{f}_t, \widehat{P}_t)}(c)}(s_0) ]  
         - \E_c[V^{\pi_t(c;\cdot)}_{\M^{(\hat{f}_t ,P_\star)}(c)}(s_0) ] 
         + \E_c[V^{\pi_t(c;\cdot)}_{\M^{(\hat{f}_t ,P_\star)}(c)}(s_0) ] 
         -\E_c[V^{\pi_t(c;\cdot)}_{\M(c)}(s_0) ] \right)
        \\
        & +
        \sum_{t = |A|+1}^T \beta_t \cdot \E_c\left[             
         \sum_{h = 0}^{H-1} \sum_{s_h \in S^c_h} 
         \frac{ q_h(s_h| \pi_t(c;\cdot),\widehat{P}^c_t)}{ {p_{min}} \sum_{i=1}^{t-1}\I[ \pi_t(c;s_h) = \pi_i(c;s_h)] } \right]
        \\
        & +
        \sum_{t = |A|+1}^T \beta_t \frac{H |S| |A|}{t} 
        +2H\sqrt{2T\log(8/\delta)} + |A| H
        \\
        = &
        \sum_{t= |A| + 1}^T \E_c[V^{\pi_t(c;\cdot)}_{\M^{(\hat{f}_t, \widehat{P}_t)}(c)}(s_0) ]  
         - \E_c[V^{\pi_t(c;\cdot)}_{\M^{(\hat{f}_t ,P_\star)}(c)}(s_0) ]
         \\
         & +
         \sum_{t= |A| + 1}^T \E_c[V^{\pi_t(c;\cdot)}_{\M^{(\hat{f}_t ,P_\star)}(c)}(s_0) ] 
         -\E_c[V^{\pi_t(c;\cdot)}_{\M(c)}(s_0) ] 
        \\
        \tag{$\beta_T \geq \beta_t, \;\; \forall t \in \{1,2,\ldots,T\}$.}
        & +
        \beta_T \cdot  \E_c\left[             
        \sum_{t = |A|+1}^T \sum_{h = 0}^{H-1} \sum_{s_h \in S^c_h} 
         \frac{ q_h(s_h| \pi_t(c;\cdot),\widehat{P}^c_t)}{ {p_{min}} \sum_{i=1}^{t-1}\I[ \pi_t(c;s_h) = \pi_i(c;s_h)] } \right]
        \\
        & +
        \sum_{t = |A|+1}^T \beta_t \frac{H |S| |A|}{t} 
        +
        2H\sqrt{2T\log(8/\delta)} + |A| H
        \\
        \tag{Holds w.p. at least $1-\delta/4$ for $T \geq |S||A|$, By~\cref{lemma:UCID-dynamixs-value-error-over-t}}
        \leq &
        {O}(H^{1.5}|S|\sqrt{|A|T  }\log(|S||A|T^2/\delta))
        \\
        \tag{By~\cref{lemma:UCID-rewards-value-error-over-t}}
        &+
        2 H\sqrt{17 \log(8|\F|T^3/\delta) |S||A| T} 
        + (1+ \log(T/|A|)) \frac{\sqrt{ 17 \log(8|\F|T^3/\delta) T |S|  |A|}} {p_{min}}
        \\
        \tag{By~\cref{lemma:Contextual-Potential-UCID} part $2$ and $\beta_T$ choice}
        & +
        (1 + \log(T/|A|) )
        \frac{\sqrt{ 17 \log(8|\F|T^3/\delta) T |S| |A|}}{p_{min}}
        \\
        \tag{By $\beta_t$ choice and $\sum_{t=1}^T \frac{1}{\sqrt{t}} \leq 2\sqrt{T}$}
        & +
        2H \sqrt{17 \log(8|\F|T^3/\delta)|S||A| T}
        \\
        & +
        2H\sqrt{2T\log(8/\delta)} + |A| H
        \\
        = &
         {O}(H^{1.5}|S|\sqrt{|A|T  }\log(|S||A|T^2/\delta))
        \\
         & +
         4H\sqrt{17 \log(8|\F|T^3/\delta) |S||A| T} 
        + 
        2(1 + \log(T/|A|)) \frac{\sqrt{ 
        17 \log(8|\F|T^3/\delta) T |S|  |A|}}
        {p_{min}}
        \\
        & + 
        2H\sqrt{2T\log(8/\delta)}  +  |A|  H 
         \\
         = &
         \widetilde{O} \left(
         H^{1.5}|S| \sqrt{T |A|}\log\frac{1}{\delta}+ 
         (H+ 1/p_{min})\cdot \sqrt{|S||A| T \log\frac{|\F|}{\delta}}+
         |A|H
         \right)
         .
    \end{align*}
    \endgroup

    Overall, by union bound, the above holds with probability at least $1-\delta$.
\end{proof}

\begin{corollary}[regret bound in terms of $\G$]\label{corl:regret-bound-ucid-g}
   For every $T \geq |S||A|$, finite functions class $\G$ for which $\F = \G^S$ and $\delta \in (0,1)$ the following holds with probability at least $1-\delta$ for the same choice of parameters $\{\beta_t\}_{t \in [T]}$.
    \begin{align*}
        \Regrv_T(RM-UCID)
        \leq
        % \tilde{O} \left(
        %  H^{1.5}|S|^{1.5} \sqrt{T |A|}\log\frac{1}{\delta}+ %dynamics
        %  (H+ 1/p_{min})\cdot|S| \sqrt{|A| T \log\frac{|\G|}{\delta}}+%rewards
        %  |A|H
        %  \right).
        \widetilde{O} \left(
         H^{1.5}|S| \sqrt{T |A|}\log\frac{1}{\delta}+ %dynamics
         (H+ 1/p_{min})\cdot |S|\sqrt{|A| T \log\frac{|\G|}{\delta}}+%rewards
         |A|H
         \right).
    \end{align*}
    %for all $t \in [T]$.
\end{corollary}
\begin{proof}
    Plug $\log(|\F|) = |S| \log(|\G|)$ in the bound of~\cref{thm:regret-bound-ucid}.
\end{proof}

\subsubsection{Computational Efficiency.}

The following lemma establishes the computational efficiency of Algorithm RM-UCID.

\begin{lemma}
    Under the good events,
    for all $t > |A|$ and any context $c \in \C$ there exist a feasible solution to the LP in~\cref{eq:LP-q} and 
    Algorithm \texttt{FOA}~(\cref{alg:foa}) run in $poly(|S|,|A|,H)$ time. 
\end{lemma}

\begin{proof}
    Assume the good event $G_2$ holds.
    Then, for all $t>|A|$ and a context $c \in \C$, 
    the true dynamics $P_\star$ and the optimal policy $\pi^\star(c;\cdot)$ induces an extended occupancy measure $q^c_\star$ which is a feasible solution for LP in~\cref{eq:LP-q}. 

    For the running time, 
    LP in~\cref{eq:LP-q} is a linear program with $poly(|S|,|A|,H)$ constrains and variables, hence it can be solved in $poly(|S|,|A|,H)$ time. 
    In addition, finding an optimal deterministic policy given an MDP can also be done in $poly(|S|,|A|,H)$ time.
    % Algorithm \texttt{FDP} (\cref{alg:FDP}) also run in $poly(|S|,|A|,H)$ time.
    Overall,the run time complexity of~\cref{alg:foa} is in $poly(|S|,|A|,H)$.
\end{proof}

\begin{corollary}
    The overall running time of Algorithm RM-UCID (\cref{alg:RM-UCID}) is in $poly(|S|,|A|,H,T)$, assuming an efficient least square regression oracle.
\end{corollary}

\begin{proof}
    \cref{alg:RM-UCID} invokes \cref{alg:foa} for $O(T^2)$ times, hence the corollary implied by the previous lemma.
\end{proof}

\section{Unknown and Context-Dependent Dynamics}\label{Appendix:UCDD}

In this section we discuss the main contribution of the paper, which is a regret minimization algorithm for the unknown, context-dependent dynamics case and its regret analysis. Bellow we elaborate the assumptions we operate under, and additional notation we use.  
\\
\noindent\textbf{Unknown and context-dependent dynamics.}{
For every context $c \in \C$ we have a potentially different dynamics $P^c_\star$ that is unknown to the learner. Recall we assume layered CMDP. 
For any context $c$, we denote by $S^c_0, S^c_1, \ldots ,S^c_H$ the disjoint layers, and  for every context $ c \in \C$ it holds that $S = \bigcup_{h \in [H]}S^c_h$. 
}

\noindent\textbf{Known minimum reachability parameter.}{
We assume that there exists $p_{min} \in (0,1]$ such that for every context $c \in \C$, layer $h \in [H]$, state $s_h \in S^c_h$ and context-dependent policy $\pi \in \Pi_{\C}$ it holds that
$
    q_h(s_h | \pi(c;\cdot), P^c_{\star}) \geq p_{min}
$. 
We remark that $p_{min}$ is {known} to the learner in this section.
}

\noindent\textbf{Dynamics approximation using least-square regression.}{
Our algorithm gets as input a finite  context-dependent dynamics class $\Fp \subseteq  S \times ( S \times A \times \C) \to [0,1]$, i.e., every function $\widetilde{P} \in \Fp$ satisfies
\[
    \sum_{s' \in S} \widetilde{P}(s'|s,a,c)  = 1, \;\;\; \forall c \in \C,\; \forall (s,a) \in S \times A.
    %h \in [H],\; (s_h,a_h) \in S^c_h \times A.
\]
Given a context-dependent dynamics $P$ and a context $c$ we denote the dynamics $P$ induced by $c$ by $P^c$, i.e., $P^c(s'|s,a)=P(s'|s,a,c) $.

\noindent\textbf{The random variable $\B$.}
For every context $c \in \C$, and state-action pair $(s,a) \in S \times A$ we define a  random variable 
$\B(P^c_\star,s,a)\in S$ which returns the next state $s'$ that observed after action $a$ was played in state $s$ and the dynamics is $P^c_\star$.
Observe that $\B(P^c_\star,s,a)$ is a random variable that distributes according to $P^c_\star(\cdot|s,a)$.
% \[
%     \B(P^c_\star,s,a) \sim P^c_\star(\cdot|s,a),
% \]
By definition it holds that
\[
    \E\left[\I\left[s' = \B(P^c_\star,s,a)\;|c,s,a\right] \right]
    =
    P^c_\star(s'|s,a)
    \;\; \forall s' \in S.
\]

% Given a context-dependent dynamics $P$ and a context $c$ we denote the dynamics $P$ induced by $c$ by $P^c$, i.e., $P^c(s'|s,a)=P(s'|s,a,c) $.

\begin{assumption}[dynamics realizability]
%Let $P$ be the true dynamics $P(s'|s,a,c) $ for a context $c$.
We assume that $\Fp$ is realizable, meaning there exist a function $P_\star \in \Fp$ which is the true context-dependent dynamics.
    % \[
    %   P^c_\star(s'|s,a) = P^c(s'|s,a) = \E[\I[s' = \B(P^c_\star,s,a)|c,s,a,s'],
    % \] 
    % where $P^c_\star(s'|s,a)=P_\star(s'|s,a,c)$.
\end{assumption}

\noindent As in previous sections, we assume an access to a least square regression (LSR) oracle 
which returns 
\[
    \widehat{P} \in \arg\min_{{P} \in \Fp }\sum_{((c,s,a,s'),b) \in D} ({P}^c(s'|s,a) - b)^2
\]
given a data set $D \in ((\C \times S \times A \times S) \times \{0,1\})^n$ that contains samples, each of a context $c$, state $s$, action $a$, next state $s'$ and a bit $b \in \{0,1\}$  which indicate whether $s'$ was observed after action $a$ was played in state $s$, for every state $s'\in S$ where the dynamics is $P^c_\star$.
}

\subsection{Regret Minimization Algorithm}

In Algorithm RM-UCDD (\cref{alg:RM-UCDD}),
we approximate both the rewards and the dynamics using LSR oracle. 
The first $|A|$ rounds are initialization rounds, where in round $i \in \{1, 2, \ldots, |A|\}$ the agent plays the policy $\pi_i$ which always chooses action $a_i$, i.e., $\pi_i(c;s) = a_i$ for every context $c \in \C$ and state $s \in S$.
\\
At round $t\geq|A|+1$, the agent computes the approximated rewards function for the context $c_t$ using the previously-selected policies $\{\pi_k(c_t;\cdot)\}_{k=1}^{t-1}$
in the same iterative computation we present in the previous sections.
\\
To approximate the dynamics, the agent uses the least square minimizer $\widehat{P}_t$.
She defines the approximated MDP for the context $c_t$, denote it $\Mhat_t(c_t) = (S, A, \widehat{P}^{c_t}_t, \widehat{r}^{c_t}_t, s_0, H)$, computes an optimal policy for it $\pi_t(c_t;\cdot)$ and runs it to generate a trajectory and update the oracles.

We remark that we feed the LSR oracle for the dynamics with samples of the form $((c_t, s^t_h,a^t_h, s'), \I[s'=s^t_{h+1}])$ for all $t\in\{1,2,\ldots,T\}$, $h \in [H-1]$ and $s' \in S$, to approximate the distribution $P^c_\star(\cdot|s,a)$ over $S$ for each state-action pair $(s,a) \in S \times A$.
%, where $\I$ is an indicator function.
\begin{algorithm}
    \caption{Regret Minimization for CMDP with Unknown Context-Dependent Dynamics (RM-UCDD)}
    \label{alg:RM-UCDD}
    \begin{algorithmic}[1]
        \STATE
        { 
            \textbf{inputs:}
            \begin{itemize}
                \item MDP parameters: 
                $S $
                %= \{S_0, S_1, \ldots , S_H\}$
                , $A$, $H$, $s_0$.
                \item Confidence parameter $\delta$ and tuning parameters $\{\beta_t,\gamma_t\}_{t=1}^T$.
                \item Minimum reachability parameter $p_{min}$.
            \end{itemize}
        } 
        %  \STATE{Initialization: for the first $|S|\cdot |A|$ rounds, for each $(s,a) \in S \times A$ in turn, run a policy at state $s$ play action $a$, regardless of the current context (and for any other state choose an action uniformly at random)}
        \STATE{\textbf{initialization:} for the first $|A|$ rounds, for each action $a_i$ in turn $i$, run the policy $\pi_i$ that
        at any state $s$ plays action $a_i$, regardless of the context.}
        \FOR{episode $t =  |A| + 1, \ldots, T$}
            \STATE{compute
            $
                \widehat{f}_t \in 
                \arg\min_{f \in \F}
                \sum_{i=1}^{t-1} \sum_{h=0}^{H-1} ( f(c_i,s^i_h,a^i_h) - r^i_h)^2
            $ using the LSR oracle
            }
            \STATE{compute
            $
                \widehat{P}_t \in 
                \arg\min_{{P} \in \Fp}
                \sum_{i=1}^{t-1} \sum_{h=0}^{H-1} \sum_{s' \in S}
                ( {P}^{c_i} (s'| s^i_h,a^i_h) - \I[s' = s^i_{h+1}])^2
            $ using the LSR oracle
            }
            \STATE{observe context $c_t \in \C$}
            \FOR{$k=|A|+1, |A|+ 2, \ldots, t$}
                \STATE{compute for all $(s,a)\in S \times A $:\\
                    $$
                        % \forall h \in [H], \; s_h \in S^c_h,\; a_h \in A: \;\;
                        %\forall (s,a)\in S \times A \;\;
                        \widehat{r}_k^{c_t}(s, a) = \widehat{f}_k(c_t,s,a) + 
                        %\min\Big\{
                        \frac{\beta_k + H|S| \gamma_k}{  p_{min}\sum_{i=1}^{k-1}\I[a = \pi_i(c_t;s)]}
                        %, 1\Big\}
                    $$}
                %\STATE{    where $\beta_k=$}
                \STATE{define the approximated MDP  $\Mhat_k(c_t) = (S, A, \widehat{P}^{c_t}_k, \widehat{r}^{c_t}_k, s_0, H)$}
                \STATE{compute $\pi_k(c_t,\cdot) \in \arg\max_{\pi \in S \to A}V^{\pi}_{\Mhat_k(c_t)}(s_0)$ using planning algorithm}
           \ENDFOR{}    
            \STATE{play $\pi_t(c_t,\cdot)$ and observe trajectory $\sigma^t = (c_t, s^t_0, a^t_0, r^t_0, s^t_1, \ldots, s^t_{H-1}, a^t_{H-1}, r^t_{H-1},s^t_H) $}
            \STATE{update the LSR oracles using $\sigma^t$}
           \ENDFOR{}
    \end{algorithmic}
\end{algorithm}

\subsection{Regret Analysis}

\subsubsection{Analysis Outline.}

For the analysis, for every $t >  |A|$ we define the following intermediate MDPs for every context $c \in \C$.
\begin{enumerate}
    \item $\M^{(f,P)}(c) = (S,A,P^c,f(c,\cdot,\cdot),s_0,H)$, for any $f \in \F$ and $P\in \Fp$.
    \item $\M^{(f_\star,P_\star)}(c)$ is the true model, where $f_\star$ is the true context dependent rewards and $P_\star$ is the true dynamics, which we also denote by $\M(c)$.
    \item $\Mhat_t(c)=\M^{(\widehat{r}_t,\widehat{P}_t)}(c) $ is the approximated MDP defined in~\cref{alg:RM-UCDD}.
    \item $\M^{(\widehat{r}_t,P)}(c) = (S,A,P^c,\widehat{r}^c_t,s_0,H)$, where $\widehat{r}^c_t$ is the context-dependent optimistic rewards function defined in~\cref{alg:RM-UCDD} and $P \in \Fp$ is a context-dependent dynamics from the dynamics class $\Fp$.
\end{enumerate}

For brevity, in the analysis outline we abuse the following notations for the contextual potential function at round $t$.
$$
    \psi_t(\pi) : = \E_c \left[ \sum_{h=0}^{H-1} \sum_{s_h \in S^c_h} \frac{q_h(s_h|\pi(c;\cdot) ,\widehat{P}^c_t)}{ p_{min}\sum_{i=1}^{t-1} \I[\pi(c;s_h)= \pi_i(c;s_h)]}\right],
$$
where $\widehat{P}_t \in \Fp$ is the least square minimizing context-dependent dynamics in round $t$.
We remark that in the detailed analysis, we use the explicit expressions of the context potential function $\psi_t$.
\\
In the following, we analyse the error caused by the dynamics approximation and the rewards approximation separately.

\noindent\textbf{Analysing the error caused by the rewards approximation.}
%(see Sub-subsection~\cref{subsubsec:CB-rewards}). 
Using a similar analysis to that we showed for the known dynamics setting (see~\cref{Appendix:KD}), 
we show in~\cref{lemma:sq-trick-rewards-UCDD} that with probability at least $1-\delta/4$ for all $t\geq |A|+1$ and context-dependent policy $\pi \in \Pi_\C$ the following holds.
$$
    \left|\E_c[V^{\pi(c;\cdot)}_{\M^{(\hat{f}_t ,\widehat{P}_t)}(c)}(s_0)] -\E_c[V^{\pi(c;\cdot)}_{\M^{({f}_\star ,\widehat{P}_t)}(c)}(s_0)]\right|
    \leq 
   \beta_t\cdot \psi_t(\pi) + \beta_t \cdot \frac{H|S||A|}{t}
   ,
$$
where $\hat{f}_t \in \F$ and $\widehat{P}_t \in \Fp$ are the least square minimizing context-dependent rewards and dynamics in round $t$, respectively.

\noindent\textbf{Analysing the error caused by the dynamics approximation.} 
%(See Sub-subsection~\cref{subsubsec:CB-dynamics}.)
We extend the four steps strategy applied to the rewards approximation, to analyse the dynamics approximation.
Such an extension is possible, due to the following key observation.

\noindent\textbf{Key observation.}
Recall $\B$ is a random variable which generate the next state $s_{h+1}$ for the context $c$ given the true dynamics associated with $c$, $P^c_\star$, the current state $s_h$ and the played action $a_h$.
Recall that by definition,  
$\B(P^c_\star,s_h,a_h) $ distributes $P^c_\star(\cdot|s_h,a_h)$.
%Our observation is that 
\begin{observation}\label{obs:key}
    Since the CMDP is layered
    %and loop free
    , given the context $c_t$ state $s_h^t$ and action $a_h^t$, we have that the random variables $\B(P^{c_t}_\star,s_h^t,a_h^t)$ and $( s^t_0, a^t_0, s^t_1, \ldots, s^t_{h-1}, a^t_{h-1}) $ are independent random variables.
\end{observation}
\noindent Using~\cref{obs:key}, we are able to adapt~\cref{lemma:4.2-agrwl}, 
%(Lemma 4.2 of \cite{agarwal2012contextual})
for the dynamics approximation using least-square regression. For more details, see~\cref{lemma:4.2-agrwl-dynamics}. 
Using~\cref{lemma:4.2-agrwl-dynamics} we obtain our uniform convergence bound for the dynamics approximation, in~\cref{lemma:UC-P}, and apply the four steps strategy presented for the rewards to the dynamics approximation as well.

\noindent\textbf{Step 1:}  
establish uniform convergence bound over any $t \geq 2$ and a fixed sequence of functions $P_2,P_3,\ldots \in \Fp$  which states the following (see~\cref{lemma:UC-lemma-5-dynamics} for more details). 
\\
 Formally, for every $\delta \in (0,1)$, with probability at least $1-\delta/4$ it holds
 %, where $s^{(i,h)}_\B \sim \B(P^{c_i}_\star,s^i_h,a^i_h)$,
 that
    \begin{align*}
        \sum_{i=1}^{t-1}   \mathop{\E}
        \Bigg[\sum_{h=0}^{H-1} \sum_{s'\in S} (P^{c_i}_t(s'| s^i_h, a^i_h) &-P^{c_i}_\star(s'| s^i_h, a^i_h) )^2 | \Hist_{i-1}\Bigg]
        \leq 
        72H|S|\log(8|\Fp|t^3/\delta)
        \\
        & +
        2 \sum_{i=1}^{t-1} \sum_{h=0}^{H-1} 
       \sum_{s'\in S}
        \Big(P_t^{c_i}(s'| s^i_h, a^i_h)- \I\Big[s'=
        s^{(i,h)}_\B\Big] \Big)^2
         - \Big( P^{c_i}_\star(s'| s^i_h, a^i_h) - \I\Big[s'=
        s^{(i,h)}_\B
        \Big] \Big)^2,
    \end{align*}
where $s^{(i,h)}_\B \sim \B(P^{c_i}_\star,s^i_h,a^i_h)$.  
We prove the above in~\cref{lemma:UC-lemma-5-dynamics}, using the results of~\cref{lemma:4.2-agrwl-dynamics,lemma:UC-P}.

\noindent\textbf{Step 2:}
constructing a multiplicative confidence bound over the expected value of any context-dependent policy with respect to the approximated and true dynamics, which holds with high probability.
%In the following, the rewards function is the optimistic approximation at round $t$.
%The confidence bound holds with high probability, in expectation over the contexts. 
\\
Formally, in~\cref{lemma:CB-policy-dynamics} we show that with probability at least $1-\delta/4$ for all $t\geq|A|+1$ and every context-dependent policy $\pi \in \Pi_\C$ it holds that
\begin{align*}
    \left|\E_c\left[V^{\pi(c;\cdot)}_{\M^{(f_\star,P_\star)}(c)}(s_0)\right] - \E_c\left[V^{\pi(c;\cdot)}_{\M^{(f_\star, \widehat{P}_t)}(c)}(s_0)\right]\right|
    \leq
    \sqrt{H|S| \psi_t(\pi)}
    \cdot \sqrt{72 H^2|S|\log(8|\Fp|t^3/\delta)}.  
\end{align*}

\noindent\textbf{Step 3:}
relax the confidence bound in step 2 to be additive.
In~\cref{lemma:sq-trick-UCDD} we show that
%with probability at least $1-\delta/4$,
under the good event of step 2,
for all $t\geq|A|+1$ and every context-dependent policy $\pi \in \Pi_\C$, for $\gamma_t = \sqrt{ \frac{18 t \log(8|\Fp|t^3/\delta) }{|S| |A|}}$, it holds that
    $$
        \left|\E_c\left[V^{\pi(c;\cdot)}_{\M^{(f_\star,P_\star)}(c)}(s_0)\right] - \E_c\left[V^{\pi(c;\cdot)}_{\M^{(f_\star, \widehat{P}_t)}(c)}(s_0)\right]\right|
        \leq 
        \gamma_t  H  |S|  \psi_t(\pi) +\gamma_t \frac{H^2 |S|^2  |A|}{t}
        ,
    $$
where $f_\star = r_\star$ is the true expected rewards function.

\noindent\textbf{Step 4:} bound the sum of contextual potentials similarly to shown for the rewards. For more details, see~\cref{lemma:Contextual-Potential-dynamics}.
\\
By combining all the steps, we obtain the optimism lemma (see~\cref{lemma:optimism-UCDD}) which states that under the good events of step 2 for both the dynamics and rewards approximations, for all $t\geq|A|+1$ it holds that
    \begin{align*}
        \E_c[V^{\pi^\star(c;\cdot)}_{\M(c)}(s_0)] -  \E_c[V^{\pi_t(c;\cdot)}_{\Mhat_t(c)}(s_0)]
        \leq 
        % H |S| \gamma_t \cdot 
        % \psi_t(\pi^\star)
        % + 
        \gamma_t \frac{H^2  |S|^2 |A|}{t}
        % +
        % \beta_t \cdot \psi_t(\pi^\star)
        + 
        \beta_t \frac{H|S| |A|}{t},
    \end{align*}
yielding the following regret bound (see~\cref{thm:regret-bound-ucdd}).
%(and an expected regret bound, see~\cref{thm:e-regret-UCDD}).

\begin{theorem}[regret bound]
For every $\delta \in (0,1)$, $T>|A|$ and finite function classes $\F$ and $\Fp$, let $\beta_t = \sqrt{ \frac{17 t \log(8|\F|t^3/\delta) }{|S||A|}}$ and  $\gamma_t = \sqrt{ \frac{18 t \log(8|\Fp|t^3/\delta) }{|S| |A|}}$ for all $t\in [T]$.  
Then, with probability at least $1-\delta$ it holds that
\begin{align*}
    \Regrv_T(RM-UCDD)\leq 
    \tilde{O}\left( 
        (H+1/p_{min})\cdot
            \left({H|S|^{3/2}\sqrt{|A|T\log\frac{|\Fp|}{\delta}}}+
            %{p_{min}},\; %dynamics
            %H^2  |S|^{3/2}  \sqrt{T|A| \log(|\Fp|/\delta) },\; %dynamics
            {\sqrt{T|S||A|\log\frac{|\F|}{\delta}}}
            %{p_{min}},%rewards
            %H\sqrt{|S||A|T \log\frac{|\F|}{\delta}}%rewards
        %\Bigg\}
        \right)\right).
\end{align*}
\end{theorem}
\noindent In~\cref{corl:regret-bound-ucdd-g} we derive a regret bound in terms of $|\G|$.

\subsubsection{Confidence Bound for Rewards Approximation.}\label{subsubsec:CB-rewards}
As before, our prove for the confidence bound for the rewards approximation consists of the following four steps.

\paragraph{Step 1: Establishing uniform convergence bound over \texorpdfstring{$\F$}{Lg}.} We use~\cref{lemma:UC-lemma-5-UCID}, presented in~\cref{Appendix:UCID}.

\paragraph{Step 2: Constructing confidence bound over policies with respect to the rewards approximation.}

\begin{lemma}[confidence bound over policies w.r.t rewards approximation]\label{lemma:CB-policy-UCDD}
    Consider~\cref{alg:RM-UCDD} 
    that at each initialization round $t \leq |A|$ plays the policy that always choose action $a_t$, and at each round $t \geq |A|+1$ 
    %an admissible 
    %non-randomized contextual MDP algorithm
    selects $\pi_t$ based on the history $\Hist_{t-1}$.
    %at each round $t \geq |A|+1$ 
    
    Then, for any $\delta \in (0,1)$,
    with probability at least $1-\delta/4$ for all $t > |A|$ and any context-dependent policy $\pi \in \Pi_\C$ the following holds.
    %the estimation error of the expected value of $\pi$ caused by the rewards approximation is bounded by
    \begin{align*}
        &\Big|\E_c\left[V^{\pi(c;\cdot)}_{\M^{({f}_\star,\widehat{P}_t)}(c)}(s_0)\right] - \E_c\left[V^{\pi(c;\cdot)}_{\M^{(\hat{f}_t,\widehat{P}_t)}(c)}(s_0)\right]\Big|
        \\
        & \leq
        \sqrt{\E_c \left[ \sum_{h=0}^{H-1} \sum_{s_h \in S^c_h} \frac{q_h(s_h|\pi(c;\cdot) ,\widehat{P}^c_t)}{ p_{min}\sum_{i=1}^{t-1} \I[\pi(c;s_h)= \pi_i(c;s_h)]}\right]}
        \cdot \sqrt{ 68H\log(8|\F|t^3/\delta)},    \end{align*}
    where $\M^{(f,\widehat{P}_t)}(c) = (S,A,\widehat{P}^c_t,f(c,\cdot,\cdot),s_0, H)$ for any $f \in \F$, and $\widehat{P}_t$ is the LSR minimizing context-dependent dynamics in round $t$.
\end{lemma}

The proof of the lemma requires minors modifications to the proof of~\cref{lemma:CB-policy}.
\begin{proof}
    For any function $f \in \F$ we consider the random variable defined in~\cref{def:Y-f-RV} 
    \[
        Y_{f,c_t,s^t_h,a^t_h,r^t_h} = (f(c_t, s^t_h,a^t_h) -
        r^t_h
        %R^{c_t}_\star(s^t_h,a^t_h)
        )^2 - (f_\star(c_t, s^t_h,a^t_h)
        -
        r^t_h
        %R^{c_t}_\star(s^t_h,a^t_h)
        )^2
    \]
    for any $t \geq 1$, context $c_t \in \C$, layer $h \in [H-1]$, state $s^t_h \in S^{c_t}_h$ and action $a^t_h \in A$.
    We remark that $(c_t, s^t_h, a^t_h) \sim \D(c_t) \cdot q_h(s^t_h, a^t_h | \pi_t(c_t; \cdot), P^{c_t}_\star)$ and $r^t_h =  R^{c_t}_\star(s^t_h,a^t_h)$.
    \\
    By the Auxiliary~\cref{claim:expectation-eq-given-hist}, for all $t \geq 2$ it holds that
    We first prove the following auxiliary claim.
    % \begin{claim}\label{claim:expectation-eq-given-hist-UCDD}
    %     For all $t \geq 2$ it holds that
        \begin{align*}
            &\sum_{i=1}^{t-1} \sum_{h=1}^{H-1}
            \mathop{\E}_{c_i, s^i_h,a^i_h}
            \left[ (\hat{f}_t(c_i,s^i_h,a^i_h) - f_\star(c_i,s^i_h,a^i_h) )^2 |\Hist_{i-1}\right]
            =
            \\
            &\E_c \left[\sum_{i=1}^{t-1} \sum_{h=0}^{H-1} \sum_{s_h \in S^c_h} q_h(s_h | \pi_i(c;\cdot),P^c_\star) (\hat{f}_t(c,s_h, \pi_i(c;s_h)) - f_\star(c,s_h, \pi_i(c;s_h)))^2 \right].
        \end{align*}
    By~\cref{lemma:UC-lemma-5-UCID}, for any $\delta \in (0,1)$ we have with probability at least $1-\delta/4$ that
    \begingroup
    \allowdisplaybreaks
    \begin{align*}
        &\sum_{i=1}^{t-1} \sum_{h=0}^{H-1} \mathop{\E}_{c_i,s^i_h,a^i_h}
        \left[ ({f}_t(c_i, s^i_h, a^i_h)-f_\star(c_i, s^i_h, a^i_h) )^2 | \Hist_{i-1}
        %,\sigma^i_{h-1}
        \right]\\
        & \leq
        68H\log(8|\F|t^3/\delta)
        +
        2 \sum_{i=1}^{t-1} \sum_{h=0}^{H-1} ({f}_t(c_i, s^i_h, a^i_h)- r^i_h )^2 - (f_\star(c_i, s^i_h, a^i_h) -r^i_h )^2,
    \end{align*}
    \endgroup
    %simultaneously for all $t \geq |A|+1$ and any fixed sequence of functions ${f}_{|A|+1}, {f}_{|A|+2}, \ldots \in \F$.
    simultaneously for all $t \geq 2$ and any fixed sequence of functions ${f}_2,{f}_3,\ldots \in \F$. 
    %, where we used the definition of $Y$'s.
    %
    %Note that where $f_t = \hat{f}_t$,
    Since $\{{f}_t\}_{t=|A|+1}^T$ are the least square minimizers at each round $t$, it holds that
    \[
        \sum_{i=1}^{t-1} \sum_{h=0}^{H-1} (\hat{f}_t(c_i, s^i_h, a^i_h)- r^i_h )^2 - (f_\star(c_i, s^i_h, a^i_h) -r^i_h )^2 \leq 0.
    \]
    Recall that $\M(c) = \M^{(f_\star,P_\star)}(c)$ for all $c \in \C$.
    By all the above, with probability at least $1-\delta/4$ for all $t \geq |A|+1$ and any context-dependent policy $\pi \in \Pi_\C$ the following holds.
    \begingroup
    \allowdisplaybreaks
    \begin{align*}
        & \left| \E_c\left[V^{\pi(c;\cdot)}_{\M^{(\hat{f}_t,\widehat{P}_t)}(c)}(s_0) \right] - \E_c\left[V^{\pi(c;\cdot)}_{\M^{(f_\star,\widehat{P}_t)}(c)}(s_0)\right]\right|
        \\
        = &
        \left| \E_c\left[V^{\pi(c;\cdot)}_{\M^{(\hat{f}_t,\widehat{P}_t)}(c)}(s_0) - V^{\pi(c;\cdot)}_{\M^{(f_\star,\widehat{P}_t)}(c)}(s_0)\right]\right|
        \\
        \tag{By the value representation using occupancy measures}
        = &
        \left| \E_c \left[ \sum_{h=0}^{H-1} \sum_{s_h \in S^c_h}\sum_{a_h \in A}  q_h(s_h,a_h | \pi(c;\cdot), \widehat{P}^c_t)(\hat{f}_t(c,s_h, a_h) - f_\star(c,s_h, a_h)) \right] \right|
        \\
        \tag{$\pi(c;\cdot)$ is a deterministic policy for all $c \in \C$}
        = & 
        \left| \E_c \left[ \sum_{h=0}^{H-1} \sum_{s_h \in S^c_h}  q_h(s_h | \pi(c;\cdot), \widehat{P}^c_t)(\hat{f}_t(c,s_h, \pi(c;s_h)) - f_\star(c,s_h, \pi(c;s_h))) \right] \right|  
        \\
        \tag{By triangle inequality}
        \leq &
        \E_c \left[ \sum_{h=0}^{H-1} \sum_{s_h \in S^c_h} 
        q_h(s_h | \pi(c;\cdot), \widehat{P}^c_t)
        \left|f_t(c,s_h, \pi(c;s_h)) - f_\star(c,s_h, \pi(c;s_h))\right| \right]
        \\
        = &
        \E_c \Bigg[  \sum_{h=0}^{H-1} \sum_{s_h \in S^c_h}
        (\sqrt{ q_h(s_h | \pi(c;\cdot), \widehat{P}^c_t)})^2
        \frac{  \sqrt{\sum_{i=1}^{t-1} \I[\pi(c;s_h)= \pi_i(c;s_h)] q_h(s_h | \pi_i(c;\cdot),P^c_\star)}}
        {\sqrt{\sum_{i=1}^{t-1} \I[\pi(c;s_h)= \pi_i(c;s_h)]q_h(s_h | \pi_i(c;\cdot),P^c_\star)}}
        \\
        \tag{Multiplication in $\frac{  \sqrt{\sum_{i=1}^{t-1} \I[\pi(c;s_h)= \pi_i(c;s_h)] q_h(s_h | \pi_i(c;\cdot),P^c_\star)}}
        {\sqrt{\sum_{i=1}^{t-1} \I[\pi(c;s_h)= \pi_i(c;s_h)]q_h(s_h | \pi_i(c;\cdot),P^c_\star)}}$}
        & \cdot
        \left|\hat{f}_t(c,s_h, \pi(c;s_h)) - f_\star(c,s_h, \pi(c;s_h))\right| \Bigg]
        \\
        \tag{Re-arranging}
        = & 
       \sum_{c \in \C} \sum_{h=0}^{H-1} \sum_{s_h \in S^c_h} \sqrt{\D(c)}
       (\sqrt{ q_h(s_h | \pi(c;\cdot), \widehat{P}^c_t)})^2
        \frac{ 1}
        {\sqrt{\sum_{i=1}^{t-1} \I[\pi(c;s_h)= \pi_i(c;s_h)]q_h(s_h | \pi_i(c;\cdot),P^c_\star)}}
        \\
       &\cdot \sqrt{\D(c)} \sqrt{\sum_{i=1}^{t-1} \I[\pi(c;s_h)= \pi_i(c;s_h)] q_h(s_h | \pi_i(c;\cdot),P^c_\star)}
        |\hat{f}_t(c,s_h, \pi(c;s_h)) - f_\star(c,s_h, \pi(c;s_h))| 
      %  }
        \\
        \tag{By Cauchy-Schwartz inequality}
         \leq &
        \sqrt{\E_c \left[ \sum_{h=0}^{H-1} \sum_{s_h \in S^c_h}  \frac{q_h^2(s_h | \pi(c;\cdot) ,\widehat{P}^c_t)}{\sum_{i=1}^{t-1} \I[\pi(c;s_h)= \pi_i(c;s_h)]q_h(s_h | \pi_i(c;\cdot),P^c_\star)}\right]}
        \\
        &\cdot \sqrt{ \E_c \left[ \sum_{h=0}^{H-1} \sum_{s_h \in S^c_h} \sum_{i=1}^{t-1}\I[\pi(c;s_h)= \pi_i(c;s_h)] q_h(s_h | \pi_i(c;\cdot),P^c_\star)(\hat{f}_t(c,s_h, \pi(c;s_h)) - f_\star(c,s_h, \pi(c;s_h)))^2 \right]}
        \\
        \leq &
        \tag{Since $q^2_h(s_h | \pi(c;\cdot), \widehat{P}^c_t) \leq q_h(s_h | \pi(c;\cdot), \widehat{P}^c_t)$}
        \sqrt{\E_c \left[ \sum_{h=0}^{H-1} \sum_{s_h \in S^c_h} \frac{q_h(s_h | \pi(c;\cdot), \widehat{P}^c_t)}{\sum_{i=1}^{t-1} \I[\pi(c;s_h)= \pi_i(c;s_h)]q_h(s_h | \pi_i(c;\cdot),P^c_\star)}\right]}
        \\
        &\cdot \sqrt{ \E_c \left[\sum_{h=0}^{H-1} \sum_{s_h \in S^c_h} \sum_{i=1}^{t-1} \I[\pi(c;s_h)= \pi_i(c;s_h)] q_h(s_h | \pi_i(c;\cdot),P^c_\star) (\hat{f}_t(c,s_h, \pi(c;s_h)) - f_\star(c,s_h, \pi(c;s_h)))^2 \right]}
        \\
        \tag{The non-zero terms are where $\pi_i(c;s_h) = \pi(c;s_h)$ and change of summing order}
        = &
        \sqrt{\E_c \left[ \sum_{h=0}^{H-1} \sum_{s_h \in S^c_h} \frac{q_h(s_h | \pi(c;\cdot), \widehat{P}^c_t)}{\sum_{i=1}^{t-1} \I[\pi(c;s_h)= \pi_i(c;s_h)]q_h(s_h | \pi_i(c;\cdot),P^c_\star)}\right]}
        \\
        &\cdot \sqrt{ \E_c \left[\sum_{i=1}^{t-1} \sum_{h=0}^{H-1} \sum_{s_h \in S^c_h}  %\I[\pi(c;s_h)= \pi_i(c;s_h)]
        q_h(s_h | \pi_i(c;\cdot),P^c_\star) (\hat{f}_t(c,s_h, \pi_i(c;s_h)) - f_\star(c,s_h, \pi_i(c;s_h)))^2 \right]}
        % \\
        % \tag{Removing the indicators only increase the sum}
        % \leq &
        % \sqrt{\E_c \left[ \sum_{h=0}^{H-1} \sum_{s_h \in S^c_h} \frac{q_h(s_h | \pi(c;\cdot), \widehat{P}^c_t)}{\sum_{i=1}^{t-1} \I[\pi(c;s_h)= \pi_i(c;s_h)]q_h(s_h | \pi_i(c;\cdot),P^c_\star)}\right]}
        % \\
        % &\cdot \sqrt{ \E_c \left[\sum_{i=1}^{t-1} \sum_{h=0}^{H-1} \sum_{s_h \in S^c_h} q_h(s_h | \pi_i(c;\cdot),P^c_\star) (\hat{f}_t(c,s_h, \pi_i(c;s_h)) - f_\star(c,s_h, \pi_i(c;s_h)))^2 \right]}
        \\
        \tag{By~\cref{claim:expectation-eq-given-hist}}
        = &
        \sqrt{\E_c \left[ \sum_{h=0}^{H-1} \sum_{s_h \in S^c_h} \frac{q_h(s_h | \pi(c;\cdot), \widehat{P}^c_t)}{\sum_{i=1}^{t-1} \I[\pi(c;s_h)= \pi_i(c;s_h)]q_h(s_h | \pi_i(c;\cdot),P^c_\star)}\right]}
        \\
        &\cdot \sqrt{ 
        %\E_c \left[ 
        \sum_{i=1}^{t-1} \sum_{h=1}^{H-1}
        %\E_{\Hist_{i-1},\sigma^i_{h-1}} \left[
        \mathop{\E}_{c_i, s^i_h,a^i_h}
        \left[ (\hat{f}_t(c_i,s^i_h,a^i_h) - f_\star(c_i,s^i_h,a^i_h) )^2 |\Hist_{i-1}
        %, \sigma^i_{h-1}\right] \right]
        \right]} 
        \\
        \tag{By~\cref{lemma:UC-lemma-5-UCID} combined with the fact that $\hat{f}_t$ is the least-square minimizer}
        \leq &
        \sqrt{\E_c \left[ \sum_{h=0}^{H-1} \sum_{s_h \in S^c_h} \frac{q_h(s_h | \pi(c;\cdot), \widehat{P}^c_t)}{\sum_{i=1}^{t-1} \I[\pi(c;s_h)= \pi_i(c;s_h)]q_h(s_h | \pi_i(c;\cdot),P^c_\star)}\right]}
        \cdot \sqrt{ 68 H \log(8|\F|t^3 /\delta)}
        \\
        \tag{By the minimum reachanility of $P^c_\star$}
        \leq &
        \sqrt{\E_c \left[ \sum_{h=0}^{H-1} \sum_{s_h \in S^c_h} \frac{q_h(s_h | \pi(c;\cdot), \widehat{P}^c_t)}{p_{min}\sum_{i=1}^{t-1} \I[\pi(c;s_h)= \pi_i(c;s_h)]}\right]}
        \cdot \sqrt{ 68 H \log(8|\F|t^3 /\delta)}
        . 
    \end{align*}
    \endgroup
Lastly, we remark that by the choice of $\pi_i$ for $i \in \{1,2,\ldots,|A|\}$, and the minimum reachability assumption for any deterministic policy $\pi \in \Pi_\C$, layer $h \in [H-1]$, state $s_h \in S^c_h$ and $t > |A|$ it holds that
\begin{align*}
    \sum_{i=1}^{t-1} \I[\pi(c;s_h)= \pi_i(c;s_h)]q_h(s_h | \pi_i(c;\cdot),P^c_\star)
    \geq
    p_{min} \cdot \sum_{i=1}^{t-1} \I[\pi(c;s_h)= \pi_i(c;s_h)]
    \geq 
     p_{min} > 0,
\end{align*}
hence the above is well defined.
\end{proof}

\paragraph{Step 3: Relax the Confidence Bound to be Additive.}
\begin{lemma}[the ``square trick'' relaxation]\label{lemma:sq-trick-rewards-UCDD}
    Under the good event of~\cref{lemma:CB-policy-UCDD},
    for all $t > |A|$ and any context-dependent policy $\pi \in \Pi_\C$ it holds that 
    \begin{align*}
        \Big|\E_c\left[V^{\pi(c;\cdot)}_{\M^{({f}_\star,\widehat{P}_t)}(c)}(s_0)\right] - \E_c\left[V^{\pi(c;\cdot)}_{\M^{(\hat{f}_t,\widehat{P}_t)}(c)}(s_0)\right]\Big|
        \leq &
        \beta_t \cdot \E_c \left[ \sum_{h=0}^{H-1} \sum_{s_h \in S^c_h} \frac{  q_h(s_h|\pi(c;\cdot) ,\widehat{P}^c_t)}{p_{min}\sum_{i=1}^{t-1} \I[\pi(c;s_h)= \pi_i(c;s_h)]}\right]
        \\
        & +  \beta_t \cdot \frac{H  |S|  |A|}{t}
        ,
    \end{align*}
    where $\beta_t = \sqrt{ \frac{17 t \log(8|\F|t^3 /\delta) }{|S| |A|}}$.
\end{lemma}

The proof is almost identical to the proof of~\cref{lemma:sq-trick}.

\begin{proof}
    Consider the following derivation,
    for $\beta_t = \sqrt{ \frac{17 t \log(8|\F|t^3/\delta) }{|S||A|}}$.
    \begingroup
    \allowdisplaybreaks
    \begin{align*}
        & \Big|\E_c\left[V^{\pi(c;\cdot)}_{\M^{({f}_\star,\widehat{P}_t)}(c)}(s_0)\right] - \E_c\left[V^{\pi(c;\cdot)}_{\M^{(\hat{f}_t,\widehat{P}_t)}(c)}(s_0)\right]\Big|
        \\
        \tag{By~\cref{lemma:CB-policy-UCDD}}
        \leq &
        \sqrt{\E_c \left[ \sum_{h=0}^{H-1} \sum_{s_h \in S^c_h} \frac{q_h(s_h|\pi(c;\cdot) ,\widehat{P}^c_t)}{p_{min}\sum_{i=1}^{t-1} \I[\pi(c;s_h)= \pi_i(c;s_h)]}\right]}
        \cdot \sqrt{ 68 H \log(8|\F|t^3/\delta)}
        \\
        = &
        \sqrt{\E_c \left[ \sum_{h=0}^{H-1} \sum_{s_h \in S^c_h} \frac{\beta_t \cdot q_h(s_h|\pi(c;\cdot) ,\widehat{P}^c_t)}{p_{min}\sum_{i=1}^{t-1} \I[\pi(c;s_h)= \pi_i(c;s_h)]}\right]}
        \cdot \sqrt{\frac{1}{\beta_t} 68 H \log(8|\F|t^3/\delta)}
        \\
        \tag{By $\beta_t$ choice}
        = & 
        \sqrt{\E_c \left[ \sum_{h=0}^{H-1} \sum_{s_h \in S^c_h} \frac{\beta_t \cdot q_h(s_h|\pi(c;\cdot) ,\widehat{P}^c_t) }{p_{min}\sum_{i=1}^{t-1} \I[\pi(c;s_h)= \pi_i(c;s_h)]}\right]}
        \cdot \sqrt{ H \sqrt{\frac{ |S| |A|}{t}} \frac{68\log(8|\F|t^3/\delta)}{\sqrt{17\log(8|\F|t^3/\delta)}}}
        \\
        = & 
        \sqrt{\E_c \left[ \sum_{h=0}^{H-1} \sum_{s_h \in S^c_h} \frac{\beta_t \cdot q_h(s_h|\pi(c;\cdot) ,\widehat{P}^c_t) }{p_{min}\sum_{i=1}^{t-1} \I[\pi(c;s_h)= \pi_i(c;s_h)]}\right]}
        \cdot \sqrt{4 H \cdot \sqrt{\frac{ |S||A|}{t}} \frac{ (\sqrt{17\log(8|\F|t^3/\delta)})^2}{\sqrt{17\log(8|\F|t^3/\delta)}}}
        \\
        = & 
        \sqrt{\E_c \left[ \sum_{h=0}^{H-1} \sum_{s_h \in S^c_h} \frac{\beta_t \cdot q_h(s_h|\pi(c;\cdot) ,\widehat{P}^c_t) }{p_{min}\sum_{i=1}^{t-1} \I[\pi(c;s_h)= \pi_i(c;s_h)]}\right]}
        \cdot \sqrt{4 H \cdot \sqrt{\frac{|S| |A|}{t}}  \sqrt{17\log(8|\F|t^3/\delta)}}
        \\
        = & 
        \sqrt{2 \cdot \E_c \left[ \sum_{h=0}^{H-1} \sum_{s_h \in S^c_h} \frac{\beta_t \cdot q_h(s_h|\pi(c;\cdot) ,\widehat{P}^c_t) }{p_{min}\sum_{i=1}^{t-1} \I[\pi(c;s_h)= \pi_i(c;s_h)]}\right]}
        \\
        & \cdot \sqrt{2 H \cdot \sqrt{\frac{|S|\cdot |A|}{t}}\sqrt{\frac{t^2}{ |S|^2\cdot |A|^2}}  \sqrt{\frac{ |S|^2\cdot |A|^2}{t^2}} \sqrt{17\log(8|\F|t^3/\delta)}}
        \\
        = & 
        \sqrt{2 \cdot \E_c \left[ \sum_{h=0}^{H-1} \sum_{s_h \in S^c_h} \frac{\beta_t \cdot q_h(s_h|\pi(c;\cdot) ,\widehat{P}^c_t) }{p_{min}\sum_{i=1}^{t-1} \I[\pi(c;s_h)= \pi_i(c;s_h)]}\right]}
        \cdot \sqrt{2 H \cdot \beta_t \sqrt{\frac{|S|^2\cdot |A|^2}{t^2}}}
        \\
        = & 
        2\cdot \sqrt{ \E_c \left[ \sum_{h=0}^{H-1} \sum_{s_h \in S^c_h} \frac{\beta_t \cdot q_h(s_h|\pi(c;\cdot) , \widehat{P}^c_t) }{p_{min}\sum_{i=1}^{t-1} \I[\pi(c;s_h)= \pi_i(c;s_h)]}\right]}
        \cdot \sqrt{ \beta_t \cdot \frac{H  |S| |A|}{t}}
        \\
        \tag{Since $2ab\leq a^2+b^2$}
        \leq &
        \E_c \left[ \sum_{h=0}^{H-1} \sum_{s_h \in S^c_h} \frac{\beta_t \cdot q_h(s_h|\pi(c;\cdot) ,\widehat{P}^c_t)}{p_{min}\sum_{i=1}^{t-1} \I[\pi(c;s_h)= \pi_i(c;s_h)]}\right]
        +  \beta_t \cdot \frac{H  |S|  |A|}{t}.
    \end{align*}
    \endgroup
\end{proof}

\paragraph{Step 4: Bounding the contextual potential for both dynamics and rewards.} We show in~\cref{lemma:Contextual-Potential-dynamics} upper bounds on the cumulative contextual potential function $\psi_t$ that is being used for both the dynamics an rewards approximation analysis.

\subsubsection{Confidence Bound for the Dynamics Approximation.}\label{subsubsec:CB-dynamics}
We apply the four steps strategy for the dynamics approximation as well.

\paragraph{Step 1: Establishing uniform convergence bound over \texorpdfstring{$\Fp$}{Lg}.}
Let $\B$ be a random variable which generate the next state $s_{h+1}$ given the true dynamics associated with $c$, $P^c_\star$, the state $s_h$ and the action $a_h$.
 $\B$ is distributed according to  $P^c_\star(\cdot|s_h,a_h)$.
By definition, $\B$ satisfies the following for all $s' \in S$.
    \[
        \E\left[\I[s' = \B(P^c_\star,s_h,a_h)] \Big|c,s_h,a_h \right]
        =
        P^c_\star(s'|s_h,a_h).
        %\;\; \forall s' \in S.
    \]

\begin{observation}
    Given the context $c_t$ state $s_h^t$ and action $a_h^t$, we have that the random variables
    $\B(P^{c_t}_\star,s_h^t,a_h^t)$ and $(s^t_0, a^t_0, s^t_1, \ldots, s^t_{h-1}, a^t_{h-1}) $ 
    are independent random variables.
\end{observation}

\begin{conclusion}
    Our samples for the dynamics approximation satisfies the requirements of~\cref{lemma:4.2-agrwl-dynamics} (see bellow). 
\end{conclusion}

The following is an adaption of Lemma $4.2$ of \citet{agarwal2012contextual} for the dynamics approximation.

\begin{lemma}\label{lemma:4.2-agrwl-dynamics}
    Fix a function $\widetilde{P} \in \Fp$. Suppose we sample context $c$ from the distribution $\D$ and than sample state $s_h$ and action $a_h$ using some policy $\pi$. 
    Let $s'\in S$ denote a possible next state and define the random variable
    % \[
    %     Y_{c,s_h,a_h,\B} = \sum_{s'} 
    %     (\widetilde{P}^c(s'|s_h,a_h) - \I[s'=\B(P^c_\star,s_h,a_h)])^2 - (P^c_\star(s'|s_h,a_h) - \I[s'=\B(P^c_\star,s_h,a_h)])^2,
    % \]
    \[
        Y_{c,s_h,a_h,\B} = \sum_{s'} 
        (\widetilde{P}^c(s'|s_h,a_h) - \I[s'=s_\B])^2 - (P^c_\star(s'|s_h,a_h) - \I[s'=s_\B])^2,
    \]
    where $s_\B \sim \B(P^c_\star,s_h,a_h)$. 
    Then, the followings hold.
    \begin{enumerate}
        \item $\mathop{\E}_{c,s_h,a_h,\B}[Y_{c,s_h,a_h,\B}] = \mathop{\E}_{c,s_h,a_h}\left[\sum_{s' \in S}\left( \widetilde{P}^c(s'|s_h,a_h) - P^c_\star(s'|s_h,a_h)\right)^2 \right]$.
        \item $\mathop{\V}_{c,s_h,a_h,\B}[Y_{c,s_h,a_h,\B}] \leq 4|S|\cdot \mathop{\E}_{c,s_h,a_h,\B}[Y_{c,s_h,a_h,\B}]$.
        \item $\mathop{\V}\left[\sum_{h=0}^{H-1} Y_{c,s_h,a_h,\B}\right] \leq 4H|S|\cdot \mathop{\E}\left[\sum_{h=0}^{H-1} Y_{c,s_h,a_h,\B}\right]$.
    \end{enumerate}
\end{lemma}

\begin{proof}
    Let us rearrange $Y_{c,s_h,a_h,\B}$ as 
    \begin{equation}\label{eq:Y-dynamics-eq}
        Y_{c,s_h,a_h,\B} = \sum_{s'\in S}(\widetilde{P}^c(s'|s_h,a_h) - P^c_\star(s'|s_h,a_h))(\widetilde{P}^c(s'|s_h,a_h) + P^c_\star(s'|s_h,a_h) - 2\I[s'=
        s_\B
        %\B(P^c_\star,s_h,a_h)
        ]).
    \end{equation}
    Consider the following derivation.
    \begingroup
    \allowdisplaybreaks
    \begin{align*}
        &\mathop{\E}_{c,s_h,a_h,\B}[Y_{c,s_h,a_h,\B}]
        \\
        = &
        \mathop{\E}_{c,s_h,a_h,\B}
        \left[ \sum_{s'}\left(\widetilde{P}^c(s'|s_h,a_h) - P^c_\star(s'|s_h,a_h)\right)\left(\widetilde{P}^c(s'|s_h,a_h) + P^c_\star(s'|s_h,a_h)
        % \\
        % &
        - 2\I[s'=
        s_\B
        %\B(P^c_\star,s_h,a_h)
        ]\right)\right]
        \\
        = &
        \mathop{\E}_{c,s_h,a_h}\left[ \mathop{\E}_{\B}
        \left[ \sum_{s'}(\widetilde{P}^c(s'|s_h,a_h) - P^c_\star(s'|s_h,a_h))(\widetilde{P}^c(s'|s_h,a_h) + P^c_\star(s'|s_h,a_h
        )
        % \\
        % &
        - 2\I[s'= s_\B
        %\B(P^c_\star,s_h,a_h)
        ]) \Big| c, s_h,a_h\right]\right] 
        \\
        = &
        \mathop{\E}_{c,s_h,a_h}
        %\mathop{\E}_{\B}
        \left[ \sum_{s'}\left(\widetilde{P}^c(s'|s_h,a_h) - P^c_\star(s'|s_h,a_h)\right) \left(\widetilde{P}^c(s'|s_h,a_h) + P^c_\star(s'|s_h,a_h)
        % \\
        % &
        - 2\E_{\B}\left[ \I[s'=s_\B
        %\B(P^c_\star,s_h,a_h)
        ] |c,s_h,a_h\right] \right)\right]
        \\
        = &
        \mathop{\E}_{c,s_h,a_h}\left[ \sum_{s'}\left(\widetilde{P}^c(s'|s_h,a_h) - P^c_\star(s'|s_h,a_h)\right)^2 \right],    
    \end{align*}
    \endgroup
    where the second identity is since $s_\B \sim \B(P^c_\star,s_h,a_h)$ and the randomness of $\B$ is determined completely by $c$, $s_h$ and $a_h$. 
    The third inequality is since the only ransom variable that depends on $\B$ is $s_\B$. 
    %in a layered MDP the distribution over $(s_h,a_h)$ depends on the played policy and layers $0$ to $h-1$, while the distribution over $s_{h+1}$ given $s_h$ and $a_h$ depends only on layer $h$. Since the layers are disjoint, given the context (which determines the MDP completely), we have that $s_\B$ and $(s_h,a_h)$ are independent.
    %
    The forth identity uses $\B$ definition which implies that for all $s' \in S$ it holds that  ${P^c_\star(s'|s,a)=\mathop{\E}_{\B}[\I[s'=\B(P^c_\star,s,a)]|c,s,a]}$,
    proving the first part of the lemma.
    
    For the second part, note that $\widetilde{P}^c(s'|s_h,a_h)$, $P^c_\star(s'|s_h,a_h)$ and $\I[s'=s_\B]$ are all between $0$ and $1$. 
    Hence from~\cref{eq:Y-dynamics-eq}, using norms inequality ($\|\cdot\|^2_1 \leq n\|\cdot\|^2_2$) we obtain
    \begingroup
    \allowdisplaybreaks
    \begin{align*}
        Y^2_{c,s_h,a_h,\B}
        = &
        \left(\sum_{s'}(\widetilde{P}^c(s'|s_h,a_h) - P^c_\star(s'|s_h,a_h))(\widetilde{P}^c(s'|s_h,a_h) + P^c_\star(s'|s_h,a_h) - 2\I[s'=
        s_\B
        %\B(P^c_\star,s_h,a_h)
        ])\right)^2
        \\
        \leq &
        |S|\sum_{s'}(\widetilde{P}^c(s'|s_h,a_h) - P^c_\star(s'|s_h,a_h))^2(\widetilde{P}^c(s'|s_h,a_h) + P^c_\star(s'|s_h,a_h) - 2\I[s'=
        s_\B
        %\B(P^c_\star,s_h,a_h)
        ])^2
        \\
        \leq &
        4|S|\sum_{s'}(\widetilde{P}^c(s'|s_h,a_h) - P^c_\star(s'|s_h,a_h))^2,
    \end{align*}
    \endgroup
    yielding the second part of the lemma since
    \begingroup
    \allowdisplaybreaks
    \begin{align*}
        \mathop{\V}_{c,s_h,a_h,\B}[Y_{c,s_h,a_h,\B}]
        \leq &
        \mathop{\E}_{c,s_h,a_h,\B}[Y^2_{c,s_h,a_h,\B}]
        \\
        \leq &
        4|S|\mathop{\E}_{c,s_h,a_h}\left[ \sum_{s'}(\widetilde{P}^c(s'|s_h,a_h) - P^c_\star(s'|s_h,a_h))^2 \right]
        \\
        = &
        4|S|\mathop{\E}_{c,s_h,a_h,\B}[Y_{c,s_h,a_h,\B}].
    \end{align*}
    \endgroup
    
    For the third part, by norms inequality and part $2$ of the lemma we have
    \begin{align*}
        \left( \sum_{h=0}^{H-1} Y_{c,s_h,a_h,\B} \right)^2
        \leq 
        H \sum_{h=0}^{H-1} Y^2_{c,s_h,a_h,\B}
        \leq
        4H|S|\sum_{s'}(\widetilde{P}^c(s'|s_h,a_h) - P^c_\star(s'|s_h,a_h))^2,
    \end{align*}
    yielding the third part of the lemma as
    \begingroup
    \allowdisplaybreaks
    \begin{align*}
        \mathop{\V}\left[\sum_{h=0}^{H-1} Y_{c,s_h,a_h,\B}\right]
        \leq &
        \mathop{\E}\left[\left(\sum_{h=0}^{H-1} Y_{c,s_h,a_h,\B}\right)^2\right]
        \\
        \leq &
        4H|S|\mathop{\E}\left[ \sum_{h=0}^{H-1} \sum_{s'}(\widetilde{P}^c(s'|s_h,a_h) - P^c_\star(s'|s_h,a_h))^2 \right]
        \\
        = &
        4H|S|\mathop{\E}\left[\sum_{h=0}^{H-1}  Y_{c,s_h,a_h,\B}\right],
    \end{align*}
    \endgroup
    as stated.
\end{proof}

\begin{definition}\label{def:Y-f-RV-dynamics}
    For every round $t \geq 1$, layer $h \in [H-1]$ and a function $\widetilde{P} \in \Fp$ we define the random variable
    \[
       Y_{\widetilde{P},c_t,s^t_h,a^t_h,s^t_{h+1}} =
       \sum_{s'\in S}
        \left(\widetilde{P}^{c_t}(s'| s^t_h, a^t_h)- \I\left[s'=s^t_{s+1}\right] \right)^2 - \left(P^{c_t}_\star(s'| s^t_h, a^t_h) - \I\left[s'=s^t_{h+1}\right] \right)^2
    \]
    where $(c_t, s^t_h, a^t_h) \sim \D(c_t) \cdot q_h(s^t_h, a^t_h| \pi_t(c_t;\cdot), P^{c_t}_\star)$ 
    %$c_t \sim \D$ is the context at round $t$, and for all $h \in [H-1]$, $(s^t_h,a^t_h,r^t_h)$ are tuples of the trajectory generated at round $t$, $\sigma^t = (s^t_0,a^t_0,r^t_0, \ldots, s^t_H)$.
    and
    %$s^{(t,h)}_\B$ is a random variable that satisfies
    $s^t_{h+1} \sim \B(P^{c_t}_\star,s^t_h,a^t_h)$.
\end{definition}

\begin{observation}\label{obs:UCDD-dist-of-y-c-s-a}
    For all $ i \in [t-1]$ and $ h \in [ H-1]$ we have that
    $$(c_i, s^i_h, a^i_h) \sim \D(c_i) \cdot q_h(s^i_h, a^i_h| \pi_i(c_i;\cdot), P^{c_i}_\star).$$
    Recall that for all $i\leq |A|$ $\pi_i$ is selected deterministically and for all $i>|A|$ $\pi_i$ is determined completely by the history $\Hist_{i-1}$. Hence, by linearity of expectation, for all $\widetilde{P} \in \Fp$ it holds that
    $$
        \mathop{\E}\left[\sum_{h=0}^{H-1} Y_{\widetilde{P},c_i,s^i_h,a^i_h,s^i_{h+1}}\Big|\Hist_{i-1}\right]
        =
        \sum_{h=0}^{H-1} \mathop{\E}_{c_i, s^i_h, a^i_h,s^i_{h+1}}\left[Y_{\widetilde{P},c_i,s^i_h,a^i_h,s^i_{h+1}}\Big|\Hist_{i-1}\right].
    $$
    Similarly,
    $$
        \mathop{\E}\left[\sum_{h=0}^{H-1} \sum_{s'\in S}(\widetilde{P}^{c_i} (s'| s^i_h, a^i_h)-P^{c_i}_\star(s'| s^i_h, a^i_h) )^2\Big|\Hist_{i-1}\right]
        =
        \sum_{h=0}^{H-1} \mathop{\E}_{c_i, s^i_h, a^i_h}\left[\sum_{s'\in S}(\widetilde{P}^{c_i} (s'| s^i_h, a^i_h)-P^{c_i}_\star(s'| s^i_h, a^i_h) )^2\Big|\Hist_{i-1}\right].
    $$
\end{observation}

\begin{observation}\label{obs:UCDD-martingale-Y}
    Since for all $i\leq |A|$ $\pi_i$ is selected deterministically and for all $i>|A|$ $\pi_i$ is determined completely by the history $\Hist_{i-1}$,
    for any function $\widetilde{P} \in \Fp$,
    it holds that $\{Z_i(\widetilde{P})\}_{i=1}^{t-1} $ is a martingale difference sequence of length $t-1$, where the filtration is $\{\Hist_i\}_{i=1}^{t-1}$ and
    $$
        Z_i(\widetilde{P}) := \mathop{\E}\left[\sum_{h=0}^{H-1} Y_{\widetilde{P},c_i,s^i_h,a^i_h,s^i_{h+1}}\Big|\Hist_{i-1}\right]
        - \sum_{h=0}^{H-1} Y_{\widetilde{P},c_i,s^i_h,a^i_h,s^i_{h+1}}. 
    $$
    (Recall that 
    %$\Hist_0$ is the empty history and
    for all 
    $i \geq 1$ we defined that $\Hist_i = (\sigma^1, \ldots, \sigma^{i})$ and $\Hist_0$ is the empty history.)
    
    In addition, since $ \mathop{\E}\left[\sum_{h=0}^{H-1} Y_{\widetilde{P},c_i,s^i_h,a^i_h,s^i_{h+1}}\Big|\Hist_{i-1}\right]$ is determined given $\Hist_{i-1}$, it holds that
    \[
        \V\left[Z_i(\widetilde{P}) \;|\Hist_{i-1}\right] = \V\left[ \sum_{h=0}^{H-1} Y_{\widetilde{P},c_i,s^i_h,a^i_h,s^i_{h+1}} \;\Big| \Hist_{i-1}\right].
    \]
\end{observation}

\begin{lemma}[uniform convergence over $\Fp$]\label{lemma:UC-P}
    For a fixed $t \geq 2$ and a fixed $\delta_t \in (0,1/e^2)$ with probability at least $1-\log_2(t-1)\delta_t$ it holds that
    \begin{align*}
        &\sum_{i=1}^{t-1} \sum_{h=0}^{H-1} \mathop{\E}_{c_i,s^i_h,a^i_h}
        \left[ \sum_{s'\in S}(\widetilde{P}^{c_i} (s'| s^i_h, a^i_h)-P^{c_i}_\star(s'| s^i_h, a^i_h) )^2 \Big| \Hist_{i-1}\right]
        \\
        & =
        \sum_{i=1}^{t-1} \mathop{\E}
        \left[\sum_{h=0}^{H-1} \sum_{s'\in S} (\widetilde{P}^{c_i}(s'| s^i_h, a^i_h)-P^{c_i}_\star(s'| s^i_h, a^i_h) )^2 \Big| \Hist_{i-1}, \right] 
        \\
        & \leq
        72H|S|\log(|\Fp|/\delta_t)
        +
        2 \sum_{i=1}^{t-1} \sum_{h=0}^{H-1} Y_{\widetilde{P},c_i,s^i_h,a^i_h,s^i_{h+1}}.
    \end{align*}
    uniformly over all  $\widetilde{P} \in \Fp$.
\end{lemma}

The following proof is similar to the proof of~\cref{lemma:UC-F}.
\begin{proof}
    Fix a function $\widetilde{P} \in \Fp$.
    % Let $\I[s'=s^i_{h+1}]$ be the realization of $\B^{c_i}(s^i_h,a^i_h,s^i_{h+1})$
    Consider the following random variables defined in~\cref{def:Y-f-RV-dynamics}
    % \[
    %   Y_{\widetilde{P},c_i,s^i_h,a^i_h} =
    %   \sum_{s'\in S}
    %     (\widetilde{P}^{c_i}(s'| s^i_h, a^i_h)- \I[s'=\B(P^c_\star, s^i_h, a^i_h)] )^2 - (P^{c_i}_\star(s'| s^i_h, a^i_h) - \I[s'=\B(P^c_\star, s^i_h, a^i_h)] )^2
    % \]
    \[
       Y_{\widetilde{P},c_i,s^i_h,a^i_h,s^i_{h+1}} =
       \sum_{s'\in S}
        \Big(\widetilde{P}^{c_i}(s'| s^i_h, a^i_h)- \I\left[s'=s^i_{h+1}\right] \Big)^2 - \Big(P^{c_i}_\star(s'| s^i_h, a^i_h) - \I\left[s'=s^i_{h+1}\right] \Big)^2
    \]
    where $s^i_{h+1} \sim \B(P^{c_i}_\star,s^i_h,a^i_h)$, for all $i \in \{1,2,\ldots,t-1\}$ and $h \in [H-1]$. 
    Note that $|Y_{\widetilde{P},c_t,s^i_h,a^i_h,s^i_{h+1}}| \leq 2$ for any $\widetilde{P},c_i,s^i_h,a^i_h,s^i_{h+1}$.\\ 
    %$\widetilde{P} \in \Fp$, round $t \in \{1,2,\ldots,T\}$, context $c_t \in C$, layer $h \in [H-1]$, state $s^t_h \in S^{c_t}_h$ action $a^t_h \in A$ and next state $s^t_{h+1} \in S^{c_t}_{h+1}$. %
    Hence, ${\left| \sum_{h=0}^{H-1} Y_{\widetilde{P},c_i,s^i_h,a^i_h,s^i_{h+1}} \right| \leq 2 H}$ for all $i \in \{1,2,\ldots,t-1\}$.
    \\
    By Freedman's inequality (\cref{lemma:fridman's-ineq}) and~\cref{obs:UCDD-martingale-Y}, for $ \delta_t <  1/e^2$, with probability at least $1-\log_2(t-1)\delta_t/|\Fp|$ it holds that
    \begin{align*}
         &\sum_{i=1}^{t-1} 
         %\sum_{h=0}^{H-1} 
         \mathop{\E}\left[
         \sum_{h=0}^{H-1} Y_{\widetilde{P},c_i,s^i_h,a^i_h,s^i_{h+1}} \Big|
         \Hist_{i-1}\right]
         -
        \sum_{i=1}^{t-1} \sum_{h=0}^{H-1} Y_{\widetilde{P},c_i,s^i_h,a^i_h,s^i_{h+1}}
        \\
        & \leq
        \; 4\sqrt{
        \sum_{i=1}^{t-1} 
        \V\left[\sum_{h=0}^{H-1}Y_{\widetilde{P},c_i,s^i_h,a^i_h,s^i_{h+1}}\Big|\Hist_{i-1} \right]\log(|\Fp|/\delta_t)}
        +
        2 \cdot 2H  \log(|\Fp|/\delta_t).
        %, \;\;\; \forall \widetilde{P} \in \Fp.
    \end{align*}
    By~\cref{lemma:4.2-agrwl-dynamics} part 4, for all $i \in \{1,2,\ldots,t-1\}$ it holds that
    \[
        \V\left[\sum_{h=0}^{H-1}Y_{\widetilde{P},c_i,s^i_h,a^i_h,s^i_{h+1}} \Big|\Hist_{i-1}\right]
        \le
        4H|S|\E\left[\sum_{h=0}^{H-1}Y_{\widetilde{P},c_i,s^i_h,a^i_h,s^i_{h+1}}\Big|\Hist_{i-1}\right].
    \]
    Therefore, by combine the two inequalities we obtain
    \begin{align*}
         &\sum_{i=1}^{t-1}  
         \mathop{\E}\left[\sum_{h=0}^{H-1} Y_{\widetilde{P},c_i,s^i_h,a^i_h,s^i_{h+1}}\Big|
         \Hist_{i-1}\right]
        \\
        & \leq \; 
        4 \sqrt{
        \sum_{i=1}^{t-1} 
        \V\left[\sum_{h=0}^{H-1} Y_{\widetilde{P},c_i,s^i_h,a^i_h,s^i_{h+1}}\Big|\Hist_{i-1}\right]\log(|\Fp|/\delta_t)
        }
        +
        2\cdot 2H \log(|\Fp|/\delta_t)
        +
        \sum_{i=1}^{t-1} \sum_{h=0}^{H-1} Y_{\widetilde{P},c_i,s^i_h,a^i_h,s^i_{h+1}}
        \\
        &\leq \;
        8 \sqrt{
        H|S|\sum_{i=1}^{t-1} 
        \E\left[\sum_{h=0}^{H-1} Y_{\widetilde{P},c_i,s^i_h,a^i_h,s^i_{h+1}} \Big|\Hist_{i-1}\right]\log(|\Fp|/\delta_t)
        }
        +
        4 H|S| \log(|\Fp|/\delta_t)
        +
        \sum_{i=1}^{t-1} \sum_{h=0}^{H-1} Y_{\widetilde{P},c_i,s^i_h,a^i_h,s^i_{h+1}}.
        %, \;\;\; \forall \widetilde{P} \in \Fp.
    \end{align*}
    This implies that for all $\widetilde{P} \in \Fp$,
    \begin{align*}
         &\sum_{i=1}^{t-1}  
         \mathop{\E}\left[\sum_{h=0}^{H-1} Y_{\widetilde{P},c_i,s^i_h,a^i_h,s^i_{h+1}}\Big|
         \Hist_{i-1}\right]
         -
        8 \sqrt{
        H|S|\sum_{i=1}^{t-1} 
        \E\left[\sum_{h=0}^{H-1} Y_{\widetilde{P},c_i,s^i_h,a^i_h,s^i_{h+1}} \Big|\Hist_{i-1}\right]\log(|\Fp|/\delta_t)
        }
        \\
        &\leq \;
        4 H|S| \log(|\Fp|/\delta_t)
        +
        \sum_{i=1}^{t-1} \sum_{h=0}^{H-1} Y_{\widetilde{P},c_i,s^i_h,a^i_h,s^i_{h+1}}.
        %, \;\;\; \forall \widetilde{P} \in \Fp.
    \end{align*}
    By adding $16H|S|\log(|\Fp|/\delta_t)$ to both sides of the inequality we obtain
    \begin{equation*}
        \left( \sqrt{\sum_{i=1}^{t-1}  \mathop{\E}\left[\sum_{h=0}^{H-1} Y_{\widetilde{P},c_i,s^i_h,a^i_h,s^i_{h+1}} \Big|
         \Hist_{i-1}\right]} - 4 \sqrt{H|S|\log(|\Fp|/\delta_t)} \right)^2
         \le
         20H|S| \log(|\Fp|/\delta_t) 
         %+ 4\log(|\Fp|/\delta_t)  
         + \sum_{i=1}^{t-1}
         \sum_{h=0}^{H-1} Y_{\widetilde{P},c_i,s^i_h,a^i_h,s^i_{h+1}}.
    \end{equation*}
    Hence,
    \begin{equation*}
        \sqrt{\sum_{i=1}^{t-1}  \mathop{\E}\left[\sum_{h=0}^{H-1} Y_{\widetilde{P},c_i,s^i_h,a^i_h,s^i_{h+1}} \Big|
         \Hist_{i-1}\right]}
         \le
         4 \sqrt{H|S|\log(|\Fp|/\delta_t)}
         +
         \sqrt{
         20H|S| \log(|\Fp|/\delta_t) 
         %+ 4\log(|\Fp|/\delta_t)  
         + \sum_{i=1}^{t-1}
         \sum_{h=0}^{H-1} Y_{\widetilde{P},c_i,s^i_h,a^i_h,s^i_{h+1}}},
    \end{equation*}
    yielding,
    \begin{equation*}
        \sum_{i=1}^{t-1}  \mathop{\E}\left[\sum_{h=0}^{H-1} Y_{\widetilde{P},c_i,s^i_h,a^i_h,s^i_{h+1}} \Big|
         \Hist_{i-1}\right]
         \le
         \left(
         4 \sqrt{H|S|\log(|\Fp|/\delta_t)}
         +
         \sqrt{
         20H|S| \log(|\Fp|/\delta_t) 
         %+ 4\log(|\Fp|/\delta_t)  
         + \sum_{i=1}^{t-1}
         \sum_{h=0}^{H-1} Y_{\widetilde{P},c_i,s^i_h,a^i_h,s^i_{h+1}}}\right)^2.
    \end{equation*}
    The latter inequality further implies (using $(a+b)^2 \le 2a^2 + 2b^2$) that for all $\widetilde{P} \in \Fp$ it holds that
    \begin{equation}\label{eq:lemma-6-final-dynamics}
        \sum_{i=1}^{t-1}  \mathop{\E}\left[ \sum_{h=0}^{H-1} Y_{\widetilde{P},c_i,s^i_h,a^i_h,s^i_{h+1}} \Big|\Hist_{i-1}\right]
        \le
        72H|S|\log(|\Fp|/\delta_t) 
        + 2\sum_{i=1}^{t-1} \sum_{h=0}^{H-1} Y_{\widetilde{P},c_i,s^i_h,a^i_h,s^i_{h+1}}.
    \end{equation}
    
    Lastly, by combining~\cref{eq:lemma-6-final-dynamics} with~\cref{lemma:4.2-agrwl-dynamics} we obtain the lemma since
    \begingroup
    \allowdisplaybreaks
    \begin{align*}
        &\sum_{i=1}^{t-1}  
        \mathop{\E}
        \left[ \sum_{h=0}^{H-1} \sum_{s'\in S}(\widetilde{P}^{c_i} (s'| s^i_h, a^i_h)-P^{c_i}_\star(s'| s^i_h, a^i_h) )^2  \;\Big|\; \Hist_{i-1}\right]
        \\
        \tag{By~\cref{obs:UCDD-dist-of-y-c-s-a}}
        &=
        \sum_{i=1}^{t-1} \sum_{h=0}^{H-1} \mathop{\E}_{c_i,s^i_h,a^i_h}
        \left[ \sum_{s'\in S}(\widetilde{P}^{c_i} (s'| s^i_h, a^i_h)-P^{c_i}_\star(s'| s^i_h, a^i_h) )^2 \;\Big|\; \Hist_{i-1}\right]
        \\
        \tag{By~\cref{lemma:4.2-agrwl-dynamics} part 1}
        &=
        \sum_{i=1}^{t-1} \sum_{h=0}^{H-1} \mathop{\E}_{c_i,s^i_h,a^i_h,s^i_{h+1}}
        \left[ Y_{\widetilde{P},c_i,s^i_h,a^i_h,s^i_{h+1}} \;\Big|\; \Hist_{i-1}\right]
        \\
        \tag{By~\cref{obs:UCDD-dist-of-y-c-s-a}}
        & = 
        \sum_{i=1}^{t-1}  
        \mathop{\E}
        \left[ \sum_{h=0}^{H-1} Y_{\widetilde{P},c_i,s^i_h,s^i_h,s^i_{h+1}} \;\Big|\; \Hist_{i-1}\right]
        \\
        \tag{By~\cref{eq:lemma-6-final-dynamics} }
        & \leq
       72H|S|\log(|\Fp|/\delta_t)
        +
        2 \sum_{i=1}^{t-1} \sum_{h=0}^{H-1} Y_{\widetilde{P},c_i,s^i_h,a^i_h,s^i_{h+1}}.
    \end{align*}
    \endgroup
\end{proof}

\begin{lemma}[uniform convergence over any sequence of estimators for the dynamics]\label{lemma:UC-lemma-5-dynamics}
    %For an arbitrary contextual MDP algorithm and 
    For any $\delta \in (0,1)$, with probability at least $1-\delta/4$ it holds that
    \begin{align*}
        &\sum_{i=1}^{t-1}   \mathop{\E}
        \left[\sum_{h=0}^{H-1} \sum_{s'\in S} (P^{c_i}_t(s'| s^i_h, a^i_h)-P^{c_i}_\star(s'| s^i_h, a^i_h) )^2 | \Hist_{i-1}\right]
        \\
        = &
        \sum_{i=1}^{t-1} \sum_{h=0}^{H-1}  \mathop{\E}_{c_i,s^i_h,a^i_h}
        \left[\sum_{s'\in S} (P^{c_i}_t(s'| s^i_h, a^i_h)-P^{c_i}_\star(s'| s^i_h, a^i_h) )^2 | \Hist_{i-1}\right]
        \\
        \leq &
        72H|S|\log(8|\Fp|t^3/\delta)\\
    %     \\
    %     & +
    %     2 \sum_{i=1}^{t-1} \sum_{h=0}^{H-1} 
    %   \sum_{s'\in S}
    %     (P_t^{c_i}(s'| s^i_h, a^i_h)- \I[s'=\B(P^c_\star, s^i_h, a^i_h)] )^2\\
    %     & - (P^{c_i}_\star(s'| s^i_h, a^i_h) - \I[s'=\B(P^c_\star, s^i_h, a^i_h)] )^2.
        & +
        2 \sum_{i=1}^{t-1} \sum_{h=0}^{H-1} 
       \sum_{s'\in S}
        \Big(P_t^{c_i}(s'| s^i_h, a^i_h)- \I\Big[s'=
        %\B(P^c_\star, s^i_h, a^i_h)
        s^i_{h+1} \Big] \Big)^2
         - \Big( P^{c_i}_\star(s'| s^i_h, a^i_h) - \I\Big[s'=
        s^i_{h+1} 
        %\B(P^c_\star, s^i_h, a^i_h)
        \Big] \Big)^2,
    \end{align*}
    where $s^i_{h+1} \sim \B(P^{c_i}_\star,s^i_h,a^i_h)$. 
    The above holds for any $t \geq 2$ and a fixed sequence of context-dependent dynamics $P_2,P_3,\ldots \in \Fp$ .
    %(one for each value of $t$).
\end{lemma}

\begin{proof}
    For a fixed $\delta \in (0,1)$, take $\delta_t = \delta/8t^3$ and apply union bound to~\cref{lemma:UC-P} with all $t \geq 2$. We have,
    \begin{align*}
        \sum_{t=1}^\infty  
        \delta_t \log(t-1)
        =
        \sum_{t=1}^\infty  \delta/8t^3  \log(t-1)
        \le
        \sum_{t=1}^\infty \frac{\delta}{8t^2}
        \le
        \frac{\delta}{4},
    \end{align*}
    Hence, by~\cref{lemma:UC-P} we have with probability at least $1-{\delta}/{4}$ that
    \begin{align*}
        &\sum_{i=1}^{t-1} \mathop{\E}
        \left[\sum_{h=0}^{H-1}  \sum_{s'\in S} (P^{c_i}_t(s'| s^i_h, a^i_h)-P^{c_i}_\star(s'| s^i_h, a^i_h) )^2 \Big| \Hist_{i-1}\right]
        \\
        & =
        \sum_{i=1}^{t-1} \sum_{h=0}^{H-1}  \mathop{\E}_{c_i,s^i_h,a^i_h}
        \left[\sum_{s'\in S} (P^{c_i}_t(s'| s^i_h, a^i_h)-P^{c_i}_\star(s'| s^i_h, a^i_h) )^2 \Big| \Hist_{i-1}\right]
        \\
        & \leq
        72H|S|\log(8|\Fp|t^3/\delta)
        +
        2 \sum_{i=1}^{t-1} \sum_{h=0}^{H-1} Y_{P_t,c_t,s^i_h,a^i_h,s^i_{h+1} }
        \\
        & =
        72H|S|\log(8|\Fp|t^3/\delta)
        +
        2 \sum_{i=1}^{t-1} \sum_{h=0}^{H-1} 
       \sum_{s'\in S}
        \Big(P_t^{c_i}(s'| s^i_h, a^i_h)- \I\Big[s'=
        %\B(P^c_\star, s^i_h, a^i_h)
        s^i_{h+1} \Big] \Big)^2
         - \Big( P^{c_i}_\star(s'| s^i_h, a^i_h) - \I\Big[s'=
        s^i_{h+1} 
        %\B(P^c_\star, s^i_h, a^i_h)
        \Big] \Big)^2,
    \end{align*}
    % \[
    %   Y_{\widetilde{P},c_i,s^i_h,a^i_h} =
    %   \sum_{s'\in S}
    %     (\widetilde{P}^{c_i}(s'| s^i_h, a^i_h)- \I[s'=s^{(i,h)}_\B] )^2 - (P^{c_i}_\star(s'| s^i_h, a^i_h) - \I[s'=s^{(i,h)}_\B] )^2
    % \]
    where $s^i_{h+1} \sim \B(P^{c_i}_\star,s^i_h,a^i_h)$.
    The above holds for any $t \geq 2$ and a fixed sequence ${P_2,P_3,\ldots \in \Fp}$.
    %(one function for each value of $t$).
\end{proof}

\paragraph{Step 2: Constructing confidence bound over policies with respect to the dynamics approximation.}

\begin{lemma}[confidence of policies over dynamics approximation]
\label{lemma:CB-policy-dynamics}
%\Orin{Additional |S|}
    % Consider 
    % %an admissible non-randomized contextual MDP 
    % ~\cref{alg:RM-UCDD} 
    % that selects $\pi_t$ based on the history $\Hist_{t-1}$ at each round $t$.
    Consider~\cref{alg:RM-UCDD}
    that at each initialization round $t \leq |A|$ plays the policy that always choose action $a_t$, and at each round $t \geq |A|+1$ 
    %an admissible 
    %non-randomized contextual MDP algorithm
    selects $\pi_t$ based on the history $\Hist_{t-1}$.
    
    Then, for any $\delta \in (0,1)$ with probability at least $1-\delta/4$ for all $t \geq |A|+1$ and every policy $\pi \in \Pi_\C$ the following holds, where $f_\star$ is the true expected rewards function.
    %the estimation error of the dynamics on the expected value of $\pi$ is as follows for the rewards function $\widehat{r}^c_t$ defined in~\cref{alg:RM-UCDD} is bounded by
    %\begin{enumerate}
        %\item 
        \begin{align*}
            (1)\;\;\; \E_c\left[V^{\pi(c;\cdot)}_{\M^{(f_\star,P_\star)}(c)}(s_0)\right] & - \E_c\left[V^{\pi(c;\cdot)}_{\M^{(f_\star, \widehat{P}_t)}(c)}(s_0)\right]
            \\
            \leq &
            \sqrt{\E_c \left[ \sum_{h=0}^{H-1} \sum_{s_h \in S^c_h} \frac{ H |S| \cdot q_h(s_h | \pi(c;\cdot),\widehat{P}^c_t)}{p_{min}\sum_{i=1}^{t-1} \I[\pi(c;s_h)= \pi_i(c;s_h)]}\right]}
            \cdot \sqrt{ 72  H^2  |S| \log(8|\Fp|t^3/\delta)}. 
        \end{align*}
        %\item 
        \begin{align*}
            (2)\;\;\;
            \E_c\left[V^{\pi(c;\cdot)}_{\M^{(f_\star, \widehat{P}_t)}(c)}(s_0)\right] & - \E_c\left[V^{\pi(c;\cdot)}_{\M^{(f_\star,P_\star)}(c)}(s_0)\right]
            \\
            \leq &
            \sqrt{\E_c \left[ \sum_{h=0}^{H-1} \sum_{s_h \in S^c_h} \frac{ H |S| \cdot q_h(s_h | \pi(c;\cdot),\widehat{P}^c_t)}{p_{min}\sum_{i=1}^{t-1} \I[\pi(c;s_h)= \pi_i(c;s_h)]}\right]}
            \cdot \sqrt{ 72 H^2  |S| \log(8|\Fp|t^3/\delta)}.
        \end{align*}
    %\end{enumerate}
\end{lemma}

The following proof is similar to shown for the rewards in~\cref{lemma:CB-policy-UCDD}.
\begin{proof}
    For any $\widetilde{P} \in \Fp$ consider the random variables defined in~\cref{def:Y-f-RV-dynamics}
    \[
       Y_{\widetilde{P},c_i,s^i_h,a^i_h,s^i_{h+1}} =
    %   \sum_{s'\in S}
        % (\widetilde{P}^{c_i}(s'| s^i_h, a^i_h)- \I[s'=\B(P^c_\star, s^i_h, a^i_h)] )^2 - (P^{c_i}_\star(s'| s^i_h, a^i_h) - \I[s'=\B(P^c_\star, s^i_h, a^i_h)] )^2
        \sum_{s'\in S}
        (\widetilde{P}^{c_i}(s'| s^i_h, a^i_h)- \I[s'=
        %\B(P^c_\star, s^i_h, a^i_h)
        s^i_{h+1}] )^2
        - (P^{c_i}_\star(s'| s^i_h, a^i_h) - \I[s'=
        %\B(P^c_\star, s^i_h, a^i_h)
        s^i_{h+1}] )^2,
    \]
   where $s^i_{h+1} \sim \B(P^{c_i}_\star,s^i_h,a^i_h)$,
    %for all $t \in \{1,2,\ldots,T\}$, $c_t \in \C$, $h \in [H]$, $s^t_h \in S^{c_t}_h$ and $a^t_h \in A$. We remark that 
    and $(c_t, s^t_h, a^t_h) \sim \D(c_t) \cdot q_h(s^t_h, a^t_h | \pi_t(c_t; \cdot), P^{c_t})$,
    for all ${t \in \{1,2,\ldots,T\}}$ and $h \in [H-1]$.
    
    We first show the following auxiliary claim.
    \begin{claim}\label{claim:expectation-eq-given-hist-UCDD-dynamics}
        For all $t \geq 2$ it holds that
        \begin{align*}
             &\sum_{i=1}^{t-1} \sum_{h=0}^{H-1}\mathop{\E}_{c_i, s^i_h,a^i_h}
            \left[\sum_{s' \in S} (\widehat{P}^{c_i}_t(s'|s^i_h,a^i_h) - P^{c_i}_\star(s'|s^i_h,a^i_h) )^2 |\Hist_{i-1}\right]
            =
            \\
            &\E_c
            \left[
            \sum_{i=1}^{t-1} \sum_{h=0}^{H-1}\sum_{s_h \in S^c_h} q_h(s_h | \pi_i(c;\cdot),P^c_\star)\sum_{s' \in S}(\widehat{P}^{c_i}_t(s'|s^i_h,\pi_i(c_i, s^i_h)) - P^{c_i}_\star(s'|s^i_h,\pi_i(c_i, s^i_h)) )^2 
            %|\Hist_{i-1}
            \right]
        \end{align*}
    \end{claim}
    
    \begin{proof}
        Recall that for all $ i\in \{1,2,\ldots,|A|\}$,  $\pi_i$ is a deterministically selected policy uses for initialization. For all $i > |A|$, $\pi_i$ is determined completely by the history $\Hist_{i-1}$
        %recalling that for all $i$, $\pi_i$ 
        and is a deterministic policy. 
        Hence, for any function $\widetilde{P} \in \Fp$, round $i \in \{1, \ldots, t-1\}$ and layer $h \in [H-1]$ the following holds.
        \begin{equation}\label{eq:lemma1-1-UCDD-dynamics}
            \begin{split}
                &\mathop{\E}_{c_i, s^i_h,a^i_h}
                \left[ \sum_{s' \in S}(\widetilde{P}^{c_i}(s'|s^i_h,a^i_h) - P^{c_i}_\star(s'|s^i_h,a^i_h) )^2 \Big|\Hist_{i-1}\right]
                \\
                = &
                \mathop{\E}_{c_i, s^i_h}\left[ \sum_{s' \in S}(\widetilde{P}^{c_i}(s'|s^i_h,\pi_i(c_i; s^i_h)) - P_\star^{c_i}(s'|s^i_h,\pi_i(c_i; s^i_h)) )^2 \Big|\Hist_{i-1}\right]
                \\
                = &
                \E_c \left[ 
                \mathop{\E}_{ s_h} \left[ \sum_{s' \in S}(\widetilde{P}^c(s'|s_h,\pi_i(c; s_h)) - P_\star^c(s'|s_h,\pi_i(c; s_h)) )^2 \Big|\Hist_{i-1},c\right]\right]
                \\
                = &
                \E_c \left[ 
                 \sum_{s_h \in S^c_h} q_h(s_h | \pi_i(c;\cdot),P^c_\star)  \sum_{s' \in S}(\widetilde{P}^c(s'|s_h,\pi_i(c; s_h)) - P_\star^c(s'|s_h,\pi_i(c; s_h)) )^2 
                 %\Big %|\Hist_{i-1}
                 \right] 
                ,
        \end{split}
    \end{equation}
    where the first identity is because $a^i_h = \pi_i(c_i; s^i_h)$ and $\pi_i$ is determined completely given $\Hist_{i-1}$. 
    The second identity is because $c_i$ is independent of $\Hist_{i-1}$, but $s^i_h$ is dependent on both $c_i$ and $\Hist_{i-1}$. (The dependency of $s^i_h$ in $\Hist_{i-1}$ is through the policy $\pi_i$).
    %and that the dependence in $\Hist_{i-1}$ affects only $\pi_i$.
    %
    The third identity is an explicit representation of the expectation over $s_h$ given the context $c$ and the history $\Hist_{i-1}$ since $\pi_i$ is determined by $\Hist_{i-1}$.
    %, that reflect es in $\pi_i$ choice.
   \\
    By summing over $i = 1,2,\ldots, t-1$ and $h \in [H-1]$ we obtain the claim since,
    \begin{align*}
        &\sum_{i=1}^{t-1} \sum_{h=1}^{H-1}
        \mathop{\E}_{c_i, s^i_h,a^i_h}
        \left[ \sum_{s' \in S}(\widetilde{P}^{c_i}(s'|s^i_h,a^i_h) - P^{c_i}_\star(s'|s^i_h,a^i_h) )^2 |\Hist_{i-1}\right]
        =
        \\
        \tag{By~\cref{eq:lemma1-1-UCDD-dynamics}}
        = &
        \sum_{i=1}^{t-1} \sum_{h=1}^{H-1}
        \E_c \left[ 
        \sum_{s_h \in S^c_h} q_h(s_h | \pi_i(c;\cdot),P^c_\star)  \sum_{s' \in S}(\widetilde{P}^c(s'|s_h,\pi_i(c; s_h)) - P_\star^c(s'|s_h,\pi_i(c; s_h)) )^2 
        %\Big %|\Hist_{i-1}
        \right]
        \\
        \tag{By linearity of expectation}
        = &
        \E_c \left[  
        \sum_{i=1}^{t-1} \sum_{h=1}^{H-1}
        \sum_{s_h \in S^c_h} q_h(s_h | \pi_i(c;\cdot),P^c_\star)  \sum_{s' \in S}(\widetilde{P}^c(s'|s_h,\pi_i(c; s_h)) - P_\star^c(s'|s_h,\pi_i(c; s_h)) )^2 
        %\Big %|\Hist_{i-1}
        \right],
    \end{align*}
    as stated.
    \end{proof}

    We now return to the proof of the lemma.
    By~\cref{lemma:UC-lemma-5-dynamics}, for any $\delta \in (0,1)$ with probability at least $1-\delta/4$ it holds that
    
    \begin{align*}
        &\sum_{i=1}^{t-1}   \mathop{\E}
        \left[\sum_{h=0}^{H-1} \sum_{s'\in S} (P^{c_i}_t(s'| s^i_h, a^i_h)-P^{c_i}_\star(s'| s^i_h, a^i_h) )^2 | \Hist_{i-1}\right]
        \\
        = &
        \sum_{i=1}^{t-1} \sum_{h=0}^{H-1}  \mathop{\E}_{c_i,s^i_h,a^i_h}
        \left[\sum_{s'\in S} (P^{c_i}_t(s'| s^i_h, a^i_h)-P^{c_i}_\star(s'| s^i_h, a^i_h) )^2 | \Hist_{i-1}\right]
        \\
        \leq &
        72H|S|\log(8|\Fp|t^3/\delta)
        % \\
        % & 
        +
        2 \sum_{i=1}^{t-1} \sum_{h=0}^{H-1} 
       \sum_{s'\in S}
        \Big(P_t^{c_i}(s'| s^i_h, a^i_h)- \I\Big[s'=
        %\B(P^c_\star, s^i_h, a^i_h)
        s^i_{h+1} \Big] \Big)^2
         - \Big( P^{c_i}_\star(s'| s^i_h, a^i_h) - \I\Big[s'=
        s^i_{h+1}
        %\B(P^c_\star, s^i_h, a^i_h)
        \Big] \Big)^2.
    \end{align*}
    %
    %     \begin{align*}
    %     &\sum_{i=1}^{t-1} \sum_{h=0}^{H-1}  \mathop{\E}_{c_i,s^i_h,a^i_h}
    %     \left[\sum_{s'\in S} (\widehat{P}^{c_i}_t(s'| s^i_h, a^i_h)-P^{c_i}_\star(s'| s^i_h, a^i_h) )^2 | \Hist_{i-1},\sigma^i_{h-1}\right]\\
    %     & \leq
    %     72|S|\log(8|\Fp|t^3H^2/\delta)
    %     %\\
    %      +
    %     2 \sum_{i=1}^{t-1} \sum_{h=0}^{H-1}  \sum_{s'\in S}
    %     (\widehat{P}^{c_i}_t(s'| s^i_h, a^i_h)- \I[s'=
    %     %\B(P^c_\star, s^i_h, a^i_h)
    %     s^{(i,h)}_\B] )^2
    %     - (P^{c_i}_\star(s'| s^i_h, a^i_h) - \I[s'=
    %     %\B(P^c_\star, s^i_h, a^i_h)
    %     s^{(i,h)}_\B] )^2,
    % \end{align*}
    %where $s^{(i,h)}_\B \sim \B(P^{c_i}_\star,s^i_h,a^i_h)$.
    %
    The above hold simultaneously for all $t \geq |A|+1$.
    %Note that where $f_t = \widehat{f}_t$,
    \\
    Since $\widehat{P}_t$ is the least square minimizer, it holds that
    \[
        \sum_{i=1}^{t-1} \sum_{h=0}^{H-1}  \sum_{s'\in S}
        (\widehat{P}^{c_i}_t(s'| s^i_h, a^i_h)- \I[s'=
        %\B(P^c_\star, s^i_h, a^i_h)
        s^i_{h+1}
        ] )^2 - (P^{c_i}_\star(s'| s^i_h, a^i_h) - \I[s'=
        %\B(P^c_\star, s^i_h, a^i_h)
        s^i_{h+1}] )^2\leq 0.
    \]

    For item $(1)$, using the above and the value difference lemma, (see~\cref{lemma:val-diff-efroni}), for every context-dependent policy $\pi \in \Pi_\C$ and round $t >  |A|$ the following holds.
    \begingroup
    \allowdisplaybreaks
    \begin{align*}
        & \E_c\left[V^{\pi(c;\cdot)}_{\M^{(f_\star,P_\star)}(c)}(s_0)\right] - \E_c\left[V^{\pi(c;\cdot)}_{\M^{(f_\star, \widehat{P}_t)}(c)}(s_0)\right]
        \\
        = &
        \E_c\left[V^{\pi(c;\cdot)}_{\M^{(f_\star,P_\star)}(c)}(s_0)-V^{\pi(c;\cdot)}_{\M^{(f_\star, \widehat{P}_t)}(c)}(s_0)\right]
        \\
        \tag{\cref{lemma:val-diff-efroni}}
        = & 
        \E_c \left[ 
        \E_{\pi(c;\cdot),\widehat{P}^c_t }
        \left[
        \sum_{h=0}^{H-1} 
        \sum_{s' \in S}
        (P^c_\star(s'|s_h,a_h) - \widehat{P}^c_t(s'|s_h,a_h))V^{\pi(c;\cdot)}_{\M^{(f_\star, P_\star)},h+1}(s')\Bigg| s_0 \right] \right] 
        \\
        \tag{By triangle inequality}
        \leq &
        \E_c \left[ 
        \E_{\pi(c;\cdot),\widehat{P}^c_t }
        \left[
        \sum_{h=0}^{H-1} 
        \sum_{s' \in S}
        |P^c_\star(s'|s_h,a_h) - \widehat{P}^c_t(s'|s_h,a_h)|V^{\pi(c;\cdot)}_{\M^{(f_\star, P_\star)},h+1}(s')  \Bigg| s_0 \right]\right]
        \\
        \tag{Since $f_\star \in [0,1]$}
        %\tag{By~\cref{corl:val-boud-r-hat}}
        \leq &
        \E_c \left[ 
        \E_{\pi(c;\cdot),\widehat{P}^c_t }
        \left[
        \sum_{h=0}^{H-1} 
        \sum_{s' \in S}
        |P^c_\star(s'|s_h,a_h) - \widehat{P}^c_t(s'|s_h,a_h)|\cdot H  \Bigg| s_0 \right]\right]
        \\
        \leq &
        \tag{Explicit representation of the expectation using occupancy measure}
        H \E_c \left[ 
        \sum_{h=0}^{H-1} 
        \sum_{s_h \in S^c_h}\sum_{a_h \in A}
        q_h(s_h,a_h| \pi(c;\cdot), \widehat{P}^c_t ) 
        %\cdot \pi(a_h|c;s_h) 
        \cdot
        \sum_{s' \in S}|P^c_\star(s'|s_h,a_h) - \widehat{P}^c_t(s'|s_h ,a_h)|   \right] 
        \\
        \tag{Since $\pi$ is a deterministic context-dependent policy}
        = &
        H \E_c \left[ 
        \sum_{h=0}^{H-1} 
        \sum_{s_h \in S^c_h}\sum_{a_h \in A}
        q_h(s_h| \pi(c;\cdot), \widehat{P}^c_t ) 
        \cdot \I[a_h=\pi(c;s_h)] 
        \cdot
        \sum_{s' \in S}|P^c_\star(s'|s_h,a_h) - \widehat{P}^c_t(s'|s_h ,a_h)|   \right] 
        \\
        \tag{The non-zero terms are where $a_h = \pi(c;s_h)$.}
        = &
         H \E_c \left[ 
        \sum_{h=0}^{H-1} 
        \sum_{s_h \in S^c_h}
        q_h(s_h| \pi(c;\cdot), \widehat{P}^c_t )
        \sum_{s' \in S}|P^c_\star(s'|s_h,\pi(c;s_h)) - \widehat{P}^c_t(s'|s_h ,\pi(c;s_h))|  \right] 
        \\
        \tag{By multiplication in $\frac{  \sqrt{\sum_{i=1}^{t-1} \I[\pi(c;s_h)= \pi_i(c;s_h)] q_h(s_h | \pi_i(c;\cdot),P^c_\star)}}
        {\sqrt{\sum_{i=1}^{t-1} \I[\pi(c;s_h)= \pi_i(c;s_h)]q_h(s_h | \pi_i(c;\cdot),P^c_\star)}}$}
        = &
        \E_c \Bigg[  \sum_{h=0}^{H-1} \sum_{s_h \in S^c_h}
        (\sqrt{H})^2 q_h(s_h | \pi(c;\cdot),\widehat{P}^c_t)
        \\
        &
        \cdot
        \frac{  \sqrt{\sum_{i=1}^{t-1} 
        \I[\pi(c;s_j)= \pi_i(c;s_h)] q_h(s_h | \pi_i(c;\cdot),P^c_\star)}}
        {\sqrt{\sum_{i=1}^{t-1} \I[\pi(c;s_h)= \pi_i(c;s_h)]q_h(s_h | \pi_i(c;\cdot),P^c_\star)}}\cdot
        \sum_{s' \in S}|P^c_\star(s'|s_h, \pi(c;s_h)) - \widehat{P}^c_t(s'|s_h, \pi(c;s_h))| \Bigg]
        \\
        \tag{Re-arranging}
        % \tag{By multiplication in $\frac{  \sqrt{\sum_{i=1}^{t-1} \I[\pi(c;s_h)= \pi_i(c;s_h)] q_h(s_h | \pi_i(c;\cdot),P^c_\star)}}
        % {\sqrt{\sum_{i=1}^{t-1} \I[\pi(c;s_h)= \pi_i(c;s_h)]q_h(s_h | \pi_i(c;\cdot),P^c_\star)}}$}
        =& 
       \sum_{c \in \C} \sum_{h=0}^{H-1} \sum_{s_h \in S^c_h}  \sum_{s'\in S} \sqrt{\D(c)}
       \sqrt{H}
        \frac{ q_h(s_h | \pi(c;\cdot),\widehat{P}^c_t)}
        {\sqrt{\sum_{i=1}^{t-1} \I[\pi(c;s_h)= \pi_i(c;s_h)]q_h(s_h | \pi_i(c;\cdot),P^c_\star)}}
        \\
       &\cdot 
       \sqrt{\D(c)} \sqrt{H} \sqrt{\sum_{i=1}^{t-1} \I[\pi(c;s_h)= \pi_i(c;s_h)] q_h(s_h | \pi_i(c;\cdot),P^c_\star)}
        |P^c_\star(s'|s,\pi(c;s)) - \widehat{P}^c_t(s'|s,\pi(c;s))| 
        \\
        % \end{align*}
        % \begin{align*}
        \tag{By Cauchy-Schwartz}
         \leq &
        \sqrt{\E_c \left[ \sum_{h=0}^{H-1} \sum_{s_h \in S^c_h} \sum_{s'\in S}  \frac{H \cdot q^2_h(s_h | \pi , \widehat{P}^c_t)}{\sum_{i=1}^{t-1} \I[\pi(c;s_h)= \pi_i(c;s_h)]q_h(s_h | \pi_i(c;\cdot),P^c_\star)}\right]}
        \\
        &\cdot \sqrt{ H \cdot \E_c \left[ \sum_{h=0}^{H-1} 
        \sum_{s_h \in S^c_h}
        \sum_{s'\in S}
        \sum_{i=1}^{t-1}
        \I[\pi(c;s_h)= \pi_i(c;s_h)] q_h(s_h | \pi_i(c;\cdot),P^c_\star)
        \cdot(P^c_\star(s'|s, \pi(c;s)) - \widehat{P}^c_t(s'|s, \pi(c;s)))^2 \right]}
        \\
        \tag{By $q^2_h(s_h | \pi(c;\cdot), \widehat{P}^c_t) \leq q_h(s_h | \pi(c;\cdot), \widehat{P}^c_t)$ and change of summing order}
        \leq &
        \sqrt{\E_c \left[ \sum_{h=0}^{H-1} \sum_{s_h \in S^c_h}  \frac{ H\cdot |S| \cdot q_h(s_h | \pi(c;\cdot) ,\widehat{P}^c_t)}{\sum_{i=1}^{t-1} \I[\pi(c;s_h)= \pi_i(c;s_h)]q_h(s_h | \pi_i(c;\cdot),P^c_\star)}\right]}
        \\
        &\cdot \sqrt{  H \E_c \left[\sum_{i=1}^{t-1}\sum_{h=0}^{H-1} \sum_{s_h \in S^c_h}  q_h(s_h | \pi_i(c;\cdot),P^c_\star)   \I[\pi(c;s_h)= \pi_i(c;s_h)]  \sum_{s'\in S} (\widehat{P}_t^c(s'|s_h, \pi(c;s_h)) - P^c_\star(s'|s_h, \pi(c;s_h)))^2 \right]}
        \\
        \tag{All terms are non-negative, and the non-zero terms are where $\pi_i(c;s_h) = \pi(c;s_h)$}
        = &
        \sqrt{\E_c \left[ \sum_{h=0}^{H-1} \sum_{s_h \in S^c_h}   \frac{ H \cdot |S| \cdot q_h(s_h | \pi(c;\cdot) ,\widehat{P}^c_t)}{\sum_{i=1}^{t-1} \I[\pi(c;s_h)= \pi_i(c;s_h)]q_h(s_h | \pi_i(c;\cdot),P^c_\star)}\right]}
        \\
        &\cdot \sqrt{ H \cdot  
        \E_c \left[
        \sum_{i=1}^{t-1} 
        \sum_{h=0}^{H-1} \sum_{s_h \in S^c_h}   
        q_h(s_h | \pi_i(c;\cdot),P^c_\star)
        \sum_{s'\in S}(\widehat{P}_t^c(s'|s_h, \pi_i(s_h)) - P^c_\star(s'|s_h, \pi_i(s_h)))^2 \right]}
        \\
        \tag{By~\cref{claim:expectation-eq-given-hist-UCDD-dynamics}}
        = &
        \sqrt{\E_c \left[ \sum_{h=0}^{H-1} \sum_{s_h \in S^c_h} 
        \frac{ H \cdot |S| \cdot q_h(s_h | \pi(c;\cdot) ,\widehat{P}^c_t)}{\sum_{i=1}^{t-1} \I[\pi(c;s_h)= \pi_i(c;s_h)]q_h(s_h | \pi_i(c;\cdot),P^c_\star)}\right]} \cdot 
        \\
        &\cdot 
        \sqrt{H
        \sum_{i=1}^{t-1}  \sum_{h=1}^{H-1} 
        \mathop{\E}_{c_i, s^i_h,a^i_h}
        \left[
        \sum_{s'\in S}
        (\widehat{P}^c_t(s'|s^i_h,a^i_h)-P^c_\star(s'|s^i_h,a^i_h))^2 \;\Bigg|\Hist_{i-1} \right] }
        \\
        \tag{By~\cref{lemma:UC-lemma-5-dynamics} combined with the fact that $\widehat{P}_t$ is the least-square minimizer}
        \leq &
        \sqrt{\E_c \left[ \sum_{h=0}^{H-1} \sum_{s_h \in S^c_h} \frac{ H \cdot |S| \cdot q_h(s_h | \pi(c;\cdot),\widehat{P}^c_t)}{\sum_{i=1}^{t-1} \I[\pi(c;s_h)= \pi_i(c;s_h)]q_h(s_h | \pi_i(c;\cdot),P^c_\star)}\right]}
        \cdot \sqrt{ H \cdot 72 H |S| \log(8|\Fp|t^3/\delta)}
        \\
        \tag{By Minimum reachability of $P^c_\star$}
        \leq &
        \sqrt{\E_c \left[ \sum_{h=0}^{H-1} \sum_{s_h \in S^c_h} \frac{ H  |S| \cdot q_h(s_h | \pi(c;\cdot),\widehat{P}^c_t)}{p_{min}\sum_{i=1}^{t-1} \I[\pi(c;s_h)= \pi_i(c;s_h)]}\right]}
        \cdot \sqrt{  72 H^2 |S| \log(8|\Fp|t^3/\delta)}
        . 
    \end{align*}
    \endgroup
    
The above proves $(1)$. 

For $(2)$, by~\cref{lemma:val-diff-simple-P} it holds that,
\begin{align*}
   &\E_c[V^{\pi(c;\cdot)}_{\M^{(f_\star,\widehat{P}_t)}(c)}(s_0)] - \E_c[V^{\pi(c;\cdot)}_{\M^{(f_\star, P_\star)}(c)}(s_0)]
    \\
    \tag{\cref{lemma:val-diff-simple-P}}
    = & 
    \E_c \left[ 
    \E_{\pi(c;\cdot),\widehat{P}^c_t }
    \left[
    \sum_{h=0}^{H-1} 
    \sum_{s' \in S}
    (\widehat{P}^c_t(s'|s_h,a_h) -  P^c_\star(s'|s_h,a_h))V^{\pi(c;\cdot)}_{\M^{(f_\star,P_\star)},h+1}(s') \right] \Bigg| s_0 \right]
    \\
    \leq &
    \E_c \left[ 
    \E_{\pi(c;\cdot),\widehat{P}^c_t }
    \left[
    \sum_{h=0}^{H-1} 
    \sum_{s' \in S}
    \left|P^c_\star(s'|s_h,a_h) -   \widehat{P}^c_t(s'|s_h,a_h)\right|V^{\pi(c;\cdot)}_{\M^{(f_\star,P_\star)},h+1}(s') \right] \Bigg| s_0 \right].
\end{align*}
Now, using an identical derivation to that showed above (from the third inequality on), we obtain $(2)$.
%(instead of~\cref{lemma:val-diff-simple-P}) and an identical derivation we obtain $(2)$.

Lastly, we remark that by the choice of $\pi_i$ for all $i \in \{1,2,\ldots,|A|\}$, and the minimum reachability assumption for any deterministic context-dependent policy $\pi \in \Pi_\C$, context $c \in \C$, layer $h \in [H-1]$ and state $s_h \in S^c_h$ it holds that
\begin{align*}
    \sum_{i=1}^{t-1} \I[\pi(c;s_h)= \pi_i(c;s_h)]q_h(s_h | \pi_i(c;\cdot),P^c_\star)
    \geq
    p_{min} \cdot \sum_{i=1}^{t-1} \I[\pi(c;s_h)= \pi_i(c;s_h)]
    \geq 
     p_{min} > 0,
\end{align*}
hence the above is well defined.
\end{proof}

\paragraph{Step 3: Relax the confidence bound to be additive.}

\begin{lemma}[the ``square trick'' relaxation for dynamics approximation]\label{lemma:sq-trick-UCDD}
    %Under the conditions 
    Under the good event of~\cref{lemma:CB-policy-dynamics} for all $t>|A|$ and any context-dependent policy $\pi \in \Pi_\C$ the followings hold for ${\gamma_t = \sqrt{ \frac{18 t \log(8|\Fp|t^3/\delta) }{|S| |A|}}}$.
    \begin{align*}
        (1) \quad
        \E_c\left[V^{\pi(c;\cdot)}_{\M^{(f_\star,P_\star)}(c)}(s_0)\right]  - \E_c\left[V^{\pi(c;\cdot)}_{\M^{(f_\star, \widehat{P}_t)}(c)}(s_0)\right]
        \leq &
        \E_c \left[ \sum_{h=0}^{H-1} \sum_{s_h \in S^c_h} \frac{H  |S|  \gamma_t \cdot  q_h(s_h | \pi(c;\cdot),\widehat{P}^c_t)}{p_{min}\sum_{i=1}^{t-1} \I[\pi(c;s_h)= \pi_i(c;s_h)]}\right]
        \\
        & + 
        \gamma_t \frac{H^2  |S|^2 |A|}{t}.
    \end{align*}
    \begin{align*}
        (2) \quad
        \E_c\left[V^{\pi(c;\cdot)}_{\M^{(f_\star, \widehat{P}_t)}(c)}(s_0)\right]
        -\E_c\left[V^{\pi(c;\cdot)}_{\M^{(f_\star,P_\star)}(c)}(s_0)\right]  
        \leq &
        \E_c \left[ \sum_{h=0}^{H-1} \sum_{s_h \in S^c_h} \frac{H  |S|\gamma_t \cdot  q_h(s_h | \pi(c;\cdot),\widehat{P}^c_t)}{p_{min}\sum_{i=1}^{t-1} \I[\pi(c;s_h)= \pi_i(c;s_h)]}\right]
        \\
        & + 
        \gamma_t \frac{H^2 |S|^2 |A|}{t}.
    \end{align*}
    % \end{enumerate}
    %Where $\gamma_t = \sqrt{ \frac{72 t \log(8|\Fp|t^3/\delta) }{|S| |A|}}$.
\end{lemma}

\begin{proof}
    For item $(1)$, consider the following derivation,
    where $\gamma_t = \sqrt{ \frac{18 t \log(8|\Fp|t^3/\delta) }{|S| |A|}}$.
    \begingroup
    \allowdisplaybreaks
    \begin{align*}
        &\E_c\left[V^{\pi(c;\cdot)}_{\M^{(f_\star,P_\star)}(c)}(s_0)\right]  - \E_c\left[V^{\pi(c;\cdot)}_{\M^{(f_\star, \widehat{P}_t)}(c)}(s_0)\right]
        \\
        \tag{By item $(1)$ in~\cref{lemma:CB-policy-dynamics}}
        \leq &
        \sqrt{\E_c \left[ \sum_{h=0}^{H-1} \sum_{s_h \in S^c_h} \frac{ H  |S| \cdot q_h(s_h | \pi(c;\cdot),\widehat{P}^c_t)}{p_{min}\sum_{i=1}^{t-1} \I[\pi(c;s_h)= \pi_i(c;s_h)]}\right]}
        \cdot \sqrt{ 72 H^2 |S| \log(8|\Fp|t^3/\delta)}
        \\
        = &
        \sqrt{\E_c \left[ \sum_{h=0}^{H-1} \sum_{s_h \in S^c_h} \frac{\gamma_t \cdot H  |S| \cdot q_h(s_h | \pi(c;\cdot),\widehat{P}^c_t)}{p_{min}\sum_{i=1}^{t-1} \I[\pi(c;s_h)= \pi_i(c;s_h)]}\right]}
        \cdot 
        \sqrt{\frac{1}{\gamma_t} 72 H^2 |S| \log(8|\Fp|t^3/\delta)}
        \\
        = & 
        \sqrt{\E_c \left[ \sum_{h=0}^{H-1} \sum_{s_h \in S^c_h} \frac{\gamma_t\cdot H  |S| \cdot q_h(s_h | \pi(c;\cdot),\widehat{P}^c_t)}{p_{min}\sum_{i=1}^{t-1} \I[\pi(c;s_h)= \pi_i(c;s_h)]}\right]}
        % \\
        % & 
        \cdot
        \sqrt{ \sqrt{\frac{|S||A|}{t}} \frac{4  H^2 |S|  18 \log(8|\Fp|t^3/\delta)}{\sqrt{18  \log(8|\Fp|t^3/\delta)}}}
        \\
        = & 
        2 \sqrt{\E_c \left[ \sum_{h=0}^{H-1} \sum_{s_h \in S^c_h} \frac{\gamma_t  \cdot H |S| \cdot q_h(s_h | \pi(c;\cdot),\widehat{P}^c_t)}{p_{min}\sum_{i=1}^{t-1} \I[\pi(c;s_h)= \pi_i(c;s_h)]}\right]}
        % \\
        % & 
        \cdot 
        \sqrt{ H^2 |S| \sqrt{ \frac{|S| |A|}{t}}  \sqrt{18\log(8|\Fp|t^3/\delta)}}
        \\
        = & 
        2\sqrt{ \E_c \left[ \sum_{h=0}^{H-1} \sum_{s_h \in S^c_h} \frac{\gamma_t \cdot H  |S| \cdot q_h(s_h | \pi(c;\cdot),\widehat{P}^c_t)}{p_{min}\sum_{i=1}^{t-1} \I[\pi(c;s_h)= \pi_i(c;s_h)]}\right]}
        \\
        & \cdot \sqrt{ H^2  |S| \sqrt{\frac{|S||A|}{t}}  \sqrt{\frac{t^2}{ |S|^2 |A|^2}} \sqrt{18\log(8|\Fp|t^3/\delta)} \sqrt{\frac{|S|^2|A|^2}{t^2}} }
        \\
        = & 
        2 \sqrt{ \E_c \left[ \sum_{h=0}^{H-1} \sum_{s_h \in S^c_h} \frac{\gamma_t \cdot H  |S| \cdot q_h(s_h | \pi(c;\cdot),\widehat{P}^c_t)}{p_{min}\sum_{i=1}^{t-1} \I[\pi(c;s_h)= \pi_i(c;s_h)]}\right]}
        \cdot 
        \sqrt{ H^2  |S|  \gamma_t \sqrt{\frac{|S|^2|A|^2}{t^2}}}
        \\
        = & 
        2\sqrt{ \E_c \left[ \sum_{h=0}^{H-1} \sum_{s_h \in S^c_h} \frac{\gamma_t \cdot H  |S| \cdot q_h(s_h | \pi(c;\cdot),\widehat{P}^c_t)}{p_{min}\sum_{i=1}^{t-1} \I[\pi(c;s_h)= \pi_i(c;s_h)]}\right]}
        \cdot 
        \sqrt{  \gamma_t \frac{H^2  |S|^2 |A|}{t}}
        \\
        \tag{Since $2ab \leq a^2 + b^2$ for all $a,b \geq 0$}
        \leq &
         \E_c \left[ \sum_{h=0}^{H-1} \sum_{s_h \in S^c_h} \frac{\gamma_t \cdot H  |S| \cdot q_h(s_h | \pi(c;\cdot),\widehat{P}^c_t)}{p_{min}\sum_{i=1}^{t-1} \I[\pi(c;s_h)= \pi_i(c;s_h)]}\right]
        +  \gamma_t \frac{H^2  |S|^2 |A|}{t}.
    \end{align*}
    \endgroup
    The above proves $(1)$.
    We obtain $(2)$ using an identical derivation, where in the first inequality we use item $(2)$ of~\cref{lemma:CB-policy-dynamics}, instead of item $(1)$.
\end{proof}

\paragraph{Step 4: Bounding the contextual potential for both dynamics and rewards.}

As in previous sections, we consider the contextual potential functions in round $t$, for ${T \geq t > |A|}$ and a context-depended policy $\pi \in \Pi_\C$. 
We abuse the notation $\psi_t$ as follows.
\begin{definition}
We denote the contextual potential functions in round $t$, for ${T \geq t > |A|}$ and a context-depended policy $\pi \in \Pi_\C$ as 
$$
    \psi_t(\pi) : = \E_c \left[ \sum_{h=0}^{H-1} \sum_{s_h \in S^c_h} \frac{q_h(s_h|\pi(c;\cdot) ,\widehat{P}^c_t)}{ p_{min}\sum_{i=1}^{t-1} \I[\pi(c;s_h)= \pi_i(c;s_h)]}\right],
$$
where $\{\pi_t \in \Pi_{\C}\}_{t=1}^T$ is the sequence of policies selected by~\cref{alg:RM-UCDD} and $\{\widehat{P}_t\}_{t=1}^T$ is the sequence of least square minimizing dynamics. 
\end{definition}

In the following lemma, we bound the sum of contextual potential functions, over the rounds $t = |A|+1, \ldots, T$.
\begin{lemma}[contextual potential lemma for dynamics approximation]\label{lemma:Contextual-Potential-dynamics}
    Let ${\{\pi_t \in \Pi_{\C}\}_{t=1}^T}$
    be the sequence of context-dependent policies selected by~\cref{alg:RM-UCDD}, and let $\{\widehat{P}_t\}_{t=1}^T$ be the sequence of least square minimizing dynamics. 
    Then, for all $T > |A|$ the following holds for the sequence of selected policies $\{\pi_t \in \Pi_{\C}\}_{t=1}^T$.
        \begin{align*}
            \sum_{t=|A| +1}^T
            \psi_t(\pi_t)
            = & 
            \sum_{t=|A| +1}^T
            \E_c \left[ \sum_{h=0}^{H-1} \sum_{s_h \in S^c_h} \frac{q_h(s_h|\pi(c;\cdot) ,\widehat{P}^c_t)}{ p_{min}\sum_{i=1}^{t-1} \I[\pi(c;s_h)= \pi_i(c;s_h)]}\right]
            \\
            \leq & 
            \frac{|S||A|}{p_{min}}(1+\log(T/|A|))
            .
        \end{align*}
\end{lemma}

\begin{proof}
    For any fixed context $c \in \C$, the following holds.
    \begingroup
    \allowdisplaybreaks
    \begin{align*}
        &\sum_{t=|A| +1}^T \sum_{h =0}^{H-1} \sum_{s_h \in S^c_h}  
        \frac{q_h(s_h| \pi_t(c;\cdot),\widehat{P}^c_t)}{p_{min}\sum_{i=1}^{t-1} \I[\pi_t(c;s_h) = \pi_i(c;s_h)]}
        \\
        \leq &
        \frac{1}{p_{min}} \sum_{h =0}^{H-1} \sum_{s_h \in S^c_h} 
        \sum_{t=|A| +1}^T
        \frac{1}{\sum_{i=1}^{t-1} \I[\pi_t(c;s_h) = \pi_i(c;s_h)]}
        \\
        \leq &
        \frac{1}{p_{min}} \sum_{h =0}^{H-1} \sum_{s_h \in S^c_h} 
        \sum_{a_h \in A}
        \sum_{i=1}^{\sum_{t=1}^T  \I[\pi_t(c;s_h) = a_h] } \frac{1}{i}
        \\
        \tag{Since $\sum_{i=1}^n \frac{1}{i} \leq 1+\log(n)$}
        \leq &
        \frac{1}{p_{min}}
        \sum_{h =0}^{H-1} \sum_{s_h \in S^c_h} \sum_{a_h \in A} \left(1 + \log\left(\sum_{t=1}^T \I[\pi_t(c;s_h) = a_h]\right) \right)
        \\
        = &
         \frac{|S||A|}{p_{min}} + \sum_{h =0}^{H-1} \sum_{s_h \in S^c_h} |A| \cdot \frac{1}{|A|}\sum_{a_h \in A}  \log\left(\sum_{t=1}^T \I[\pi_t(c;s_h) = a_h]\right)
        \\
         \tag{By Jansen's inequality, since $\log$ is concave}
         \leq &
        \frac{|S||A|}{p_{min}} + \sum_{h =0}^{H-1} \sum_{s_h \in S^c_h} |A| \cdot   \log\left(\frac{1}{|A|}\sum_{a_h \in A}\sum_{t=1}^T \I[\pi_t(c;s_h) = a_h]\right)
        \\
        = &
        \frac{|S||A|}{p_{min}} + \sum_{h =0}^{H-1} \sum_{s_h \in S^c_h} |A| \cdot   \log\left(\frac{1}{|A|}\sum_{t=1}^T \sum_{a_h \in A} \I[\pi_t(c;s_h) = a_h]\right)
        \\
        \tag{For all $t$, $\sum_{a_h \in A} \I[\pi_t(c;s_h) = a_h]=1$ since $\pi_t$ is a deterministic policy}
         = &
        \frac{|S||A|}{p_{min}} + \sum_{h =0}^{H-1} \sum_{s_h \in S^c_h} |A| \cdot   \log\left(\frac{T}{|A|}\right)
        \\
        = &
        \frac{|S||A|}{p_{min}}( 1 +  \log(T/|A|)).
    \end{align*}
    \endgroup
    By taking an expectation over $c$ on both sides of the inequality, we obtain the lemma.
\end{proof}

\subsubsection{Regret Bound}

\begin{lemma}[optimism]\label{lemma:optimism-UCDD}
Under the good events of~\cref{lemma:CB-policy-UCDD,lemma:CB-policy-dynamics} for any $t \geq |A|+1$ it holds that
    \begin{align*}
        \E_c\left[V^{\pi^\star(c;\cdot)}_{\M(c)}(s_0)\right] -  \E_c\left[V^{\pi_t(c;\cdot)}_{\Mhat_t(c)}(s_0)\right]
        \leq 
        \gamma_t \frac{H^2 |S|^2 |A|}{t}
        +
        \beta_t \frac{H|S||A|}{t},
    \end{align*}
    where $\beta_t = \sqrt{ \frac{17 t \log(8|\F|t^3/\delta) }{|S||A|}}$ and  $\gamma_t = \sqrt{ \frac{18 t \log(8|\Fp|t^3/\delta) }{|S| |A|}}$.
\end{lemma}

\begin{observation}\label{obs:value-decomp-UCDD}
    Note that for every fixed context $c \in \C$, dynamics $P \in \Fp$ and context-dependent policy $\pi \in \Pi_\C$ the following holds for all $t \in [T]$.
    \begingroup
    \allowdisplaybreaks
    \begin{align*}
        V^{\pi(c;\cdot)}_{\M^{(\widehat{r}_t, P)}(c)}(s_0)
        = &
        \sum_{h = 0}^{H-1} \sum_{s_h \in S^c_h} \sum_{a_h \in A}
        q_h(s_h, a_h |\pi(c;\cdot),P^c) \cdot \widehat{r}_t^c(s_h, a_h)
        \\
        = &
        \sum_{h = 0}^{H-1} \sum_{s_h \in S^c_h} \sum_{a_h \in A}
        q_h(s_h, a_h | \pi(c;\cdot),P^c) \cdot 
        \left( \hat{f}_t (c,s_h,a_h) + \frac{\beta_t + H|S|\gamma_t}{p_{min}\sum_{i=1}^{t-1}\I[a_h = \pi_i(c;s_h)] }\right)
        \\
        \tag{Since $\pi$ is a deterministic policy}
        = &
        \sum_{h = 0}^{H-1} \sum_{s_h \in S^c_h} \sum_{a_h \in A}
        q_h(s_h|  \pi(c;\cdot),P^c) \I[a_h = \pi(c;s_h)] \cdot 
        \left( \hat{f}_t (c,s_h,a_h) + \frac{\beta_t+ H|S|\gamma_t}{p_{min}\sum_{i=1}^{t-1}\I[a_h = \pi_i(c;s_h)] }\right)
        \\
        = &
        \sum_{h = 0}^{H-1} \sum_{s_h \in S^c_h} 
        q_h(s_h| \pi(c;\cdot),P^c) \cdot 
        \left( \hat{f}_t (c, s_h,  \pi(c;s_h)) + \frac{\beta_t + H|S|\gamma_t}{p_{min}\sum_{i=1}^{t-1}\I[ \pi(c;s_h) = \pi_i(c;s_h)] }\right)
        \\
        = &
        V^{\pi(c;\cdot)}_{\M^{(\hat{f}_t,P)}(c)}(s_0) +  \sum_{h=0}^{H-1} \sum_{s_h \in S^c_h} \frac{ q_h(s_h|\pi(c;\cdot),P^c) \cdot (\beta_t + H|S|\gamma_t)}{  p_{min}\sum_{i=1}^{t-1} \I[\pi(c;s_h)= \pi_i(c;s_h)]}
        .
    \end{align*}
    \endgroup
    Clearly, the implied identity holds when taking the expectation on both sides of the equation.
\end{observation}

\begin{proof}
Assume the good events hold and consider the following derivation.
    \begingroup
    \allowdisplaybreaks
    \begin{align*}
        \E_c[V^{\pi^\star(c;\cdot)}_{\M(c)}(s_0)]
        = &
        \E_c[V^{\pi^\star(c;\cdot)}_{\M^{(f_\star,P_\star)}(c)}(s_0)]
        \\
        \tag{By~\cref{lemma:sq-trick-UCDD}, item $(1)$}
        \leq &
        \E_c[V^{\pi^\star(c;\cdot)}_{\M^{(f_\star,\widehat{P}_t)}(c)}(s_0)] 
        +  
        \E_c\left[             
        \sum_{h = 0}^{H-1} \sum_{s_h \in S^c_h} 
        \frac{H|S|\gamma_t \cdot q_h(s_h| \pi^\star(c;\cdot), \widehat{P}^c_t)}{  p_{min}\sum_{i=1}^{t-1}\I[ \pi^\star(c;s_h) = \pi_i(c;s_h)] } \right]
        + 
        \gamma_t \frac{H^2|S|^2|A|}{t}
        \\
        \tag{By~\cref{lemma:sq-trick-rewards-UCDD}}
        \leq &
        \E_c[V^{\pi^\star(c;\cdot)}_{\M^{(\hat{f}_t,\widehat{P}_t)}(c)}(s_0)] 
        + 
        \E_c \left[ \sum_{h=0}^{H-1} \sum_{s_h \in S^c_h} \frac{H  |S| \gamma_t \cdot q_h(s_h | \pi^\star(c;\cdot),\widehat{P}^c_\star)}{p_{min}\sum_{i=1}^{t-1} \I[\pi^\star(c;s_h)= \pi_i(c;s_h)]}\right]
        % \\
        % & 
        + 
        \gamma_t \frac{H^2 |S|^2  |A|}{t}
        \\
        & +
        \E_c \left[ \sum_{h=0}^{H-1} \sum_{s_h \in S^c_h} \frac{\beta_t  \cdot q_h(s_h | \pi^\star(c;\cdot),\widehat{P}^c_\star)}{p_{min}\sum_{i=1}^{t-1} \I[\pi^\star(c;s_h)= \pi_i(c;s_h)]}\right]
        + 
        \beta_t \frac{H|S||A|}{t}
        \\
        \tag{By~\cref{obs:value-decomp-UCDD}}
        \\
        = &
        \E_c[V^{\pi^\star(c;\cdot)}_{\M^{(\widehat{r}_t,\widehat{P}_t)}(c)}(s_0)] 
        + 
        \gamma_t \frac{H^2  |S|^2 |A|}{t}
        + 
        \beta_t \frac{H|S| |A|}{t}
        \\
        \tag{By $\Mhat_t$ definition}
        = &
        \E_c[V^{\pi^\star(c;\cdot)}_{\Mhat_t(c)}(s_0)] 
        + 
        \gamma_t \frac{H^2  |S|^2 |A|}{t}
        + 
        \beta_t \frac{H|S| |A|}{t}
        \\
        \tag{$\pi_t(c; \cdot)$ is the optimal policy of $\Mhat_t(c)$, for all $c \in \C$}
        \leq &
        \E_c[V^{\pi_t(c;\cdot)}_{\Mhat_t(c)}(s_0)] 
        + 
        \gamma_t \frac{H^2  |S|^2 |A|}{t}
        + 
        \beta_t \frac{H|S||A|}{t},
    \end{align*}
    \endgroup
    as the lemma states.
\end{proof}

Recall the regret, which defined as
$
   \Regrv_T(\text{ALG}) 
    :=
    \sum_{t=1}^T V^{\pi^\star(c_t;\cdot)}_{\M(c_t)}
    -
    V^{\pi_t(c_t;\cdot)}_{\M(c_t)}
$. The following theorem establish our main result.
\begin{theorem}[regret bound]\label{thm:regret-bound-ucdd}
    For any $T \geq 1$, finite functions classes $\F$ and $\Fp$ and $\delta \in (0,1)$, with probability at least $1-\delta$ it holds that
    \begin{align*}
        \Regrv_T(RM-UCDD)
        \leq 
        \tilde{O}\left( 
        (H+ 1/p_{min}) \cdot \left({H|S|^{3/2}\sqrt{|A|T\log\frac{|\Fp|}{\delta}}}+{\sqrt{T|S||A|\log\frac{|\F|}{\delta}}}
        \right)
        +|A|H
        \right),
    \end{align*} 
    for the choice in $\beta_t = \sqrt{ \frac{17 t \log(8|\F|t^3/\delta) }{|S||A|}}$ and  $\gamma_t = \sqrt{ \frac{18 t \log(8|\Fp|t^3/\delta) }{|S| |A|}}$ for all $t \in [T]$.
\end{theorem}

\begin{proof}
    We derive a regret bound under the good events of~\cref{lemma:CB-policy-UCDD,lemma:CB-policy-dynamics}. Both events hold with probability at least $1-{\delta}/{2}$.
     
    Consider the martingale difference sequence $\{Y_t\}_{t=1}^T$ and the filtration $\{\Hist_{t}\}_{t=1}^T$ where 
    $$
        Y_t 
        := 
        V^{\pi^{\star}(c_t,\cdot)}_{\M(c_t)}(s_0)  - V^{\pi_t(c_t,\cdot)}_{\M(c_t)}(s_0)
        -
        %\E_{\Hist_{t-1}}\left[ 
        \E_{c_t}\left[ V^{\pi^{\star}(c_t,\cdot)}_{\M(c_t)}(s_0)  - V^{\pi_t(c_t,\cdot)}_{\M(c_t)}(s_0) \Big|\Hist_{t-1}\right] %\right]
        .
    $$
    %and $\Hist_0$ is the empty history.
    Clearly, for all $t$, $|Y_t|\leq 2H$, $Y_t$ is determined completely by the histories $\Hist_1, \ldots ,\Hist_t$ and $\E_{c_t}\left[ Y_t |\Hist_{t-1}\right] = 0$.
    Hence, by Azuma's inequality,
    with probability at least $1-\delta/2$ it holds that
    \begingroup
    \allowdisplaybreaks
    \begin{align*}
        &\Regrv_T( RM-UCDD)
        \\
        = &
        \sum_{t = 1}^T V^{\pi^\star(c_t;\cdot)}_{\M(c_t)}(s_0) - V^{\pi_t(c_t;\cdot)}_{\M(c_t)}(s_0) 
        \\
        \tag{By Azuma's inequality, holds w.p. at least $1-\delta/2$.}
        \leq &
        \sum_{t =  1}^T \E_{c_t}[V^{\pi^\star(c_t;\cdot)}_{\M(c_t)}(s_0)|\Hist_{t-1}] - \E_{c_t}[V^{\pi_t(c_t;\cdot)}_{\M(c_t)}(s_0)|\Hist_{t-1}]
        + 
        2H\sqrt{2T\log(4/\delta)} 
        \\
        \tag{Since $\pi_t$ is determined by $\Hist_{t-1}$, and $c_t$ and $\pi^\star$ are independent of the history, we can omit the conditioning on $\Hist_{t-1}$}
        = &
        \sum_{t = 1}^T \E_c[V^{\pi^\star(c;\cdot)}_{\M(c)}(s_0)] - \E_c[V^{\pi_t(c;\cdot)}_{\M(c)}(s_0)]
        + 
        2H\sqrt{2T\log(4/\delta)} 
        \\
        \leq &
        \sum_{t = |A| + 1}^T \E_c[V^{\pi^\star(c;\cdot)}_{\M(c)}(s_0)] - \E_c[V^{\pi_t(c;\cdot)}_{\M(c)}(s_0)]
        + 
        2H\sqrt{2T\log(4/\delta)} +|A|H
        \\
        \tag{By the optimism lemma (\cref{lemma:optimism-UCDD})}
        \leq &
        \sum_{t = |A| + 1}^T 
        \E_c[V^{\pi_t(c;\cdot)}_{\Mhat_t(c)}(s_0)] - \E_c[V^{\pi_t(c;\cdot)}_{\M(c)}(s_0)]
        \\
        & +        
        \sum_{t = |A| + 1}^T\gamma_t \frac{H^2  |S|^2  |A|}{t}
        \\
        & +
        \sum_{t = |A| + 1}^T\beta_t \frac{H |S| |A|}{t} 
        \\
        & + 
        2H\sqrt{2T\log(4/\delta)} +|A|H
        \\
        \tag{By adding and subtracting $\E_c[V^{\pi_t(c;\cdot)}_{\M^{(f_\star,\widehat{P}_t)}(c)}(s_0)]$}
        = &
        \sum_{t =|A| + 1}^T
        \E_c[V^{\pi_t(c;\cdot)}_{\Mhat_t(c)}(s_0)] 
        - 
        \E_c[V^{\pi_t(c;\cdot)}_{\M^{(f_\star,\widehat{P}_t)}(c)}(s_0)]
        +
        \sum_{t =|A| + 1}^T
        \E_c[V^{\pi_t(c;\cdot)}_{\M^{(f_\star,\widehat{P}_t)}(c)}(s_0)]
        - 
        \E_c[V^{\pi_t(c;\cdot)}_{\M(c)}(s_0)]
        \\
        & +
        \sum_{t =|A| + 1}^T \gamma_t \frac{H^2  |S|^2  |A|}{t}
        \\
        & +
        \sum_{t =|A| + 1}^T \beta_t \frac{H |S| |A|}{t}
        \\
        & + 
        2H\sqrt{2T\log(4/\delta)} +|A|H
        \\
        \tag{By~\cref{obs:value-decomp-UCDD} applied for $\pi_t$}
        = &
        \sum_{t =|A| + 1}^T 
        \E_c[V^{\pi_t(c;\cdot)}_{\M^{(\hat{f}_t, \widehat{P}_t)}(c)}(s_0)] 
        - 
        \E_c[V^{\pi_t(c;\cdot)}_{\M^{(f_\star,\widehat{P}_t)}(c)}(s_0)]
        \\
        & +
        \sum_{t =|A| + 1}^T
        (\beta_t + H|S|\gamma_t) \E_c \left[ \sum_{h=0}^{H-1} \sum_{s_h \in S^c_h} \frac{q_h(s_h|\pi_t(c;\cdot) ,\widehat{P}^c_t)}{ p_{min}\sum_{i=1}^{t-1} \I[\pi_t(c;s_h)= \pi_i(c;s_h)]}\right]
        \\
        & +
        \sum_{t =|A| + 1}^T
        \E_c[V^{\pi_t(c;\cdot)}_{\M^{(f_\star,\widehat{P}_t)}(c)}(s_0)]
        - 
        \E_c[V^{\pi_t(c;\cdot)}_{\M(c)}(s_0)]
        \\
        & +
        \sum_{t =|A| + 1}^T \gamma_t \frac{H^2 |S|^2  |A|}{t}
        \\
        & +
        \sum_{t =|A| + 1}^T \beta_t\frac{ H  |S| |A|}{t}
        \\
        & + 
        2H\sqrt{2T\log(4/\delta)} +|A|H
        \\
        \tag{By~\cref{lemma:sq-trick-rewards-UCDD}}
        \leq &
        2\sum_{t =|A| + 1}^T \beta_t
        \E_c \left[ \sum_{h=0}^{H-1} \sum_{s_h \in S^c_h} \frac{q_h(s_h|\pi_t(c;\cdot) ,\widehat{P}^c_t)}{ p_{min}\sum_{i=1}^{t-1} \I[\pi_t(c;s_h)= \pi_i(c;s_h)]}\right]
        + 
        2\sum_{t =|A| + 1}^T \beta_t\frac{ H  |S| |A|}{t}
        \\
        \tag{By~\cref{lemma:sq-trick-UCDD}, item $(2)$}
        & +
        2\sum_{t =|A| + 1}^T
        H|S|\gamma_t \E_c \left[ \sum_{h=0}^{H-1} \sum_{s_h \in S^c_h} \frac{q_h(s_h|\pi_t(c;\cdot) ,\widehat{P}^c_t)}{ p_{min}\sum_{i=1}^{t-1} \I[\pi_t(c;s_h)= \pi_i(c;s_h)]}\right]
         +
         2\sum_{t =|A| + 1}^T \gamma_t \frac{H^2 |S|^2  |A|}{t}
        \\
        & + 
        2H\sqrt{2T\log(4/\delta)} +|A|H
        \\
        \tag{Since for all $t \in [T]$, $\beta_t \leq \beta_T$ and $\gamma_t \leq \gamma_T$}
        \leq &
        2(\beta_T + H |S|\gamma_T )  \sum_{t =|A| + 1}^T
        \E_c \left[\sum_{h=0}^{H-1} \sum_{s_h \in S^c_h} \frac{q_h(s_h|\pi_t(c;\cdot) ,\widehat{P}^c_t)}{ p_{min}\sum_{i=1}^{t-1} \I[\pi_t(c;s_h)= \pi_i(c;s_h)]}\right]
        \\
        & + 
        2\sum_{t =|A| + 1}^T \beta_t\frac{ H  |S| |A|}{t}
        \\
        & +
         2\sum_{t =|A| + 1}^T \gamma_t \frac{H^2 |S|^2  |A|}{t}
        \\
        & + 
        2H\sqrt{2T\log(4/\delta)} +|A|H
        \\
        \tag{By~\cref{lemma:Contextual-Potential-dynamics}}
        \leq &
        2(\beta_T + H |S|\gamma_T ) 
        \frac{|S||A|}{p_{min}}(1+ \log(T/|A|))
        \\
        \tag{Since  $\beta_t = \sqrt{ \frac{17 t \log(8|\F|t^3/\delta) }{|S||A|}}$}
        & + 
        2\sum_{t =|A| + 1}^T \sqrt{ \frac{17  \log(8|\F|t^3/\delta) t}{|S||A|}} \frac{ H  |S| |A|}{t}
        \\
        \tag{Since $\gamma_t = \sqrt{ \frac{18 t \log(8|\Fp|t^3/\delta) }{|S| |A|}}$}
        & +
         2\sum_{t =|A| + 1}^T  \sqrt{ \frac{18 t \log(8|\Fp|t^3/\delta) }{|S| |A|}} \frac{H^2 |S|^2  |A|}{t}
        \\
        & + 
        2H\sqrt{2T\log(4/\delta)} +|A|H
        \\
        \tag{By $\beta_T$ and $\gamma_T$ choice}
        \leq &
        2\left(\sqrt{ \frac{17 T \log(8|\F|T^3/\delta) }{|S||A|}}+ H |S|\sqrt{ \frac{18 T \log(8|\Fp|t^3/\delta) }{|S| |A|}}\right) 
        \frac{|S||A|}{p_{min}}(1+ \log(T/|A|))
        \\
        & + 
        2H \sqrt{ 17 \log(8|\F|T^3/\delta)|S||A| }\sum_{t = 1}^T  \frac{1}{\sqrt{t}}
        \\
        & +
         2H^2 |S|^{3/2}  \sqrt{18  \log(8|\Fp|T^3/\delta)|A| } \sum_{t = 1}^T  \frac{1}{\sqrt{t}}
        \\
        & + 
        2H\sqrt{2T\log(4/\delta)} +|A|H
        \\
        \leq &
        \frac{2}{p_{min}}\sqrt{17 T |S||A| \log(8|\F|T^3/\delta) }(1+ \log(T/|A|))
        \\
        & +
        \frac{2 H |S|^{3/2}}{p_{min}}
        \sqrt{ 18 T |A| \log(8|\Fp|t^3/\delta) }
       (1+ \log(T/|A|))
        \\
        \tag{$\sum_{t=1}^T \frac{1}{\sqrt{t}}\leq 2\sqrt{T}$}
        & + 
        4H \sqrt{ 17 \log(8|\F|T^3/\delta)T|S||A| }
        \\
        \tag{$\sum_{t=1}^T \frac{1}{\sqrt{t}}\leq 2\sqrt{T}$}
        & +
         4H^2 |S|^{3/2}  \sqrt{18  \log(8|\Fp|T^3/\delta)T|A| } 
        \\
        & + 
        2H\sqrt{2T\log(4/\delta)} +|A|H
        \\
        = &
        \widetilde{O}\Bigg( 
        (H+1/p_{min}) \cdot 
        \left(H|S|^{3/2}\sqrt{T|A|\log\frac{|\Fp|}{\delta}} + \sqrt{T|S||A|\log\frac{|\F|}{\delta}} \right) + |A|H
        \Bigg)
        .
    \end{align*}
    \endgroup
 Since the good events hold w.p. at least $1-\delta/2$, by union bound combined with Azuma's inequality we obtain the theorem.   
\end{proof}

\begin{corollary}[regret bound in terms of $\G$]\label{corl:regret-bound-ucdd-g}
   For every $T \geq 1$, finite functions class $\G$ ($\F = \G^S$) and $\Fp$ and $\delta \in (0,1)$ the following holds with probability at least $1-\delta$ for the same choice of parameters $\{\beta_t, \gamma_t\}_{t \in [T]}$.
    \begin{align*}
    \Regrv_T(RM-UCDD)\leq 
    \tilde{O}\left( 
        (H+1/p_{min})\cdot
            \left({H|S|^{3/2}\sqrt{T|A|\log\frac{|\Fp|}{\delta}}}+
            {|S|\sqrt{T|A|\log\frac{|\G|}{\delta}}} + |A|H
        \right)\right).
    \end{align*}
\end{corollary}
\begin{proof}
    Plug $\log(|\F|) = |S| \log(|\G|)$ in the bound of~\cref{thm:regret-bound-ucdd} and obtain the corollary.
    
\end{proof}

\newpage
\section{Lower Bound}\label{Appendix:LB}

We present a lower bound for layered CMDP using the lower bound for CMAB presented by~\citet{agarwal2012contextual}, in which $K=|A|$, $\G \subseteq \C \times A \to [0,1]$ and $N \in \N$.

\begin{theorem}[Theorem $5.1$,~\citet{agarwal2012contextual}]\label{thm:lower-bound-CMAB}
For every $N$ and $K$ such that $\ln N / \ln K \leq T$, and every algorithm $\mathfrak{A}$, there exist a functions class $\G %\subset \C \times A \to [0,1]
$ of 
cardinality at most $N$ and a distribution $D(c,r)$  for which
the realizability assumption holds, but the expected regret of $\mathfrak{A}$ is $\Omega(\sqrt{KT \ln  N/ \ln K})$.
\end{theorem}

% \begin{theorem}[Theorem $5.1$,~\cite{agarwal2012contextual}]\label{thm:lower-bound-CMAB}
% For every $N \in \N$ and $K = |A|$ such that $\ln N / \ln K \leq T$, and every algorithm $\mathfrak{A}$, there exist a functions class $\F \subset \C \times A \to [0,1]$ of 
% cardinality at most $N$ and a distribution $D(c,r)$  for which
% the realizability assumption holds, such that the expected regret of $\mathfrak{A}$ is $\Omega(\sqrt{KT \ln  N/ \ln K})$.
% \end{theorem}

We present a lower bound for \textbf{layered CMDP}, where the dynamics is known and context-independent. The rewards are context-dependent. Clearly, it implies a lower bound for the unknown dynamics setting. 
%A similar lower bound can be shown for unknown and context-dependent dynamics since one can always encode the reward in the transition probabilities function.
%We remark that similarly to the non-contextual MDP, in the non-layered case, an additional $H$ factor is expected.
%
%In addition
We remark that if the horizon length is $H=1$, since we assume that there is a unique start state and the CMDP is layered, the lower bound for CMAB stated in~\cref{thm:lower-bound-CMAB} holds. In the following theorem we show a lower bound for horizon $H \geq 2$.

\usetikzlibrary{positioning,angles,quotes}
\begin{figure}
    \centering
    \caption{Lower bound illustration}\label{fig:lower-bound}
  \begin{tikzpicture}[auto,node distance=8mm,>=latex,font=\small]
    \tikzstyle{round}=[thick,draw=black,circle]
    %\node[below right=0mm and 5mm of s0] (c) {$c$}
    \node[round] (s0) {$s_0$};
    \node[above right=6mm and -10mm of s0] (c) {$c$};
    \node[round,above right=0mm and 15mm of s0] (s1) {$s^1_1$};
    \node[round,below right=0mm and 15mm of s0] (s2) {$s^1_2$};
    \node[below right=12mm and 15mm of s0] (dots1) {$\ldots$};
    \node[round,below right=20mm and 15mm of s0] (ss1) {$s^1_{M}$};
    
    \draw[->] (c) -- (s0);
    \draw[->] (s0) to node [above left] {$\frac{1}{M}$} (s1);
    \draw[->] (s0) to node [above left] {$\frac{1}{M}$} (s2);
    %\draw[->] (s0) to node  {$\frac{1}{M}$} (dots1);
    \draw[->] (s0) to node [above left] {$\frac{1}{M}$} (ss1);
    
    %\layer 2
    \node[round,above right=4mm and 30mm of s0] (s12) {$s^2_1$};
    \node[round,below right=0mm and 30mm of s0] (s22) {$s^2_2$};
    \node[below right=12mm and 30mm of s0] (dots2) {$\ldots$};
    \node[round,below right=20mm and 30mm of s0] (ss2) {$s^2_{M}$};
    
    %\draw[->] (c) -- (s0);
    %\draw[->] (s1) -- (s12);
    \draw[->] (s1) to node {$1$} (s12);
    \draw[->] (s2) to node {$1$} (s22);
    %\draw[->] (dots1) to node {$1$} (dots2);
    \draw[->] (ss1) to node {$1$} (ss2);
    
    %\layer dots
    \node[above right=6mm and 40mm of s0] (s12d) {$\ldots$};
    \node[below right=2mm and 40mm of s0] (s22d) {$\ldots$};
    \node[below right=12mm and 40mm of s0] (dots2d) {$\ldots$};
    \node[below right=22mm and 42mm of s0] (ss2d) {$\ldots$};
    
    \draw[->] (s12) -- (s12d);
    \draw[->] (s22) -- (s22d);
    %\draw[->] (dots2) -- (dots2d);
    \draw[->] (ss2) -- (ss2d);
    
    %\layer H
    \node[round,above right=4mm and 50mm of s0] (s1H) {$s^{H-1}_1$};
    \node[round,below right=-1mm and 50mm of s0] (s2H) {$s^{H-1}_2$};
    \node[below right=12mm and 50mm of s0] (dotsH) {$\ldots$};
    \node[round,below right=20mm and 51mm of s0] (ssH) {$s^{H-1}_{M}$};
    
    \draw[->] (s12d) -- (s1H);
    \draw[->] (s22d) -- (s2H);
    %\draw[->] (dots2d) -- (dotsH);
    \draw[->] (ss2d) -- (ssH);
    
    \node[round,above right=-5mm and 70mm of s0] (sH) {$s_{H}$};
    
    \draw[->] (s1H) to node {$1$} (sH);
    \draw[->] (s2H) to node {$1$} (sH);
    %\draw[->] (dotsH) -- (sH);
    \draw[->] (ssH) to node {$1$} (sH);
    
    % \Loop[dist=1cm,dir=SOEA](s1H)
    % \Loop[dist=1cm,dir=SOEA](s2H)
    % \Loop[dist=1cm,dir=SOEA](ssH)
\end{tikzpicture}
\end{figure}
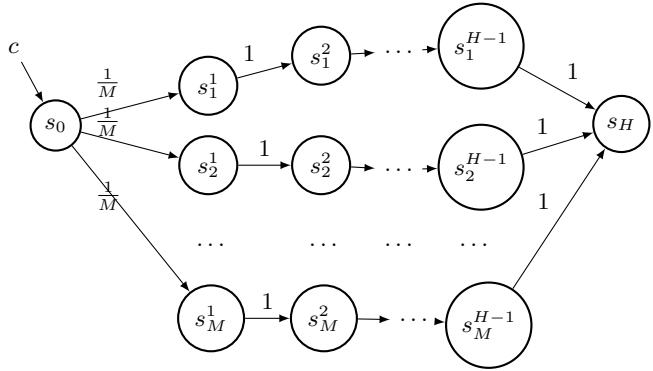

\begin{theorem}[lower bound for CMDP]\label{thm:LB-CMDP}
    Let $\delta \in (0,1)$, horizon $H \geq 2$ and ${M, N \in \N}$. 
    Let ${T \geq {8 M \log \frac{|S|}{\delta}+ 2M\ln N / \ln |A|}}$ and consider a CMDP $(\C,S,A, \M)$ for which ${|S| = M \cdot (H-1) + 2}$.

    Then, for any algorithm $\mathfrak{A}$, there exist a  base function class $\G \subseteq (\C \times A \to [0,1])$ of cardinality at most $N$ and a distribution $D(c,s,a,r)$ for which the realizability assumption holds for $\F = \G^S$, and with probability at least $1-\delta$, the expected regret of $\mathfrak{A}$ is 
    %$\Omega(\sqrt{TH|S||A|\ln|\F|/ \ln|S||A|})$.
    $\Omega \left(
    \sqrt{T  H |S| |A| \ln(N)/\ln(|A|)} \right)$.
\end{theorem}

\begin{proof}
Let $\delta \in (0,1)$, horizon $H \geq 2$, $N, M \in \N$, and $T$ that satisfy the requirements stated in the theorem.

Consider the following layered CMDP, $(\C,S,A,\M)$, where $\C \subseteq \R^d$, $\M$ maps a context $c \in \C$ to the MDP ${\M(c) = (S,A,P^c_\star,r^c_\star, s_0, H)}$ that is defined as follows.\\
$S = \{S_0, S_1, \ldots, S_{H}\}$ is a layered states space contains $M \cdot (H-1) +2$ states for which $S_0 = \{s_0\}$ and $S_H = \{s_H\}$, meaning, $s_0$ and $s_H$ are unique start and final states, respectively. 
In addition, for all $i \in \{1, 2, \ldots, H-1 \}$ we have $S_i = \{s^i_1, s^i_2, \ldots ,s^i_{M}\}$, meaning, each layer $i \in \{1,\dots,H-1\}$ contains $M \geq 1$ states and the layers are disjoint.

Let $A = \{a_1, \ldots, a_K\}$ be a set of $K$ actions.
We define a context-independent dynamics $P^c_\star = P$ for every context $c \in \mathcal{C}$, as follows.
\begingroup
\allowdisplaybreaks
\begin{align*}
    &P(s|s_0, a) = \frac{1}{M} \cdot \I[s \in S_1], \;\; \forall a \in A.\\
    &P(s|s^i_j, a) = \I[s = s^{i+1}_j] , \;\; \forall a \in A,\; i \in \{1,\ldots, H-2\},\; j\in  \{1,\ldots,M\}.\\
    &P(s|s^{H-1}_j, a) = \I[s = s_H],\; \forall a \in A,\; j \in \{1,\ldots, M\}.
\end{align*}
\endgroup
For illustration, see~\cref{fig:lower-bound}.

We assume that the dynamics $P$ is known to the learner.
$r^c_\star$ is an unknown context-dependent expected rewards function.

\begin{claim}\label{remark:LB-constuction}
    Minimizing regret in the above CMDP is equivalent to minimizing regret in $|S|-1$ CMAB instances.
\end{claim}

\begin{proof}
    The start state $s_0$ is unique and identical for all of the contexts, and, any choice of an action will move the agent to a uniformly at random chosen state in layer $1$.
    \\
    In layers $h=1, 2, \ldots H-1$ the dynamics is deterministic. In any state $s\in S_h$, for any choice of an action $a \in A$ the agent will move to the exactly same state, regardless of the context.
    \\
    That is, to maximize the cumulative reward, the agent needs to play in each state the arm with the maximal immediate expected reward. 
    %the planing in such a CMDP translates to choosing the optimal action with respect to the immediate expected reward in each state. 
    %(i.e., for each state, $H$-steps planing is not needed!).
    In addition,
    %the start state $s_0$ is unique and identical for all of the contexts, and 
    the final state $s_H$ has zero rewards.
    Hence, since the dynamics is known to the learner, minimizing regret in this CMDP is equivalent to minimizing regret in $|S|-1$ CMAB instances. 
\end{proof}

%Since we assume the dynamics is known and context-independent and 
By all the above and since the horizon is $H$, as explain in~\cref{remark:LB-constuction},
%to solve 
 solving the above CMDP instance
 %we need to solve 
 is equivalent to solving $|S|-1 = M \cdot (H-1) + 1$ CMAB instances, 
%since we need to solve 
one 
%CMAB problem
 for each state $s \in S\setminus \{s_H\}$.
% for every action $a \in A$.

By~\cref{thm:lower-bound-CMAB} for any algorithm
%and state $s \in S\setminus \{s_H\}$ 
there exists a realizable functions class
$\mathcal{G} \subseteq (\C \times A \to [0,1])$ for which $|\mathcal{G}| \leq N$ and the expected regret of the algorithm on the appropriate CMAB problem
%defined by $s$
is $\Omega(\sqrt{T |A| \ln  N/ \ln |A|})$.
Hence, for any algorithm 
$\mathfrak{A}$, 
consider the function class  
$
    \F = \G^S \subseteq (\C \times S 
    %\setminus \{s_H\})
    \times A \to [0,1])
$, 
where $\mathcal{G}$ is the ``hard'' function class for the CMAB problem. 
\begin{remark}
    While the function class $\G$ is the same for all states, the true rewards function of an individual state $s$ ,might be different. Meaning, for two states $s \neq s'$, the true reward functions are $g^\star_s, g^\star_{s'} \in \G$ (respectively) are not necessarily identical. Hence, the CMAB instances defined by $s$ and $s'$ are different, and uncorrelated.
\end{remark}
% defined by a state $s \in S \setminus \{s_H\}$. 
% Since for all $s \in S \setminus \{s_H\}$ it holds that $|\mathcal{G}_s| \leq N$, we obtain that $|\F| \leq N\cdot|S| 
% %\leq N \cdot M \cdot H
% $.

Consider a $T$ rounds run of any algorithm $\mathfrak{A}$.
For every state $s \in S$ and round $t \in \{1,2,\ldots, T\}$, let $X^t_s$ be a Bernoulli random variable which indicates whether state $s$ was visited in round $t$.
(In other words, it indicates whether $s$ appears in the $t$'th trajectory generated by the algorithm).
By our construction, for all $s \in S \setminus \{s_0, s_H\}$ and $t \in \{1,2,\ldots, T\}$ we have that $\mathbb{P}[X^t_s = 1] = \frac{1}{M}$ and thus $\E[X^t_s] = \frac{1}{M}$. 
For all $s \in S$ let $X_s = \sum_{t=1}^T X^t_s$. It holds that $\E[X_s] = \frac{T}{M}$ for all  $s \in S \setminus \{s_0, s_H\}$. 
For the start and final states $s_0$ and $s_H$ we have that $X_{s_0} = X_{s_H} =T$ with probability $1$.

Hence, by multiplicative Chernoff bound, for all $s \in S \setminus \{s_0,s_H\}$ it holds that 
\begin{align*}
    \Pr\left[X_s \leq \frac{1}{2}\cdot \frac{T}{M}\right]
    \leq \exp{\left(\frac{-T/M \cdot 0.25}{2}\right)}
    =
    \exp{\left(\frac{-T}{8M}\right)}.
\end{align*}
Thus, by union bound
\begin{align*}
    \Pr\left[\exists s \in S :\; X_s \leq \frac{T}{2 M}\right]
    \leq 
    \sum_{s \in S} \Pr\left[X_s \leq  \frac{T}{2M}\right]
    \leq
    |S|\exp{\left(\frac{-T}{8M}\right)},
\end{align*}
which implies that
\begin{equation*}
    \Pr \left[\forall s \in S:\; X_s \geq  \frac{T}{2M}\right]
    \geq 1- |S|\exp{\left(\frac{-T}{8M}\right)}.
\end{equation*}
Hence, for $T \geq  8M \cdot \log \frac{|S|}{\delta}$ it holds that
\begin{equation}\label{eq:prob-visits}
    \Pr\left[\forall s \in S:\; X_s \geq  \frac{T}{2M}\right]
    \geq 1- \delta.
\end{equation}
Recall we demanded ${T \geq {8 M \log \frac{|S|}{\delta}+ 2M\ln N / \ln |A|}}$. 
Let $\Reg(X,s)$ denote the expected regret of the CMAB instance defined by state $s$ given that it was visited $X$ times.

By all the above, we have for any algorithm $\mathfrak{A}$ with probability at least $1-\delta$ that 

\begingroup
\allowdisplaybreaks
\begin{align*}
    \Reg_T(\mathfrak{A})
    & =
    \sum_{s \in S \setminus \{s_H\}}\E_{X_s}\left[\;\Reg(X_s, s)\;\right]
    \\
    \tag{Holds w.p. at least $1-\delta$, by~\cref{eq:prob-visits}, since the expected regret is non-decreasing}
    & \geq
    %\E\left[
    \sum_{s \in  S \setminus \{s_H\}}\Reg\left(\frac{T}{2M}, s\right)
    %\right]
    \\
    \tag{By~\cref{thm:lower-bound-CMAB}, since $\frac{T}{2M} \geq \ln N / \ln |A|$}
    % & \geq 
    % \Omega \left(
    % \E\left[\sum_{s \in S} \sqrt{X_s \cdot |A| \ln(N)/\ln(|A|)} \right] \right)
    % \\
    % \tag{Holds w.p. at least $1-\delta$}
    & \geq
    \Omega \left(
    \sum_{s \in  S \setminus \{s_H\}} \sqrt{\frac{T}{M} \cdot |A| \ln(N)/\ln(|A|)} \right)
    \\
    & =
    \Omega \left(
    |S| \sqrt{\frac{T}{M} \cdot |A| \ln(N)/\ln(|A|)} \right)
    \\
    \tag{Since $M = \frac{|S|-2}{H-1}$}
    & =
    \Omega \left(
    |S| \sqrt{\frac{T H}{|S|} \cdot |A| \ln(N)/\ln(|A|)} \right)
    \\
    &=
    \Omega \left(
    \sqrt{T  H  |S|  |A| \ln(N)/\ln(|A|)} \right)
    ,
    % \\
    % \tag{Since $N \geq |\F|/|S|$}
    % &\geq
    % \Omega \left(
    % \sqrt{T  H |S| |A| \ln(|\F|/|S|)/\ln(|A|)} \right).
\end{align*}
\endgroup
% Here, the last inequality holds since $|F| \leq N |S|$, hence,
% \begin{align*}
%     \frac{\ln{N}}{\ln{|A|}} + 1
%     \geq
%      \frac{\ln{N}}{\ln{|S|} + \ln{|A|}} + \frac{\ln{|S|}}{\ln{|S|} + \ln{|A|}}
%      =
%      \frac{\ln{N |S|}}{\ln{|S||A|}}
%      \geq 
%      \frac{\ln{|\F|}}{\ln{|S||A|}}
%      .
% \end{align*}
% The above lower bound holds since we solve $M\cdot (H-1) +1$ CMAB problem, one for each state, where the probability to visit each state $s \in S \setminus \{s_0\}$ is $T/M$ in expectation. For each CMAB problem we have a lower bound of $\Omega \left( \sqrt{ T  |A| \ln  N/ \ln |A|} \right)$ on the expected regret.
%The above holds with probability at least $1-\delta$.
as stated.
\end{proof}

% \begin{corollary}
%     Under the conditions of~\cref{thm:LB-CMDP},
%     for any $\delta \in (0, 1/2]$,
%     the expected regret is lower bounded by
%     $\Omega(\sqrt{TH|S||A|\ln(N)/ \ln(|A|)})$ with probability $1$.
% \end{corollary}

% \begin{proof}
%     Take any $\delta \in (0, 1/2]$ and consider the results of~\cref{thm:LB-CMDP}.
    
%     Let $G$ denote the good event in which every state $s$ was visited at least $\frac{T}{2M}$ times. In the proof of~\cref{thm:LB-CMDP} we showed that $\Prob[G] \geq 1-\delta$.
%     Hence, the corollary follows by total expectation low when conditioning on the event $G$.
    
%     % Follows by total expectation when conditioning on the following event $G$. Here,
%     % $G$ is the good event in which every state $s$ was visited at least $\frac{T}{2M}$ times. In the proof of~\cref{thm:LB-CMDP} we showed that $\Prob[G] \geq 1-\delta$. Hence, the corollary follows. 
% \end{proof}

\section{Auxiliary Lemmas}

%\Orin{to validate we use those lemmas only for deterministic policies, or extend to }

\begin{lemma}[value-difference, Corollary $1$ of~\citet{efroni2020optimistic}]\label{lemma:val-diff-efroni}
    Let $M$, $M'$ be any $H$-finite horizon MDP. Then, for any two policies $\pi$, $\pi'$ the following holds
    \begin{align*}
        V^{\pi,M}_1(s) -  V^{\pi',M'}_1(s) =&
        \\
        =&
        \sum_{h=1}^{H-1} \E \left[ \langle  Q^{\pi,M}_h(s_h, \cdot) , \pi_h(\cdot|s_h) - \pi'_h(\cdot|s_h) \rangle |s_1 = s, \pi',M' \right]
        \\
        & +
        \sum_{h=1}^{H-1} \E \left[ 
        c_h(s_h,a_h) -  c'_h(s_h,a_h)
        + (p_h(\cdot|s_h,a_h)-p'_h(\cdot|s_h,a_h)) V^{\pi,M}_{h+1}
        |s_h = s, \pi',M' \right] .       
    \end{align*}
\end{lemma}

Bellow we present two additional \textbf{well-known} versions of the above value-difference lemma.
%whose induced by~\cref{lemma:val-diff} or its proof.
%Both of the versions are induced by the above lemma or the proof of it, showed in \cite{efroni2020optimistic}.
We present them mainly for our convenience in using them, and provide proofs only for completeness.

\begin{lemma}[value difference where the dynamics is $P$]\label{lemma:val-diff-simple-P}
    Let $\pi$ be a deterministic policy. Let $M = (S,A,P,r,s_0,H)$ and $\overline{M} = (S,A,\overline{P},\overline{r},s_0, H)$ be two $H$-finite horizon MDPs. Then, for any $s \in S$ and $h \in [H-1]$ it holds that
    \begin{align*}
        &V^\pi_{M,h}(s) -  V^\pi_{\overline{M},h}(s) =\\
        &\E_{\pi, P}
        \left[ \sum_{h'=h}^{H-1} 
        \left(r(s_{h'}, a_{h'}) - \overline{r}(s_{h'}, a_{h'})    
        + \sum_{s' \in S}(P(s'|s_{h'}, a_{h'}) - \overline{P}(s'|s_{h'}, a_{h'}))V^\pi_{\overline{M},{h+1}(s')}\right) \Big| s_h = s \right].
    \end{align*}

\end{lemma}

\begin{remark}
    Since from the final state (i.e., the state at time $H$) there are no transitions or rewards, for completeness we define
    \[
        V^\pi_{M,H}(s) = V^\pi_{\overline{M},H}(s) = 0
        ,\;\; \forall s \in S.
    \]
\end{remark}

\begin{proof}
    %Fix Markov deterministic policy $\pi$.
    We prove the lemma by backwards induction on $h$.
    \\
    \underline{Base case:} $h=H-1$.
    By Bellman equations for every state $s \in S$ the following holds.
    \begingroup
    \allowdisplaybreaks
    \begin{align*}
        V^\pi_{M,H-1}(s) -  V^\pi_{\overline{M},H-1}(s)
        = &
        r(s, \pi(s)) + \underbrace{\mathop{\E}_{s' \sim P(\cdot | s, \pi(s))}\left[ V^\pi_{M,H}(s') \right]}_{=0} 
        \\
        & - 
        \left( \overline{r}(s, \pi(s)) + \underbrace{\mathop{\E}_{s' \sim \overline{P}(\cdot | s, \pi(s))}\left[ V^\pi_{\overline{M},H}(s') \right]}_{=0} \right)
        \\
        = &  
        r(s, \pi(s)) - \overline{r}(s, \pi(s))
        \\
        \underbrace{=}_{(1)} &
        \E_{\pi,P}[r(s_{H-1}, \pi(s_{H-1})) - \overline{r}(s_{H-1}, \pi(s_{H-1}))| s_{H-1} = s]
        \\
        \underbrace{=}_{(2)} &
        \E_{\pi,P}\Big[r(s_{H-1}, \pi(s_{H-1})) - \overline{r}(s_{H-1}, \pi(s_{H-1}))
        \\
        & + \sum_{s' \in S} (P(s'|s, \pi(s)) - \overline{P}(s'|s,\pi(s)))\cdot V^\pi_{\overline{M},H}(s') \Big| s_{H-1} = s \Big],
    \end{align*}
    \endgroup
    where identity $(1)$ is since given that $s_{H-1} = s$, the expectation over $\pi$ and $P$ has no effect on both terms. $(2)$ is since $V^\pi_{\overline{M},H}(s) = 0$ for all $s' \in S$.
    \\
    \underline{Induction step:} we assume the induction hypothesis holds for all $k \in [h+1, H]$ and prove for $h$.
    Consider the following derivation for any state $s \in S$.
    \begingroup
    \allowdisplaybreaks
    \begin{align*}
        &V^\pi_{M,h}(s) -  V^\pi_{\overline{M},h}(s) 
        \\
        \underbrace{=}_{(1)} & 
        r(s, \pi(s)) + \mathop{\E}_{s' \sim P(\cdot | s, \pi(s))}\left[ V^\pi_{M,h+1}(s') \right]
        - \left( \overline{r}(s, \pi(s)) + \mathop{\E}_{s' \sim \overline{P}(\cdot | s, \pi(s))}\left[ V^\pi_{\overline{M},h+1}(s') \right] \right)
        \\
        \underbrace{=}_{(2)} & 
        r(s, \pi(s)) -  \overline{r}(s, \pi(s)) 
        + 
        \mathop{\E}_{s' \sim P(\cdot | s, \pi(s))}\left[ V^\pi_{M,h+1}(s') \right]
        -
        \mathop{\E}_{s' \sim P(\cdot | s, \pi(s))}\left[ V^\pi_{\overline{M},h+1}(s') \right]
        \\
        & +
        \mathop{\E}_{s' \sim P(\cdot | s, \pi(s))}\left[ V^\pi_{\overline{M},h+1}(s') \right]
        -
        \mathop{\E}_{s' \sim \overline{P}(\cdot | s, \pi(s))}\left[ V^\pi_{\overline{M},h+1}(s') \right]
        \\
        \underbrace{=}_{(3)} & 
        r(s, \pi(s)) -  \overline{r}(s, \pi(s)) 
        + 
        \mathop{\E}_{s' \sim P(\cdot | s, \pi(s))}\left[ V^\pi_{M,h+1}(s') -  V^\pi_{\overline{M},h+1}(s') \right]
        \\
        & +
        \mathop{\E}_{s' \sim P(\cdot | s, \pi(s))}\left[ V^\pi_{\overline{M},h+1}(s') \right]
        -
        \mathop{\E}_{s' \sim \overline{P}(\cdot | s, \pi(s))}\left[ V^\pi_{\overline{M},h+1}(s') \right]
        \\
        \underbrace{=}_{(4)} &
        r(s, \pi(s)) -  \overline{r}(s, \pi(s)) 
         + 
        \mathop{\E}_{s' \sim P(\cdot | s, \pi(s))} \Big [\E_{\pi,P} \Big[ \sum_{h' = h+1} \Big(r(s_{h'}, a_{h'}) - \overline{r}(s_{h'}, a_{h'})  \\
        & + \sum_{s'' \in S}(P(s'' |s_{h'}, a_{h'}) - \overline{P}(s'' |s_{h'}, a_{h'}) V^\pi_{\overline{M},h'+1}(s'')\Big)\Big| s_{h+1} = s' \Big]\Big] \\
        & +\sum_{s'' \in S}(P(s'' |s, \pi(s)) - \overline{P}(s'' |s, \pi(s)) V^\pi_{\overline{M},h+1}(s'')
        \\
        \underbrace{=}_{(5)} &
         \mathop{\E}_{\pi,P} \left[ \sum_{h' = h}^{H-1} \left(r(s_{h'}, a_{h'}) - \overline{r}(s_{h'}, a_{h'})  + \sum_{s' \in S}\left(P(s' |s_{h'}, a_{h'}) - \overline{P}(s' |s_{h'}, a_{h'})\right) V^\pi_{\overline{M},h'+1}(s')\right) \Big| s_{h} = s \right].
    \end{align*}
    \endgroup
     Here, 
     $(1)$ is by Bellman equations for the value functions.
     $(2)$ is by adding and subtracting $\mathop{\E}_{s \sim P(\cdot | s, \pi(s))}\left[ V^\pi_{\overline{M},h+1}(s') \right] $ and re-organizing.
    $(3)$ is by linearity of expectation.
    $(4)$ is by the induction hypothesis.
    $(5)$ is since the expectation over $s' \sim P(\cdot| s,\pi(s))$ translates to the expectation induced by $\pi$ and $P$ when applied on $s_{h+1}$ given that $s_h = s$. Hence, given that $s_h = s$ and since $\pi$ is deterministic, taking expectation over $\pi$ and $P$ has no influence on both $r(s_{h}, a_{h})= r(s, \pi(s))$ and $\overline{r}(s_{h}, a_{h}) = \overline{r}(s, \pi(s))$.
\end{proof} 

% On the other hand we have a simpler version of~\cref{lemma:val-diff}:

\begin{lemma}[value difference where the dynamics is $\overline{P}$]\label{lemma:val-diff-simple-P-bar}
    Let $\pi$ be a deterministic policy. Let $M = (S,A,P,r,s_0,H)$ and $\overline{M} = (S,A,\overline{P},\overline{r},s_0,H)$ be two $H$-finite horizon MDPs. Then, for any $s \in S$ and $h \in [H-1]$ it holds that
    \begin{align*}
        &V^\pi_{M,h}(s) -  V^\pi_{\overline{M},h}(s) =\\
        &\E_{\pi, \overline{P}}
        \left[ \sum_{h'=h}^{H-1} 
        \left(r(s_{h'}, a_{h'}) - \overline{r}(s_{h'}, a_{h'})    + \sum_{s' \in S}(P(s'|s_{h'}, a_{h'}) - \overline{P}(s'|s_{h'}, a_{h'}))V^\pi_{M,{h+1}(s')}\right) \Big| s_h = s \right].
    \end{align*}
    % where for completeness we define
    % \[
    %     V^\pi_{M,H}(s) = V^\pi_{\overline{M},H}(s) = 0
    %     ,\;\; \forall s \in S.
    % \]
\end{lemma}

\begin{proof}
    %Fix Markov deterministic policy $\pi$.
    We prove the lemma by backwards induction on $h$.
    \\
    \underline{Base case:} $h=H-1$.
    By Bellman equations we have for every state $s \in S$: 
    \begingroup
    \allowdisplaybreaks
    \begin{align*}
        V^\pi_{M,H-1}(s) -  V^\pi_{\overline{M},H-1}(s)
        = &
        r(s, \pi(s)) + \underbrace{\mathop{\E}_{s' \sim P(\cdot | s, \pi(s))}\left[ V^\pi_{M,H}(s') \right]}_{=0} 
        \\
        & - 
        \left( \overline{r}(s, \pi(s)) + \underbrace{\mathop{\E}_{s' \sim \overline{P}(\cdot | s, \pi(s))}\left[ V^\pi_{\overline{M},H}(s') \right]}_{=0} \right)
        \\
        = &  
        r(s, \pi(s)) - \overline{r}_{H-1}(s, \pi(s))
        \\
        \underbrace{=}_{(1)} &
        \E_{\pi,\overline{P}}[r(s_{H-1}, \pi(s_{H-1})) - \overline{r}_{H-1}(s_{H-1}, \pi(s_{H-1}))| s_{H-1} = s]
        \\
        \underbrace{=}_{(2)} &
        \E_{\pi,\overline{P}}\Big[r(s_{H-1}, \pi(s_{H-1})) - \overline{r}_{H-1}(s_{H-1}, \pi(s_{H-1}))
        \\
        & + \sum_{s' \in S} (P(s'|s, \pi(s)) - \overline{P}(s'|s,\pi(s)))\cdot V^\pi_{M,H}(s') \Big| s_{H-1} = s \Big],
    \end{align*}
    \endgroup
    where identity $(1)$ is since given that $s_{H-1} = s$, the expectation over $\pi$ and $\overline{P}$ has no effect on both terms. $(2)$ is since $V^\pi_{M,H}(s) = 0$ for all $s' \in S$.
    \\
    \underline{Induction step:} we assume correctness for all $k \in [h+1, H]$ and prove for $h$.
    Consider the following derivation for any state $s \in S$.
    \begingroup
    \allowdisplaybreaks
    \begin{align*}
        &V^\pi_{M,h}(s) -  V^\pi_{\overline{M},h}(s) 
        \\
        \underbrace{=}_{(1)} & 
        r(s, \pi(s)) + \mathop{\E}_{s' \sim P(\cdot | s, \pi(s))}\left[ V^\pi_{M,h+1}(s') \right]
        - \left( \overline{r}(s, \pi(s)) + \mathop{\E}_{s' \sim \overline{P}(\cdot | s, \pi(s))}\left[ V^\pi_{\overline{M},h+1}(s') \right] \right)
        \\
        \underbrace{=}_{(2)} & 
        r(s, \pi(s)) -  \overline{r}(s, \pi(s)) 
        + 
        \mathop{\E}_{s' \sim P(\cdot | s, \pi(s))}\left[ V^\pi_{M,h+1}(s') \right]
        -
        \mathop{\E}_{s' \sim \overline{P}(\cdot | s, \pi(s))}\left[ V^\pi_{M,h+1}(s') \right]
        \\
        & +
        \mathop{\E}_{s' \sim \overline{P}(\cdot | s, \pi(s))}\left[ V^\pi_{M,h+1}(s') \right]
        -
        \mathop{\E}_{s' \sim \overline{P}(\cdot | s, \pi(s))}\left[ V^\pi_{\overline{M},h+1}(s') \right]
        \\
        \underbrace{=}_{(3)} & 
        r(s, \pi(s)) -  \overline{r}(s, \pi(s)) 
        + 
        \mathop{\E}_{s' \sim P(\cdot | s, \pi(s))}\left[ V^\pi_{M,h+1}(s') \right]
        -
        \mathop{\E}_{s' \sim \overline{P}(\cdot | s, \pi(s))}\left[ V^\pi_{M,h+1}(s') \right]
        \\
        & +
        \mathop{\E}_{s' \sim \overline{P}(\cdot | s, \pi(s))}\left[ V^\pi_{M,h+1}(s') -  V^\pi_{\overline{M},h+1}(s') \right]
        \\
        \underbrace{=}_{(4)} &
        r(s, \pi(s)) -  \overline{r}(s, \pi(s)) 
        +\sum_{s'' \in S}(P(s'' |s, \pi(s))- \overline{P}(s'' |s, \pi(s))\cdot V^\pi_{M,h+1}(s'')
        \\
         & +
         \mathop{\E}_{s' \sim \overline{P}(\cdot | s, \pi(s))} \Big [\E_{\pi,\overline{P}} \Big[ \sum_{h' = h+1} \Big(r(s_{h'}, a_{h'}) - \overline{r}(s_{h'}, a_{h'})  \\
        & + \sum_{s'' \in S}(P(s'' |s_{h'}, a_{h'}) - \overline{P}(s'' |s_{h'}, a_{h'})\cdot V^\pi_{M,h'+1}(s'')\Big)\Big| s_{h+1} = s' \Big]\Big] \\
        \underbrace{=}_{(5)} &
         \mathop{\E}_{\pi,\overline{P}} \left[ \sum_{h' = h}^{H-1} \left(r(s_{h'}, a_{h'}) - \overline{r}(s_{h'}, a_{h'})  + \sum_{s' \in S}\left(P(s' |s_{h'}, a_{h'}) - \overline{P}(s' |s_{h'}, a_{h'})\right)\cdot V^\pi_{M,h'+1}(s')\right) \Big| s_{h} = s \right].
    \end{align*}
    \endgroup
     Here, 
     $(1)$ is by Bellman equations for the value functions.
     $(2)$ is by adding and subtracting $\mathop{\E}_{s \sim \overline{P}(\cdot | s, \pi(s))}\left[ V^\pi_{M,h+1}(s') \right] $ and re-organizing.
    $(3)$ is by linearity of expectation.
    $(4)$ is by the induction hypothesis.
    $(5)$ is since the expectation over $s' \sim \overline{P}(\cdot| s,\pi(s))$ translates to the expectation induced by $\pi$ and $\overline{P}$ when applied on $s_{h+1}$ given that $s_h = s$. Hence given that $s_h = s$ and $\pi$ is deterministic, taking expectation over $\pi$ and $\overline{P}$ has no influence on both $r(s_{h}, a_{h})= r(s, \pi(s))$ and $\overline{r}(s_{h}, a_{h}) = \overline{r}(s, \pi(s))$.
\end{proof} 

\begin{lemma}[Bretagnolle Huber-Carol inequality, see e.g.,~\citet{MannorMT-RLbook}]\label{lemma:Bretagnolle-Huber-Carol}
    Let $X$ be a random variable taking values in
    $\{1,2,\ldots,k\}$ where $\Prob[X = i] = p_i$. Assume we sample $X$ for $n$ times and observe the value $i$ in $\hat{n}_i$ outcomes. Then,
    \begin{align*}
        \Prob\left[\sum_{i=1}^k \Big| \frac{\hat{n}_i}{n} -p_i \Big|  \geq \lambda \right] \leq 2^{k+1} \exp{(-n\lambda^2/2)}. 
    \end{align*}
\end{lemma}

\begin{corollary}
    With probability at least $1-\delta$  we have
    \[
        \sum_{i=1}^k \Big| \frac{\hat{n}_i}{n} -p_i \Big|  \leq \lambda
    \]
    for any
    \[
        \lambda \geq 
        \sqrt{\frac{2}{n}\ln(1/\delta)+ (k+1)\ln(2)}.
    \]
\end{corollary}

\end{document}